\documentclass{article}

    \PassOptionsToPackage{numbers, compress}{natbib}

    \usepackage[final]{neurips_2020}

\usepackage[utf8]{inputenc} %
\usepackage[T1]{fontenc}    %
\usepackage{hyperref}       %
\usepackage{url}            %
\usepackage{booktabs}       %
\usepackage{amsfonts}       %
\usepackage{nicefrac}       %
\usepackage{microtype}      %
\usepackage{comment}   
\usepackage{wrapfig}        %

\usepackage[dvipsnames,table,xcdraw]{xcolor}

\usepackage{graphicx}
\usepackage{subfigure}
\usepackage{booktabs}
\usepackage{enumitem}

\usepackage{adjustbox}
\usepackage{amsmath}
\usepackage{array}
\usepackage{amsthm}
\usepackage{makecell}
\usepackage{diagbox}
\usepackage{footnote}
\usepackage{tablefootnote}
\usepackage[para,online,flushleft]{threeparttable}
\usepackage{multirow}
\usepackage{arydshln}

\usepackage{amsmath,amsfonts,bm}

\def\1{\bm{1}}

\DeclareMathAlphabet{\mathsfit}{\encodingdefault}{\sfdefault}{m}{sl}
\SetMathAlphabet{\mathsfit}{bold}{\encodingdefault}{\sfdefault}{bx}{n}

\usepackage{amsmath,epsfig}
\usepackage{graphics,amsmath,amsfonts,amssymb,epsfig,latexsym,graphicx,amssymb}
\usepackage{subfigure}   
\usepackage{bm}
\usepackage{makecell}
\usepackage{diagbox} 
\usepackage{slashbox}
\usepackage{stfloats}  
\usepackage{footnote}
\usepackage{algorithm} 
\usepackage{algorithmic} 
\usepackage{hyperref} 

\usepackage[bottom]{footmisc}
\usepackage{hyperref}
\usepackage{url}
\usepackage{amsmath}
\usepackage{bm}
\usepackage{amsfonts}
\usepackage{algorithm}
\usepackage{algorithmic}
\usepackage{enumitem,graphicx}
\usepackage{tabularx}
\usepackage{array}
\usepackage{svg}

\usepackage{grffile}  %
\usepackage[normalem]{ulem}
\usepackage[font=footnotesize]{caption}

\usepackage[subfigure]{tocloft}
\usepackage{hyperref}

\newcommand{ \dR }{\mathbb{R}}

\newcommand{\Exp }{ \mathbb{E} }
\newcommand{\Pg }{ p_{ \rm g } }
\newcommand{\Pd }{ p_{ \rm data } }

\newcommand{\barn }{m}
\newcommand{\tablesize }{ \scriptsize } %
\newcommand{\equationsize}{ \scriptsize } %
\newcommand{\equationsizeReg}{\footnotesize}

\newtheorem{lemma}{Lemma}
\newtheorem{prop}{Proposition}
\newtheorem{prop2}{Proposition}[section]
\newtheorem{thm}{Theorem}
\newtheorem{thm2}{Theorem}[section]
\newtheorem{Def}{Definition}[section]
\newtheorem{coro}{Corollary}[section]

\newtheorem{claim}{Claim}[section]
\newtheorem{conjecture}{Conjecture}[section]
\newtheorem{assumption}{Assumption}[section]

\newif\ifsubmit
  \submitfalse

\ifsubmit
\newcommand{\tf}[1]{}
\newcommand{\rev}[1]{}
\newcommand{\as}[1]{}
\newcommand{\ruoyu}[1]{}
\else
\newcommand{\tf}[1]{{\color{Maroon}Tian: #1}}
\newcommand{\rev}[1]{{\color{ForestGreen} Reviewer: #1}}
\newcommand{\as}[1]{{\small{\color{blue}A: #1}}}
\newcommand{\ruoyu}[1]{{\color{red}{[Ruoyu: #1]}}}
\fi

\newif \ifaddsupplement

\addsupplementtrue

\newif \iflonger

 \longerfalse

\newif \ifNIPSonly

 \NIPSonlyfalse

\title{Towards a Better Global Loss Landscape of GANs}

\author{%
 Ruoyu Sun, Tiantian Fang,  Alex Schwing  \\
  University of Illinois at Urbana-Champaign \\
  \texttt{ruoyus,tf6,aschwing@illinois.edu} \\
}

\begin{document}

\maketitle

\vspace{-0.5cm}

\begin{abstract}
Understanding of GAN training is still very limited. One major challenge is its  non-convex-non-concave min-max objective, which may lead to sub-optimal local minima. In this work, we perform a global landscape analysis of the empirical loss of GANs. We prove that a class of separable-GAN, including the original JS-GAN, has exponentially many bad basins which are perceived as mode-collapse. We also study the relativistic pairing GAN (RpGAN) loss which couples the generated samples and the true samples. We prove that RpGAN has no bad basins. Experiments on  synthetic data show that the predicted bad basin can indeed appear in training. We also perform  experiments to support our theory that RpGAN has a better landscape than separable-GAN. For instance, we empirically show that RpGAN performs better than separable-GAN with relatively narrow neural nets. The code is available at \url{https://github.com/AilsaF/RS-GAN}.

\end{abstract}

\section{Introduction}
Generative Adversarial Nets (GANs) \citep{goodfellow2014generative} are a successful method for learning data distributions. Current theoretical efforts to advance understanding of GANs often focus on statistics or optimization.

On the statistics side, \citet{goodfellow2014generative} built a link between the min-max formulation and the J-S (Jenson-Shannon) distance.
\citet{arjovsky2017towards}  and \citet{arjovsky2017wasserstein} 
proposed an alternative loss function based on the Wasserstein distance.  \citet{arora2017generalization} studied the generalization error and showed that both the Wasserstein distance and J-S distance are not generalizable (i.e., both require an exponential number of samples). 
Nevertheless, \citet{arora2017generalization} argue that the real metric used in practice differs from the two statistical distances, and can be generalizable with a proper  discriminator. \citet{bai2018approximability} and \citet{lin2018pacgan} analyzed the potential ``lack of diversity'': two different distributions can have the same loss, which may cause mode collapse. \citet{bai2018approximability} argue that proper balancing of generator and discriminator permits both generalization and diversity.

 On the optimization side,  cyclic behavior (non-convergence) is well recognized \cite{mescheder2018training, balduzzi2018mechanics,gidel2018negative,berard2019closer}. This is a generic issue for min-max optimization:
 a first-order algorithm may cycle around a stable point, converge very slowly or even diverge.
  The convergence issue can be alleviated by more advanced optimization algorithms such as optimism (\citet{daskalakis2017training}), averaging (\citet{yazici2018unusual}) and extrapolation (\citet{gidel2018variational}). 

Besides convergence, another general optimization challenge is to avoid sub-optimal local minima.
 It is an important issue in non-convex optimization (e.g., \citet{zhang2017hitting,sun2020optimization}), 
 and has received  great attention in matrix factorization \cite{ge2016matrix,bhojanapalli2016global,chi2019nonconvex} and supervised learning \cite{jacot2018neural, lee2019wide,allen2018convergence,zou2018stochastic,du2018gradient}.  
  For GANs, the aforementioned works \cite{mescheder2018training, balduzzi2018mechanics,gidel2018negative,berard2019closer} either analyze convex-concave games or perform local analysis. Hence they do not touch the global optimization issue of non-convex problems. 
\citet{mescheder2018training}  and   \citet{feizi2017understanding} prove global convergence only for simple settings where the true data distribution is a single point or a single Gaussian distribution. The global analysis of GANs for a fairly general data distribution is still a rarely touched direction. 

The global analysis of GANs is an interesting direction for the following reasons. \textbf{First}, from a theoretical perspective, it is an indispensable piece for a complete theory. To put our work in perspective, we 
compare representative works in supervised learning with  works on GANs in Tab.~\ref{table:compare}.
\textbf{Second},  it may help to  understand mode collapse. 
 \citet{bai2018approximability} conjectured that a lack of diversity may be caused by optimization issues, albeit convergence analysis works \cite{mescheder2018training, balduzzi2018mechanics,gidel2018negative,berard2019closer} do not link non-convergence to mode collapse. Thus we suspect that mode collapse is at least partially related to sub-optimal local minima, but a formal theory is still lacking. 
\textbf{Third}, it may help to understand the training process of GANs. Even understanding a simple 
two-cluster experiment is challenging because the loss values of min-max optimization are fluctuating during training. Global analysis can provide an additional lens in demystifying the training process.

\iffalse 
Third, judging based on BigGAN \cite{brock2018large}, 
increasing the width of deep nets can greatly improve GANs.
This is reminiscent of the popular belief that increasing the width of deep nets in image classification smoothes the landscape \cite{{livni2014computational,geiger2018jamming,li2018visualizing}}. %
We suspect that GANs  suffer from a bad landscape which may explain why increasing the width in GANs has a significant effect.  If this explanation is valid, we conjecture that a GAN loss with the better landscape can better operate with narrower deep nets, thus reducing the model size. 
\fi 
Additional related work is reviewed in  Appendix~\ref{sec: related works}.

\begin{table}[t]
\vspace{-0.5cm}
    \caption{{Comparison of theoretical works.}}
    \label{table:compare}
    \centering
    \begin{adjustbox}{max width=0.95\textwidth}
        \begin{threeparttable}
            \begin{tabular}{|c|c|c||c|c|}
    \hline
 \multirow{2}{*}{ } & \multicolumn{2}{c||}{Supervised Learning} & \multicolumn{2}{c|}{GANs}
                \\ 
    \cline{2-3} \cline{4-5}
& paper & brief description    & paper & brief description \\        
                \hline
                Generalization analysis     
             &   \cite{bartlett2017spectrally}     &
                 generalization bound for neural-nets
                     &  \cite{arora2017generalization}  & 
                generalization bound for GANs
                \\ \hline
                Convergence analysis &   \cite{reddi2018convergence} 
                 & 
                 \begin{tabular}[c]{@{}c@{}} convex problem, divergence of Adam \\
                convergence of AMSGrad \end{tabular}
                 &  \cite{daskalakis2017training}  & 
                \begin{tabular}[c]{@{}c@{}} bi-linear game, non-convergence of GDA \\
                convergence of optimistic GDA \end{tabular}
                \\ \hline
                \begin{tabular}[c]{@{}c@{}}  Global landscape \\ 
     \end{tabular}
                 &  \cite{nguyen2018loss} \cite{li2018over}    &  
                     \begin{tabular}[c]{@{}c@{}} Any distinct input data \\
            Wide neural-nets have no sub-optimal basins   \end{tabular}
                 &     \textbf{This work}  &  
                 \begin{tabular}[c]{@{}c@{}} Any distinct input data \\
              SepGAN has bad basins; RpGAN does not  \end{tabular}
                \\ \hline
            \end{tabular}
            \begin{tablenotes}
                \item$^*$  {\footnotesize {
            This table does NOT show a complete list of works. The goal is to list various types of works. Only one or two works are listed as examples of that class.
                }
                }
            \end{tablenotes}
        \end{threeparttable}
    \end{adjustbox}
    \vspace{-0.5cm}
\end{table}

\iffalse 
\begin{tabular}[c]{@{}c@{}} 
 Optimization + generalization \\ 
              (relatively strong assumptions)
              \end{tabular}
               &    \citet{brutzkus2017globally} %
               & \begin{tabular}[c]{@{}c@{}} 
               Gaussian input data, \\ 
                    global convergence of SGD     \end{tabular}
         &   \citet{farnia2018convex}   & 
         \begin{tabular}[c]{@{}c@{}} Gaussian input data, \\ 
              quadratic G, linear D \end{tabular}
               \\ \hline  
\fi

\textbf{Challenges and our solutions.}  While the idea of a global analysis is natural, there are a few  obstacles. First,  it is hard to follow a common path of supervised learning \cite{jacot2018neural, lee2019wide,allen2018convergence,zou2018stochastic,du2018gradient} 
to prove global convergence of gradient descent for GANs, because the dynamics of non-convex-non-concave games are much more complicated.
Therefore,  we resort to a  \textit{landscape analysis}. Note that our approach resembles an 
``equilibrium analysis'' in game theory. Second, it was not clear which formulation can cure the landscape issue of JS-GAN. Wasserstein GAN (W-GAN) is a candidate, but its landscape is hard to analyze due to the extra constraints. After analyzing the issue of JS-GAN, we realize that the idea of ``paring'', which is implicitly used by W-GAN, is enough to cure the issue.
This leads us to consider relativistic pairing GANs (RpGANs)
\cite{jolicoeur2018relativistic,JolicoeurMartineau2019OnRF} that couple the true data and generated data\footnote{In fact, we proposed this loss in a first version of this paper, but later found that \cite{jolicoeur2018relativistic,JolicoeurMartineau2019OnRF} considered the same loss. We adopt their name RpGAN from \cite{JolicoeurMartineau2019OnRF}.}. 
We prove that RpGANs have a better landscape than
separable-GANs (generalization of JS-GAN).
Third, it was not clear whether the theoretical finding affects practical training. We make a few conjectures based on our landscape theory
and design experiments to verify those. Interestingly, the experiments match the conjectures quite well. 
      
     \textbf{Our contributions.}
     This work provides a global landscape analysis of the empirical version of GANs. 
     Our contributions are summarized as follows: 
     \begin{itemize}[noitemsep,topsep=0pt,parsep=0pt,partopsep=0pt]
\item \textit{Does the original JS-GAN have a good landscape, provably?}
For JS-GAN \cite{goodfellow2014generative}, we prove that the outer-minimization problem  has exponentially many sub-optimal strict local minima.  Each strict local minimum corresponds to a mode-collapse situation. We also extend this result to a class of separable-GANs, covering hinge loss and least squares loss.

         \item
\textit{Is there a way to improve the landscape, provably?}
We study a class of relativistic paring GANs (RpGANs)  \cite{jolicoeur2018relativistic} that pair the true data and the generated data in the loss function. We prove that the outer-minimization problem of RpGAN has no bad strict local minima, improving upon separable-GANs. 

 \item \textit{Does the improved landscape lead to any empirical benefit?}
 Based on our theory, we predict that  RpGANs are more robust to data,
 network width and initialization than their separable counter-parts, and our experiments
 support our prediction.
  Although the empirical benefit of RpGANs was observed before \cite{jolicoeur2018relativistic}, the aspects we demonstrate are closely related to our landscape theory. In addition, using synthetic experiments we explain why
  mode-collapse (as bad basins) can slow down JS-GAN training.

     \end{itemize}

\section{Difference of Population Loss and Empirical Loss}
\label{sec: empirical and population}
\vspace{-0.1cm}

\citet{goodfellow2014generative} proved that the population
loss of GANs is convex in the space of probability densities.
We highlight that this convexity  %
highly depends on a simple property of the population loss,
which may vanish in an empirical setting.

Suppose $\Pd$ is the data distribution, $\Pg $ is a generated  distribution and
 $ D \in C_{(0,1)} (\mathbb{R}^d) $, where $ C_{(0,1)} (\mathbb{R}^d)  $ is the set of continuous functions with
 domain $ \mathbb{R}^d $ and codomain $(0,1)$.  Consider the JS-GAN formulation \citep{goodfellow2014generative}
\begin{equation}  
\begin{split}
	\min_{ \Pg  }  \phi_{\rm JS}( \Pg; \Pd ) , \text{ where } 
	\phi_{\rm JS}( \Pg ; \Pd) = \sup_{ D   }  \Exp_{x \sim \Pd ,   y \sim \Pg  } [ \log( D (x) ) + \log (1 - D( y ) ) ].   \notag 
\end{split}
\end{equation}

\begin{claim}\label{claim 1 Goodfellow}
(\cite[in proof of Prop. 2]{goodfellow2014generative})
The objective function $ \phi_{\rm JS}( \Pg; \Pd ) $  is convex in $ \Pg $.  
\end{claim}

The proof utilizes two facts: first, the supremum of (infinitely many) convex functions is convex; second,  
$ \Exp_{x \sim \Pd ,   y \sim \Pg  } [ \log( D (x) ) + \log (1 - D( y ) ) ] $ is  a linear function of $\Pg $. The second fact is the essence of the argument, which we restate below in a more general form.

\begin{claim}\label{claim 2 any expectation}
$ \Exp_{ y \sim \Pg  } [ f^{\rm arb}( y ) ] $ is 
a linear  function of $ \Pg  $, 
where $ f^{\rm arb}( y) $ is an arbitrary function of $ y $. 
\end{claim}

Claim \ref{claim 2 any expectation} implies that 
 $ \min_{ \Pg  }  \Exp_{ y \sim \Pg  } [ f^{\rm arb}( y ) ] $ is a convex problem.
One approach to solve it is to draw finitely many samples (particles) $y_i, i=1,\dots, n$  from $\Pg$, and approximate the population loss by the empirical loss. See  Fig.~\ref{fig1:fig_two_views} for a comparison
of the probability space and the particle space.
For an arbitrarily complicated function such as $ f^{\rm arb} (y) = \sin( \| y \|^{8} + 2 \| y \|^3 + \log( \| y \|^4 + 1 ) ) $,
the population loss is convex in  $ \Pg $, but clearly the empirical loss is non-convex in $ (y_1, \dots, y_n) $.
This example indicates that studying the empirical loss may better reveal the difficulty of the problem (especially with a limited number of samples).  See Appendix~\ref{Sec: Parameter Space} for more discussions.

We focus on the empirical loss in this work. 
Suppose there are $ n $ data points $ x_1, \dots, x_n  $.  
We sample $ n $ latent variables $z_1, \dots, z_n \in  \mathbb{R}^{d_z} $  according to a  rule (e.g., i.i.d.\ Gaussian) and generate artificial data $ y_i = G( z_i ), i=1,\dots, n. $ 
The empirical version of JS-GAN addresses
$\min_{  Y }   \phi_{\rm JS}(  Y, X )$ where
\begin{equation}\label{JSGAN min-max finite, using D}
    \phi_{\rm JS}(  Y , X )  \triangleq  \sup_{ D }   
   \frac{1}{2n}  \sum_{i=1}^n [  \log( D(x_i ) ) +    \log ( 1 - D(y_i) ) ].
\end{equation}
Note that the empirical loss is considered in \citet{arora2017generalization} as well, but they study the generalization properties. We focus on the optimization properties, which is complementary to their work.

\begin{figure}[t]
\vspace{-0.5cm}
\centering
    \begin{tabular}{cc}
        \includegraphics[height = 2.5 cm]{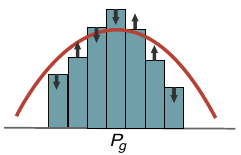}
        &
        \includegraphics[ height=3cm]{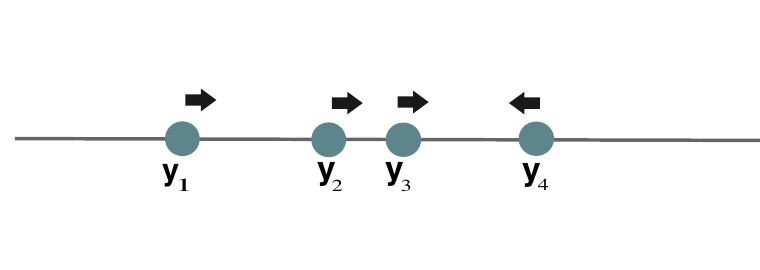}
        \\[-0.1cm]
        (a) & (b)
    \end{tabular}
    \vspace{-0.3cm}
    \caption{(a) Population loss: probability density changes;  (b) Empirical loss: samples move. 
    }
    \label{fig1:fig_two_views}
\end{figure}

 \section{Landscape Analysis of GANs: Intuition and Toy Results}
\label{sec: intuition and toy results}
\vspace{-0.1cm}
In this section, we discuss the main intuition and present  results
for a 2-point distribution.  %

\setlength{\columnsep}{10pt}
\begin{wrapfigure}{r}{0.37\textwidth}
	\vspace{-0.3cm}
	\centering
	\includegraphics[width=0.95\linewidth]{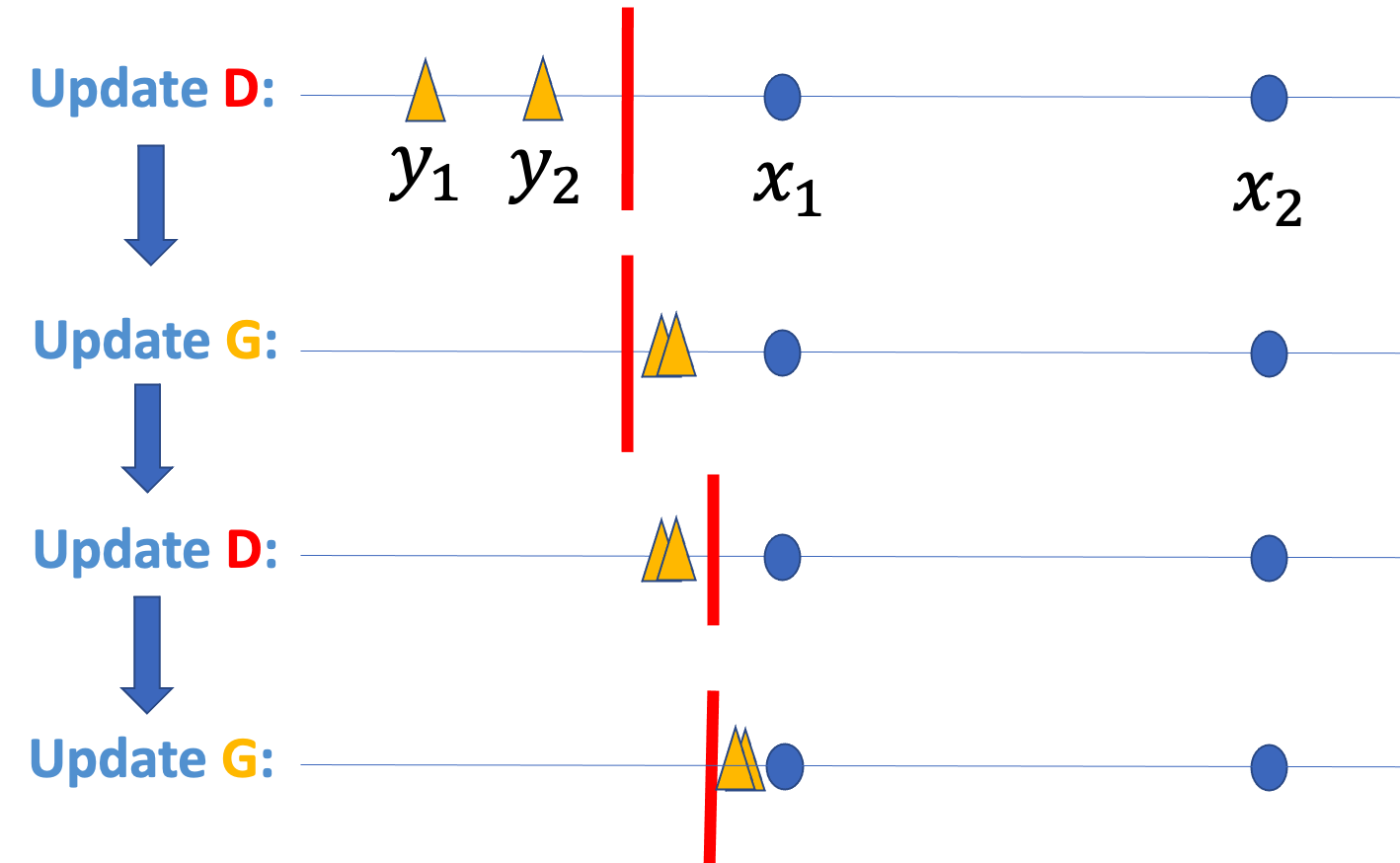}   \\
	\vspace{-0.1cm}
	\caption{{
			{ \small  Issue of separable-GAN (including JS-GAN). After updating $G$, fake data crosses  boundary to fool  $D$; after updating $D$, they are separated by $D$. Fake data may be stuck near $x_1$. 	%
		}
	}}
	\label{fig1 stuck at local min}
	\vspace{-0.5cm}
\end{wrapfigure}

\textbf{Intuition of Bad ``Local Minima'' and Separable-GAN:} 
Consider an empirical data distribution consisting
of two samples $x_1, x_2 \in \mathbb{R}. $
The generator produces two data points $ y_1, y_2  $ to match $x_1, x_2$. 
We illustrate the training process of JS-GAN in Fig.~\ref{fig1 stuck at local min}. Initially, $y_1, y_2$ are far from $x_1, x_2$, 
thus the discriminator can easily separate true data and fake data.
After the generator update,  $ y_1, y_2$ cross the decision boundary to fool the discriminator. Then, after the discriminator update, the decision boundary moves and can again separate true data and fake data. As iterations progress,
$ y_1, y_2 $ and the decision boundary may stay close to $ x_1 $, causing mode-collapse.

\iffalse 
\setlength{\columnsep}{10pt}
\begin{figure}[t]
\vspace{-0.2cm}
\centering
    \begin{tabular}{cccc}
        \includegraphics[width=0.22\linewidth]{figure/fig_new_1a.png} & \includegraphics[width=0.22\linewidth]{figure/fig_new_1b.png} &
        \includegraphics[width=0.22\linewidth]{figure/fig_new_1c.png} &
        \includegraphics[width=0.22\linewidth]{figure/fig_new_1d_CP.png}
        \\
        (a) Step 1  & (b) Step 2 & (c) Step 3 & (d) Breaking locality 
    \end{tabular}
    \vspace{-0.1cm}
    \caption{ \footnotesize{ (a)-(c) illustrate issue of JS-GAN. (a): D  separates fake and true data. (b): fake data crosses  boundary to fool  discriminator. (c): D moves a bit. (d): breaking locality by ``personalized'' judgement.  }
    }
    \label{fig1 stuck at local min old }
\vspace{-0.3cm}
\end{figure}
\fi

The intuition above is the starting point of this work. We notice that \citet{unterthiner2018coulomb,li2018implicit} presented somewhat similar intuition, and \citet{kodali2017convergence} suggested the connection between mode collapse and a bad equilibrium. Nevertheless, \citet{li2018implicit,kodali2017convergence} do not present a theoretical result, and \citet{unterthiner2018coulomb} uses a significantly
different formulation from standard GANs.
See Appendix
\ref{sec: related works} for more.

We point out that a major reason for the above issue is a single decision boundary which judges the generated samples. Therefore, this issue exists not only for the JS-GAN, but also for a large class of GANs which
we call separable-GANs: 
\begin{equation}\label{sepa GAN min-max finite}
	\min_Y \sup_f \sum_{i=1}^n h_1 ( f(x_i) ) +  h_2 ( - f(y_i ) ) ,
\end{equation}
where $h_1, h_2 $ are fixed scalar functions, such as $h_1(t) =
h_2(t) = - \log (1 + e^{-t}) $ and $ h_1(t) = h_2(t) = - \max \{ 0,  1 - t  \} $,  and  $ f $ is chosen from a function space
 (e.g., a set of neural-net functions). 
 
 \textbf{Pairing as Solution: Rp-GAN.} 
A natural solution is to use a different ``decision boundary'' for 
every generated point, e.g., pairing $x_i$ and $ y_i $,
as illustrated in Fig.~\ref{fig1b: break locality}. 

\begin{wrapfigure}{r}{0.35\textwidth}
	\centering
	\includegraphics[width=0.95\linewidth]{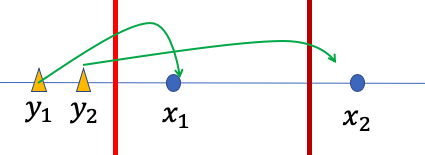}  \\
	\vspace{-0.1cm}
	\caption{{ \small Idea of RpGAN: breaking locality by ``personalized'' judgement.
	}}
	\label{fig1b: break locality}
\end{wrapfigure}
A suitable loss  is  the relativistic paring GAN (RpGAN)\footnote{
 Our motivation of considering RpGAN because it breaks locality, thus possibly admitting a better landscape. This motivation  is somewhat different from \citet{jolicoeur2018relativistic,JolicoeurMartineau2019OnRF}. }
\begin{equation}\label{rela GAN min-max finite}
  \min_Y \sup_f \sum_{i=1}^n h( f(x_i) - f(y_i ) ),
\end{equation}
where $h$ is a fixed scalar function and $ f $ is chosen from a function space. RS-GAN (relative standard GAN) is a special case where $ h( t ) = -\log( 1 + e^{ - t })$.
More specifically, RS-GAN addresses $\min_{  Y }   \phi_{\rm RS}(  Y, X )$ where
\begin{equation}\label{CPGAN min-max finite}
    \phi_{\rm RS}(  Y , X )  \triangleq  \sup_{  f }   \frac{1}{ n }  \sum_{i = 1}^n \log \frac{1}{  1 \!+\!   \exp( f  ( y_i )   \!-\!  f (  x_i )    )     )  }.
\end{equation}
W-GAN \cite{arjovsky2017towards} can be viewed as a variant of RpGAN where $ h( t ) = t $, with extra Lipschitz constraint.

We wonder how the issue of seperable-GANs
relates to ``local minima'' and  how  ``pairing'' helps. 
We present results for  JS-GAN and RS-GAN for the two-point case below.

\iffalse 
For the two-point example we just described, the RS-GAN discriminator provides two boundaries, one at $ -0.2 $ and one at $ 0.5 $. Consequently, $ y_1 $ will try  to cross $ -0.2$ to approach $x_1 = 0$, and $ y _2 $ will try to cross $ 0.5 $ to approach
$x_2 = 1$. Thus it can better learn the true distribution. 
Therefore, RS-GAN is adopting a ``personalized criterion'' for the generated data.
Following this intuition, we will rigorously prove next that 
RS-GAN has a better landscape than JS-GAN.  
\fi

\iffalse 
We remark that a similar intuition was mentioned
in \citet{goodfellow2016deep} that
``the generator is thus asked to map every $ z$ value to the single
$x$  coordinate that the discriminator believes is most likely to be real rather than fake'', 
\fi

\textbf{Global Landscape of 2-Point Case:} %
\iffalse 

\fi 
\iffalse 
Consider two samples $z_1, z_2 $, and the task of empirical
optimization is to learn $ G $ such that $ G( z_1 ) = y_1 , G(z_2) = y_2 $ match  $  x_1, x_2   $. We consider the function space optimization,
thus we consider optimization over $y_1, y_2 $. 
\fi 
Depending on the positions of $y_1, y_2$, there are four 
 states $ s_0 , s_{ 1 \text{a} }, s_{ 1 \text{b} }, s_2  $. 
 They represent the four cases
$   |  \{ x_1, x_2  \}  \!\cap\!  \{ y_1, y_2  \}   | = 0 $, 
 $  y_1 = y_2 \in  \{ x_1, x_2  \} $, 
$   |  \{ x_1, x_2  \}  \!\cap\!  \{ y_1, y_2  \}   | \!\!=\!\! 1, $
and $   \{ x_1, x_2  \}  \!=\! \{ y_1, y_2  \}   $ respectively.
Training often starts from the ``no-recovery'' state $ s_0 $,
and ideally should end at the ``perfect-recovery'' state $ s_2 $. %
There are two intermediate states: $ s_{1a}$ means all generated points
fall into one mode (``mode collapse'');
$s_{1b} $ means one generated point is the true data point while the other
is not a desired data point, which we  call ``mode dropping''\footnote{Both may be called mode collapse. Here we differentiate ``mode collapse'' and ``mode dropping''.}.
The first three states can transit to each other (assuming continuous
change of $Y$), but only $ s_{ 1 \text{b} }  $ can transit to $ s_2 $.
We illustrate the landscape of  $ \phi_{\rm JS}( Y; X ) $
 and $ \phi_{\rm RS}(Y; X) $ in Fig.~\ref{fig1_GAN_landscape},
by indicating the values in different states.
The detailed computation is given next.

\iffalse 
\begin{figure}[t]
\vspace{-0.5cm}
\centering
    \begin{tabular}{cc}
        \includegraphics[width=0.35\linewidth]{figure/fig_new_2a} & \includegraphics[width=0.35\linewidth]{figure/fig_new_2b} \\
        (a) JS-GAN & (b) RS-GAN
    \end{tabular}
    \vspace{-0.3cm}
    \caption{The four states with loss values. Left: JS-GAN.
    Right: RS-GAN. $s_0$ means recovering no data point, the starting state. $ s_{ 1 \text{a} } $  means $  y_1 = y_2 \in  \{ x_1, x_2  \} $,
a mode collapse pattern. $  s_{ 1 \text{b} }  $ means 
$   |  \{ x_1, x_2  \}  \!\cap\!  \{ y_1, y_2  \}   | \!\!=\!\! 1 $,
a mode dropping pattern. $  s_{ 2 }  $ means recovering both data points, a desired destination. The numbers are their outer objective value $ \phi ( Y; X ) = \sup_{ f } L(Y ,  f; X) . $ For JS-GAN, the outer loss function is $ \phi_{\rm JS}(Y; X) $; for RS-GAN, the outer loss function is $ \phi_{\rm RS}(Y; X) $. }
    \label{transition graph}
\vspace{-0.3cm}
\end{figure}

\fi

\textbf{JS-GAN 2-Point Case:} 
The range of $ \phi_{\rm JS}( Y , X ) $ is $ [ - \log 2, 0 ] $.
The value for the four states are: 
\begin{claim}\label{2-dim GAN all values}       %
The minimal value of $\phi_{\rm JS}(  Y , X ) $ is $ - \log 2 $, achieved at $ \{ y_1, y_2  \}  = \{ x_1, x_2  \} $. 
    $$
   \phi_{\rm JS}(Y, X) =  \begin{cases}
    - \log 2 \approx -0.6931    %
    &      \hspace{-0.2cm}\text{if }  \{ x_1, x_2  \}  \!=\! \{ y_1, y_2  \},   \\
    -\log 2 /2   \approx   -0.3467   %
    &     \hspace{-0.2cm}\text{if }     |  \{ x_1, x_2  \}  \!\cap\!  \{ y_1, y_2  \}   | \!\!=\!\! 1, \\
\frac{1}{4} (  2 \log 2 - 3 \log 3 )  \approx  -0.4774     %
    &      \hspace{-0.2cm}\text{if }    y_1 = y_2 \in  \{ x_1, x_2  \} ,    \\
    0       
    &     \hspace{-0.2cm}\text{if }     |  \{ x_1, x_2  \}  \!\cap\!  \{ y_1, y_2  \}   | \!\!=\!\! \emptyset.
    \end{cases}  
    $$
\end{claim}
We  illustrate the landscape of  $  \phi_{\rm JS}(Y, X) $ in Fig.~\ref{fig1_GAN_landscape}(a).
As a corollary of the above claim, the outer optimization  of the original GAN has a bad strict local minimum at state $ s_{\rm 1a} $ (a mode-collapse).
\begin{coro}\label{coro of bad basin in GAN}
    $ \bar{Y} = (   x_1, x_1  ) $   is a  sub-optimal strict local-min of the function $  g(Y) =     \phi_{\rm JS}(Y, X)  .  $ 
\end{coro}

\textbf{RS-GAN 2-Point Case:} 
The range is still  $ \phi_{\rm RS}(  Y , X ) \in [ - \log 2,  0 ] $. The values are:
\begin{claim}\label{claim: 2-dim RS-GAN all values}
The minimal value of $\phi_{\rm RS}(  Y , X ) $ is $ - \log 2 $, achieved at $ \{ y_1, y_2  \}  = \{ x_1, x_2  \} $.   In addition,
$$
\phi_{\rm RS}(  Y , X )  =  \begin{cases}
- \log 2 \approx -0.6931                 &      \text{if }  \{ x_1, x_2  \}  = \{ y_1, y_2  \} ,  \\
- \frac{1}{2}   \log 2   \approx -0.3466         &     \text{if }    | \{ i: \exists j, \text{ s.t. } x_i  = y_j  \} |,       \\
0          &     \text{otherwise. }
\end{cases}  
$$
\end{claim} 
We  illustrate  $  \phi_{\rm RS}( Y , X) $ in Fig.~\ref{fig1_GAN_landscape}(b). Importantly, note that the only basin is the global minimum. In contrast, the landscape of JS-GAN contains a bad basin at a mode-collapsed pattern. 

The proofs of Claim \ref{2-dim GAN all values} and
Claim \ref{claim: 2-dim RS-GAN all values} are given
in Appendix \ref{appen: toy proof}. 
We briefly explain the main insight provided by these proofs.
For the mode-collapsed pattern $s_{\rm 1a}$,
the loss value of JS-GAN is
$ - \frac{1}{4}  \min_{s, t} [ \log ( 1 + e^{-t} ) +  2 \log ( 1 + e^t)
 + \log ( 1 + e^{-s} ) ]  
 =  \frac{1}{4} ( \log \frac{1}{3} + 2 \log \frac{2}{3})  \approx -0.48
\neq - \frac{ r }{ 2 } \log 2 $ for any integer $r$. 
This creates an ``irregular'' value among other loss values of 
the form $ - \frac{ r }{ 2 } \log 2 $. 
In contrast,  for pattern $s_{\rm 1a}$,
the loss value of RS-GAN is $ - \frac{1}{2}\min_{s, t}
[ \log ( 1 + e^{ t - t }  ) + \log (1 + e^{ t - s } ) ] =  -\frac{1}{2} \log 2 $, which is of the form $ - \frac{ r }{ 2 } \log 2 $. 
Therefore, for the $2$-point case, RS-GAN has a better landscape because it avoids the ``irregular'' value  of JS-GAN due to its ``pairing''.
 This insight is the foundation of the general theory presented
 in the next section.

\begin{figure}[t]
\vspace{-0.5cm}
\centering
    \begin{tabular}{cc}
        \includegraphics[width=0.33\linewidth]{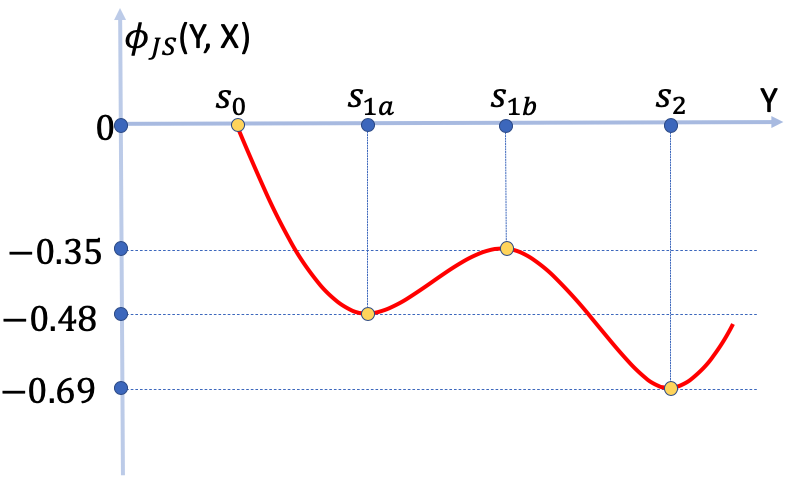} & \includegraphics[width=0.35\linewidth]{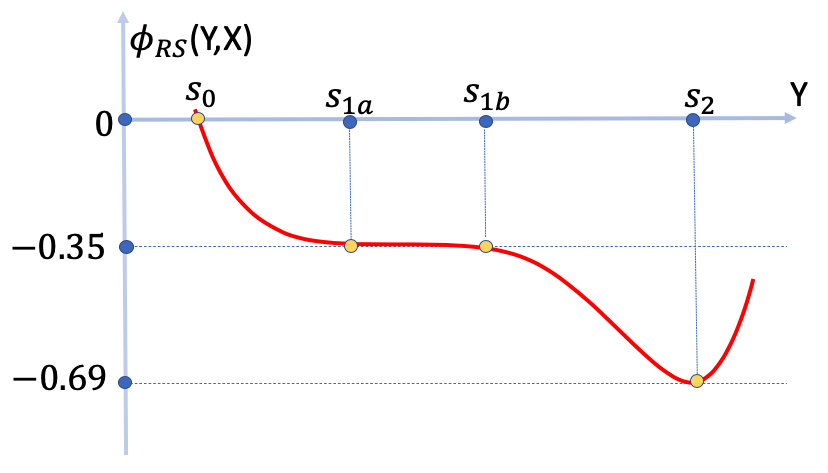} \\[-0.5cm]
        {\footnotesize (a) JS-GAN} &  {\footnotesize (b) RS-GAN}
    \end{tabular}
    \vspace{-0.3cm}
    \caption{ \small  Landscape for GAN outer optimization  $ \min_Y  \phi(Y, X) $.
        It is not a rigorous figure because:
        (i) there are only four possible values, thus the function is piece-wise linear while we use smooth curves for accessibility.    %
        (ii) the landscape should be two-dimensional, but 
        we  illustrate them in 1D space. 
    Nevertheless, it is still useful for understanding
    GAN training, as discussed later in Section \ref{sec: case study}
    and Appendix~\ref{subsec: details of experiments}.
    }
    \label{fig1_GAN_landscape}
\vspace{-0.3cm}
\end{figure}

\vspace{-0.1cm}

\section{Main Theoretical Results}
\label{subsec: main results}

\subsection{Landscape Results in Function Space} 
We present our main theoretical results, extending the landscape results from $ n = 2 $ to general $ n $. 

Denote 
$\xi( m ) \triangleq  \sup_{ t \in \mathbb{R} } ( h_1 (t) + m h_2 ( - t) ) $. 
\begin{assumption}\label{assumption 1}
 $ \sup_{t \in \mathbb{R} } h_1 (t) =  \sup_{t \in \mathbb{R} } h_2 (t)  =  0 $.  
\end{assumption}

\begin{assumption}\label{assumption 2}
   $     \xi( m ) > m \xi(1) ,  ~ \forall ~ m \in [2, n].  $
\end{assumption}

\begin{assumption}\label{assumption 3}
 $ \xi (m) < \xi (m - 1), ~ \forall ~  m \in [1, n]  $.
\end{assumption}

It is easy to prove that under Assumption \ref{assumption 1}, 
$ \xi(m-1) \geq \xi(m) \geq m \xi(1)$ 
always holds.
Assumption \ref{assumption 2} and  Assumption \ref{assumption 3}
require strict inequalities, thus do not always hold (e.g., for constant functions).
Nevertheless, most non-constant functions satisfy these assumptions. 

The separable-GAN (SepGAN) problem (empirical loss, function space) is
\begin{equation}\label{sep GAN def}
 \min_{ Y \in \mathbb{R}^{d \times n}  }  g_{\rm SP}( Y ) , \text{ where }   g_{\rm SP}( Y )  = \frac{1}{2n }  \sup_{  f \in C( \mathbb{R}^d ) }
       \sum_{i = 1}^n    [  h_1 (  f( x_i ) ) + h_2  ( - f(y_i ) )  ].
\end{equation}

\begin{thm}\label{prop: GAN all values, extension}
 Suppose   $x_1 , x_2, \dots, x_n \in \mathbb{R}^d $ are distinct.
Suppose  $h_1, h_2 $ satisfy Assumptions \ref{assumption 1},
\ref{assumption 2} and \ref{assumption 3}. Then 
for separable-GAN loss $ g_{\rm SP}( Y )$  defined in Eq.~\eqref{sep GAN def}, we have: 
    (i)    The global minimal value is 
    $ - \frac{1}{2} \sup_{t \in \mathbb{R} } ( h_1(t) + h_2(-t))   $, which
    is achieved iff  $ \{   y_1  ,  \dots, y_n  \}  = \{ x_1, \dots, x_n  \}  $. 
    (ii)    If $y_i \in \{ x_1, \dots, x_n  \}, i\in\{1,2, \dots, n\}  $ and $y_i = y_j $ for some $ i \neq j$, then $ Y $ is a sub-optimal strict local minimum. Therefore, $ g_{\rm SP}( Y )$  has  $ (  n^n - n !  )$  sub-optimal strict local minima. 
\end{thm}

Remark 1: $h_1(t) = h_2(t) = - \log (1 + e^{-t}) $  satisfy Assumptions \ref{assumption 1}, \ref{assumption 2} and \ref{assumption 3}, thus Theorem \ref{prop: GAN all values, extension} applies
to JS-GAN. It also applies to hinge-GAN with $ h_1(t) = h_2(t) = - \max \{ 0,  1 - t  \} $ and LS-GAN (least-square GAN) 
with $h_1(t) = - (1 - t)^2, h_2(t) = - t^2$.

Next we consider RpGANs. 
The RpGAN problem (empirical loss, function space) is
\begin{equation}\label{phi R def}
  \min_{ Y \in \mathbb{R}^{d \times n}  }  g_{\rm R} ( Y ) , \text{ where }    g_{\rm R} ( Y )  = \frac{1}{ n } 
\sup_{  f \in C( \mathbb{R}^d ) } \sum_{i = 1}^n    [  h (  f( x_i ) - f(y_i ) )  ]. 
\end{equation}

\begin{Def}
(global-min-reachable)
 We say a point $ w $ is global-min-reachable for a function $ F(w) $
 if there exists a continuous path from $ w $ to one global minimum of $ F$ along which the value of $ F(w) $ is non-increasing. 
\end{Def}

\begin{assumption}\label{assumption cp-1}
 $ \sup_{t \in \mathbb{R} } h(t) = 0 $
and $ h( 0 ) < 0 .  $
\end{assumption}

\begin{assumption}\label{assumption cp-2}
 $ h $ is a concave function in $\mathbb{R} $. 
\end{assumption}

\iffalse 
\begin{assumption}\label{assumption cp-2}
  $ \inf_{ t \in \mathbb{R } } h(t) = -\infty.   $
\end{assumption}
\fi

\begin{thm}\label{prop: RS-GAN all values, extension}
 Suppose   $x_1 , x_2, \dots, x_n \in \mathbb{R}^d $ are distinct.
Suppose  $h$ satisfies Assumptions \ref{assumption cp-1} and
 \ref{assumption cp-2}. Then for RpGAN loss $ g_{\rm R} $
defined in Eq.~\eqref{phi R def}: 
(i)    The global minimal value is  $ h(0)  $, which is achieved iff  $ \{   y_1  ,  \dots, y_n  \}  = \{ x_1, \dots, x_n  \}  $. 
(ii) Any $Y $ is global-min-reachable for the function $g_{\rm R}( Y )  $. 
\end{thm}

This result  sanity checks  the loss $ g_{\rm R}( Y  ) $:
its global minimizer is indeed the desired empirical distribution.
In addition, it establishes a significantly different optimization landscape for RpGAN.

Remark 1: $ h(t) =  - \log(1 + e^{-t})$ satisfies
 Assumption \ref{assumption cp-1} and \ref{assumption cp-2}, thus Theorem 
 \ref{prop: RS-GAN all values, extension} applies to RS-GAN. 
 It also applies to Rp-hinge-GAN with $h(t) = 
- \max \{ 0, a - t \} $ and Rp-LS-GAN with $h(t) = - (a - t)^2$, for any positive constant $a$.

Remark 2: The W-GAN loss is $ \frac{1}{n} \sup_f \sum_i h(f(x_i) - f(y_i))$ where $ h( t) = t $; however, since $ \sup_{t} h(t) = \infty  $ it does not satisfy Assumption \ref{assumption cp-1}. The unboundedness of $h(t) = t$ necessitates extra constraints, 
which make the landscape analysis of W-GAN challenging; see Appendix~\ref{appen: W-GAN}.  Analyzing the landscape of W-GAN is an interesting future work. 

\iffalse 
Remark 3: While
we do recommend practitioners to consider RpGAN, we expect that there can be other losses with similar or better landscape properties than RpGAN. 
Part of our goal is to emphasize the importance of
a global landscape analysis of GANs. 
\fi

To prove Theorem \ref{prop: GAN all values, extension}, careful computation suffices; see  Appendix \ref{appen: proof of Thm 1}.
The proof of Theorem \ref{prop: RS-GAN all values, extension} is a bit involved.
We first build a graph with nodes representing $x_i$'s and $y_i$'s, 
then decompose the graph into cycles and trees, 
and finally compute the loss value by grouping the terms according to cycles and trees and calculate the contribution of each cycle and tree. 
The detailed proof is given in Appendix  \ref{appen: proof of Thm 2}.

\subsection{Landscape Results in Parameter Space} 

We now consider a deep net generator $G_w$ with $w \in \mathbb{R}^K $
and a deep net discriminator $ f_{\theta}  $ with  $ \theta \in \mathbb{R}^{ J } $. Different from before, where we optimize over $ y_i$ and $ f $ (function space), we now optimize over $ w $ and $\theta$ (parameter space).

We first present a technical assumption.
For $ Z = ( z_1, \dots, z_n ) \in \mathbb{R}^{d_z \times n} $, $Y = (y_1, \dots, y_n) \in \mathbb{R}^{d \times n} $
and $\mathcal{W} \subseteq \mathbb{R}^K $, define 
a set $ G^{-1}(Y; Z, \mathcal{W} )
\triangleq \{ w \in \mathcal{W} \mid G_w( z_i ) = y_i, ~\forall ~i \} $.
  \begin{assumption}\label{G repres assumption 2}
  (path-keeping property of generator net):
  For any distinct $z_1, \dots, z_n \in \mathbb{R}^{d_z} $, any continuous path $Y(t), t \in [0, 1]$ in the space 
 $ \mathbb{R}^{d \times n}  $ and any $w_0 \in G^{-1}(Y (0) ; Z , \mathcal{W})  $, 
 there is continuous path $ w(t), t \in [0,1]  $ such that
 $ w(0) = w_0 $ and $ Y(t) = G_{w(t)}(Z)  , t \in [0, 1] $. 
\end{assumption}

Intuitively, this assumption relates the paths in the function
space to the paths in the parameter space,
 thus the results in function space can be transferred
 to the results in parameter space. 
The formal results involve two extra assumptions on representation power of $f_{\theta}$ and $G_w$ (see Appendix~\ref{appen: param space results} for details).
Informal results are as follows: 

\begin{prop}\label{prop: GAN param space}(informal)
Consider the separable-GAN problem  $  \min_{ w \in \mathbb{R}^{K}  } 
    \varphi_{\rm sep}( w ) , $ where
\begin{equation} 
\varphi_{\rm sep}( w )  = \sup_{ \theta }
  \frac{1}{2 n} \sum_{i=1}^n   [ h_1 (   f_{ \theta}( x_i )  ) + h_2 (  - f_{\theta} ( G_w( z_i ) )  )   ].
       \end{equation}
  Suppose $h_1, h_2$ satisfy the  assumptions of Theorem \ref{prop: GAN all values, extension}. 
  Suppose $ G_{w}  $ satisfies Assumption \ref{G repres assumption 2} (with certain $\mathcal{W}$). Suppose $  f_{\theta} $ and $G_w$ have enough representation power (formalized in Appendix~\ref{appen: param space results}).
Then there exist at least $ (  n^n - n !  )$ distinct $ w  \in \mathcal{W } $ that are not global-min-reachable for  $ \varphi_{\rm sep}( w )  $. 
\end{prop}

\begin{prop}\label{prop: CPGAN param space}(informal)
 Consider the RpGAN problem 
 $  \min_{ w \in \mathbb{R}^{K}  }  \varphi_{\rm R}( w ) , $  where  
 \begin{equation} 
 \varphi_{\rm R}( w )  = \sup_{ \theta }  \frac{1}{ n } \sum_{i=1}^n   [ h ( f_{ \theta}( x_i ) )  -  f_{ \theta}( G_w(z_i) )  ]. 
 \end{equation}
  Suppose $h $ satisfies  the  assumptions of Theorem \ref{prop: RS-GAN all values, extension}. Suppose $ G_{w}  $ and $  f_{\theta} $ satisfy
  the same assumptions as Proposition \ref{prop: GAN param space}.
 Then any $ w  \in \mathcal{W} $ is global-min-reachable for
 $ \varphi_{\rm R}( w )  $.
\end{prop}

Remark 1: The existence of a decreasing path does not necessarily
mean an algorithm can follow it. Nevertheless, our results 
already distinguish SepGAN and RpGAN. We will illustrate that
these results can improve our understanding of GAN training  
in Sec.~\ref{sec: case study}, and present experiments supporting
 our theory in Sec.~\ref{sec:experiments}.

Remark 2: The two results rely on a few 
assumptions of  neural-nets including Assumption
\ref{G repres assumption 2}. 
These assumptions can be satisfied by certain
over-parameterized neural-nets, in which case $ \mathcal{W} $ is a certain dense subset of $ \mathbb{R}^K $ or $ \mathbb{R}^K $ itself. 
For details see Appendix \ref{app-sub: suffcient conditions for repres assumptions}.

\subsection{Discussion of Implications}

These results distinguish the SepGAN and RpGAN landscapes.
\iffalse 
Intuitively, the JS-GAN landscape has many basins of attraction, and a local search  on $ w $ might fall into one of them,
exhibiting mode collapse. 
For RS-GAN, there is no strict local-min that traps a local search,
thus the algorithm may more likely converge to a global minimum.
\fi 
Theoretically, there is evidence regarding the benefit of losses without sub-optimal basins.  \citet{bovier2004metastability} proved that it takes the Langevin diffusion at least $ e^{\omega(h)}$ time to escape a depth-$h$ basin. A recent work \cite{zhang2017hitting} proved that the hitting time of SGLD (stochastic gradient Langevin dynamics) is positively related to the height of the barrier, and SGLD may escape basins with low barriers relatively fast. 
 The theoretical insight is that a landscape without a bad basin permits better quality solutions or a faster convergence to good-quality solutions. 

We now discuss the possible gap between our theory and practice. 
We proved that a mode collapse $ Y^* $ is a bad basin in the generator space, which indicates that $( Y^*, D^*( Y^* )  )$ is
an attractor in the joint space of $(Y,D)$ and hard to escape by gradient descent ascent (GDA). In GAN training, the dynamics are not the same as GDA dynamics due to various reasons (e.g., sampling, unequal $D$ and $G$ updates),
and basins could be escaped with enough training time
  (e.g., \cite{zhang2017hitting}).
In addition, a randomly initialized $ (Y, D) $ might be
far away from the basins at $( Y^*, D^*( Y^* ))$, and properly chosen
 hyper-parameters (e.g., learning rate) may re-position the dynamics so as to avoid attraction to bad basins.
Further, it is known that adding neurons can smooth the landscape of deep nets (e.g., eliminating bad basins in neural-nets \cite{li2018over}),
thus wide nets might help escape basins in the $(Y, D)$-space faster. 
In short, the effect of bad basins may be mitigated via the following factors: 
(i) proper initial $ D $ and $Y$;
(ii) long enough training time;
(iii) wide neural-nets; 
(iv) enough hyper-parameter tuning. 
These factors make it relatively hard to 
detect the existence of bad basins and their influences. 
We  support our landscape theory, by identifying
 differences of SepGAN and RpGAN in synthetic and real-data experiments.

\section{Case Study of Two-Cluster Experiments}\label{sec: case study}
\vspace{-0.3cm}

Although in Section \ref{sec: intuition and toy results}
we argue that, \textit{intuitively}, mode collapse can happen for training
JS-GAN for two-point generation, it does not necessarily mean  mode collapse really appears in practical training.
 We discuss a two-cluster experiment, an extension of two-point generation,
in order to build a link between theory and practice. We aim to understand the following question: does mode collapse really appear as a ``basin'', and how does it affect training? 

Suppose the true data are two clusters around $ c_1 =0 $ and $ c_2 = 4 $. We sample $ 100 $ points from the two clusters as $x_i$'s,
and sample $z_1, \dots, z_{100}$  uniformly from an interval.
 We use 4-layer neural-nets for the discriminator and generator.
 We use the non-saturating versions of JS-GAN and RS-GAN.
 
 \iffalse 
The training process of GANs is difficult to understand, even for the two-cluster experiments. One major difficulty is the lack of a proper metric to monitor the training process. The plots of D loss and G loss over iterations are often drawn, but many patterns of the plots are hard to interpret.
We demonstrate that the theory can help provide a partial understanding of the fluctuation. Based on this understanding, we provide an explanation of why RS-GAN is faster than JS-GAN. 
\fi

\textbf{Mode collapse as bad basin can appear.} 
We visualize the movement of fake data in Fig.~\ref{fig:2clustercomp}, 
and plot the loss value of D (indicating the discriminator) over iterations in Fig.~\ref{fig5 D loss and D image}(a,b). %
 Interestingly, the minimal $D$ losses are around $ 0.48$, which is the value of $ \phi_{\rm JS} $ at state $s_{1a}$. 
It is easy to check that the optimal $ D = D^*(s_{1a}) $ for a mode collapse state $s_{1 \rm a}$ satisfies $ \{ D ( c_1), D(c_2) \}  = \{ 1 , 1/3\} $, 
 and Fig.~\ref{fig5 D loss and D image}(c)
 shows that at iteration $ 2800 $ the $D $ actually 
  becomes $ D^* $. 
This provides a concrete example that  training 
gets stuck at a mode collapse due to the  bad-basin-effect. 
 We also notice that there are a few more attempts to approach 
 the bad attractor $ ( s_{1a}, D^*(s_{1a} ) )  $ (e.g., from iteration
 $2000$ to $2500$).
 In RS-GAN training, the minimal loss is around $ 0.35 $,
  which is also the value of $ \phi_{\rm RS} $ at state $s_{1a}$.
The attracting power of $(s_{1a}, D^*( s_{1a}))$ is weaker than for  JS-GAN. 
 Thus it only attracts the iterates for a very short time.
 RS-GAN needs 800 iterations to escape, which is about 3 times faster than the escape for JS-GAN.

 \begin{figure}[t]
\centering
\begin{tabular}{>{\centering}m{11mm}>{\centering}m{16mm} >{\centering}m{16mm}>{\centering}m{16mm} >{\centering}m{16mm} >{\centering}m{16mm} >{\centering\arraybackslash}m{16mm}}
\small{JS-GAN:} & 
\includegraphics[width=15mm]{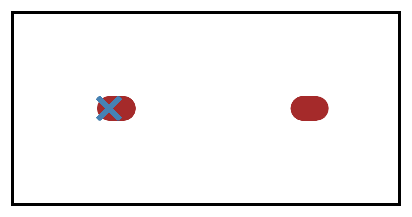}&\includegraphics[width=15mm]{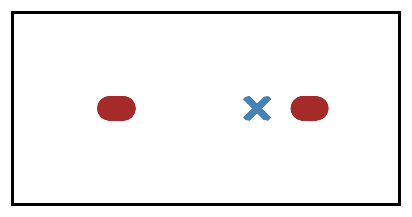}&\includegraphics[width=15mm]{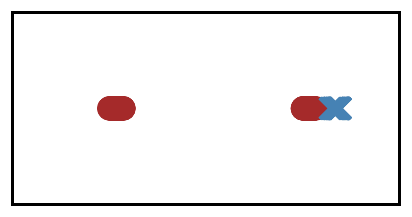}&\includegraphics[width=15mm]{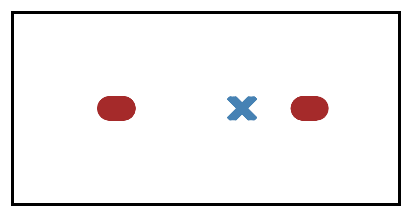}&\includegraphics[width=15mm]{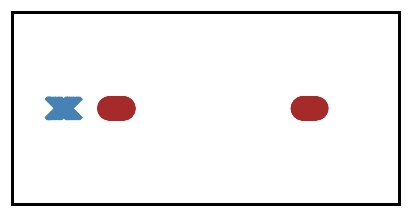}&\includegraphics[width=15mm]{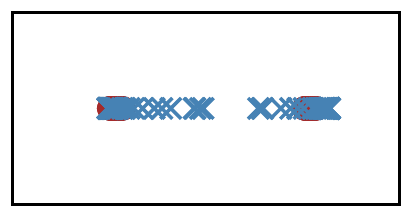}\\
\small{RS-GAN:} &  \includegraphics[width=15mm]{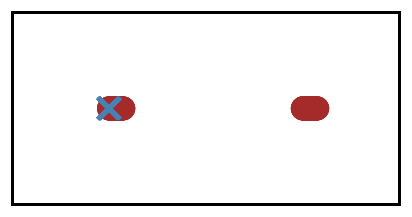}&\includegraphics[width=15mm]{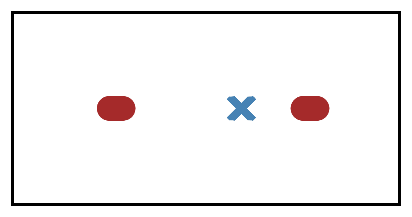}&\includegraphics[width=15mm]{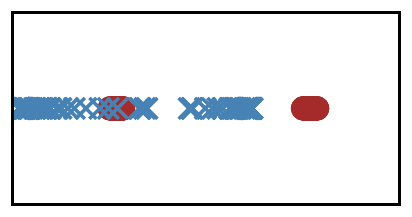}&\includegraphics[width=15mm]{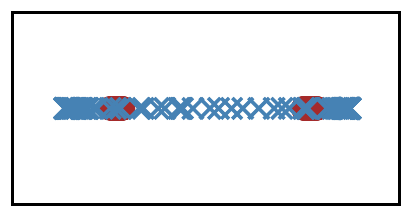}&\includegraphics[width=15mm]{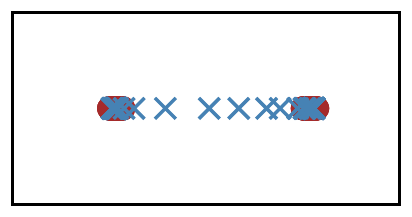}&\includegraphics[width=15mm]{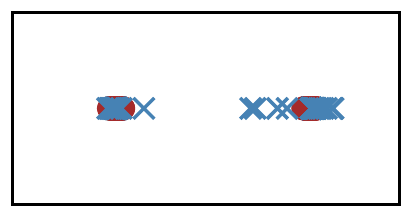}\\
\end{tabular}
\caption{Training process of JS-GAN and RS-GAN for two-cluster data.
True data are red, fake data are blue. RS-GAN escapes from mode collapse faster
than JS-GAN. 
}
\label{fig:2clustercomp}
\vspace{-0.2cm}
\end{figure}

  \begin{figure}[t]
\vspace{-0.2cm}
\centering
    \begin{tabular}{ccc}
        \includegraphics[width=0.3\linewidth]{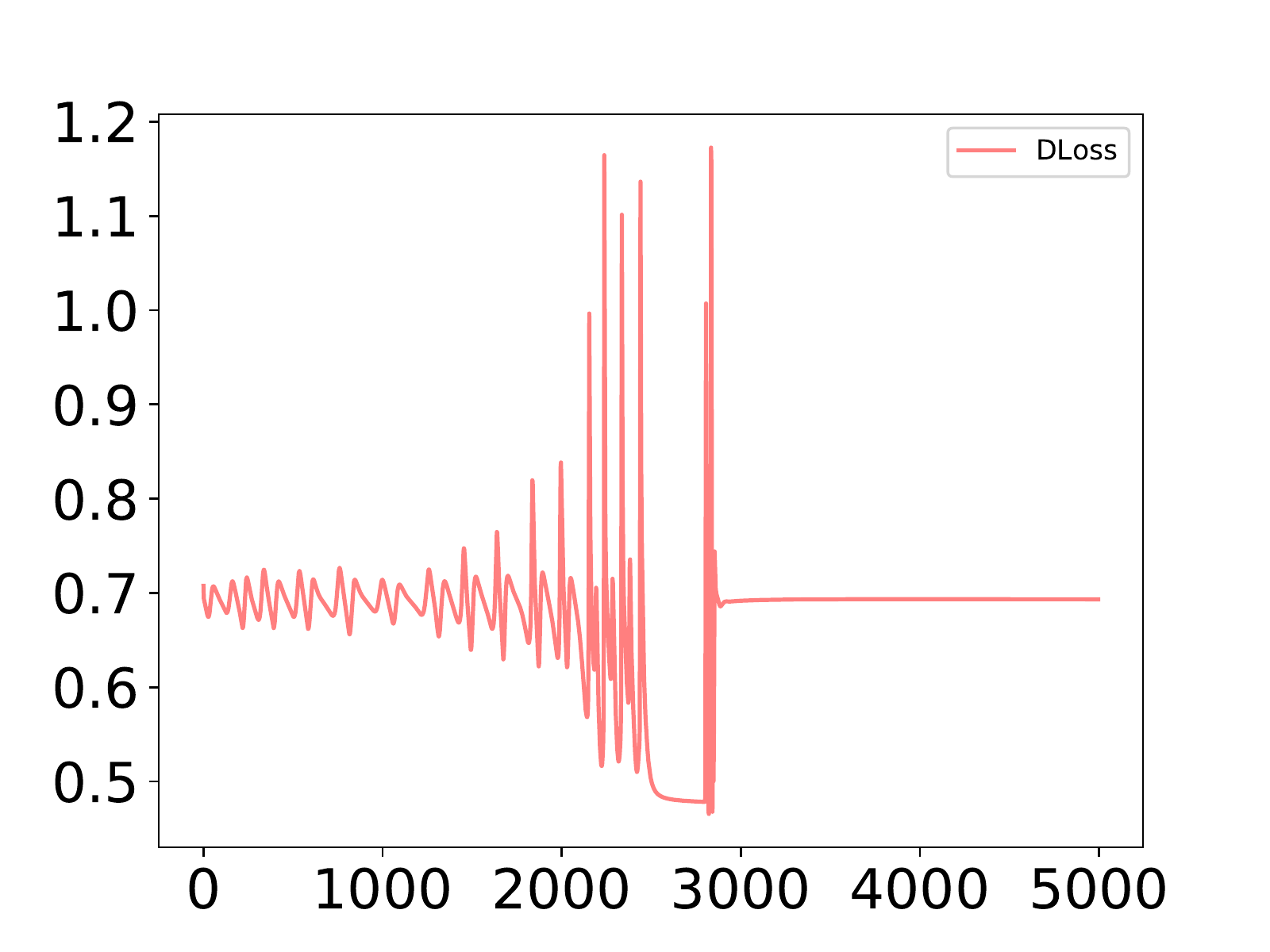} & \includegraphics[width=0.3\linewidth]{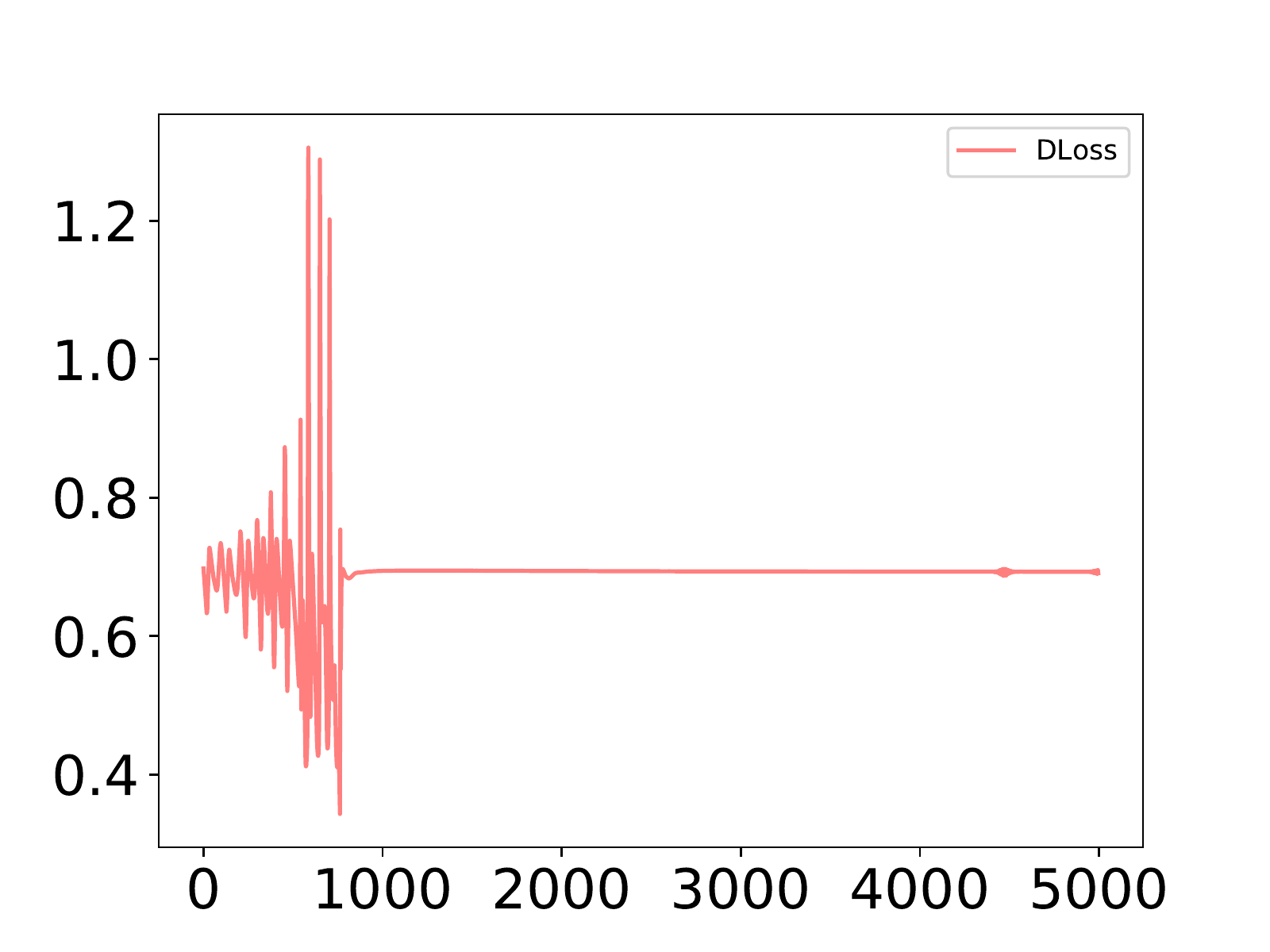} &
          \includegraphics[width=0.3\linewidth]{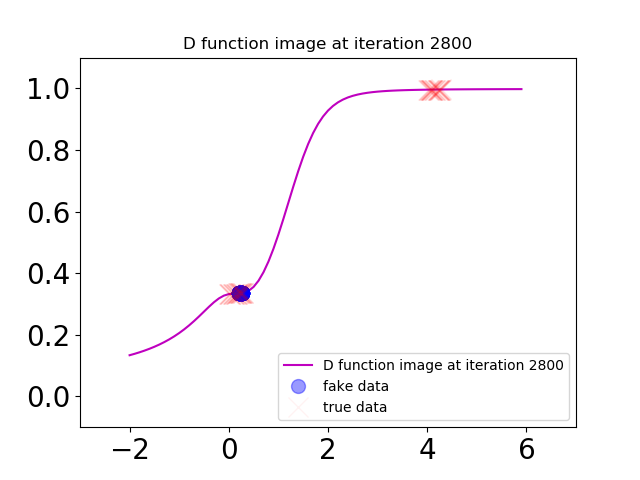}
        \\[-2mm]
        {\footnotesize (a) JS-GAN}  & {\footnotesize (b) RS-GAN} &  {\footnotesize (c) JS-GAN, D image}
    \end{tabular}
    \vspace{-0.2cm}
    \caption{(a) and (b): Evolution of $D$ loss over iterations. 
    RS-GAN is 3-4$\times$ faster than JS-GAN.
    (c) For JS-GAN training in (a), we plot $ (Y,D) $ together at iteration 2800. $ Y $ are represented in blue points, and they are near $c_1 = 0$.
   $ D $ is near the optimal $ D^*(s_{1 \rm a}) $
   since $ D( 0 )  \approx 1/3 $ and $ D( 4) \approx 1. $
  Interestingly, this bad attractor $ ( Y, D ) $ is similar to the one discussed in  Fig.~\ref{fig1:fig_two_views}, so the intuition of ``local-min''  is verified in (c). 
    }
    \label{fig5 D loss and D image}
\vspace{-0.3cm}
\end{figure}

\textbf{Effect of width:} 
We see a clear effect of width on convergence speed.
As the networks become wider, both JS-GAN and RS-GAN
 converge faster. We find that the reason of faster convergence
is because wider nets make JS-GAN escape mode collapse faster.
 See details in Appendix~\ref{subsec: details of experiments}.

More experiment details and findings are 
presented in  Appendix~\ref{subsec: details of experiments}.

\iffalse 
We draw the D loss for one such case in Fig. \ref{fig5 D loss and D image}(c).
We observe that for JS-GAN the D loss is concentrated around value $0.48$
from 5k to 50k iterations, indicating that the training was constantly trapped
by the bad basin at $( s_{1a}), D^*( s_{1a}) )  $. 
In contrast, RS-GAN recovers the two modes within 2k iterations. 
This experiment verifies that the mode collapse can
possibly trap the training for a very long time (if not forever),
even with inexact $D$ updates in practice. 
\fi

\iffalse 
We suspect this is due to two possible reasons.
First, width can create more paths on the landscape, allowing higher chance
to escape a bad attracting region.
Second, width can also allow the generated points to travel faster (intuition of \cite{jacot2018neural}).
\fi

\iffalse 
Again, we emphasize that in most scenarios, 
we do not expect $Y$-space basins to completely trap the iterates,
but rather to slow down the training process. 
\fi

\begin{table}[t]
\centering
\setlength{\tabcolsep}{2pt}
\scalebox{0.8}{
\begin{tabular}{l|cccc|cccc}
\toprule
&  \multicolumn{4}{c}{ \bf{CIFAR-10}}     & \multicolumn{4}{c}{ \bf{STL-10}} \\
          & Inception Score $\uparrow$ &  FID $\downarrow$ & FID Gap & Model size & Inception Score $\uparrow$ &  FID $\downarrow$ & FID Gap & Model size\\ \midrule
 Real Dataset  &  11.24$\pm$0.19    &   5.18  & &  & 24.45$\pm$0.41    &  5.34 & \\
 \hline 
 \multicolumn{5}{l}{\bf{Standard CNN} }  \\
 WGAN-GP  & 6.68$\pm$0.06  &  39.66  &  & &  8.11$\pm$0.09  &  55.64  &  &   \\
JS-GAN  & 6.27$\pm$0.10 &  49.13  &  \multirow{2}{*}{15.34} & \multirow{2}{*}{100\%} & 8.01$\pm$0.07    &    50.38  & \multirow{2}{*}{2.16} & \multirow{2}{*}{100\%} \\
RS-GAN  & 7.02$\pm$0.07       &  33.79   &  &  &  7.62$\pm$0.08    &   52.54   &  &  \\ %
JS-GAN+ SN  & 7.42$\pm$0.08  &  28.07  & \multirow{2}{*}{0.91} & \multirow{2}{*}{100\%}  &  8.32$\pm$0.10    &  44.06   & \multirow{2}{*}{0.18} & \multirow{2}{*}{100\%}    \\
RS-GAN+ SN  &  7.32$\pm$0.08     &  27.16   &  &  &  8.29$\pm$0.13    &  43.88  &  & \\ 
JS-GAN+SN; GD channel/2  & 6.85$\pm$0.08  &  33.90   &  \multirow{2}{*}{1.16} & \multirow{2}{*}{29.0\%}  & 7.69$\pm$0.05   &   57.16   & \multirow{2}{*}{4.69} & \multirow{2}{*}{32.9\%}   \\
RS-GAN+SN; GD channel/2  & 6.74$\pm$0.04       &  32.74   &  &   &  7.95$\pm$0.10    &  52.47   &  &  \\ 
JS-GAN + SN; GD channel/4  & 5.83$\pm$0.07  &  52.63  &  \multirow{2}{*}{7.26} & \multirow{2}{*}{9.2\%} & 6.90$\pm$0.06     &  72.96     & \multirow{2}{*}{9.35}  &\multirow{2}{*}{11.9\%}   \\
RS-GAN + SN; GD channel/4  &  5.94$\pm$0.09  &  45.37  &  &    &  7.27$\pm$0.11    & 63.61 &  &  \\\hline
\multicolumn{5}{l}{\bf{ResNet}}  \\
JS-GAN+ SN  &  8.12$\pm$0.14   &  20.13   & \multirow{2}{*}{0.82} & \multirow{2}{*}{100\%}   & 8.87$\pm$0.07 &  36.33  & \multirow{2}{*}{1.56} & \multirow{2}{*}{100\% }    \\
RS-GAN + SN  &   7.92$\pm$0.13    & 19.31  &  &    &  8.96$\pm$0.10 & 34.77  &  &   \\
JS-GAN + SN; GD channel/2  & 7.67$\pm$0.04 &  23.29   & \multirow{2}{*}{1.51}  &\multirow{2}{*}{27.5\%}   &   8.45$\pm$0.05  & 44.39   & \multirow{2}{*}{2.21} &\multirow{2}{*}{29.0\% } \\
RS-GAN + SN; GD channel/2  &  7.63$\pm$0.07     &  21.78  &  &   &  8.47$\pm$0.09    & 42.18  &  &  \\
JS-GAN + SN; GD channel/4  &  6.65$\pm$0.06   &  45.20   & \multirow{2}{*}{13.94} &\multirow{2}{*}{10.4\%}  &8.21 $\pm$0.12  & 53.57  & \multirow{2}{*}{1.48} &\multirow{2}{*}{9.2\%}  \\
RS-GAN+ SN; GD channel/4  & 7.08$\pm$0.05       &  31.26   & &   &  8.46$\pm$0.11   & 52.09   & & \\ 
JS-GAN + SN; BottleNeck   & 7.60$\pm$0.07  &  26.98   &  \multirow{2}{*}{1.54}  & \multirow{2}{*}{16.8\%}  &   8.29$\pm$0.05   &  50.38  & \multirow{2}{*}{3.80} &\multirow{2}{*}{19.2\%}  \\
RS-GAN+ SN; BottleNeck   &  7.57$\pm$0.09     &   25.44  &  &   &  8.52$\pm$0.11    &  46.58  & & \\
\bottomrule
\end{tabular}
}
\caption{Inception score (IS) (higher is better) and Frech\'et Inception distance (FID) (lower is better)  for JS-GAN, WGAN-GP and RS-GAN on CIFAR-10 and STL-10. 
We also show FID gap between JS-GAN and RS-GAN, and show the relative model size of narrow nets vs.\ regular nets (``regular'': CNN and ResNet of \cite{miyato2018spectral}).  }
\label{Tab:metric}
\vspace{-0.6cm}
\end{table}

\vspace{-0.2cm}
\section{Real Data Experiments}
\label{sec:experiments}
\vspace{-0.2cm}

RpGANs have been tested by  \citet{jolicoeur2018relativistic}, and are shown to be better than their SepGAN counterparts in a variety of settings\footnote{That paper tested a number of variants, and some of them are not directly covered by our results. }.
In addition, RpGAN and its variants have been used in super-resolution (ESRGAN) \cite{wang2018esrgan} and a few recent GANs \cite{xiangli2020real,berthelot2020creating}.
Therefore, the effectiveness of RpGANs has been justified  to some extent. We do not attempt to re-run the experiments merely for the purpose of justification. 
Instead, our goal is to use experiments to support our landscape theory.

Based on the discussions in Sec. \ref{sec: empirical and population},
Sec.~\ref{subsec: main results}
and Sec.~\ref{sec: case study}, 
 we conjecture that 
RpGANs have a bigger advantage over SepGAN (A) with narrow deep nets, (B) in high resolution image generation, (C) with imbalanced data.
Finally, (D) there exists some bad initial $ D $ that makes
  SepGANs  much worse than RpGANs.
  In the main text, we present results on
  the logistic loss (i.e., JS-GAN and RS-GAN). 
  Results on other losses are given in the appendix.

  \iffalse 
  We also show that the benefit of EMA (exponential moving average)~\cite{yazici2018unusual} and RS-GAN can be additive. 
In the appendix, we test two more corollaries of our landscape theory:
faster convergence of RS-GAN, and a large performance gap 
for a certain initial $D$. 
These experiments  support  our conjecture that 
RS-GAN has a better landscape than JS-GAN which leads to a performance gap.
\fi 

\textbf{Experimental setting for (A).}
For (A), we test on CIFAR-10 and STL-10 data. 
For the optimizer, we use Adam with the discriminator's learning rate  $0.0002$. For CIFAR-10 on ResNet, we set $\beta_1=0$ and $\beta_2=0.9$ in Adam; for others, $\beta_1=0.5$ and $\beta_2=0.999$. 
We tune the generator's learning rate and run $100k$ iterations in total. 
We report the Inception score (IS) and Frech\'et Inception distance (FID). IS and FID are evaluated on $50k$ and $10k$ samples respectively. 
More details of the setting are shown in Appendix \ref{expsetting},
and the experimental settings for other cases besides (A) are shown
 in the corresponding parts in the appendix. 
Generated images are shown in Appendix \ref{sec: high resolution}.

\textbf{Regular architecture and effect of spectral norm (SN).}
We use the two neural architectures in~\cite{miyato2018spectral}: 
standard CNN and ResNet, and report results in Table~\ref{Tab:metric}.
First, without spectral normalization (SN),
 RS-GAN achieves much higher accuracy than JS-GAN and WGAN-GP on CIFAR-10.
Second, with SN, RS-GAN achieves 1-2 points lower FID 
 score than JS-GAN, i.e., it's slightly better.
 We suspect that SN smoothens the landscape, thus greatly reducing the gap between JS-GAN and RS-GAN. Note that the scores of JS-GAN and WGAN-GP (both without and with SN) are comparable to or better than the scores in Table 2 of \citet{miyato2018spectral}.

 \textbf{Narrow nets.}
For both CNN and ResNet, we reduce the number of channels for all convolutional layers in the  generator and  discriminator to (1) half, (2) quarter and (3) bottleneck (for ResNet structure). The experimental results are provided in Table~\ref{Tab:metric}.
We consider the gap between RS-GAN and JS-GAN for regular width
as a baseline. 
For narrow nets,
the gap between RS-GAN and JS-GAN is similar or larger in most cases,
and can be much larger (e.g. $ > 13 $ FID) in some cases.
The fluctuations in the gaps are consistent with landscape theory: if JS-GAN training  gets stuck at a bad basin then the performance is bad; if it converges to a good basin, then the performance is reasonably good. 
In CIFAR-10, compared to SN-GAN with the conventional ResNet (FID=20.13), we can achieve a relatively close result by using RS-GAN with 28\% parameters (half channel, FID=21.78).

\iffalse 
\footnote{
Many empirical works  design compact neural-nets, e.g., network pruning \cite{han2015deep} and architecture search \cite{tan2019efficientnet}.
In contrast, the study of compact GANs is relatively scarce (e.g., \cite{li2020gan}).
See the appendix for a discussion. 
}
\fi

\textbf{High resolution data experiments.} Sec.~\ref{sec: empirical and population} discusses that the non-convexity of JS-GAN will become a more severe issue when the number of samples is limited compared to the data space (e.g., high resolution space or limited data points). We conduct experiments with LSUN Church and Tower images of size $256\times256$. %
RS-GAN can generate higher visual quality images than JS-GAN (Appendix~\ref{sec: high resolution}). Similarly, using another model architecture,  \cite{jolicoeur2018relativistic} achieves a better FID score with RSGAN on the CAT dataset, which contains a small number of images (e.g., 2k $256 \times 256$ images).

\textbf{Imbalanced data experiments.} 
For imbalanced data, we find more evidence for the existence of JS-GAN's bad basins.The reason:  JS-GAN would have a deeper bad basin, and hence a higher  chance to get stuck. We conduct ablation experiments on 2-cluster data and MNIST. Both cases show that JS-GAN ends up with mode collapse while RS-GAN can generate data with proportions similar to the imbalanced true data. Check Appendix~\ref{sec: imbalance} for more.

\begin{wrapfigure}{r}{0.35\textwidth}
\centering
\vspace{-0.3cm}
\includegraphics[width=1\linewidth]{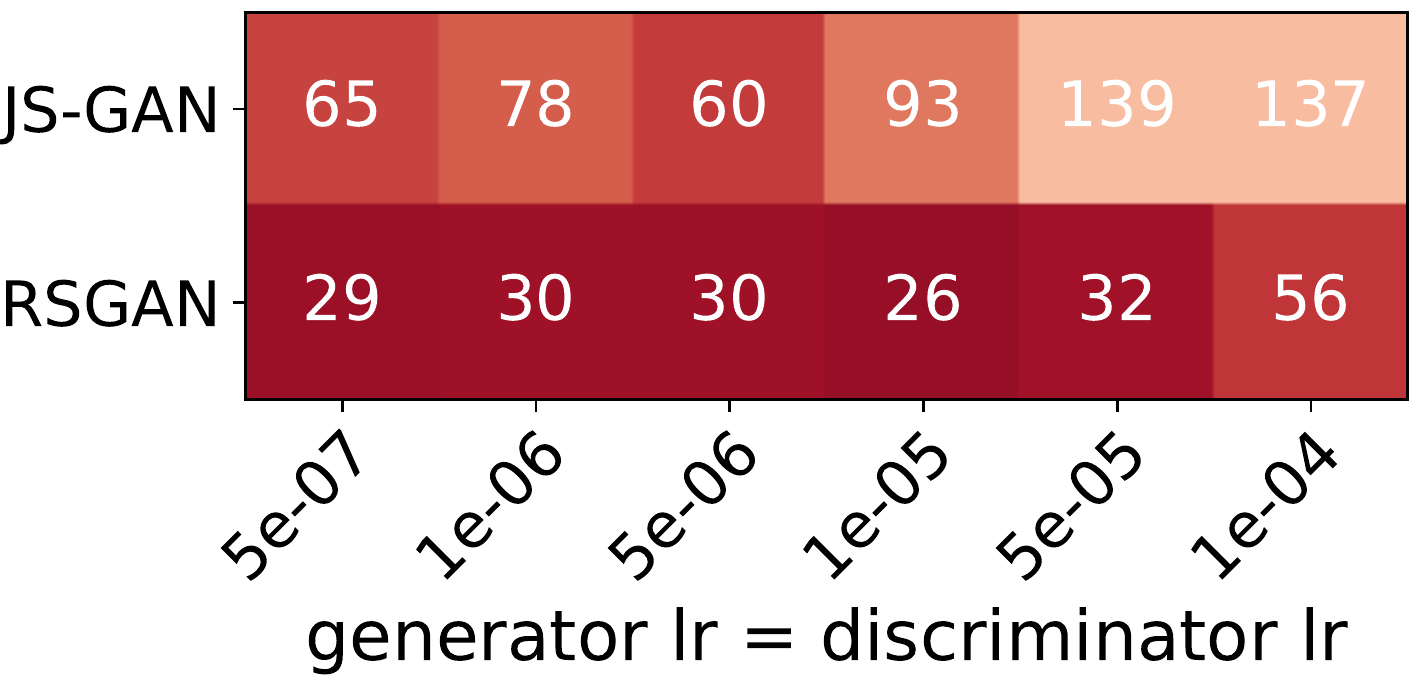}
   \label{fig8: MNIST}
\vspace{-0.7cm}
\end{wrapfigure}
\textbf{Bad initial point experiments.} 
A better landscape is more robust to initialization. %
On MNIST data, we find  a  discriminator (not random)
which permits RS-GAN to converge to a much better solution 
than JS-GAN when used as the starting point.
The FID scores are reported in the  table to the right. The gap is at least 30 FID scores (a much higher gap than the gap for a random initialization). 
Check Appendix \ref{app: bad initial experiment} for more. 

\iffalse
\begin{figure}[htbp!]
\centering
\vspace{-0.3cm}
\includegraphics[width=0.4\linewidth]{figure/MNSIT_FID_goodGcrapyD.pdf}
   \label{fig8: MNIST}
\vspace{-0.3cm}
\end{figure}
\fi

\textbf{Combining with EMA.}
It is known that non-convergence can be alleviated via EMA \cite{yazici2018unusual},
and our theory predicts that the global landscape issue can be alleviated by RpGAN. Non-convergence and global landscape are orthogonal: no matter whether iterates are near a sub-optimal local basin or a globally-optimal basin, the algorithm may cycle. 
Therefore, we conjecture that the effect of EMA 
and the effect of RS-GAN are ``additive''.
Our simulations show that EMA can improve both JS-GAN and RS-GAN,
and the gap is approximately preserved after adding EMA. 
Combining EMA and RS-GAN, we achieve a similar result to the baseline (JS-GAN + SN, no EMA, FID = 20.13) using 16.8\% parameters (Resnet with bottleneck plus EMA, FID=21.38). See Appendix~\ref{expsetting} for more.

\textbf{General RpGAN:} %
We conduct additional experiments on other losses, including hinge loss and least squares loss. See Appendix~\ref{sec: hingereal} and~\ref{sec: lsreal} for more.

\vspace{-0.2cm}

\section{Conclusion}
\vspace{-0.3cm}

Global optimization landscape,
together with statistical analysis and convergence analysis, 
are important theoretical angles. %
In this work, we study the global landscape of GANs. Our major questions are: (1) Does the original JS-GAN formulation have a good landscape? (2) If not, 
is there a simple way to improve the landscape in theory?
(3) Does the improved landscape lead to better performance?
 First, studying the empirical versions of SepGAN (extension of
 JS-GAN) we prove that it has exponentially many bad basins,
 which are mode-collapse patterns. 
 Second, we prove that a simple coupling idea (resulting in RpGAN)
 can remove  bad basins in theory. 
 Finally, we verify a few predictions based on the landscape
 theory, e.g., RS-GAN has a bigger advantage over JS-GAN
for narrow nets.

 \newpage
 \section*{Acknowledgements} This work is supported in part by NSF under Grant $\#$ 1718221, 2008387, 1755847 and MRI $\#$1725729, and NIFA award 2020-67021-32799. We thank Sewoong Oh for pointing
out the connection of the earlier version of our work to \cite{jolicoeur2018relativistic}.
 
 \iffalse 
 formulations. We assume that the true data distribution is the empirical distribution consisting of $n $ points. In this formulation,  the samples are moving during the optimization process, which is what is happening in practical training. This analysis  also mimics an $n$-mode distribution. In fact, it captures the macro-learning part of the learning process. 

We show that using this perspective,  the JS-GAN formulation has exponentially many bad basins, and they correspond to mode collapse. We also show that  RS-GAN does not have bad basins. Further, we analyze the training dynamics for  log-linear discriminators and show that there is a global Lyapunov function for RS-GAN. Our simulation results on synthetic data and real data also provide some empirical evidence that RS-GAN has a better landscape than JS-GAN.
\fi 

 \ifNIPSonly

\clearpage

\section*{Broader Impact}
Generative adversarial nets (GANs) are an important tool for modeling of high-dimensional distributions. However, the theoretical understanding of GANs is still limited. While recent work has investigated GANs form a statistical and an optimization viewpoint, little is known about their global landscape.

This paper is a first step to add theory about the global landscape of GANs. We think this research will have a societal impact as it enables practitioners to make a more informed decision about the type of loss function that should be optimized.

For example, we show that RS-GAN has benefits: (1) fewer bad basins, permitting a more stable optimization; (2) better results for narrow deep net generators, permitting its use on smaller devices, and promoting the development of smart devices and smart home services; (3) better performance on high-resolution image generation, which can be helpful in the fashion, animation, film and television industries. 

Since we are focusing on the optimization of GANs,  we do not think that this research has any ethical disadvantages beyond those of GANs. Illegal fake images or videos may be the main concern related to GAN.

We hope that the research community finds these results interesting and that people will join our effort in uncovering more theory for GANs.

\fi 

{\footnotesize
\bibliography{reference}

\begin{thebibliography}{92}
\providecommand{\natexlab}[1]{#1}
\providecommand{\url}[1]{\texttt{#1}}
\expandafter\ifx\csname urlstyle\endcsname\relax
  \providecommand{\doi}[1]{doi: #1}\else
  \providecommand{\doi}{doi: \begingroup \urlstyle{rm}\Url}\fi

\bibitem[Adler and Lunz(2018)]{adler2018banach}
J.~Adler and S.~Lunz.
\newblock Banach wasserstein gan.
\newblock In \emph{NeurIPS}, 2018.

\bibitem[Allen-Zhu et~al.(2019)Allen-Zhu, Li, and Song]{allen2018convergence}
Z.~Allen-Zhu, Y.~Li, and Z.~Song.
\newblock A convergence theory for deep learning via over-parameterization.
\newblock In \emph{ICML}, 2019.

\bibitem[Arjovsky and Bottou(2017)]{arjovsky2017towards}
M.~Arjovsky and L.~Bottou.
\newblock Towards principled methods for training generative adversarial
  networks.
\newblock In \emph{ICLR}, 2017.

\bibitem[Arjovsky et~al.(2017)Arjovsky, Chintala, and
  Bottou]{arjovsky2017wasserstein}
M.~Arjovsky, S.~Chintala, and L.~Bottou.
\newblock Wasserstein gan.
\newblock In \emph{ICML}, 2017.

\bibitem[Arora et~al.(2017)Arora, Ge, Liang, Ma, and
  Zhang]{arora2017generalization}
S.~Arora, R.~Ge, Y.~Liang, T.~Ma, and Y.~Zhang.
\newblock Generalization and equilibrium in generative adversarial nets
  ({GAN}s).
\newblock In \emph{ICML}, 2017.

\bibitem[Azizian et~al.(2019)Azizian, Mitliagkas, Lacoste-Julien, and
  Gidel]{azizian2019tight}
W.~Azizian, I.~Mitliagkas, S.~Lacoste-Julien, and G.~Gidel.
\newblock A tight and unified analysis of extragradient for a whole spectrum of
  differentiable games.
\newblock \emph{arXiv preprint arXiv:1906.05945}, 2019.

\bibitem[Bai et~al.(2018)Bai, Ma, and Risteski]{bai2018approximability}
Y.~Bai, T.~Ma, and A.~Risteski.
\newblock Approximability of discriminators implies diversity in gans.
\newblock \emph{arXiv preprint arXiv:1806.10586}, 2018.

\bibitem[Balduzzi et~al.(2018)Balduzzi, Racaniere, Martens, Foerster, Tuyls,
  and Graepel]{balduzzi2018mechanics}
D.~Balduzzi, S.~Racaniere, J.~Martens, J.~Foerster, K.~Tuyls, and T.~Graepel.
\newblock The mechanics of n-player differentiable games.
\newblock \emph{arXiv preprint arXiv:1802.05642}, 2018.

\bibitem[Bartlett et~al.(2017)Bartlett, Foster, and
  Telgarsky]{bartlett2017spectrally}
P.~L. Bartlett, D.~J. Foster, and M.~J. Telgarsky.
\newblock Spectrally-normalized margin bounds for neural networks.
\newblock In \emph{NeurIPS}, 2017.

\bibitem[Bengio and LeCun(2007)]{Bengio+chapter2007}
Y.~Bengio and Y.~LeCun.
\newblock Scaling learning algorithms towards {AI}.
\newblock In \emph{Large Scale Kernel Machines}. MIT Press, 2007.

\bibitem[Berard et~al.(2019)Berard, Gidel, Almahairi, Vincent, and
  Lacoste-Julien]{berard2019closer}
H.~Berard, G.~Gidel, A.~Almahairi, P.~Vincent, and S.~Lacoste-Julien.
\newblock A closer look at the optimization landscapes of generative
  adversarial networks.
\newblock \emph{arXiv preprint arXiv:1906.04848}, 2019.

\bibitem[Berthelot et~al.(2017)Berthelot, Schumm, and Metz]{berthelot2017began}
D.~Berthelot, T.~Schumm, and L.~Metz.
\newblock Began: Boundary equilibrium generative adversarial networks.
\newblock \emph{arXiv preprint arXiv:1703.10717}, 2017.

\bibitem[Berthelot et~al.(2020)Berthelot, Milanfar, and
  Goodfellow]{berthelot2020creating}
D.~Berthelot, P.~Milanfar, and I.~Goodfellow.
\newblock Creating high resolution images with a latent adversarial generator.
\newblock \emph{arXiv preprint arXiv:2003.02365}, 2020.

\bibitem[Bhojanapalli et~al.(2016)Bhojanapalli, Neyshabur, and
  Srebro]{bhojanapalli2016global}
S.~Bhojanapalli, B.~Neyshabur, and N.~Srebro.
\newblock Global optimality of local search for low rank matrix recovery.
\newblock In \emph{NeurIPS}, 2016.

\bibitem[Bianchini and Gori(1996)]{bianchini1996optimal}
M.~Bianchini and M.~Gori.
\newblock Optimal learning in artificial neural networks: A review of
  theoretical results.
\newblock \emph{Neurocomputing}, 1996.

\bibitem[Bi{\'n}kowski et~al.(2018)Bi{\'n}kowski, Sutherland, Arbel, and
  Gretton]{binkowski2018demystifying}
M.~Bi{\'n}kowski, D.~J. Sutherland, M.~Arbel, and A.~Gretton.
\newblock Demystifying mmd gans.
\newblock In \emph{ICLR}, 2018.

\bibitem[Bovier et~al.(2004)Bovier, Eckhoff, Gayrard, and
  Klein]{bovier2004metastability}
A.~Bovier, M.~Eckhoff, V.~Gayrard, and M.~Klein.
\newblock Metastability in reversible diffusion processes i. sharp asymptotics
  for capcities and exit times.
\newblock \emph{JEMS}, 2004.

\bibitem[Brock et~al.(2018)Brock, Donahue, and Simonyan]{brock2018large}
A.~Brock, J.~Donahue, and K.~Simonyan.
\newblock Large scale gan training for high fidelity natural image synthesis.
\newblock \emph{arXiv preprint arXiv:1809.11096}, 2018.

\bibitem[Chi et~al.(2019)Chi, Lu, and Chen]{chi2019nonconvex}
Y.~Chi, Y.~M. Lu, and Y.~Chen.
\newblock Nonconvex optimization meets low-rank matrix factorization: An
  overview.
\newblock \emph{IEEE Transactions on Signal Processing}, 67\penalty0
  (20):\penalty0 5239--5269, 2019.

\bibitem[Chu et~al.(2019)Chu, Blanchet, and Glynn]{chu2019probability}
C.~Chu, J.~Blanchet, and P.~Glynn.
\newblock Probability functional descent: A unifying perspective on gans,
  variational inference, and reinforcement learning.
\newblock \emph{arXiv preprint arXiv:1901.10691}, 2019.

\bibitem[Cully et~al.(2017)Cully, Chang, and Demiris]{cully2017magan}
R.~W.~A. Cully, H.~J. Chang, and Y.~Demiris.
\newblock Magan: Margin adaptation for generative adversarial networks.
\newblock \emph{arXiv preprint arXiv:1704.03817}, 2017.

\bibitem[Daskalakis and Panageas(2018)]{daskalakis2018limit}
C.~Daskalakis and I.~Panageas.
\newblock The limit points of (optimistic) gradient descent in min-max
  optimization.
\newblock In \emph{NeurIPS}, 2018.

\bibitem[Daskalakis et~al.(2018)Daskalakis, Ilyas, Syrgkanis, and
  Zeng]{daskalakis2017training}
C.~Daskalakis, A.~Ilyas, V.~Syrgkanis, and H.~Zeng.
\newblock Training gans with optimism.
\newblock In \emph{ICLR}, 2018.

\bibitem[Deshpande et~al.(2018)Deshpande, Zhang, and
  Schwing]{deshpande2018generative}
I.~Deshpande, Z.~Zhang, and A.~Schwing.
\newblock Generative modeling using the sliced wasserstein distance.
\newblock In \emph{CVPR}, 2018.

\bibitem[Deshpande et~al.(2019)Deshpande, Hu, Sun, Pyrros, Siddiqui, Koyejo,
  Zhao, Forsyth, and Schwing]{IDeshpandeCVPR2019}
I.~Deshpande, Y.-T. Hu, R.~Sun, A.~Pyrros, N.~Siddiqui, S.~Koyejo, Z.~Zhao,
  D.~Forsyth, and A.~G. Schwing.
\newblock {Max-Sliced Wasserstein Distance and its use for GANs}.
\newblock In \emph{CVPR}, 2019.

\bibitem[Ding et~al.(2019)Ding, Li, and Sun]{ding2019sub}
T.~Ding, D.~Li, and R.~Sun.
\newblock Sub-optimal local minima exist for almost all over-parameterized
  neural networks.
\newblock \emph{arXiv preprint arXiv:1911.01413}, 2019.

\bibitem[Du et~al.(2018)Du, Lee, Li, Wang, and Zhai]{du2018gradient}
S.~S. Du, J.~D. Lee, H.~Li, L.~Wang, and X.~Zhai.
\newblock Gradient descent finds global minima of deep neural networks.
\newblock \emph{arXiv preprint arXiv:1811.03804}, 2018.

\bibitem[Farnia and Ozdaglar(2020)]{farnia2020gans}
F.~Farnia and A.~Ozdaglar.
\newblock Gans may have no nash equilibria.
\newblock \emph{arXiv preprint arXiv:2002.09124}, 2020.

\bibitem[Farnia and Tse(2018)]{farnia2018convex}
F.~Farnia and D.~Tse.
\newblock A convex duality framework for gans.
\newblock In \emph{NeurIPS}, 2018.

\bibitem[Feizi et~al.(2017)Feizi, Farnia, Ginart, and
  Tse]{feizi2017understanding}
S.~Feizi, F.~Farnia, T.~Ginart, and D.~Tse.
\newblock Understanding gans: the lqg setting.
\newblock \emph{arXiv preprint arXiv:1710.10793}, 2017.

\bibitem[Ge et~al.(2016)Ge, Lee, and Ma]{ge2016matrix}
R.~Ge, J.~D. Lee, and T.~Ma.
\newblock Matrix completion has no spurious local minimum.
\newblock In \emph{NeurIPS}, 2016.

\bibitem[Geiger et~al.(2018)Geiger, Spigler, d'Ascoli, Sagun, Baity-Jesi,
  Biroli, and Wyart]{geiger2018jamming}
M.~Geiger, S.~Spigler, S.~d'Ascoli, L.~Sagun, M.~Baity-Jesi, G.~Biroli, and
  M.~Wyart.
\newblock The jamming transition as a paradigm to understand the loss landscape
  of deep neural networks.
\newblock \emph{arXiv preprint arXiv:1809.09349}, 2018.

\bibitem[Gidel et~al.(2018)Gidel, Berard, Vignoud, Vincent, and
  Lacoste-Julien]{gidel2018variational}
G.~Gidel, H.~Berard, G.~Vignoud, P.~Vincent, and S.~Lacoste-Julien.
\newblock A variational inequality perspective on generative adversarial
  networks.
\newblock \emph{arXiv preprint arXiv:1802.10551}, 2018.

\bibitem[Gidel et~al.(2019)Gidel, Hemmat, Pezeshki, Lepriol, Huang,
  Lacoste-Julien, and Mitliagkas]{gidel2018negative}
G.~Gidel, R.~A. Hemmat, M.~Pezeshki, R.~Lepriol, G.~Huang, S.~Lacoste-Julien,
  and I.~Mitliagkas.
\newblock Negative momentum for improved game dynamics.
\newblock In \emph{AISTATS}, 2019.

\bibitem[Goodfellow et~al.(2014)Goodfellow, Pouget-Abadie, Mirza, Xu,
  Warde-Farley, Ozair, Courville, and Bengio]{goodfellow2014generative}
I.~Goodfellow, J.~Pouget-Abadie, M.~Mirza, B.~Xu, D.~Warde-Farley, S.~Ozair,
  A.~Courville, and Y.~Bengio.
\newblock Generative adversarial nets.
\newblock In \emph{NeurIPS}, 2014.

\bibitem[Gulrajani et~al.(2017)Gulrajani, Ahmed, Arjovsky, Dumoulin, and
  Courville]{gulrajani2017improved}
I.~Gulrajani, F.~Ahmed, M.~Arjovsky, V.~Dumoulin, and A.~Courville.
\newblock Improved training of wasserstein gans.
\newblock In \emph{NeurIPS}, 2017.

\bibitem[Huang et~al.(2017)Huang, Li, Poursaeed, Hopcroft, and
  Belongie]{huang2016stacked}
X.~Huang, Y.~Li, O.~Poursaeed, J.~Hopcroft, and S.~Belongie.
\newblock Stacked generative adversarial networks.
\newblock In \emph{CVPR}, 2017.

\bibitem[Jacot et~al.(2018)Jacot, Gabriel, and Hongler]{jacot2018neural}
A.~Jacot, F.~Gabriel, and C.~Hongler.
\newblock Neural tangent kernel: Convergence and generalization in neural
  networks.
\newblock In \emph{NeurIPS}, 2018.

\bibitem[Jin et~al.(2019)Jin, Netrapalli, and Jordan]{jin2019minmax}
C.~Jin, P.~Netrapalli, and M.~I. Jordan.
\newblock Minmax optimization: Stable limit points of gradient descent ascent
  are locally optimal.
\newblock \emph{arXiv preprint arXiv:1902.00618}, 2019.

\bibitem[Johnson and Zhang(2019)]{johnson2019framework}
R.~Johnson and T.~Zhang.
\newblock A framework of composite functional gradient methods for generative
  adversarial models.
\newblock \emph{IEEE transactions on pattern analysis and machine
  intelligence}, 2019.

\bibitem[Jolicoeur-Martineau(2018)]{jolicoeur2018relativistic}
A.~Jolicoeur-Martineau.
\newblock The relativistic discriminator: a key element missing from standard
  gan.
\newblock In \emph{ICLR}, 2018.

\bibitem[Jolicoeur-Martineau(2019)]{JolicoeurMartineau2019OnRF}
A.~Jolicoeur-Martineau.
\newblock On relativistic f-divergences.
\newblock In \emph{ICML}, 2019.

\bibitem[Karras et~al.(2019)Karras, Laine, and Aila]{karras2019style}
T.~Karras, S.~Laine, and T.~Aila.
\newblock A style-based generator architecture for generative adversarial
  networks.
\newblock In \emph{Proceedings of the IEEE conference on computer vision and
  pattern recognition}, pages 4401--4410, 2019.

\bibitem[Karras et~al.(2020)Karras, Laine, Aittala, Hellsten, Lehtinen, and
  Aila]{karras2020analyzing}
T.~Karras, S.~Laine, M.~Aittala, J.~Hellsten, J.~Lehtinen, and T.~Aila.
\newblock Analyzing and improving the image quality of stylegan.
\newblock In \emph{Proceedings of the IEEE/CVF Conference on Computer Vision
  and Pattern Recognition}, pages 8110--8119, 2020.

\bibitem[Kodali et~al.(2017)Kodali, Abernethy, Hays, and
  Kira]{kodali2017convergence}
N.~Kodali, J.~Abernethy, J.~Hays, and Z.~Kira.
\newblock On convergence and stability of gans.
\newblock \emph{arXiv preprint arXiv:1705.07215}, 2017.

\bibitem[Kolouri et~al.(2018)Kolouri, Martin, and Rohde]{kolouri2018sliced}
S.~Kolouri, C.~E. Martin, and G.~K. Rohde.
\newblock Sliced-wasserstein autoencoder: An embarrassingly simple generative
  model.
\newblock \emph{arXiv preprint arXiv:1804.01947}, 2018.

\bibitem[Lee et~al.(2019)Lee, Xiao, Schoenholz, Bahri, Novak, Sohl-Dickstein,
  and Pennington]{lee2019wide}
J.~Lee, L.~Xiao, S.~Schoenholz, Y.~Bahri, R.~Novak, J.~Sohl-Dickstein, and
  J.~Pennington.
\newblock Wide neural networks of any depth evolve as linear models under
  gradient descent.
\newblock In \emph{Advances in neural information processing systems}, pages
  8572--8583, 2019.

\bibitem[Lei et~al.(2019)Lei, Lee, Dimakis, and Daskalakis]{lei2019sgd}
Q.~Lei, J.~D. Lee, A.~G. Dimakis, and C.~Daskalakis.
\newblock Sgd learns one-layer networks in wgans.
\newblock \emph{arXiv preprint arXiv:1910.07030}, 2019.

\bibitem[Li et~al.(2017{\natexlab{a}})Li, Chang, Cheng, Yang, and
  P{\'o}czos]{li2017mmd}
C.-L. Li, W.-C. Chang, Y.~Cheng, Y.~Yang, and B.~P{\'o}czos.
\newblock Mmd gan: Towards deeper understanding of moment matching network.
\newblock In \emph{NeurIPS}, 2017{\natexlab{a}}.

\bibitem[Li et~al.(2018{\natexlab{a}})Li, Ding, and Sun]{li2018over}
D.~Li, T.~Ding, and R.~Sun.
\newblock Over-parameterized deep neural networks have no strict local minima
  for any continuous activations.
\newblock \emph{arXiv preprint arXiv:1812.11039}, 2018{\natexlab{a}}.

\bibitem[Li et~al.(2017{\natexlab{b}})Li, Madry, Peebles, and
  Schmidt]{li2017limitations}
J.~Li, A.~Madry, J.~Peebles, and L.~Schmidt.
\newblock On the limitations of first-order approximation in gan dynamics.
\newblock \emph{arXiv preprint arXiv:1706.09884}, 2017{\natexlab{b}}.

\bibitem[Li et~al.(2018{\natexlab{b}})Li, Madry, Peebles, and
  Schmidt]{li2017towards}
J.~Li, A.~Madry, J.~Peebles, and L.~Schmidt.
\newblock Towards understanding the dynamics of generative adversarial
  networks.
\newblock In \emph{ICML}, 2018{\natexlab{b}}.

\bibitem[Li and Malik(2018)]{li2018implicit}
K.~Li and J.~Malik.
\newblock Implicit maximum likelihood estimation.
\newblock \emph{arXiv preprint arXiv:1809.09087}, 2018.

\bibitem[Li et~al.(2017{\natexlab{c}})Li, Schwing, Wang, and Zemel]{LiNIPS2017}
Y.~Li, A.~G. Schwing, K.-C. Wang, and R.~Zemel.
\newblock {Dualing GANs}.
\newblock In \emph{NeurIPS}, 2017{\natexlab{c}}.

\bibitem[Liang et~al.(2018{\natexlab{a}})Liang, Sun, Lee, and
  Srikant]{liang2018adding}
S.~Liang, R.~Sun, J.~D. Lee, and R.~Srikant.
\newblock Adding one neuron can eliminate all bad local minima.
\newblock In \emph{Advances in Neural Information Processing Systems}, pages
  4350--4360, 2018{\natexlab{a}}.

\bibitem[Liang et~al.(2018{\natexlab{b}})Liang, Sun, Li, and
  Srikant]{liang2018understanding}
S.~Liang, R.~Sun, Y.~Li, and R.~Srikant.
\newblock Understanding the loss surface of neural networks for binary
  classification.
\newblock \emph{arXiv preprint arXiv:1803.00909}, 2018{\natexlab{b}}.

\bibitem[Liang et~al.(2019)Liang, Sun, and Srikant]{liang2019revisiting}
S.~Liang, R.~Sun, and R.~Srikant.
\newblock Revisiting landscape analysis in deep neural networks: Eliminating
  decreasing paths to infinity.
\newblock \emph{arXiv preprint arXiv:1912.13472}, 2019.

\bibitem[Lin et~al.(2018)Lin, Khetan, Fanti, and Oh]{lin2018pacgan}
Z.~Lin, A.~Khetan, G.~Fanti, and S.~Oh.
\newblock Pacgan: The power of two samples in generative adversarial networks.
\newblock In \emph{NeurIPS}, 2018.

\bibitem[Liu and Chaudhuri(2018)]{liu2018inductive}
S.~Liu and K.~Chaudhuri.
\newblock The inductive bias of restricted f-gans.
\newblock \emph{arXiv preprint arXiv:1809.04542}, 2018.

\bibitem[Livni et~al.(2014)Livni, Shalev-Shwartz, and
  Shamir]{livni2014computational}
R.~Livni, S.~Shalev-Shwartz, and O.~Shamir.
\newblock On the computational efficiency of training neural networks.
\newblock In \emph{NeurIPS}, 2014.

\bibitem[Makkuva et~al.(2019)Makkuva, Taghvaei, Oh, and
  Lee]{makkuva2019optimal}
A.~V. Makkuva, A.~Taghvaei, S.~Oh, and J.~D. Lee.
\newblock Optimal transport mapping via input convex neural networks.
\newblock \emph{arXiv preprint arXiv:1908.10962}, 2019.

\bibitem[{Mao} et~al.(2016){Mao}, {Li}, {Xie}, {Lau}, {Wang}, and
  {Smolley}]{lsgan2016mao}
X.~{Mao}, Q.~{Li}, H.~{Xie}, R.~Y.~K. {Lau}, Z.~{Wang}, and S.~P. {Smolley}.
\newblock {Least Squares Generative Adversarial Networks}.
\newblock \emph{arXiv e-prints}, 2016.

\bibitem[Mao et~al.(2017)Mao, Li, Xie, Lau, Wang, and
  Paul~Smolley]{mao2017least}
X.~Mao, Q.~Li, H.~Xie, R.~Y. Lau, Z.~Wang, and S.~Paul~Smolley.
\newblock Least squares generative adversarial networks.
\newblock In \emph{ICCV}, 2017.

\bibitem[Mazumdar et~al.(2019)Mazumdar, Jordan, and
  Sastry]{mazumdar2019finding}
E.~V. Mazumdar, M.~I. Jordan, and S.~S. Sastry.
\newblock On finding local nash equilibria (and only local nash equilibria) in
  zero-sum games.
\newblock \emph{arXiv preprint arXiv:1901.00838}, 2019.

\bibitem[Mescheder et~al.(2018)Mescheder, Geiger, and
  Nowozin]{mescheder2018training}
L.~Mescheder, A.~Geiger, and S.~Nowozin.
\newblock Which training methods for gans do actually converge?
\newblock In \emph{ICML}, 2018.

\bibitem[Metz et~al.(2017)Metz, Poole, Pfau, and
  Sohl-Dickstein]{metz2016unrolled}
L.~Metz, B.~Poole, D.~Pfau, and J.~Sohl-Dickstein.
\newblock Unrolled generative adversarial networks.
\newblock In \emph{ICLR}, 2017.

\bibitem[Miyato et~al.(2018)Miyato, Kataoka, Koyama, and
  Yoshida]{miyato2018spectral}
T.~Miyato, T.~Kataoka, M.~Koyama, and Y.~Yoshida.
\newblock Spectral normalization for generative adversarial networks.
\newblock In \emph{ICLR}, 2018.

\bibitem[Mohamed and Lakshminarayanan(2016)]{mohamed2016learning}
S.~Mohamed and B.~Lakshminarayanan.
\newblock Learning in implicit generative models.
\newblock \emph{arXiv preprint arXiv:1610.03483}, 2016.

\bibitem[Mroueh and Sercu(2017)]{mroueh2017fisher}
Y.~Mroueh and T.~Sercu.
\newblock Fisher gan.
\newblock In \emph{NeurIPS}, 2017.

\bibitem[Mroueh et~al.(2017)Mroueh, Sercu, and Goel]{mroueh2017mcgan}
Y.~Mroueh, T.~Sercu, and V.~Goel.
\newblock Mcgan: Mean and covariance feature matching gan.
\newblock \emph{arXiv preprint arXiv:1702.08398}, 2017.

\bibitem[Nagarajan and Kolter(2017)]{nagarajan2017gradient}
V.~Nagarajan and J.~Z. Kolter.
\newblock Gradient descent gan optimization is locally stable.
\newblock In \emph{NeurIPS}, 2017.

\bibitem[Nguyen and Hein(2017)]{nguyen2017loss}
Q.~Nguyen and M.~Hein.
\newblock The loss surface of deep and wide neural networks.
\newblock In \emph{ICML}, 2017.

\bibitem[Nguyen et~al.(2018)Nguyen, Mukkamala, and Hein]{nguyen2018loss}
Q.~Nguyen, M.~C. Mukkamala, and M.~Hein.
\newblock On the loss landscape of a class of deep neural networks with no bad
  local valleys.
\newblock \emph{arXiv preprint arXiv:1809.10749}, 2018.

\bibitem[Nowozin et~al.(2016)Nowozin, Cseke, and Tomioka]{nowozin2016f}
S.~Nowozin, B.~Cseke, and R.~Tomioka.
\newblock f-gan: Training generative neural samplers using variational
  divergence minimization.
\newblock In \emph{NeurIPS}, 2016.

\bibitem[Poole et~al.(2016)Poole, Alemi, Sohl-Dickstein, and
  Angelova]{poole2016improved}
B.~Poole, A.~A. Alemi, J.~Sohl-Dickstein, and A.~Angelova.
\newblock Improved generator objectives for gans.
\newblock \emph{arXiv preprint arXiv:1612.02780}, 2016.

\bibitem[Radford et~al.(2016)Radford, Metz, and
  Chintala]{radford2015unsupervised}
A.~Radford, L.~Metz, and S.~Chintala.
\newblock Unsupervised representation learning with deep convolutional
  generative adversarial networks.
\newblock In \emph{ICLR}, 2016.

\bibitem[Reddi et~al.(2018)Reddi, Kale, and Kumar]{reddi2018convergence}
S.~J. Reddi, S.~Kale, and S.~Kumar.
\newblock On the convergence of adam and beyond.
\newblock In \emph{ICLR}, 2018.

\bibitem[Salimans et~al.(2016)Salimans, Goodfellow, Zaremba, Cheung, Radford,
  Chen, and Chen]{salimans2016improved}
T.~Salimans, I.~Goodfellow, W.~Zaremba, V.~Cheung, A.~Radford, X.~Chen, and
  X.~Chen.
\newblock Improved techniques for training gans.
\newblock In \emph{NeurIPS}, 2016.

\bibitem[Sanjabi et~al.(2018)Sanjabi, Ba, Razaviyayn, and
  Lee]{sanjabi2018convergence}
M.~Sanjabi, J.~Ba, M.~Razaviyayn, and J.~D. Lee.
\newblock On the convergence and robustness of training gans with regularized
  optimal transport.
\newblock In \emph{NeurIPS}, 2018.

\bibitem[Sun et~al.(2020)Sun, Li, Liang, Ding, and Srikant]{sun2020global}
R.~Sun, D.~Li, S.~Liang, T.~Ding, and R.~Srikant.
\newblock The global landscape of neural networks: An overview.
\newblock \emph{IEEE Signal Processing Magazine}, 37\penalty0 (5):\penalty0
  95--108, 2020.

\bibitem[Sun(2020)]{sun2020optimization}
R.-Y. Sun.
\newblock Optimization for deep learning: An overview.
\newblock \emph{Journal of the Operations Research Society of China}, pages
  1--46, 2020.

\bibitem[Tran et~al.(2017)Tran, Ranganath, and Blei]{Tran2017DeepAH}
D.~Tran, R.~Ranganath, and D.~M. Blei.
\newblock Deep and hierarchical implicit models.
\newblock In \emph{NeurIPS}, 2017.

\bibitem[Unterthiner et~al.(2018)Unterthiner, Nessler, Seward, Klambauer,
  Heusel, Ramsauer, and Hochreiter]{unterthiner2018coulomb}
T.~Unterthiner, B.~Nessler, C.~Seward, G.~Klambauer, M.~Heusel, H.~Ramsauer,
  and S.~Hochreiter.
\newblock Coulomb gans: Provably optimal nash equilibria via potential fields.
\newblock In \emph{International Conference on Learning Representations}, 2018.

\bibitem[Venturi et~al.(2018)Venturi, Bandeira, and Bruna]{venturi2018spurious}
L.~Venturi, A.~S. Bandeira, and J.~Bruna.
\newblock Spurious valleys in two-layer neural network optimization landscapes.
\newblock \emph{arXiv preprint arXiv:1802.06384}, 2018.

\bibitem[Wang et~al.(2018)Wang, Yu, Wu, Gu, Liu, Dong, Qiao, and
  Change~Loy]{wang2018esrgan}
X.~Wang, K.~Yu, S.~Wu, J.~Gu, Y.~Liu, C.~Dong, Y.~Qiao, and C.~Change~Loy.
\newblock Esrgan: Enhanced super-resolution generative adversarial networks.
\newblock In \emph{ECCV}, 2018.

\bibitem[Wu et~al.(2019)Wu, Huang, Li, Thoma, and Van~Gool]{wu2017sliced}
J.~Wu, Z.~Huang, W.~Li, J.~Thoma, and L.~Van~Gool.
\newblock Sliced wasserstein generative models.
\newblock In \emph{CVPR}, 2019.

\bibitem[Xiangli et~al.(2020)Xiangli, Deng, Dai, Loy, and Lin]{xiangli2020real}
Y.~Xiangli, Y.~Deng, B.~Dai, C.~C. Loy, and D.~Lin.
\newblock Real or not real, that is the question.
\newblock \emph{arXiv preprint arXiv:2002.05512}, 2020.

\bibitem[Yaz{\i}c{\i} et~al.(2019)Yaz{\i}c{\i}, Foo, Winkler, Yap, Piliouras,
  and Chandrasekhar]{yazici2018unusual}
Y.~Yaz{\i}c{\i}, C.-S. Foo, S.~Winkler, K.-H. Yap, G.~Piliouras, and
  V.~Chandrasekhar.
\newblock The unusual effectiveness of averaging in gan training.
\newblock In \emph{ICLR}, 2019.

\bibitem[Zhang et~al.(2018)Zhang, Goodfellow, Metaxas, and
  Odena]{zhang2018self}
H.~Zhang, I.~Goodfellow, D.~Metaxas, and A.~Odena.
\newblock Self-attention generative adversarial networks.
\newblock In \emph{ICML}, 2018.

\bibitem[Zhang et~al.(2020)Zhang, Xiao, Sun, and Luo]{zhang2020single}
J.~Zhang, P.~Xiao, R.~Sun, and Z.-Q. Luo.
\newblock A single-loop smoothed gradient descent-ascent algorithm for
  nonconvex-concave min-max problems.
\newblock \emph{arXiv preprint arXiv:2010.15768}, 2020.

\bibitem[Zhang et~al.(2017)Zhang, Liang, and Charikar]{zhang2017hitting}
Y.~Zhang, P.~Liang, and M.~Charikar.
\newblock A hitting time analysis of stochastic gradient langevin dynamics.
\newblock \emph{arXiv preprint arXiv:1702.05575}, 2017.

\bibitem[Zou et~al.(2018)Zou, Cao, Zhou, and Gu]{zou2018stochastic}
D.~Zou, Y.~Cao, D.~Zhou, and Q.~Gu.
\newblock Stochastic gradient descent optimizes over-parameterized deep relu
  networks.
\newblock \emph{arXiv preprint arXiv:1811.08888}, 2018.

\end{thebibliography}
}
\bibliographystyle{abbrvnat}

\ifaddsupplement

\appendix

\onecolumn

\clearpage

\section*{\Large Appendix: Towards a Better Global Loss Landscape of GANs}

 The code is available at \href{https://github.com/AilsaF/RS-GAN}{https://github.com/AilsaF/RS-GAN}.
This appendix consists of additional experiments, related work, proofs, other results and various discussions. 

\iflonger 
\begin{itemize}
    \item \textbf{Additional experiments: } Include (a) more discussions of 2-cluster experiments (following Sec.~\ref{sec: case study}); (b) experiments for imbalanced data; (c) experiments for hinge loss and least-squares loss; (d) experiments for bad initial point (following Sec.~\ref{sec:experiments}); (e) experiments for high resolution image generation (following Sec.~\ref{sec: empirical and population})
    
    \item \textbf{Related work, discussions, proofs and other results:}
Include: related work; discussions of empirical loss (following Sec.~\ref{sec: empirical and population}); proofs of the main results Theorem \ref{prop: GAN all values, extension} and Theorem \ref{prop: RS-GAN all values, extension};
  details and proofs of other results in Sec. \ref{subsec: main results}. 
\end{itemize}
\fi 

{\footnotesize 
\tableofcontents
}

\section{ Related Work }\label{sec: related works}

We provide a more detailed overview of related work in this section. 

\textbf{Global analysis in supervised learning.}
Recently,  global landscape analysis has attracted much attention. See \citet{sun2020optimization,sun2020global,bianchini1996optimal} for  surveys
and \cite{liang2018adding,liang2019revisiting,ding2019sub,liang2018understanding,jacot2018neural,allen2018convergence,zou2018stochastic,du2018gradient}
for some recent works. It is widely believed that wide networks have a nice loss landscape and thus local minima are less of a concern
(e.g., \cite{livni2014computational,geiger2018jamming,li2018over}).
 However, this claim only holds for supervised learning, and it is not clear whether local  minima cause training difficulties for GANs.

\iffalse 
A few recent works in supervised learning \cite{jacot2018neural,allen2018convergence,zou2018stochastic,du2018gradient} 
 proved the convergence of gradient descent to global minima for wide neural networks.
They implicitly utilize the assumption that the ``shell'' loss function $\ell(\hat{y}, y)$
 is convex; in other words, the loss function is convex
 in the function space of $\hat{y} =f(x)$ where
 $f$ is the neural network.  These works further argue that a wide neural net does not cause extra training difficulties in the parameter space. 
Although the convexity of the ``shell'' loss function holds for supervised learning (e.g., quadratic loss and cross entropy loss are convex), it does not hold  for GANs. We think the lack of progress on the global convergence analysis of GANs is 
partially due to the lack of understanding of the optimization landscape of the shell loss function. This is one of the main motivations for us to study which shell loss function has a good landscape.
\fi 

\textbf{Single-mode analysis.}
For single-mode data, \citet{feizi2017understanding} 
and  \citet{mescheder2018training} provide a global analysis of  GANs. They consider a 
single point $0$ and a single Gaussian respectively. 
 \iflonger  
 \cite{mescheder2018training} consider the case that the true distribution $\Pd$ is a single point $ 0 $.
  Interestingly, we found that their proof cannot be directly generalized to the case that $ \Pd $ is a single non-zero point. 
 Thus even for the single non-zero point setting, it is natural to use the  RpGAN formulation for a better global landscape. 
 \cite{feizi2017understanding} consider learning a Gaussian $\Pd$ by a new ``quadratic GAN''.
 \fi 
\citet{feizi2017understanding} differs from ours in a few aspects.
First, they consider the single-mode setting which does not
have an issue of mode collapse. %
Second, they assume $\Pd$ is a Gaussian distribution, while we consider an arbitrary  empirical distribution.
  Third, they analyze ``quadratic-GAN,'' which is not common in practice,
while we analyze commonly used GAN formulations (including JS-GAN).

\iffalse 
One may wonder whether the work by  \citet{feizi2017understanding} can
be extended to a multi-Gaussian distribution. As we discuss later,
there is a macro-learning effect and micro-learning effect for learning a multi-mode distribution. Our work on $n$-point distributions captures the macro-learning effect, and \citet{feizi2017understanding} captures the micro-learning effect for Gaussian data.
It is  an interesting future work to combine the analysis of \citet{feizi2017understanding} and our analysis to study the multi-Gaussian case.
\fi 

\textbf{Mode collapse.}
Mode collapse  is one of the major challenges for GANs which received a lot of attention. 
There are a few high-level hypotheses, such as improper loss functions  \cite{arjovsky2017towards, arora2017generalization} and weak discriminators   \cite{metz2016unrolled,salimans2016improved,arora2017generalization,li2017towards}. Interestingly, RpGAN both changes the loss function and improves the discriminator. 
The theoretical analysis of mode collapse is relatively scarce.
\citet{lin2018pacgan} makes a key observation that two
distributions with the same total variation (TV) distance to true distribution 
do not exhibit the same degree of mode collapse. 
\iffalse 
\sout{For instance, consider $ Q = U[0,1] $,  $ P_1 =  U[ 0.2, 1]$ and $P_2 = 0.6 U [0, 0.5   ]  +   1.4 U ( [ 0.5, 1 ] )  $, where $U[a,b]$ is the uniform distribution on $[a,b]$. Then $TV(Q, P_1) = TV (Q, P_2) $, but $P_1$ has mode collapse while $P_2$ does not. They proposed to pack the samples, i.e., consider the distance of the product distribution $ P^m $ and $Q^m$. The main theoretical result is that ``$TV(P^m , Q^m)$  is a better loss to penalize strong mode collapse than $TV(P, Q)$.''}
\fi 
They proposed to pack the samples (PacGAN) to alleviate mode collapse. 
This work is rather different from ours.
First, they analyze the TV distance, %
while we analyzed SepGANs and RpGANs.
Second, their analysis is statistical, while
our analysis is about optimization. %
As for the empirical guidance, RpGAN and PacGAN are complimentary and can be used together (suggested by the author of \cite{jolicoeur2018relativistic}).
 There are a few more works that discuss mode collapse and/or local minima; we defer the discussion to Appendix \ref{app-sub: related works on local min}.

  \textbf{Theoretical studies of loss functions. }
The early work on GANs \cite{goodfellow2014generative} built a link between the min-max formulation and the J-S distance to justify the formulation.
 \citet{arjovsky2017towards} pointed out some possible drawbacks of J-S distance, and proposed a new loss based on Wasserstein distance, referred to as WGAN. Later, \citet{arora2017generalization} point out that 
both Wasserstein distance and J-S distance are not generalizable,
but they also argued that this is not too scary since people
are not directly minimizing these two distances but a class of metrics referred to as 
``neural-network distance.''

\textbf{Convergence analysis.}
Many recent works analyze convergence of GANs and/or min-max optimization, e.g.,
 \citep{daskalakis2017training,daskalakis2018limit,azizian2019tight,gidel2018negative,mazumdar2019finding,yazici2018unusual,jin2019minmax,sanjabi2018convergence,zhang2020single}. These works
 often only analyze local stability or convergence
    to local minima (or stationary points), making it different from our work.
     \citet{lei2019sgd} studied the convergence of WGAN,
 but restricted to 1-layer neural nets.

 \textbf{Other theoretical analysis.}
 There are a few other theoretical analysis of GANs, e.g., \cite{mohamed2016learning,liu2018inductive,farnia2018convex,binkowski2018demystifying,balduzzi2018mechanics,li2017limitations,makkuva2019optimal,lei2019sgd}.
 Most of these works are not directly related to our work.

 \iffalse 
 \citet{mohamed2016learning} analyzed GANs in the general
  framework of implicit generative models, and summarized
 four approaches to solve the problem.  
 \citet{liu2018inductive} analyzed the inductive bias
 in the restricted f-GAN: for instance,
 they analyzed the  globally optimal solution of a KL-GAN with linear
 discriminator. Different from our work,
 they do not analyze the local optima or the convergence.  \citet{farnia2018convex}
 proposed a convex duality framework to explain why regularizers used
 in f-GAN and W-GAN can greatly improve the performance.
 \citet{binkowski2018demystifying} discussed  the bias in the gradient during GAN training and the effect of this bias.
 \citet{balduzzi2018mechanics} studied
 the dynamics of $n$-player games. 
 \citet{li2017limitations} studied the dynamics
  of a Gaussian mixture model with two modes,
  but they do not extend their result to a more general setting.
   \citet{makkuva2019optimal} studied the problem of learning an optimal transport mapping using input convex neural networks. 
\citet{lei2019sgd} studied the global convergence of WGAN with 1-layer neural network. \tf{maybe we want to eliminate this paragraph if these works don't connect to ours directly}
\fi

\textbf{Other GAN Variants.} %
There are many GAN variants, e.g., WGAN~\cite{arjovsky2017wasserstein, arjovsky2017towards,gulrajani2017improved}
and variants~\cite{wu2017sliced, kolouri2018sliced, adler2018banach,deshpande2018generative,IDeshpandeCVPR2019},
$f$-GAN \cite{nowozin2016f},
SN-GAN \cite{miyato2018spectral}, self-attention GAN 
\cite{zhang2018self}, StyleGAN \cite{karras2019style,karras2020analyzing}
and many more~\cite{mao2017least,mroueh2017fisher, berthelot2017began, mroueh2017mcgan, cully2017magan, LiNIPS2017, li2017mmd, salimans2016improved, nowozin2016f,poole2016improved,metz2016unrolled,huang2016stacked,radford2015unsupervised,Bengio+chapter2007,li2017mmd}.  
Our analysis framework (analyzing global landscape of empirical loss) can potentially be applied to more variants mentioned above.

\vspace{-0.2cm}
 \subsection{Related Works on Local Minima and Mode Collapse}\label{app-sub: related works on local min}
 We discuss a few related works on local minima and mode collapse, including
 \citet{kodali2017convergence}, \citet{li2018implicit} 
 and \citet{unterthiner2018coulomb} that are mentioned in the main text. 
 
\textbf{DRAGAN.}
 \citet{kodali2017convergence} suggested the connection between mode collapse and a bad equilibrium based on the following empirical observation: a sudden increase of the gradient norm of the discriminator during training is associated with a sudden drop of the IS score. 
 However, \citet{kodali2017convergence} don't 
 present formal theoretical results %
  on the relation between mode collapse and a bad equilibrium.

\textbf{IMLE.}
 \citet{li2018implicit} proposed implicit maximum likelihood estimation (IMLE).
 \iflonger 
 some ideas that are  related to RpGANs and our analysis in Section~\ref{sec: intuition and toy results}.  \citet{li2018implicit} argue that a GAN allows mode dropping
 since it minimizes the distinguishability between data and samples. 
 In contrast,  maximum likelihood ensures that each sample is
close to some data example, thus ``disallowing'' mode dropping. 
To combine the advantages of GANs (implicit model) and maximum likelihood 
(explicit model), they propose implicit maximum likelihood estimation (IMLE).
\fi 
The empirical version of IMLE in the parameter space is the following:
{\equationsizeReg
\begin{equation}
  \min_{ w }  \sum_{j=1}^n \min_{ i \in \{1, \dots, m \} }
   \| x_i -  G_w( z_j ) \|^2.
\end{equation}
}
In other words, for each generated sample $y_j = G_w( z_j )$,
the loss is the distance from $y_j$ to the closest true sample $ x_i $. 
Interestingly, IMLE and RpGAN  both couple the true data and the fake data in the loss.
The differences are two fold:
first, IMLE does not have an extra discriminator $f_{\theta}$,
while RpGAN has;
second, IMLE compares $y_j$ with all $x_i $ (so as to find the nearest
neighbor) while RpGAN compares $y_j$ with an arbitrary $ x_j $. 
See Table~\ref{tab: compare RpGAN, IMLE and ColumbGAN} for a comparison.
Note that \citet{li2018implicit} 
don't  present formal theoretical results on the landscape. 
\iflonger 
We suspect that their proposed loss has a better loss landscape than JS-GAN,
but cannot ensure the elimination of bad basins. This requires further
theoretical analysis. 
\fi 

\begin{table}[htbp]
 \footnotesize 
\caption{Models that couple true data and fake data in the loss}
    \centering
    \begin{threeparttable}
    \begin{tabular}{|c|c|c|c|}
      \hline 
   Model name  &  Empirical form of loss \tnote{i}  & 
    Form of coupling & Optimization \\
      \hline 
   RpGAN \cite{jolicoeur2018relativistic}  &  
   $  \max_{f } \sum \limits_j h( f(  x_j ) - f ( y_j )     ) $ & 
 pairing  &  min-max  \tnote{ii}  \\
     \hline 
   RaGAN \tnote{iii} $\; $ \cite{jolicoeur2018relativistic}  &  
   $   \max_f \sum \limits_j h( \frac{1}{ n } \sum\limits_{i=1}^n f ( x_i ) \! - \! f ( y_j )   ) $ &  comparing with average & min-max  \\
  \hline 
  (max-)sliced-WGAN &  
   $   \max \limits_{ |f|_L \leq 1 } \sum \limits_{i=1}^n [  f(X)_{(i)} - f(Y)_{(i)}  ]^2 $ \tnote{iv}  & pairing sorted output & min-max   \\
  \cite{deshpande2018generative,IDeshpandeCVPR2019}   &  
  &  &   \\
    \hline 
   IMLE \cite{li2018implicit} 
    & $ \sum \limits_j \min_{ i \in [ n ]} \| y_j -  x_{ i }  \|^2  $  &    comparing with closest &  min \\
      \hline 
      Coulomb-GAN  &  $ \sum_{i, j} k ( x_i ,  x_j ) 
 + \sum_{i, j} k ( y_i ,  y_j ) $ 
 &   & non-zero-sum   \\
\cite{unterthiner2018coulomb} & $  - 2 \sum_{i, j} k ( x_i , y_j )   \quad \tnote{v} $  & all-pairs   & game \tnote{vi}  \\
      \hline 
    \end{tabular}
     \begin{tablenotes}
     \scriptsize 
     \item[i] We show the empirical form of the loss 
 in the function space. Rigorously speaking, the provided form is the
  the loss for one mini-batch; in practice, in different iterations
  of SGD we will use different samples of $x_i, y_j$. 
  For the emprical loss in the parameter space, we shall replace $ f $ by $f_{\theta}$
  and $y_j $ by $ G_w(z_j) $. 
     \item[ii] %
Besides the  zero-sum game form (min-max form), RpGAN can be easily modified to a non-zero-sum game form (``non-saturating version'' proposed in \cite{goodfellow2014generative}). 
    \item[iii] The precise expression of RaGAN (relativistic averaging GAN) shall be 
  $  \sum_j h_1( \frac{1}{ n } \sum_{ i =1}^n f_{\theta} ( x_i ) - f_{\theta} ( y_j )   )
  + \sum_i h_2 ( \frac{1}{ n } \sum_{ j =1}^n f_{\theta} ( y_j ) - f_{\theta} ( x_i )   )  $, but for simplicity we only present one term in the table.
 \item[iv] Here $  f(X)_{(1)} \leq \dots \leq f(X)_{ (n) }  $ and 
$  f(Y)_{(1)} \leq \dots \leq f(Y)_{ (n) }  $ are the sorted versions of $f(x_i)$'s and $f(y_i)$'s respectively.
 \item[v] Here $k$ is the Coulomb kernel, defined
 as $  k(u, v) = \frac{1}{ (\sqrt{ \| u - v \|^2 + \epsilon^2 })^{\alpha } } $
where $ u, v\in \mathbb{R}^d$,  $ \alpha \leq d - 2 $ and $ \epsilon > 0 $. 
The original form of Coulomb-GAN is a non-zero-sum game, but
 it is straightforward to transfer the formulation to a pure minimization form since
the discriminator-minimization problem has a closed form solution
(used in the proof of \cite[Theorem 2]{unterthiner2018coulomb}). 
We presented the transformed minimization problem here. 
  \item[vi] Coulomb-GAN is presented as a non-zero-sum game, but as mentioned
  earlier it can be transformed to a minimization problem. 
The original Coulomb-GAN uses a smoothing operator in the generator loss;
in this empirical form, we omit the smoothing operator for easier comparison (thus
it is not the same as Coulomb-GAN). In the table, we show the resulting loss in the pure minimization form. 
Unlike SepGAN and RpGAN that can be written as either min-max form or non-zero-sum game form, we point out that there is no min-max form for Coulomb-GAN, since the design principle of Coulomb-GAN is very different from typical GANs.
    \vspace{0.2cm}
      \end{tablenotes}
\end{threeparttable}
\label{tab: compare RpGAN, IMLE and ColumbGAN}
\end{table}

\textbf{Coulomb-GAN.}
 \citet{unterthiner2018coulomb} argued that mode collapse can be
 a local Nash equilibrium in an example of two clusters (see \cite[Appendix A.1]{unterthiner2018coulomb}).
They further proposed ColumbGAN and claimed that every local Nash equilibrium is a global Nash equilibrium (see \cite[Theorem 2]{unterthiner2018coulomb}). Their study is different from ours in a few aspects.
\textbf{First}, they
still consider the pdf $p_g$, though  restrict the possible movement of  $ p_g $
(according to a continuity equation).
In contrast, we consider the empirical loss in particle space. 
\textbf{Second}, the bad landscape of JS-GAN is discussed
in words for the 2-cluster case \cite[Appendix A.1]{unterthiner2018coulomb}, but not
formally proved. In contrast, we prove rigorous result for the general case.
\textbf{Third}, they do not study parameter space (though with informal discussion).
 \textbf{Fourth}, they do not present landscape-related experiments,
 such as the narrow-net experiments we have done.
 
\iflonger 
There is a major difference from our analysis of the two-cluster example: \cite[Appendix A.1]{unterthiner2018coulomb} still considers the pdf $p_g$, though it restricts the possible movement of  $ p_g $ based on the gradient update; in contrast, we consider the particle space. 
In addition, \cite[Appendix A.1]{unterthiner2018coulomb}
is intended to be intuitive, not a  mathematical proof. 
\fi 
\iffalse 
There are a few differences though. 
 \textbf{First}, 
\cite[Appendix A.1]{unterthiner2018coulomb} still considers
the pdf $p_g$, though it restricts the possible movement of  $ p_g $ based on the gradient update. 
In contrast, we consider the particle space, which is different from \cite[Appendix A.1]{unterthiner2018coulomb} and  \cite[Appendix A.1]{unterthiner2018coulomb}. \tf{why two 86, appendix A.1}
\textbf{Second}, the argument in \cite[Appendix A.1]{unterthiner2018coulomb} 
is intended to be intuitive, %
while our Corollary~\ref{coro of bad basin in GAN} is formally proved.
In fact, the feasible region of the pdf is not formally defined in \cite[Appendix A.1]{unterthiner2018coulomb}, thus \cite[Appendix A.1]{unterthiner2018coulomb} does not constitute a formal proof. 
For the case of two clusters, it is ambiguous what the exact local Nash equilibria are, and the argument can only handle approximate local Nash equilibria. For instance,
\fi

\textbf{Common idea: Coupling true data and fake data.}
Interestingly, similar to IMLE and RpGAN,
  ColumbGAN also coupled the true data and fake data in the loss functions.
 RpGAN, RaGAN (a variant of RpGAN considered in \cite{jolicoeur2018relativistic}), IMLE and ColumbGAN differ in two aspects: the specific form of coupling (pairing, comparing with average,
  comparing with the closest, all possible pairs), and the specific form of optimization
  (pure minimization, min-max, non-zero-sum game). 
  See the comparison in Table~\ref{tab: compare RpGAN, IMLE and ColumbGAN}.
  It is interesting that all three lines of work choose to couple true data and fake data
  to resolve the issue of mode collapse.
We suspect
it is hard to prove similar results on the landscape of empirical loss for IMLE and Coulomb-GAN.

\textbf{Relation to (max)-sliced Wasserstein GAN.} 
We point out that the sliced Wasserstein GAN (sliced-WGAN) \citep{deshpande2018generative} and the max-sliced Wasserstein GAN (max-sliced-WGAN) \citep{IDeshpandeCVPR2019} also couple the true data
 and fake data. 
For any function $f$, denote $f(X) = (f(x_1), \dots, f(x_n))$ and 
$f(Y) = (f(y_1), \dots, f(y_n))$.
The empirical version of the max-sliced Wasserstein GAN
can be written as 
\begin{equation}\label{eq of max sliced GAN, 1st version}
 \min_{Y} \max_{ |f|_L \leq 1 } W_2(   f(  X ) ,   f( Y )  )^2 .
\end{equation}
Here  $ f $ is a neural net with codomain $\mathbb{R}$, and $W_2$ is the Wasserstein-2-distance.
Denote $  f(X)_{(1)} \leq \dots \leq f(X)_{ (n) }  $ and 
$  f(Y)_{(1)} \leq \dots \leq f(Y)_{ (n) }  $ as the sorted versions of $f(x_i)$'s and $f(y_i)$'s respectively.
Then Eq.~\eqref{eq of max sliced GAN, 1st version} is equivalent to
\begin{equation}\label{eq of max sliced GAN, 2nd version}
\text{(max-)sliced-WGAN}\footnote{
{\scriptsize Note that max-sliced-WGAN in \citet{IDeshpandeCVPR2019}  uses
$ \min_{Y} \max_{ \| v \| \leq 1,  | g |_L \leq 1 } W_2(   v^T g(  X ) ,  v^T g( Y )  )^2 , $ while sliced-WGAN in \citet{deshpande2018generative} uses
$\min_{Y} \mathbb{E}_{\|v\|=1} \max_{ | g |_L \leq 1 } W_2(   v^T g(  X ) ,  v^T g( Y )  )^2$.  %
In Eq.~\eqref{eq of max sliced GAN, 1st version}  we use $ f(u) $ to replace $v^T g(u)$ 
to simplify the expression; although technically, $f$ and $v^T g$ are not equivalent,
 this minor difference does not affect our discussion.
}
}
: \min_{Y} \max_{ |f|_L \leq 1 } \sum_{i=1}^n [  f(X)_{(i)} - f(Y)_{(i)}  ]^2.
\end{equation}
 This form is quite close to RpGAN (when $h(t) = t^2 $): the only differences
 are the sorting of $f(X), f(Y)$ and the extra constraint $|f|_L \leq 1 $.
 The extra constraint $|f|_L \leq 1 $ is due to unbounded $h$, and can be removed
 if we use an upper bounded $ h $ (which leads to a sorting version of RpGAN).
 See the comparison of max-sliced-WGAN with RpGAN and other models 
 in Table \ref{tab: compare RpGAN, IMLE and ColumbGAN}. 

\textbf{Nash equilibria for Gaussian data.}
A very recent work \citet{farnia2020gans}  shows that for a non-realizable case (with a linear generator) Nash equilibria may not exist for learning a Gaussian distribution. This setting is quite different from ours.

\section{2-Cluster Experiments: Details and More Discussions}\label{subsec: details
of experiments}

In this part, we present details of the experiments in Section~\ref{sec: case study} and other complementary experiments.

\textbf{Experimental Setting.}\label{sec: 2clustersetting}
 The code is provided in ``GAN$\_$2Cluster.py''.
We sample $ 100 $ points from two clusters of data near $0$ and $4$ (roughly $50$ 
in each cluster). 
We use GD with momentum parameter $ 0.9 $ for both $D$ and $G$. 
The default learning rate is (Dlr, Glr) $= (10^{-2}, 10^{-2} )  $.
\iflonger 
The learning rates are chosen to be these because
bigger learning rates such as $( 2 \times 10^{-2}, 2 \times 10^{-2} )   $  sometimes lead to unstable results, and smaller learning rates make the algorithm slower.
\fi 
The default inner-iteration-number for the discriminator and the generator are (DIter, GIter) $= (10, 10 )  $.
The discriminator and generator net are a 4-layer network (with 2 hidden layers) with sigmoid activation and tanh activation respectively.
The default neural network width (Dwidth, Gwidth) $= ( 10,5 )  $.
We will also discuss the results of other hyperparameters. 
The default number of training iterations is  MaxIter $= 5000 $.
We use the non-saturating versions for both JS-GAN and RS-GAN.

\iffalse 
\textbf{Variations in reproduction.} 
We fix the random seed to be $ 2 $ to generate Figure \ref{fig5 D loss and D image} (for other seeds, the D loss plots are similar).   %
Across different runs, the D loss plots are not the same, which
may be due to numerical errors. 
We show two runs for JS-GAN and two runs for RS-GAN in Figure \ref{fig 6 standard 2 cluster}.
\fi 

\textbf{Understanding the effect of mode collapse, by checking D loss evolution and data movement.}
In the main text, we discussed that mode collapse can slow down training of JS-GAN.
For easier understanding of the training process, we add the visualization of the data movement (which is possible since we are dealing with 1-dimensional data)
in Figure~\ref{fig 6 standard 2 cluster}. 
We use the y-axis to denote the data position, and x-axis to denote the iteration. 
The blue curves represent the movement of all fake data during training, 
and the red straight lines represent the position of true data (two clusters).
The training time may vary across different runs, but overall the time for JS-GAN is about 2-4 times longer than that for RS-GAN.
\iflonger 
In a few cases, JS-GAN gets completely attracted to 
mode collapse for a certain period, in which case
the corresponding $D$ function is almost the same
as the optimal discriminator for mode collapse pattern $ D^*(s_{1 \rm a}) $ 
(similar to Figure~\ref{fig5 D loss and D image}(c)).
\fi 

\iffalse 
The training process consists of two stages:
in Stage 1, the cluster of generated points jump  between the two clusters; in Stage 2, the generated points split and converge to two separate clusters. 
In the discriminator loss, the transition point from fluctuation
to staying flat indicates the transition from Stage 1 to Stage 2. 
The length of Stage 1 (time spent on escaping mode collapse) may vary across different runs, but overall the escaping time for JS-GAN is about 2-4 times
longer than that for RS-GAN.
\fi

\iffalse 
For JS-GAN, for about half of the time we get a similar D-loss plot as Figure \ref{fig5 D loss and D image}, and another half of the time the D-loss can be stuck around $0.48$ for longer time.
For RS-GAN, the convergence speed is only slightly different.
Different runs exhibit the same major message of Figure  \ref{fig5 D loss and D image} that RS-GAN escapes mode collapse faster than JS-GAN. 
 In different runs, the movement of the generated data 
  are closely related to the D-loss plot: when
  D loss converges to around 0.7, the generated data begin to split
  into two clusters and then it takes another 1k-2k iterations
  to converge to the true data clusters. 
 For JS-GAN, when the D loss gets stuck at $ 0.48 $, the
  D image is similar to Figure \ref{fig5 D loss and D image} (c),
  i.e., the training gets stuck at the predicted attractor
  $ ( s_{1\rm a}, D^*( s_{1 \rm a}) ) $.  
  \fi

\begin{figure}[t]
\vspace{-0.2cm}
\centering
    \begin{tabular}{cccc}
        \includegraphics[width=0.2\linewidth, height = 1.8cm]{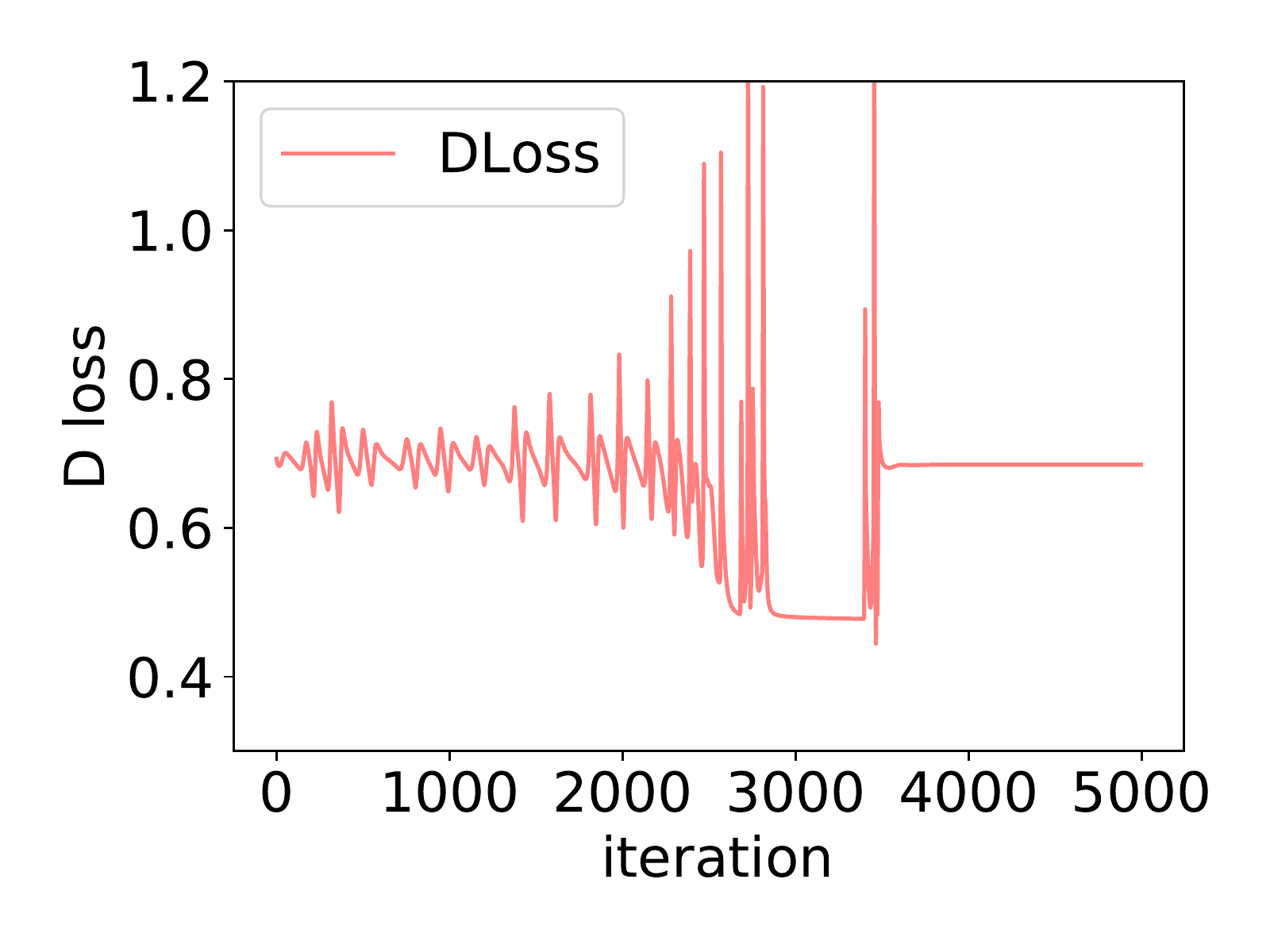} & \includegraphics[width=0.2\linewidth, height = 1.8cm]{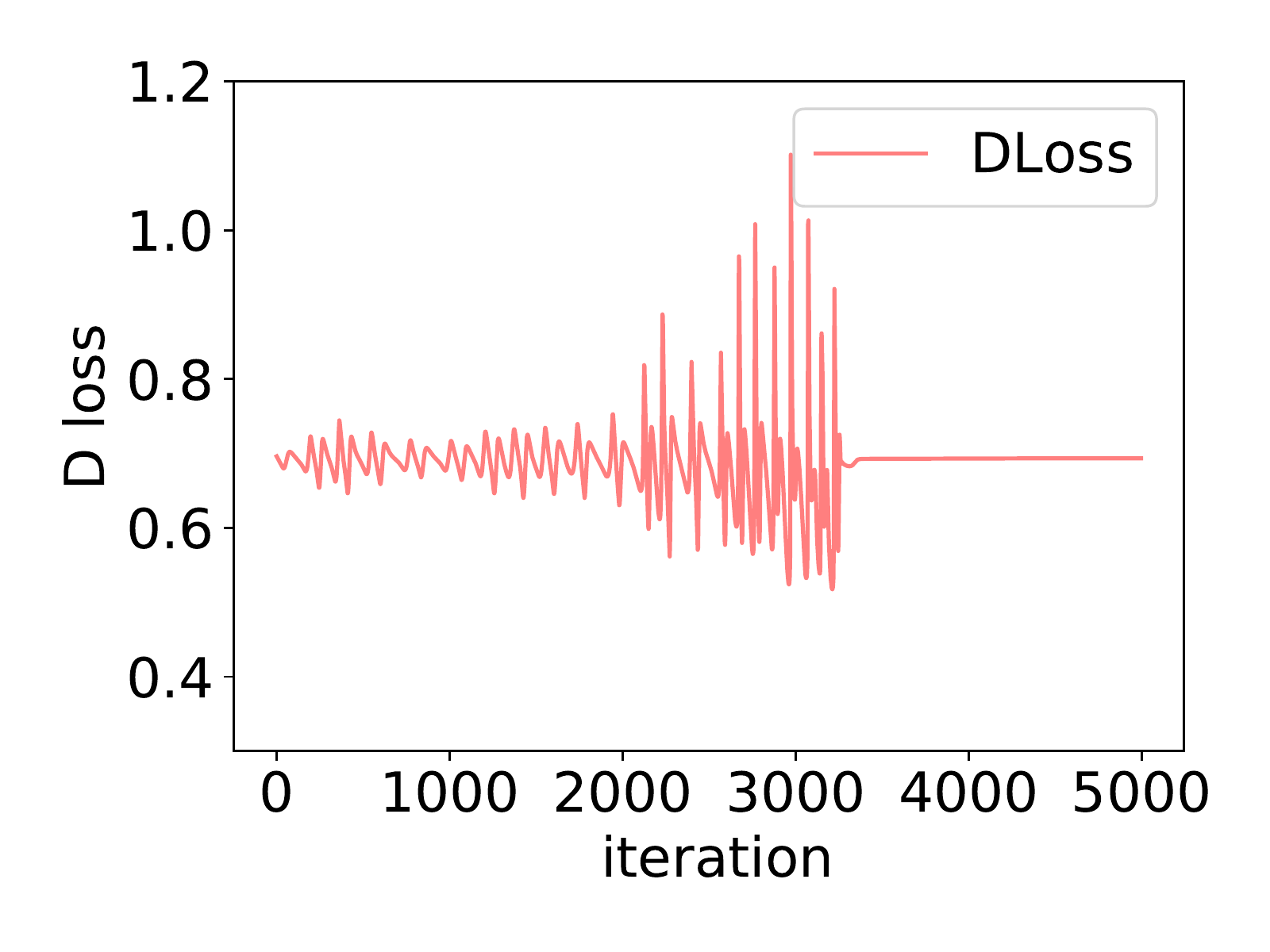} &
        \includegraphics[width=0.2\linewidth, height = 1.8cm]{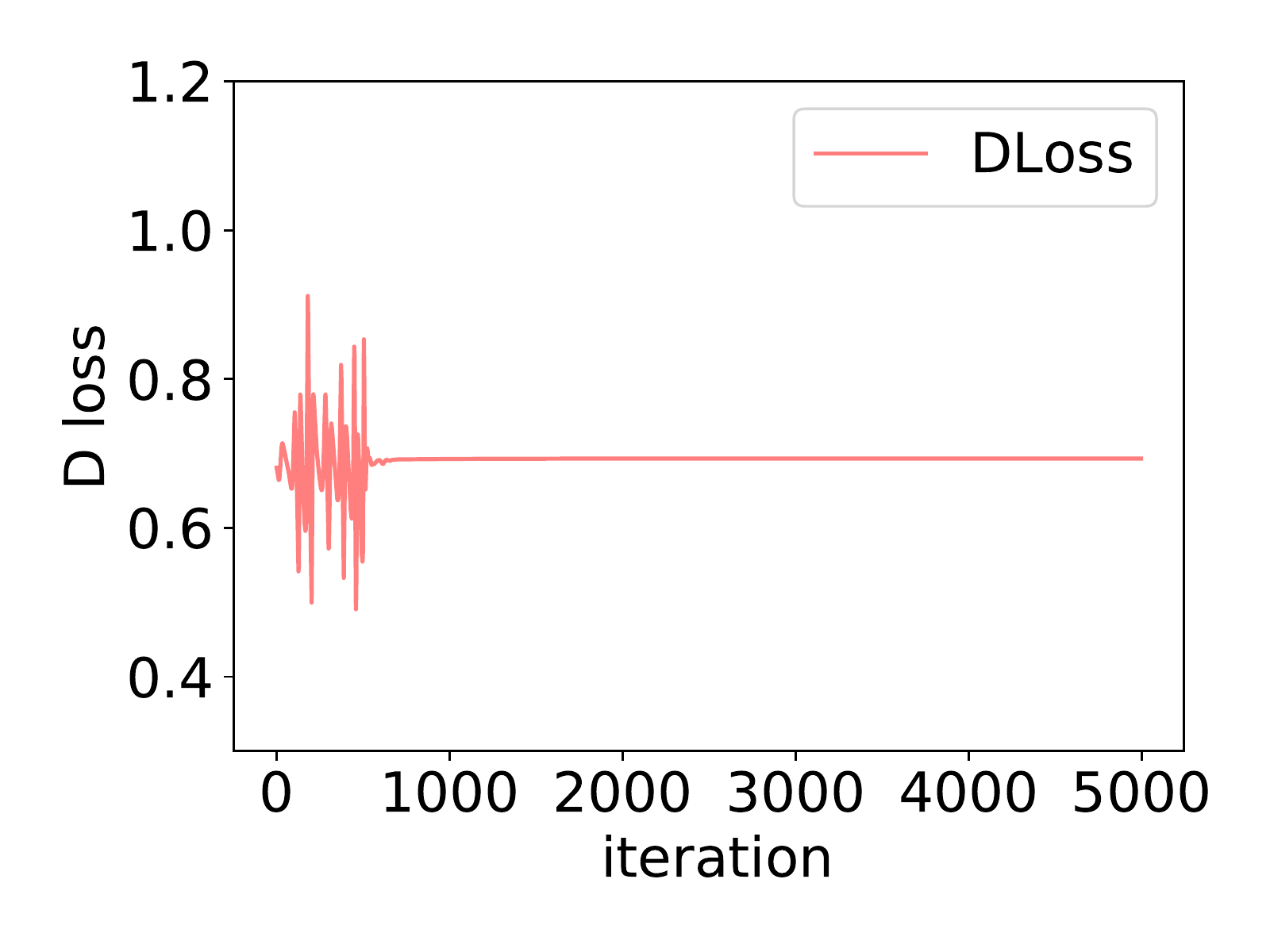} &
        \includegraphics[width=0.2\linewidth, height = 1.8cm]{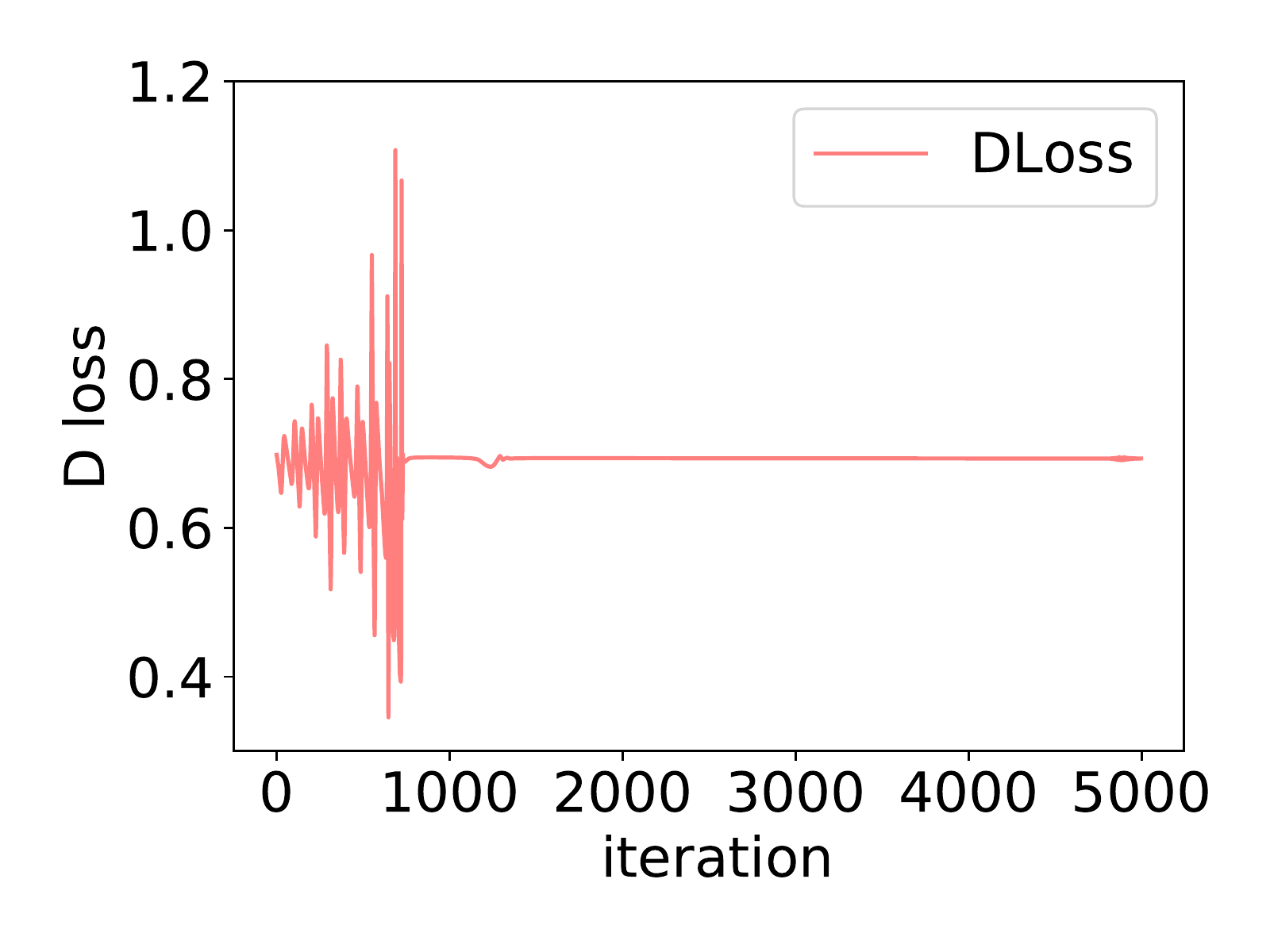} 
        \\
        \includegraphics[width=0.2\linewidth, height = 1.8cm]{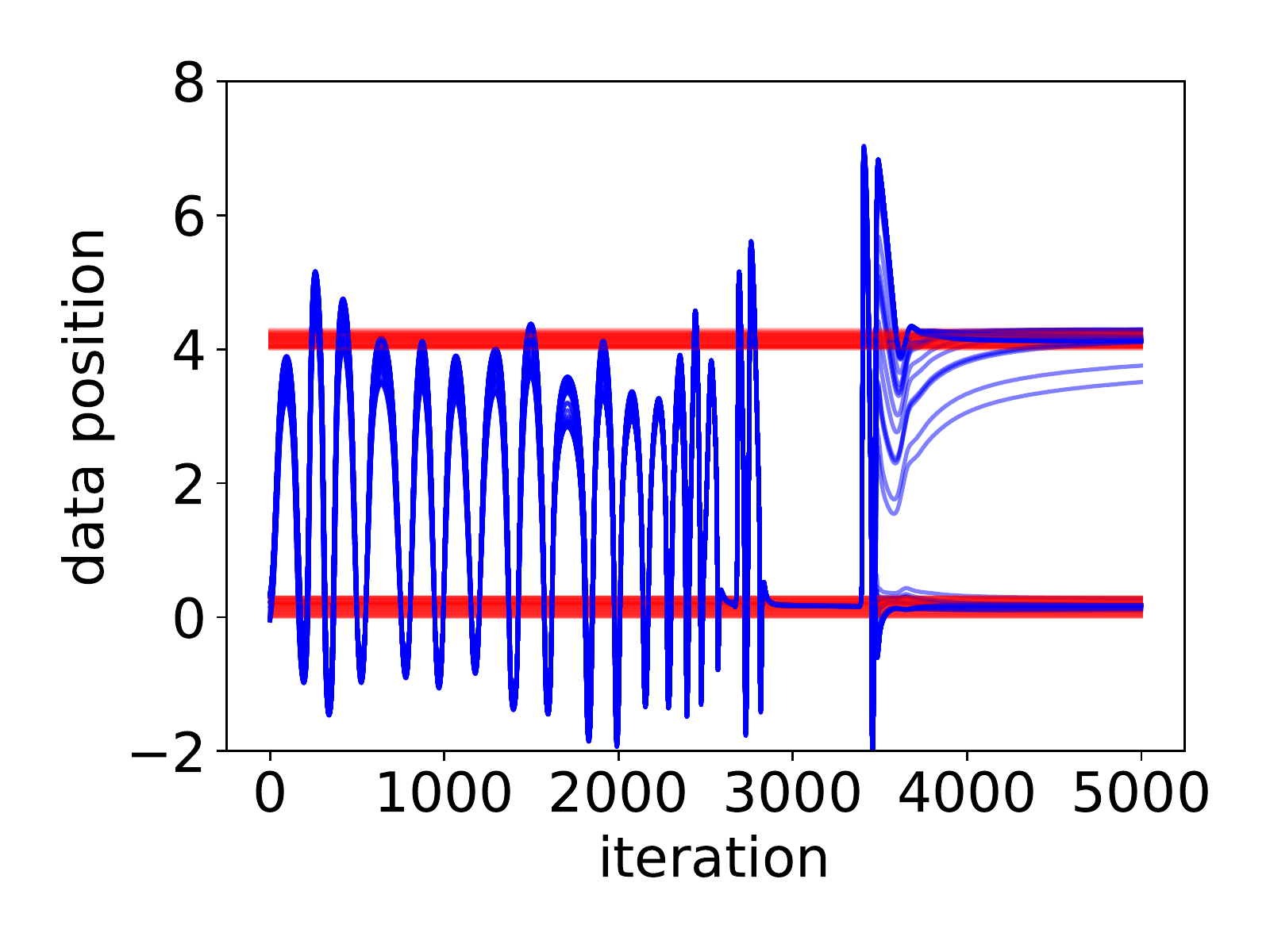} & \includegraphics[width=0.2\linewidth, height = 1.8cm]{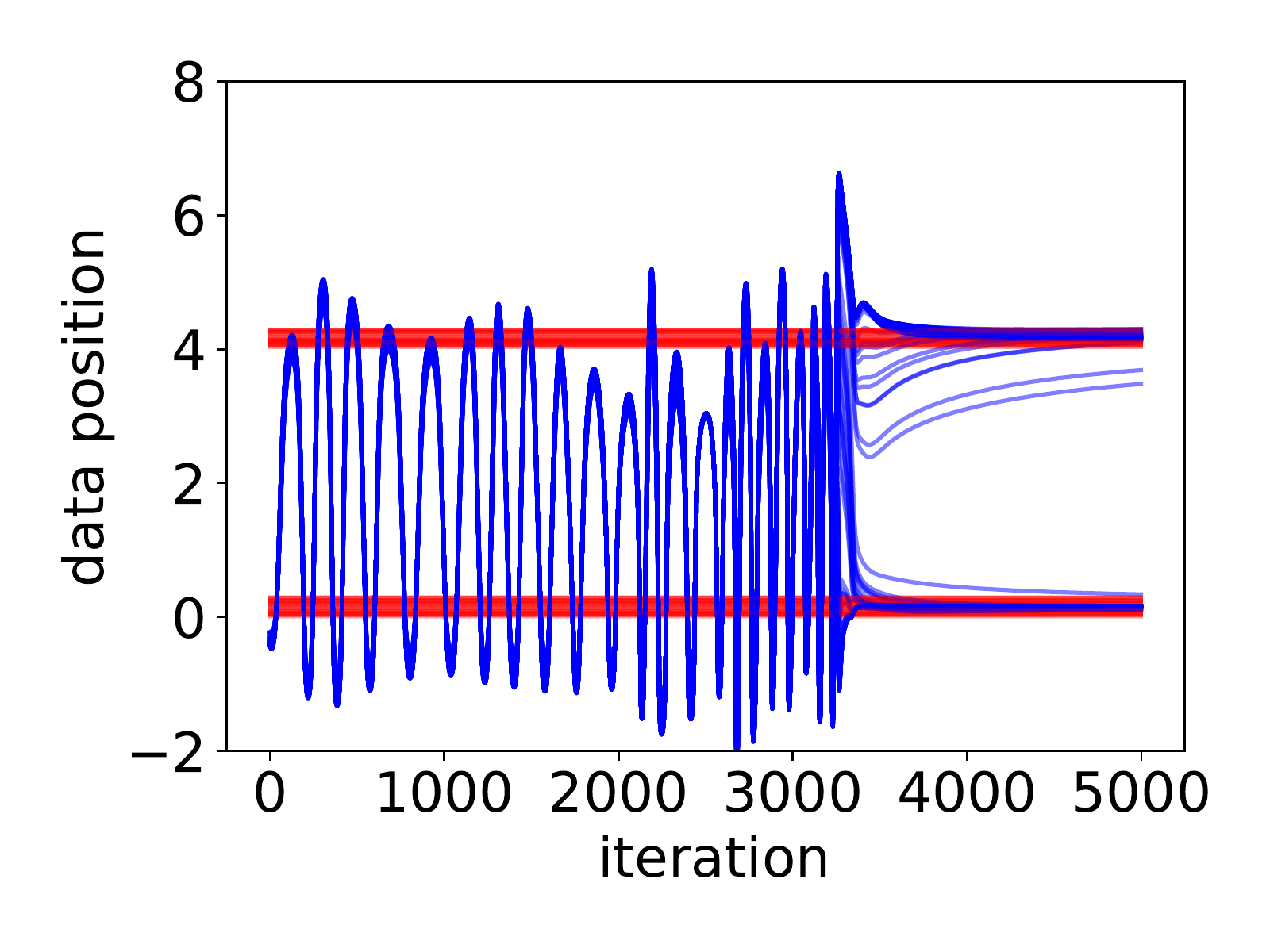} &
        \includegraphics[width=0.2\linewidth, height = 1.8cm]{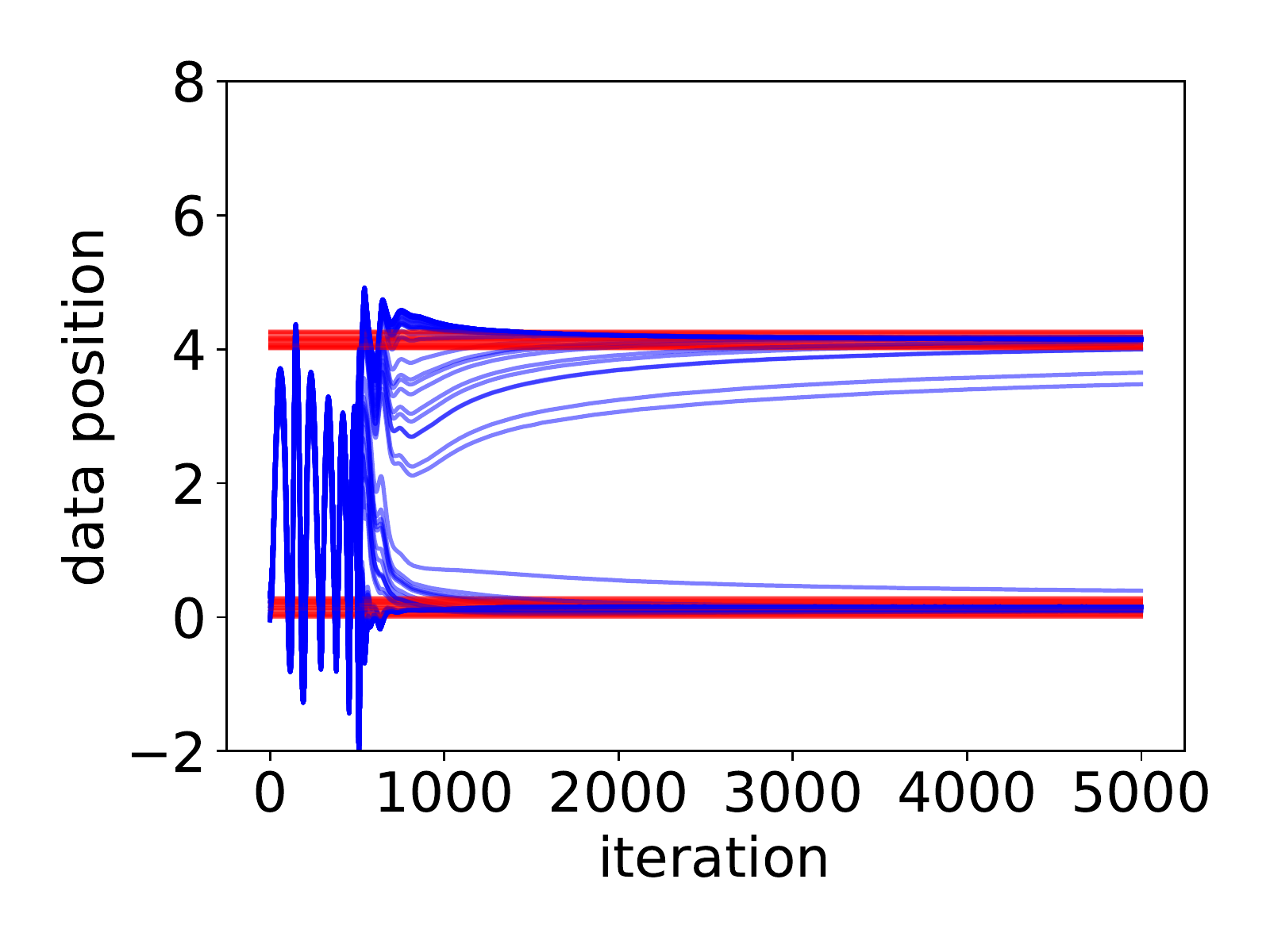} &
        \includegraphics[width=0.2\linewidth, height = 1.8cm]{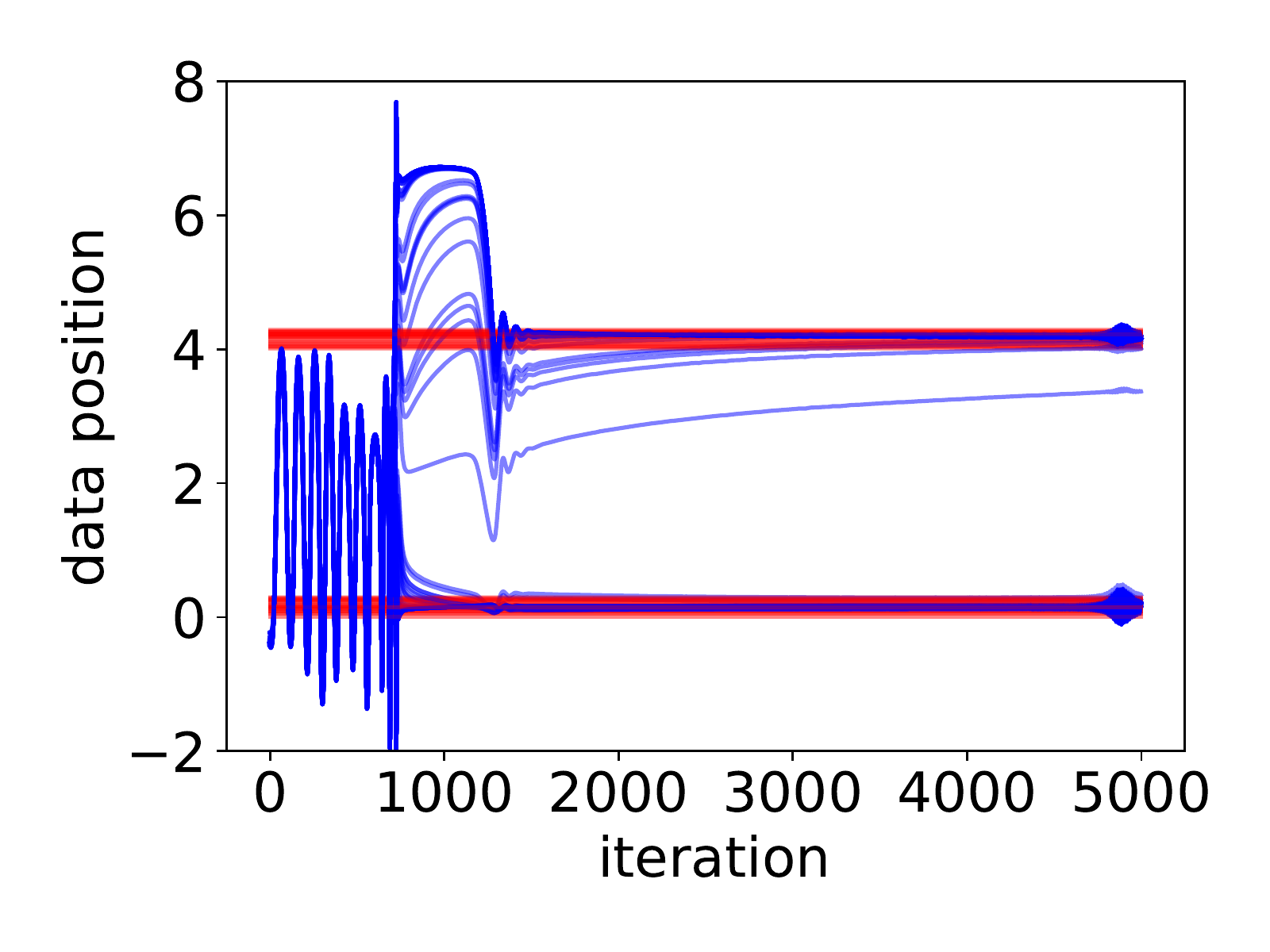} 
        \\
{\scriptsize (a) JS-GAN 1st run } & 
{\scriptsize  (b) JS-GAN 2nd run } &
{\scriptsize  (c) RS-GAN 1st run } &
{\scriptsize  (d) RS-GAN 2nd run }
    \end{tabular}
    \vspace{-0.2cm}
\captionsetup{font={scriptsize}} 
    \caption{Comparison of JS-GAN and RS-GAN for
    two different runs. First row: D loss; second row: fake data movement during training.
    }
    \label{fig 6 standard 2 cluster}
\vspace{-0.1cm}
\end{figure}

\textbf{Effect of width.}
The default width is (Dwidth, Gwidth) $= ( 10,5 )  $. 
We tested two other settings:  $ ( 20, 10  ) $ and $ (5, 3) $.
For the wide-network setting, the convergence of both JS-GAN and RS-GAN are much faster,
but RS-GAN is still faster than JS-GAN in most cases; see Fig.~\ref{fig7 wide net 2 cluster}.
For  the narrow-network setting, RS-GAN can recover two modes in all five runs, while JS-GAN fails in two of the five runs (within 5k iterations). 
See  Fig.~\ref{fig7 narrow net 2 cluster} for one success case of JS-GAN
and one failure case of JS-GAN.
In the failure case, JS-GAN completely gets stuck at mode collapse,
and the $D$ loss is stuck at around $0.48 $, consistent with our theory.

\iflonger 
Comparing Figure~\ref{fig7 wide net 2 cluster}
to Figure~\ref{fig 6 standard 2 cluster},
we  see that the range of the $D$ loss is much
smaller in the wide-network setting, for both JS-GAN and RS-GAN.
In particular, for JS-GAN wide-net setting the $D$ loss is often larger than 0.6,
 while for JS-GAN standard-net setting the $D$ loss often gets close to the critical value $ 0.48 $. That means that for the wide network,
the iterates are not fully attracted to the bad basin $ (s_{1 \rm a}, D^*(s_{1 \rm a})) $, but just attracted half-way. 
 \fi

\begin{figure}[t]
\vspace{-0.2cm}
\centering
    \begin{tabular}{cccc}
   \includegraphics[width=0.2\linewidth, height = 1.8cm]{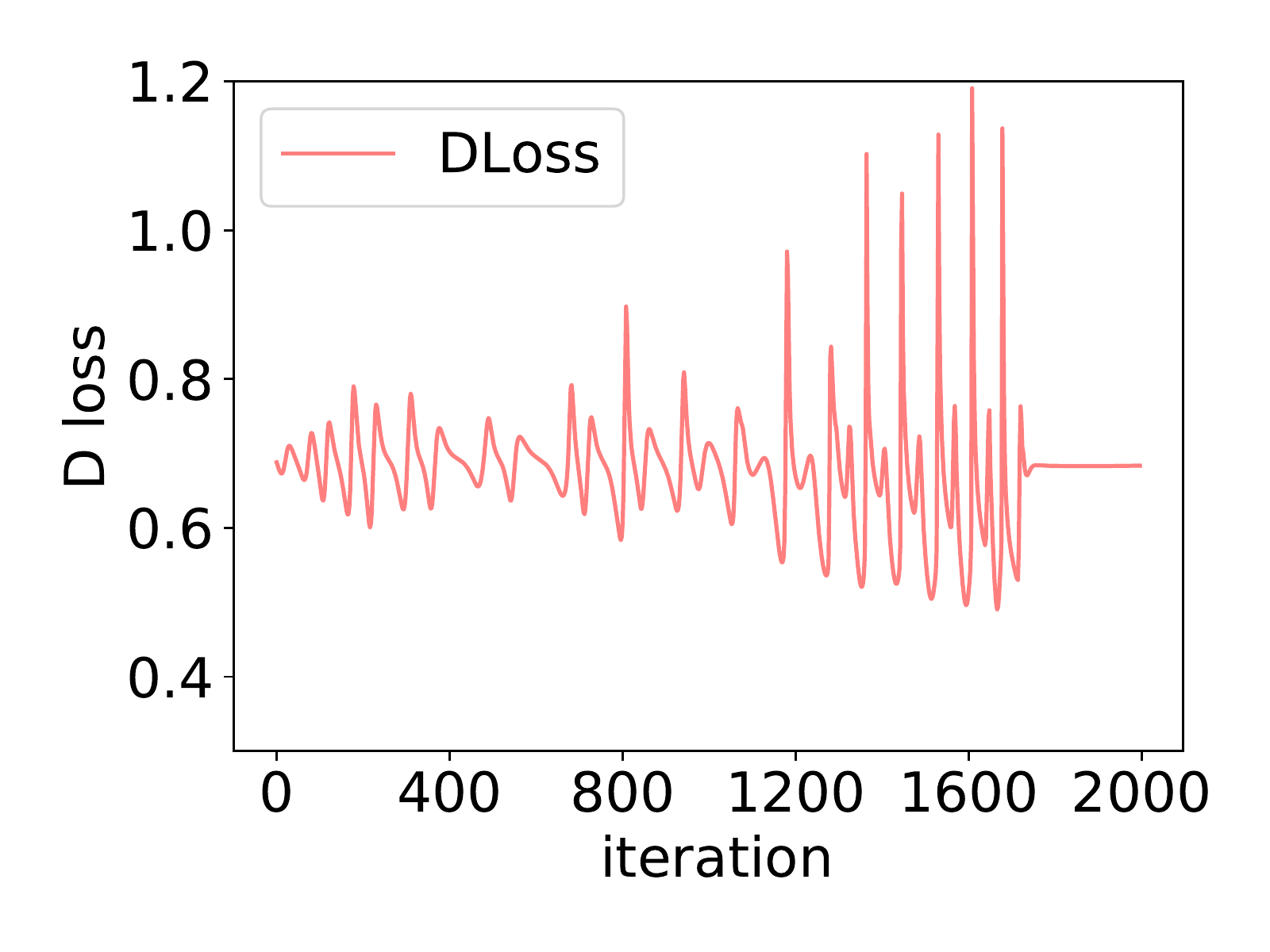} & 
        \includegraphics[width=0.2\linewidth, height = 1.8cm]{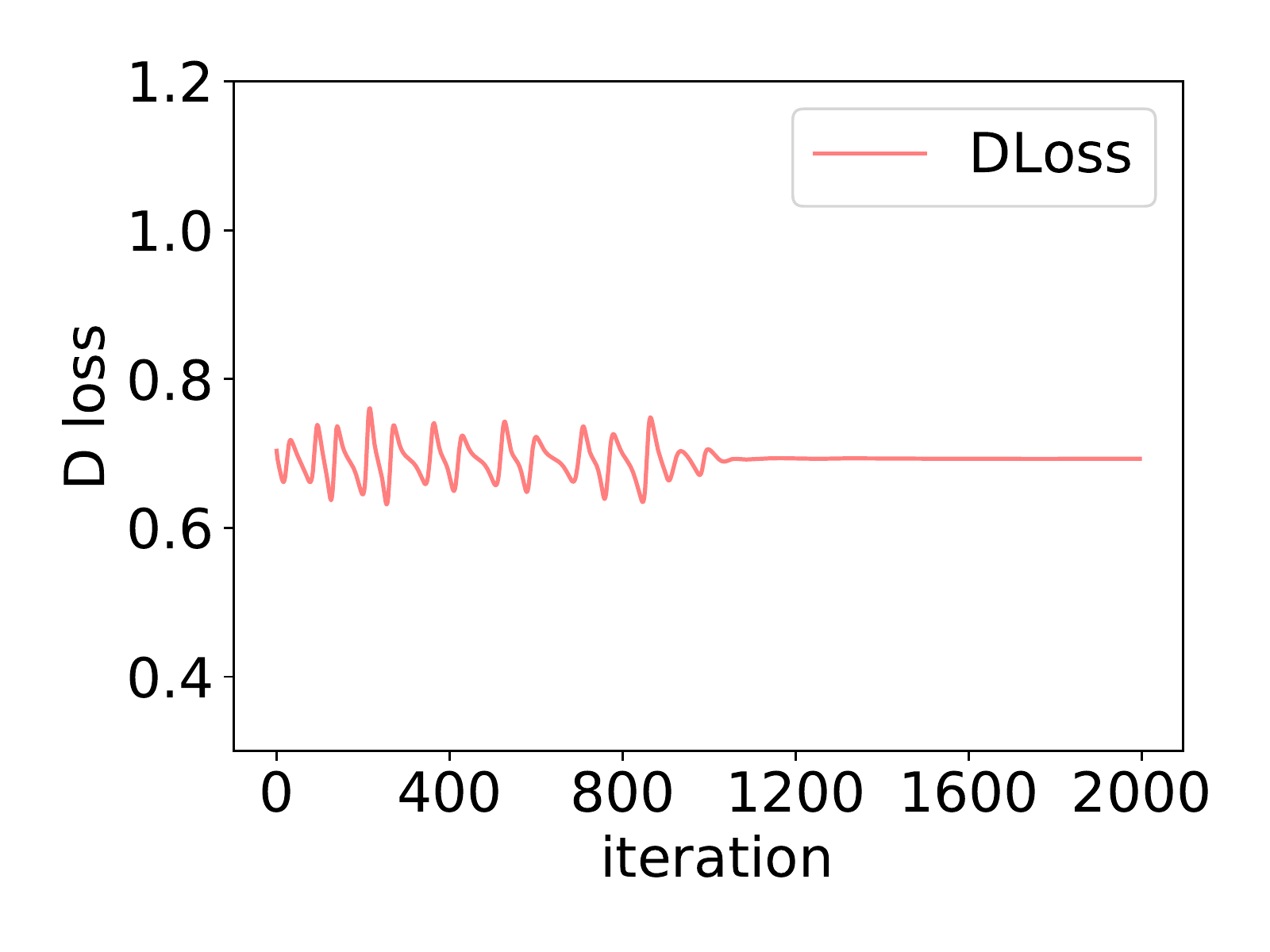} &
       \includegraphics[width=0.2\linewidth, height= 1.8cm]{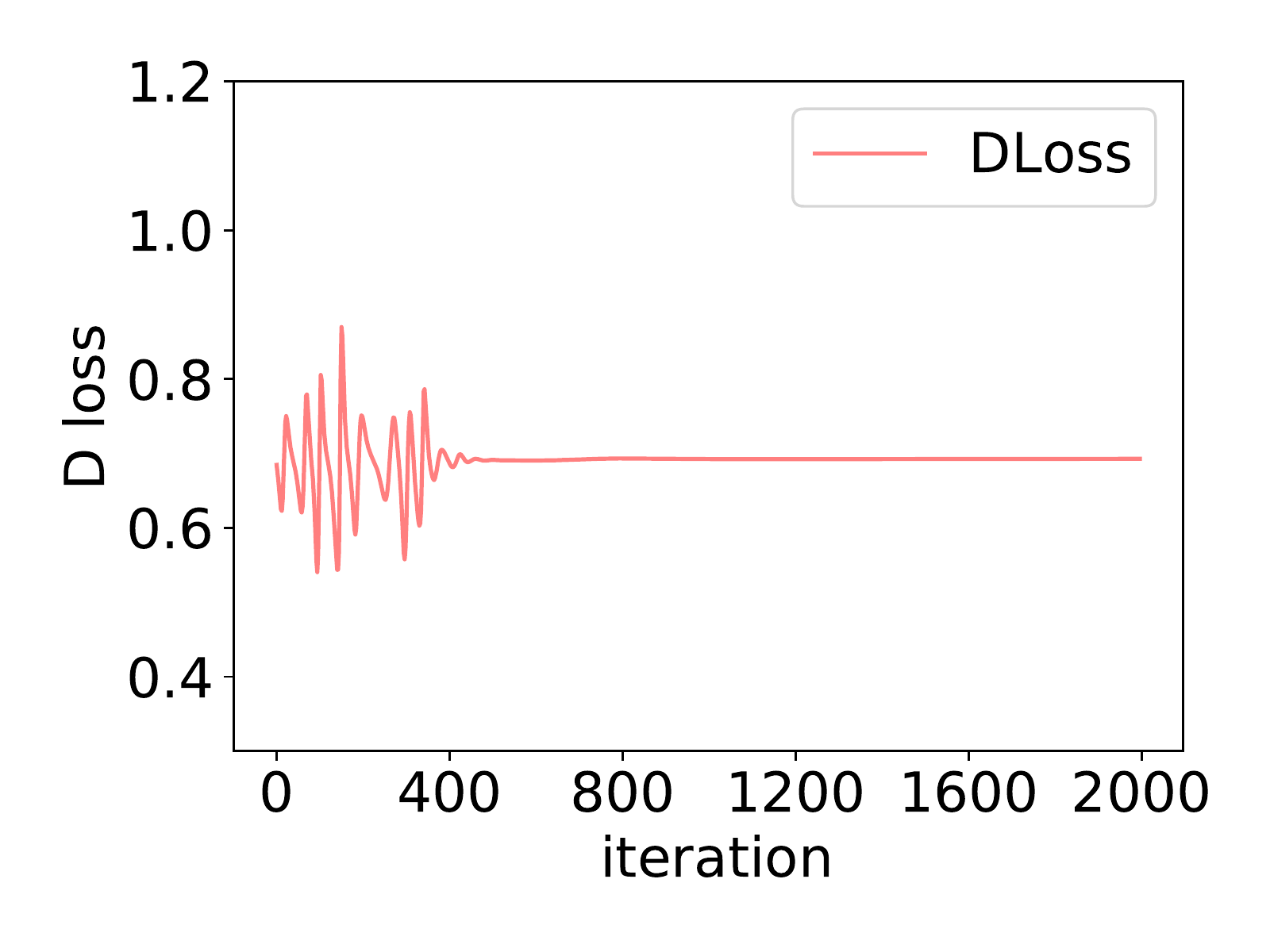} &
               \includegraphics[width=0.2\linewidth, height = 1.8cm]{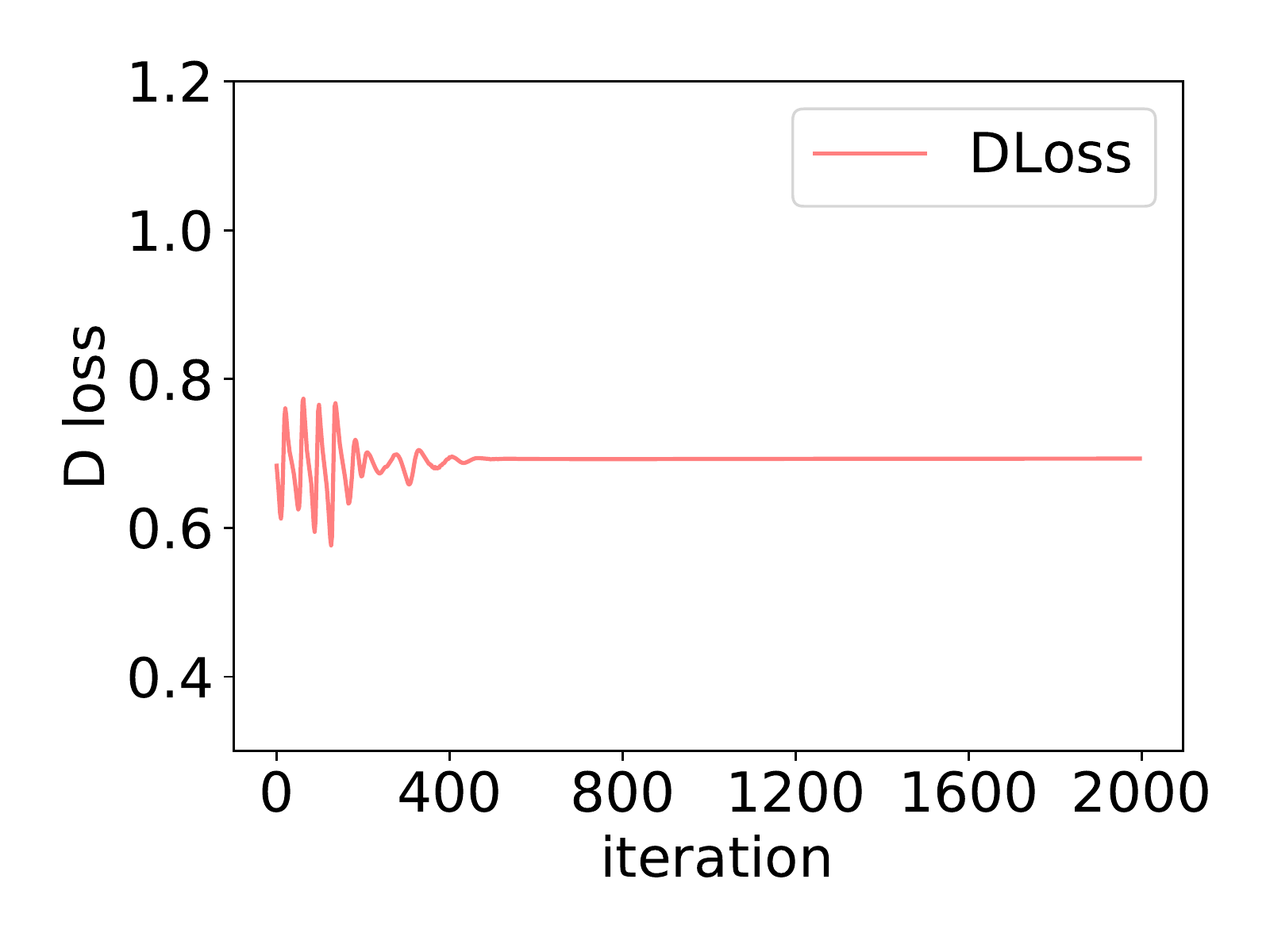} 
        \\
{\scriptsize (a) JS-GAN 1st run } & 
{\scriptsize  (b) JS-GAN 2nd run } &
{\scriptsize  (c) RS-GAN 1st run } &
{\scriptsize  (d) RS-GAN 2nd run }
    \end{tabular}
    \vspace{-0.2cm}
\captionsetup{font={scriptsize}} 
    \caption{Wide network (Dwidth, Gwidth) $= ( 20, 10)  $: JS-GAN and RS-GAN in two different runs.  %
    Compare to regular widths (Dwidth, Gwidth) $= ( 10, 5)  $,
    both GANs converge faster. 
Anyhow, RS-GAN is still 2-3 times faster than JS-GAN.
    }
    \label{fig7 wide net 2 cluster}
\vspace{-0.3cm}
\end{figure}

\iffalse 
\includegraphics[width=0.2\linewidth, height = 1.8cm]{FigApp/JS_Dloss_Dit10_Git10_seed2_epoch2000_Dw20_Gw10_run1.pdf} &
 \includegraphics[width=0.2\linewidth, height = 1.8cm]{FigApp/JS_Dloss_Dit10_Git10_seed2_epoch2000_Dw20_Gw10_run3.pdf} 
 \includegraphics[width=0.2\linewidth, height = 1.8cm]{FigApp/CP_Dloss_Dit10_Git10_seed2_epoch2000_Dw20_Gw10_run5.pdf} &
        \includegraphics[width=0.2\linewidth, height = 1.8cm]{FigApp/CP_Dloss_Dit10_Git10_seed2_epoch2000_Dw20_Gw10_run7.pdf} 
\fi

\begin{figure}[t]
\vspace{-0.0cm}
\centering
    \begin{tabular}{cccc}
   \includegraphics[width=0.2\linewidth, height = 1.8cm]{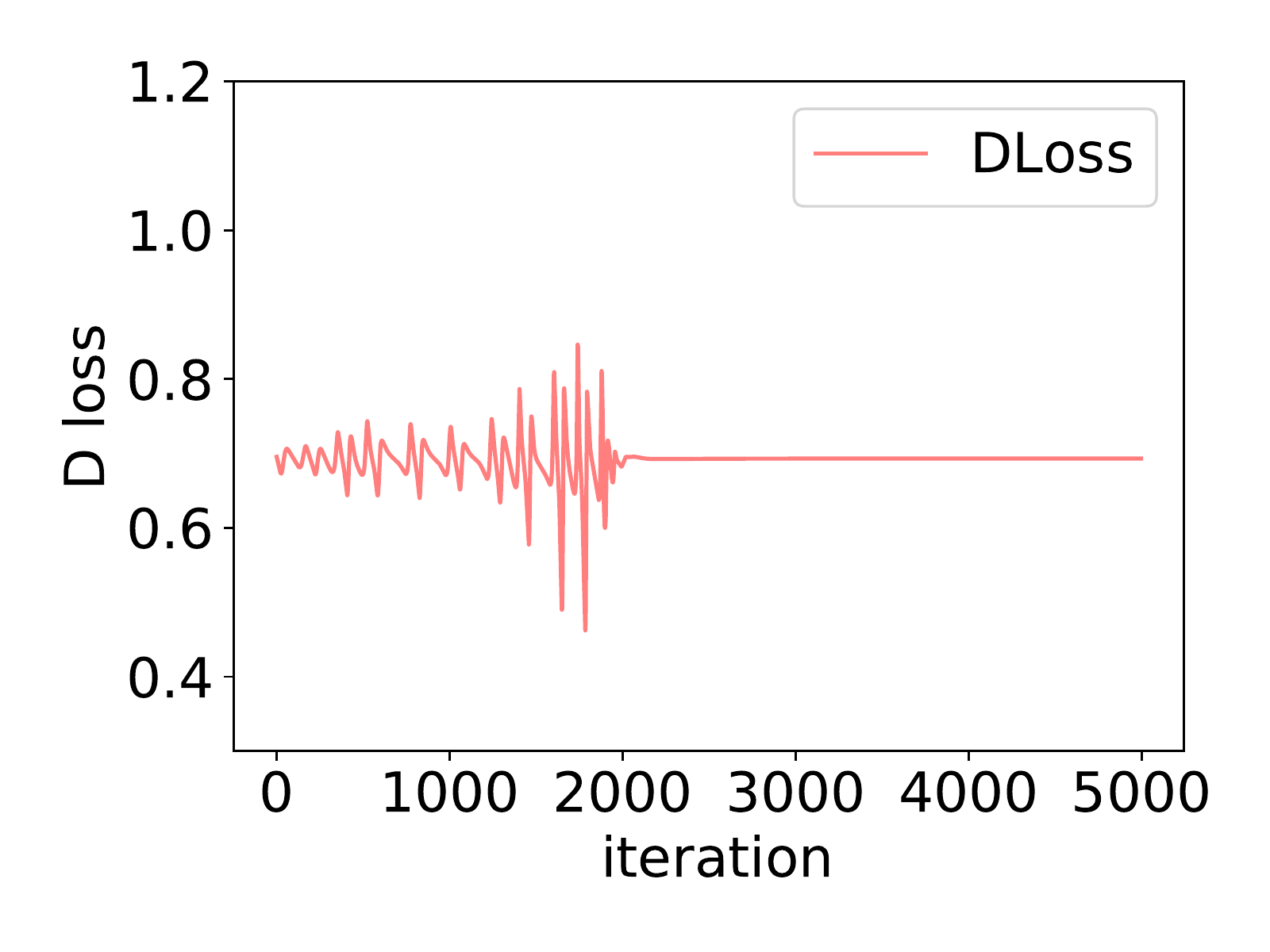} & \includegraphics[width=0.2\linewidth, height = 1.8cm]{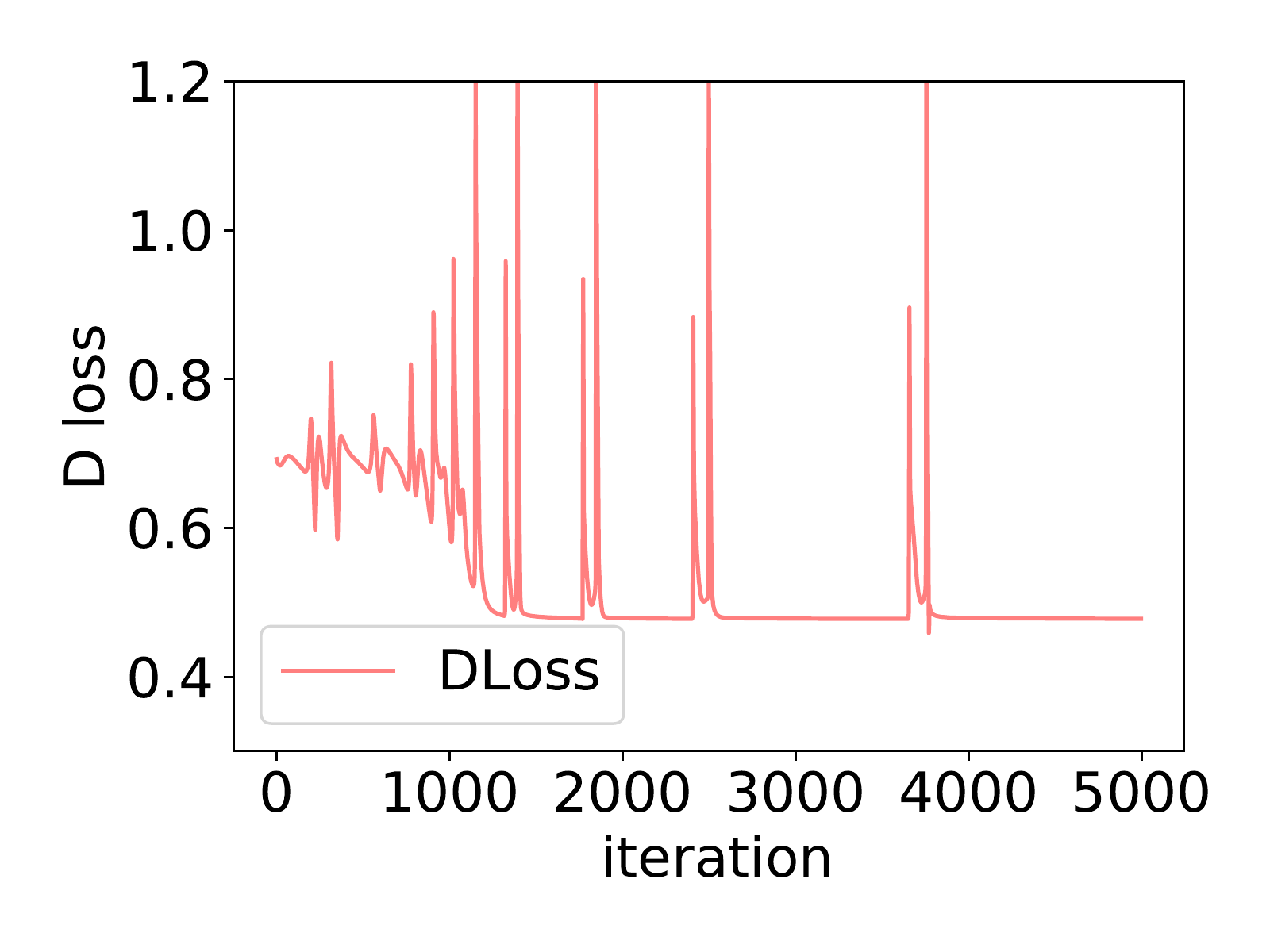} &
        \includegraphics[width=0.2\linewidth, height = 1.8cm]{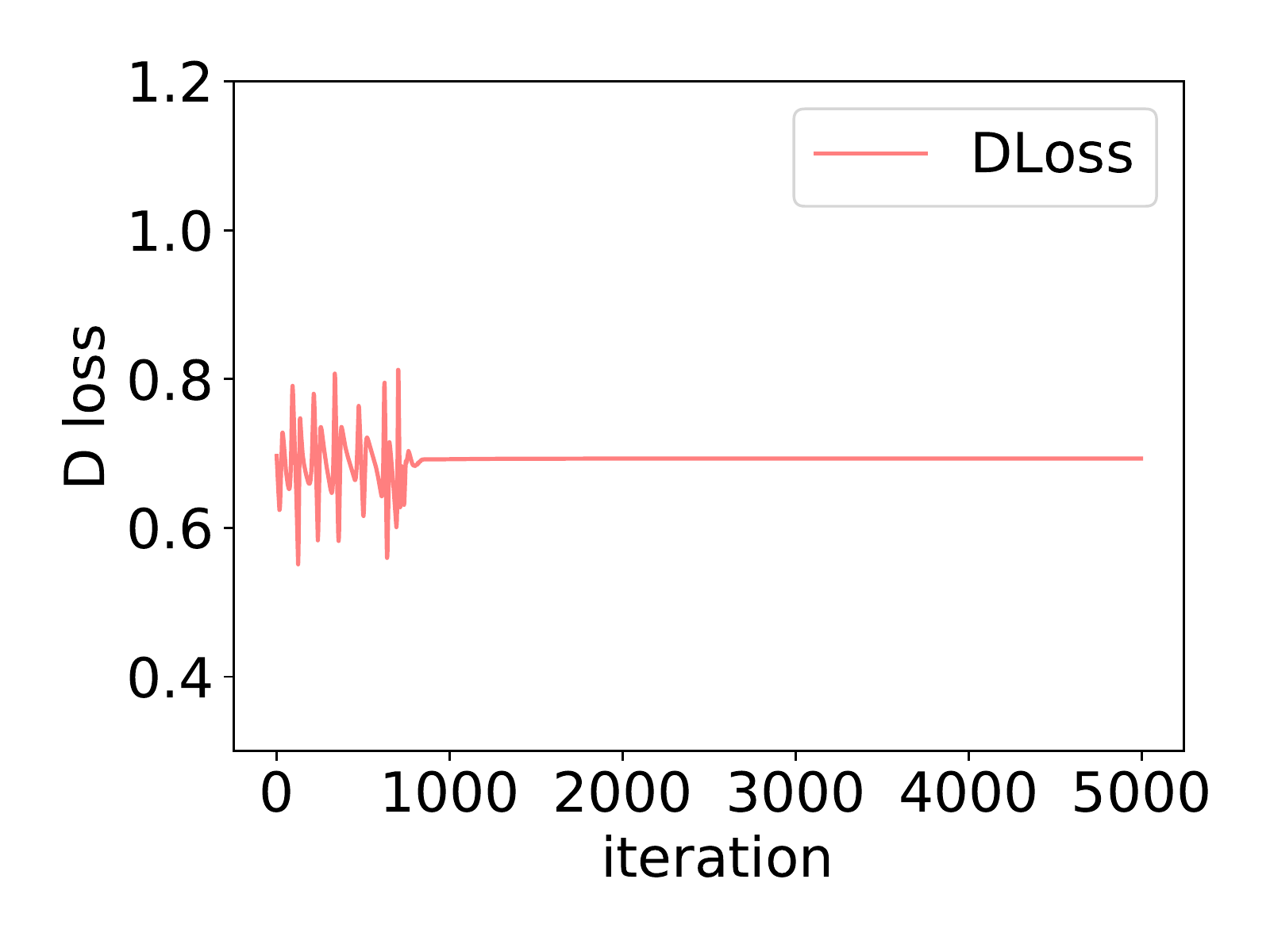} &
        \includegraphics[width=0.2\linewidth, height = 1.8cm]{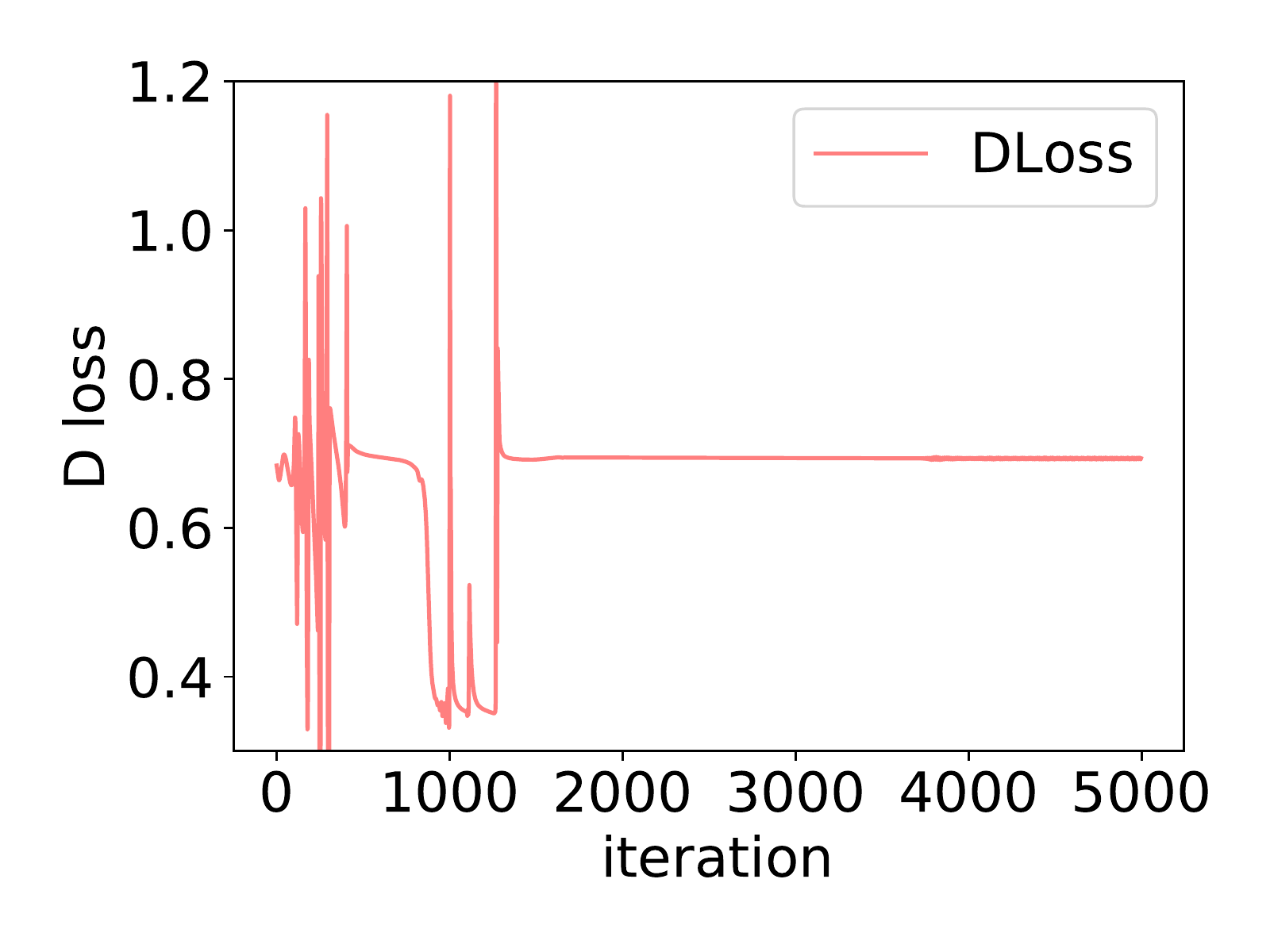}
        \\
 \includegraphics[width=0.2\linewidth, height= 1.8cm]{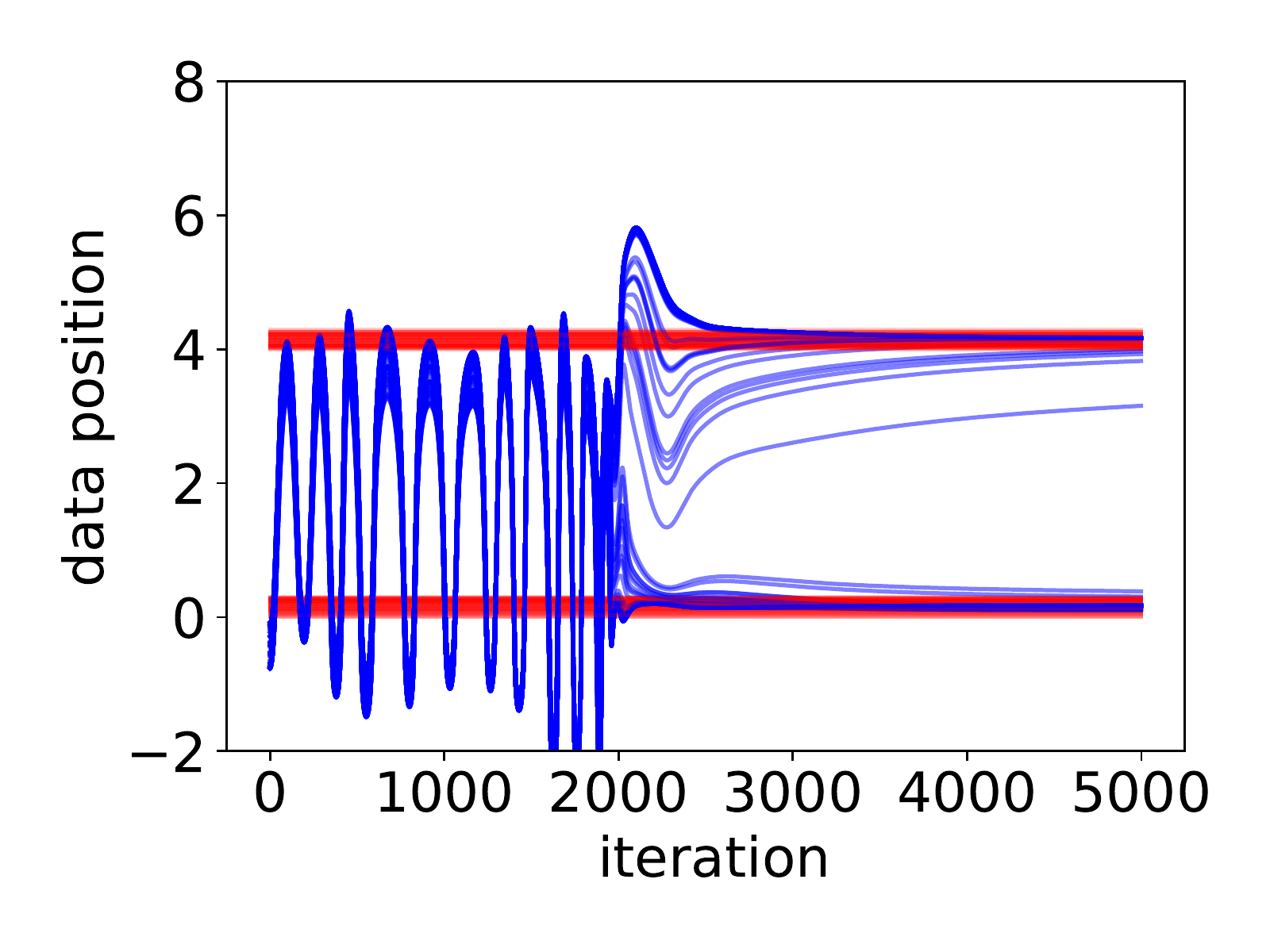} & \includegraphics[width=0.2\linewidth, height = 1.8cm]{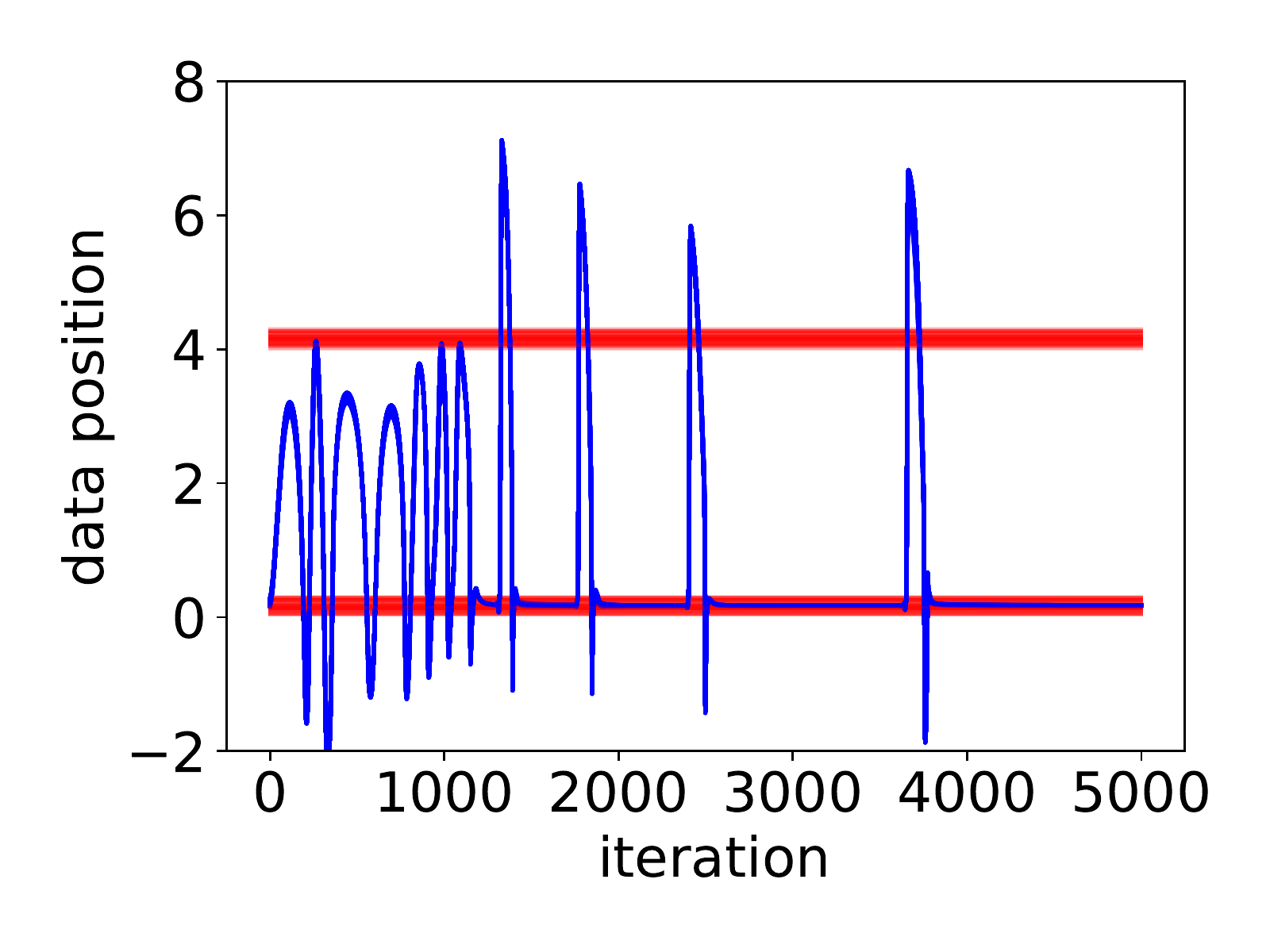} &
        \includegraphics[width=0.2\linewidth, height = 1.8cm]{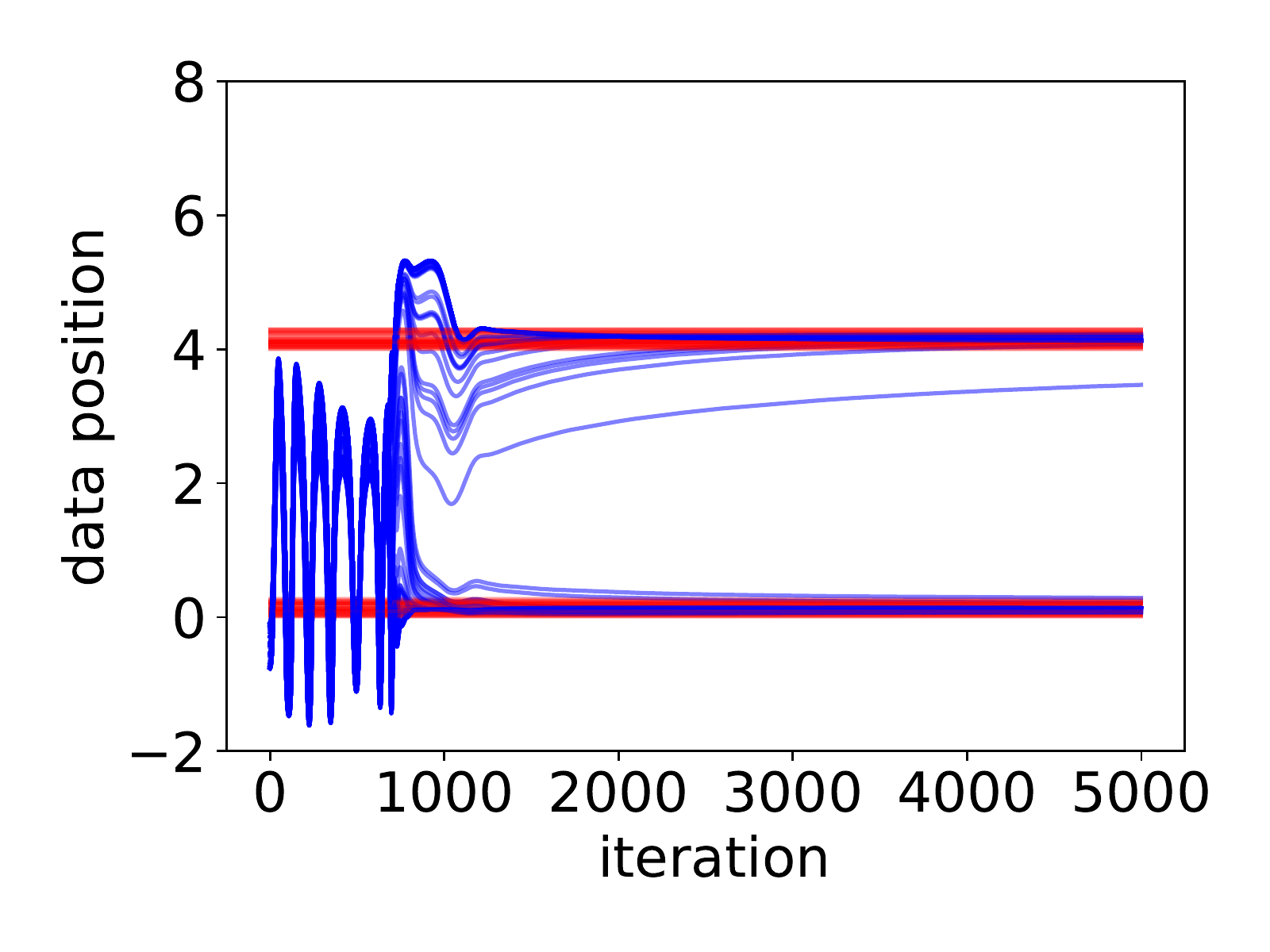} &
        \includegraphics[width=0.2\linewidth, height = 1.8cm]{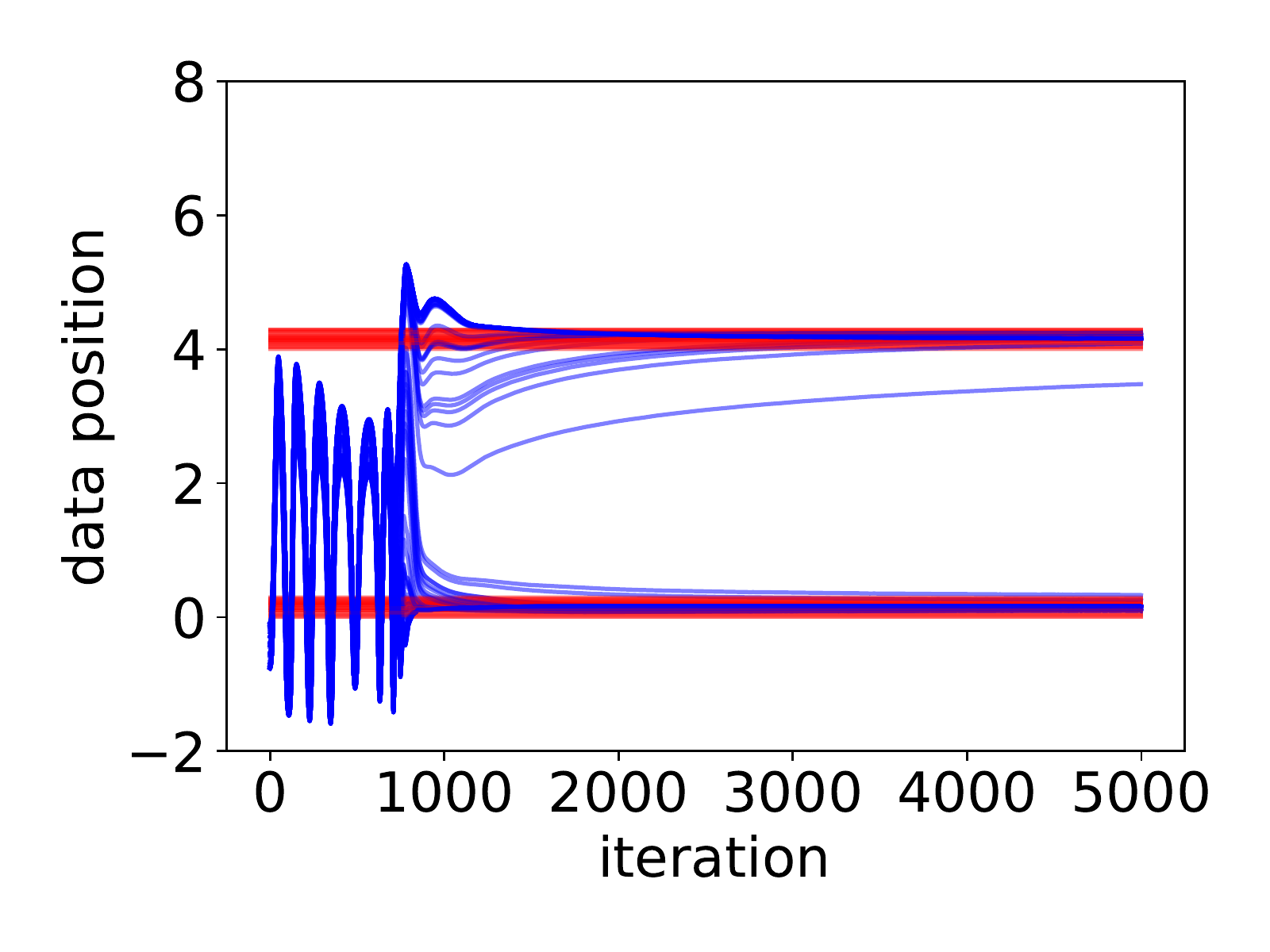} 
        \\
{\scriptsize (a) JS-GAN 1st run } & 
{\scriptsize  (b) JS-GAN 2nd run } &
{\scriptsize  (c) RS-GAN 1st run } &
{\scriptsize  (d) RS-GAN 2nd run }
    \end{tabular}
    \vspace{-0.2cm}
\captionsetup{font={scriptsize}} 
    \caption{Narrow network setting: Comparison of JS-GAN and RS-GAN in two runs.  
    RS-GAN is a few times faster than JS-GAN in general.
    Compare to default widths (D width 10, G width 5), 
    both GANs converge slower.
    In one case (b), JS-GAN gets stuck at mode collapse.
    }
    \label{fig7 narrow net 2 cluster}
\vspace{-0.3cm}
\end{figure}

\iffalse 
\textbf{Other lessons.}
GAN training is sensitive to hyper-parameters, even for
this two-cluster setting. 
To ensure that our findings are not biased by hyper-parameter choice,
we study a large set of hyper-parameters and
summarize a few findings below.
First, tuning hyper-parameters (learning rate, $G$ and $D$  iterations)
can affect the speed of escaping mode collapse.
 Proper hyper-parameters can help JS-GAN avoid 
  converging to the bad $D^*$ as shown in 
  Fig.~\ref{fig:2clustercomp}. We can recognize this
  situation if  the $D$ loss is fluctuating but does not hit $ 0.48 $. In this case, $(s_{1 \rm a}, D^*(s_{1 \rm a}))$ only slows down the training by attracting the iterates half-way. 
In any case, due to the differing attraction power, RS-GAN is still faster than JS-GAN. Second, mode collapse can be due to  neural-net issues  (e.g., training getting stuck). %
However, such cases do not exhibit the pattern reported in Fig.~\ref{fig5 D loss and D image}.
For details of these findings (and other findings),
see Appendix~\ref{subsec: details of experiments}. 
\fi

\textbf{Other hyperparameters.}
Besides the width, the learning rates
  and (DIter, GIter) will also affect the training process. 
 As for (DIter, GIter), we use $(10, 10)$ as default,
 but other choices such as $ (5, 2) $ and $(1 , 1)$ also work. 
As for learning rates, we use $ (0.01, 0.01) $ as default,
but smaller learning rates such as  $ (0.001, 0.001) $ also work.
Different from the default hyper-parameters,
for some hyper-parameters, the D loss of JS-GAN does not reach $0.48$,
indicating that the basin only attracts the iterates half-way. 
Nevertheless, in most settings RS-GAN is still faster than JS-GAN.

\iffalse 
the training behavior for other hyper-parameters
are qualitatively similar to the behavior in Figure~\ref{fig 6 standard 2 cluster} and Figure~\ref{fig7 wide net 2 cluster}. 
\fi

\iflonger 
We present a few main findings from various hyper-parameter settings below. 
The two-stage behavior is observed for almost all experiments (except narrow nets). 
 In the first stage, the generated points swing between the two true clusters; in the second stage, there are two possibilities: a) success case, where the generated points split into two clusters and converge to the true clusters; b) failure case, where the generated points fall into one mode. We believe that the attraction power of mode collapse makes the generated points swing between the two true clusters, and whether the training leads to success or failure case depends on the hyper-parameter choice (proper balance of $D$ and $G$). The convergence time depends on the time to escape mode collapse. 
 The major reason, as we argued, is due to the 
smaller attraction power of mode collapse for RS-GAN,
which is supported by our theoretical result that
RS-GAN optimization landscape has no bad basin. 

\textbf{Judgement of success of failure based on D loss}
As a corollary, we can tell whether the training succeeds
or not by checking the $D$ loss plot. 
If the $D$ loss is still fluctuating, then we need to train longer.
If the $D$ loss gets stuck at 0.48 for JS-GAN (or 0.35 for RS-GAN),
then we can claim the iterates probably fell into mode collapse.
If the $D$ loss first fluctuates and then converges to 0.7, 
then the training probably successfully recovers the two modes.
We highlight again that the above simple criterion of success was observed
for almost all experiments except narrow nets.
This is nontrivial since we expected that divergence
or strange behavior could happen; but we find that with long enough
training time (and not too crazy learning rates), 
the training either succeeds or falls into mode collapse. 
\fi

\section{Result and Experiments for Imbalanced Data Distribution}\label{sec: imbalance}
In the main results, we assume $x_i$'s are distinct.
In this section, we allow $x_i$'s to be in general positions, i.e., they can overlap.
The 2-point model can only approximate two balanced  clusters. 
Allowing $x_i$'s to overlap, we are able to analyze imbalanced two clusters. 
We will show: 
(i) a theoretical result for 2-cluster data; (ii) experiments on
imbalanced 2-cluster data and MNIST. 

\subsection{Imbalanced Data: Math Results for Two-Clusters}\label{sec: imbalance formulation}
Assume there are $n$ true data points $X = (x_1, \dots, x_n)$  in two modes with proportion $\alpha$ and $1-\alpha$ respectively, where $\alpha>0.5$.
More precisely, assume $x_1  = x_2 = \dots = x_{n \alpha } $ 
and $ x_{n\alpha+1  } = \dots = x_n$, 
and denote two multi-sets $ \mathcal{X}_1 = \{ x_1  , x_2 , \dots , x_{n \alpha } \} $
 and $ \mathcal{X}_2 = \{ x_{n\alpha+1  }  , x_2 , \dots , x_{n } \} $.
Denote $ Y =  (  y_1, \dots, y_n )  $ as the tuple of all generated points,
and let $ \mathcal{Y} $ be the multiset $ \{  y_1, \dots, y_n \} . $

\begin{claim}\label{claim of JS GAN imbalanced}
Consider the JS-GAN loss defined in Eq.~\eqref{JSGAN min-max finite, using D},
where $X$ is defined above. We have 
\begin{equation}
\begin{split}
\phi_{\rm JS}(Y, X) = q_{\alpha}(m_1) + q_{1 - \alpha}(m_2), \text {if } | 
    \mathcal{X}_1 \!\cap\!  \mathcal{Y}  |\!\!=\!\! m_1  , | 
\mathcal{X}_2 \!\cap\! \mathcal{Y} |\!\!=\!\! m_2 ,  \\
 \text{where } q_{\alpha}(m) \triangleq \frac{\alpha}{2}\log (\alpha n) + \frac{m}{2n}\log m - \frac{\alpha n+m}{2n}\log (\alpha n+m)  .
 \end{split}
\end{equation}
As a result, the global minimal loss is $ - \log 2 $, which is achieved iff
    $ \mathcal{Y}   = \mathcal{X}_1 \cup  \mathcal{X}_2  $. 
\end{claim}

\iffalse 
As a result, we have
   \begin{equation*}
   \phi_{\rm JS}(Y, X) = 
  \begin{cases}
    - \log 2 \approx -0.6931 
    &      \hspace{-0.2cm}\text{if }  \mathcal{Y} = \mathcal{X}_1 \cup  \mathcal{X}_2   \\
q_{\alpha}(m_1) + q_{1 - \alpha}(m_2) 
    &      \hspace{-0.2cm}\text{if } |  \mathcal{X}_1 \!\cap\! \mathcal{Y} |\!\!=\!\! m_1 , | \mathcal{X}_2 \!\cap\!  \mathcal{Y} |\!\!=\!\! m_2,  \\
    0       &     \hspace{-0.2cm}\text{if }     |  \mathcal{X}_i \!\cap\! \mathcal{Y}  | \!\!=\!\! 0, i=1, 2
    \end{cases}  
   \end{equation*}
   \fi

\begin{coro}\label{coro of strict local min imbalanced}
 Suppose $ \hat{Y} = ( \hat{y}_1, \dots, \hat{y}_n ) $ satisfies
  $ |  \mathcal{X}_1  \cap  \hat{ \mathcal{Y} } |\!\!=\!\! n_1   , | 
\mathcal{X}_2  \cap  \hat{ \mathcal{Y} } |\!\!=\!\!  n - n_1,  $
where $ \hat{ \mathcal{Y} } = \{ \hat{y}_1, \dots, \hat{y}_n  \} $ is the multiset of all $ \hat{y}_j$'s, then  $ \hat{Y} $ is a strict local minimum.
Moreover, if $ n_1 \neq  n \alpha , $
then $ \hat{Y} $ is a sub-optimal strict local minimum.
\end{coro}

The proofs of Claim \ref{claim of JS GAN imbalanced} and Corollary \ref{coro of strict local min imbalanced} are given in Appendix \ref{app-sub: proof of imbalanced}.

\begin{figure}
    \centering
    \includegraphics[height = 2.5cm, width = 10cm]{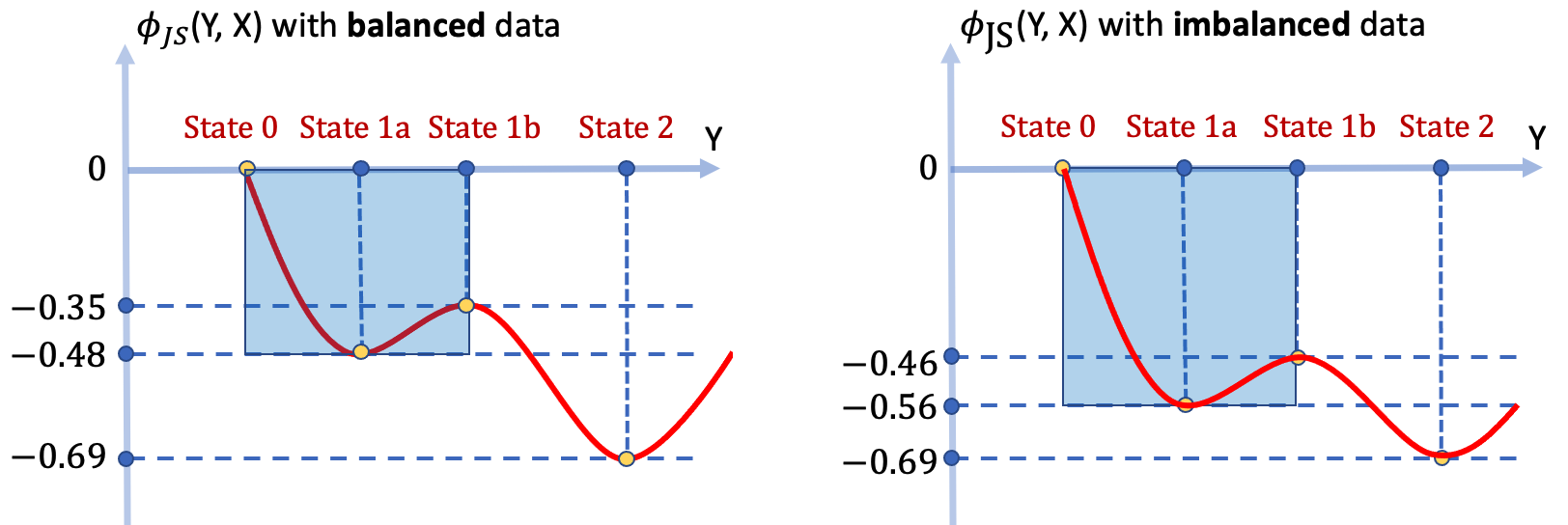}
    \captionsetup{font={scriptsize}}
    \caption{Illustration of the landscape
     of JS-GAN for balanced two clusters with $\alpha = 0.5$ (left)
     and imbalanced two clusters with $ \alpha = 2/3 $ (right).
  Denote $ m_i \triangleq  | \mathcal{X}_i \!\cap\! \mathcal{Y} | , i=1,2$.
Here state 0, state 1a, state 1b, state 2 represent 
$(m_1, m_2) = ( 0, 0)$, $(n \alpha, 0)$, $(n \alpha,  0),$
$( n \alpha, n (1 - \alpha) )  $ respectively.
By Claim \ref{claim of JS GAN imbalanced}, for $\alpha = 1/2$, $q_{\alpha}(n) \approx -0.48$ and $q_{\alpha}(\alpha n) \approx -0.35$; for $\alpha = 2/3$, $q_{\alpha}(n) \approx -0.56$ and $q_{\alpha}(\alpha n) \approx -0.46$. 
Different from the 2-point-case landscape in Fig ~\ref{fig:2clustercomp}, 
 there should be some intermediate patterns (satisfying  $ m_1  \leq  n,  m_2  =  0)$, but for simplicity we do not show them. 
From state 1a to state 2, $Y$ can go through state 1b  or go through state 0, but we only show the path through state 1b.
 We view the gap between state 0 and state 1a as an approximation of the ``depth'' of the basin. }
    \label{fig: imbalanced}
\end{figure}

\iflonger 
It is easy to show that $ q_{ \alpha}(n)$, the loss value at $ (m_1, m_2) = (n , 0)$, is  $q_{ \alpha}(n) = \frac{\alpha }{2} \log \frac{ \alpha  }{ \alpha + 1 } + \frac{1}{2} \log \frac{1}{ \alpha + 1}. $
The computation is as follows:
\begin{align*}
q_{\alpha}(n) & = \frac{ \alpha }{2 } \log (n \alpha ) + \frac{n}{2n } \log n - \frac{ \alpha n + n }{ 2 n }
\log( \alpha n + n ) \\
& = \frac{\alpha }{2} \log \frac{ \alpha  }{ \alpha + 1 } + \frac{1}{2} \log \frac{1}{ \alpha + 1} . 
\end{align*}
This value is independent of $n$ and only depends on $\alpha$; in addition, 
$  q_{\alpha}( n   ) $ is a strictly decreasing function of $\alpha$. 
When $\alpha = 1/2$, $ q_{\alpha}(n)  = \frac{1}{4} \log \frac{1}{3} +  \frac{1}{2} \log \frac{2}{3} \approx 0.4774.  $ When $\alpha = 2/3$, $ q_{\alpha}(n) \approx -0.5608 $.  When $\alpha = 1$, $ q_{\alpha}(n) = -\log 2 \approx -0.6931 $. 
\fi 
 
 Denote $ m_1 \triangleq  | \mathcal{X}_2 \!\cap\! \mathcal{Y} |,  m_2 \triangleq  | \mathcal{X}_1 \!\cap\!  \mathcal{Y}  | . $ The value $  q_{\alpha}( n   ) $ indicates the value of $\phi(Y, X)$ at the mode collapsed pattern (state 1a) where $m_1 = n, m_2 = 0$. 
Note that  $q_{ \alpha}(n) = \frac{\alpha }{2} \log \frac{ \alpha  }{ \alpha + 1 } + \frac{1}{2} \log \frac{1}{ \alpha + 1} $
 is a strictly decreasing function of $\alpha$. 
 When $\alpha = 1/2$, $ q_{\alpha}(n)  = \frac{1}{4} \log \frac{1}{3} +  \frac{1}{2} \log \frac{2}{3} \approx -0.4774 $; when $\alpha = 2/3$, $ q_{\alpha}(n) \approx -0.5608 $. 
The more imbalanced the data are (larger $\alpha$),
the smaller $ q_{\alpha}(n)  $, and further the deeper the basin.
In Figure \ref{fig: imbalanced}, we compare the loss landscape  of the balanced case $\alpha = 1/2$ and the imbalanced case $\alpha = 2/3$.

\iflonger 
To illustrate the effect of $\alpha$ (which measures the level of imbalance)
on the loss landscape, we compare the loss landscape for $\alpha = 1/2$ and $\alpha = 2/3$
in Figure \ref{fig: imbalanced}. 
 The mode collapse creates a deeper basin for the imbalanced data case (right
 sub-figure of Figure \ref{fig: imbalanced}) than the balanced data case (the left sub-figure of Figure \ref{fig: imbalanced}, which is the same as Fig.~\ref{fig1_GAN_landscape}(a)).
 \fi 

We suspect that the deeper basin in the imbalanced case will make it harder to escape mode collapse for JS-GAN. We then make the following prediction:
for JS-GAN, mode collapse is a more severe issue for imbalanced data than it is for balanced data. For RS-GAN, the performance does not change much as data becomes more imbalanced.
\iffalse 
\begin{align*}
&  \textit{Prediction: for JS-GAN, mode collapse is a more severe issue for imbalanced data than balanced data. }   \\
 & \quad \quad \textit{ For RS-GAN, the performance does not change much as data become more imbalanced. }
\end{align*}
\fi 
We will verify this prediction in the next subsections.

\subsection{Experiments}
\textbf{2-Cluster Experiments.}
For the balanced case, the experiment is described in Appendix~\ref{sec: 2clustersetting}. 
Both JS-GAN and RS-GAN can converge to the two-mode-distribution.
For the imbalanced case where $\alpha=\frac{2}{3}$,
with other hyper-parameters unchanged, JS-GAN falls into mode collapse while 
RS-GAN generates the true distribution (2/3 in mode 1 and 1/3 in mode 2) (see Fig.~\ref{2 cluster imbalance}). 
The loss $\phi_{\rm JS}(Y, X)$ ends up at  approximately -0.56, which
matches Claim \ref{claim of JS GAN imbalanced}.

\textbf{MNIST experiments.} 
To ease visualization, we create an MNIST  sub-dataset only containing 5's and 7's. We use the CNN structure of Tab.~\ref{table: CNN structure} and train for $30k$ iterations. For the balanced case, the number of 5's and 7's are identical (i.e., ratio 1:1). Both JS-GAN and RS-GAN generate a roughly equal number of 5's and 7's, as shown in Fig.~\ref{mnist imbalance}(a,b). For the imbalanced case with $4$ times more $7$'s than $5$'s (ratio 1:5), JS-GAN only generates 7's, while RS-GAN
generates 13 5's among 64 generated samples, aligning with the true data distribution (see Fig.~\ref{mnist imbalance}(c,d)).

The above two experiments verify our earlier prediction
that RS-GAN is robust to imbalanced data while JS-GAN easily gets stuck at
mode collapse for imbalanced data.

\begin{figure}[t]
\vspace{-0.2cm}
\centering
    \begin{tabular}{cccc}
   \includegraphics[width=0.2\linewidth, height= 1.8cm]{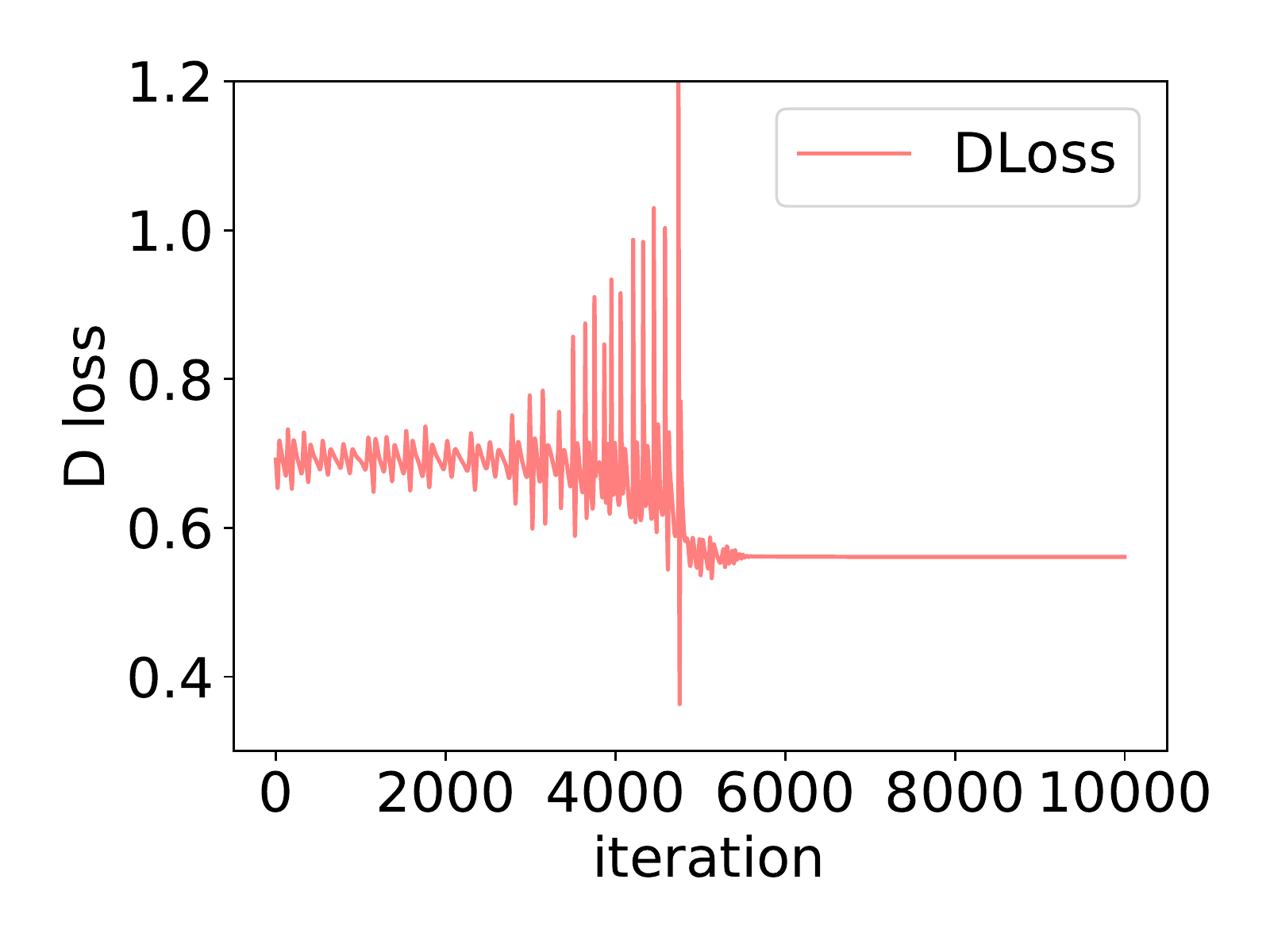} 
   & \includegraphics[width=0.2\linewidth, height = 1.8cm]{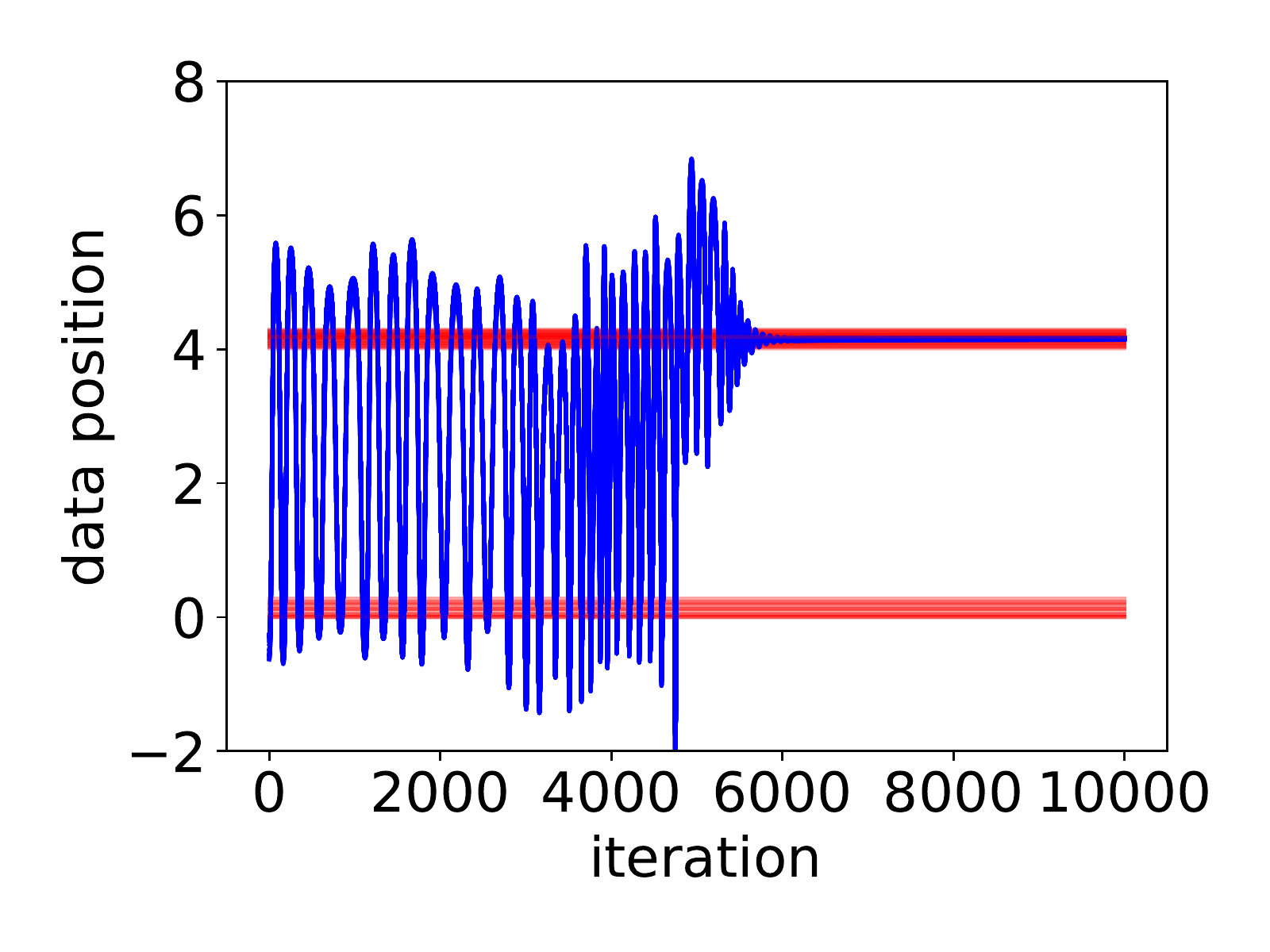} &
 \includegraphics[width=0.2\linewidth, height= 1.8cm]{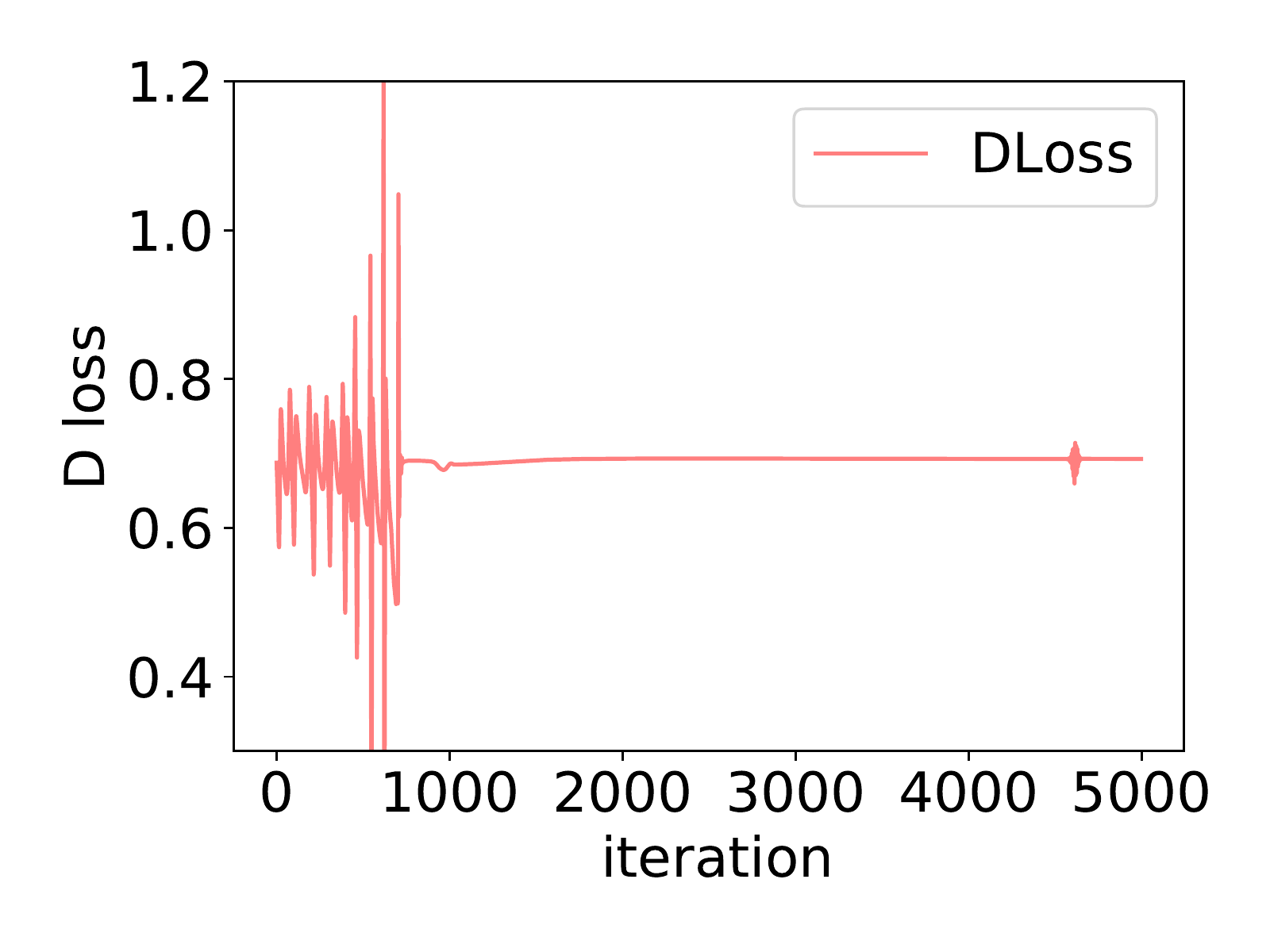} 
 & \includegraphics[width=0.2\linewidth, height =
 1.8cm]{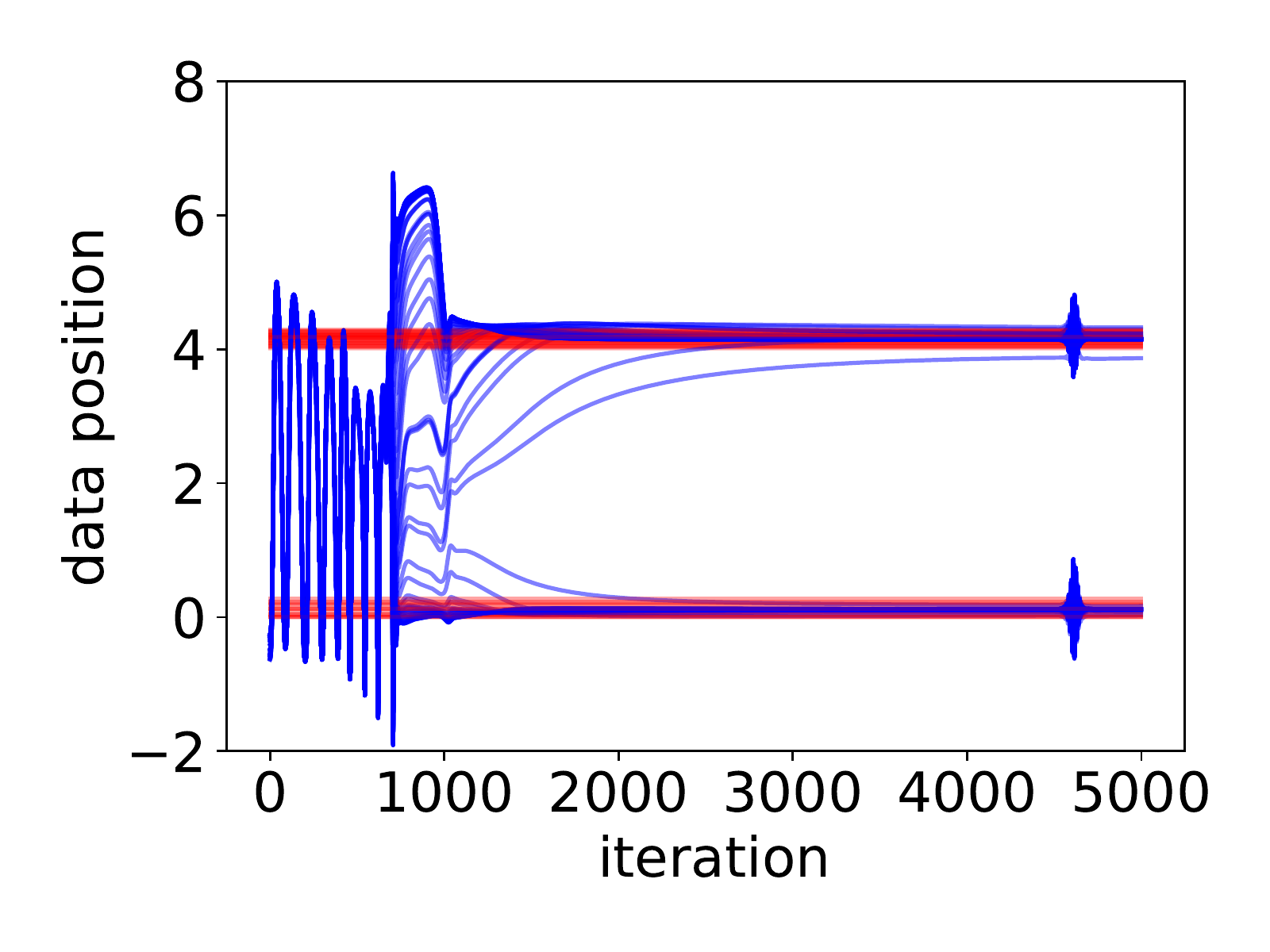} 
        \\
{\scriptsize (a) JS-GAN D loss } & 
{\scriptsize  (b) JS-GAN Data Evolution } &
{\scriptsize  (c) RS-GAN  D loss  } &
{\scriptsize  (d) RS-GAN Data Evolution }
    \end{tabular}
    \vspace{-0.1cm}
    \caption{Imbalanced 2-cluster result: comparison of JS-GAN in (a) and (b), and RS-GAN in (c) and (d). (a) and (c): evolution of D loss;
    (b) and (d): data position movement during training. 
    }
    \label{2 cluster imbalance}
\end{figure}

\begin{figure}[t]
\vspace{-0.2cm}
\centering
    \begin{tabular}{cccc}
   \includegraphics[width=2.3cm, height = 2.3cm]{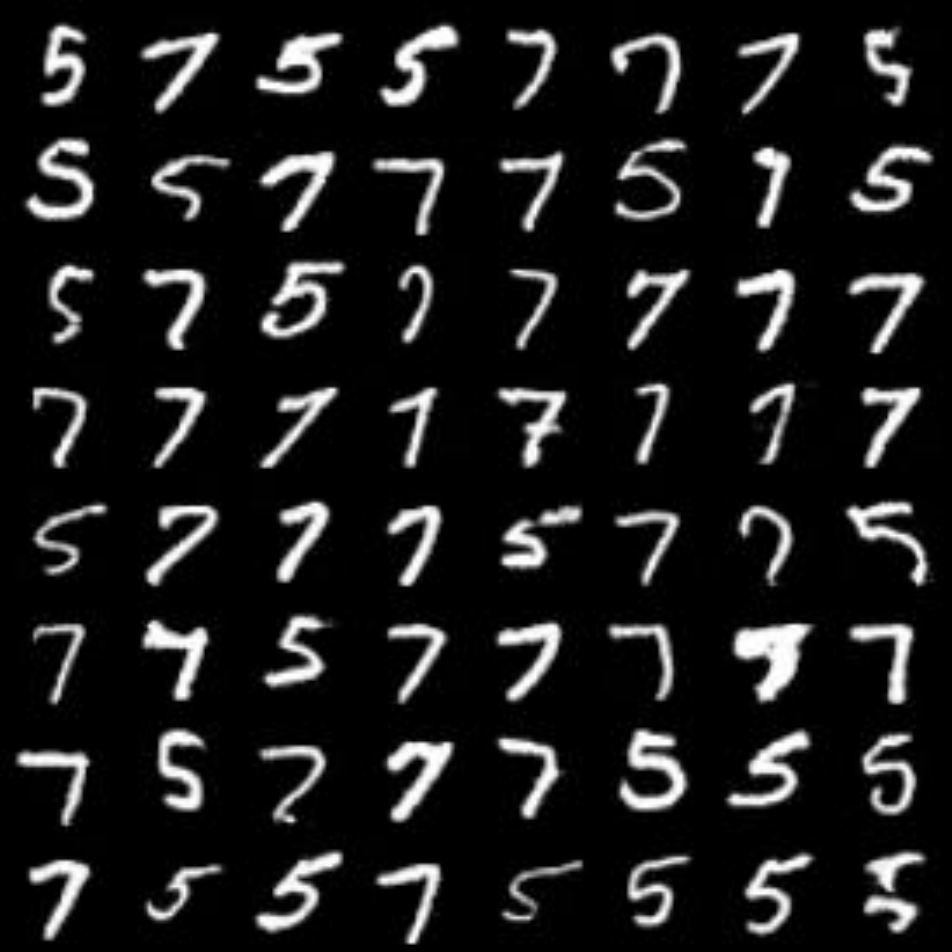} & \includegraphics[width=2.3cm,  height = 2.3cm]{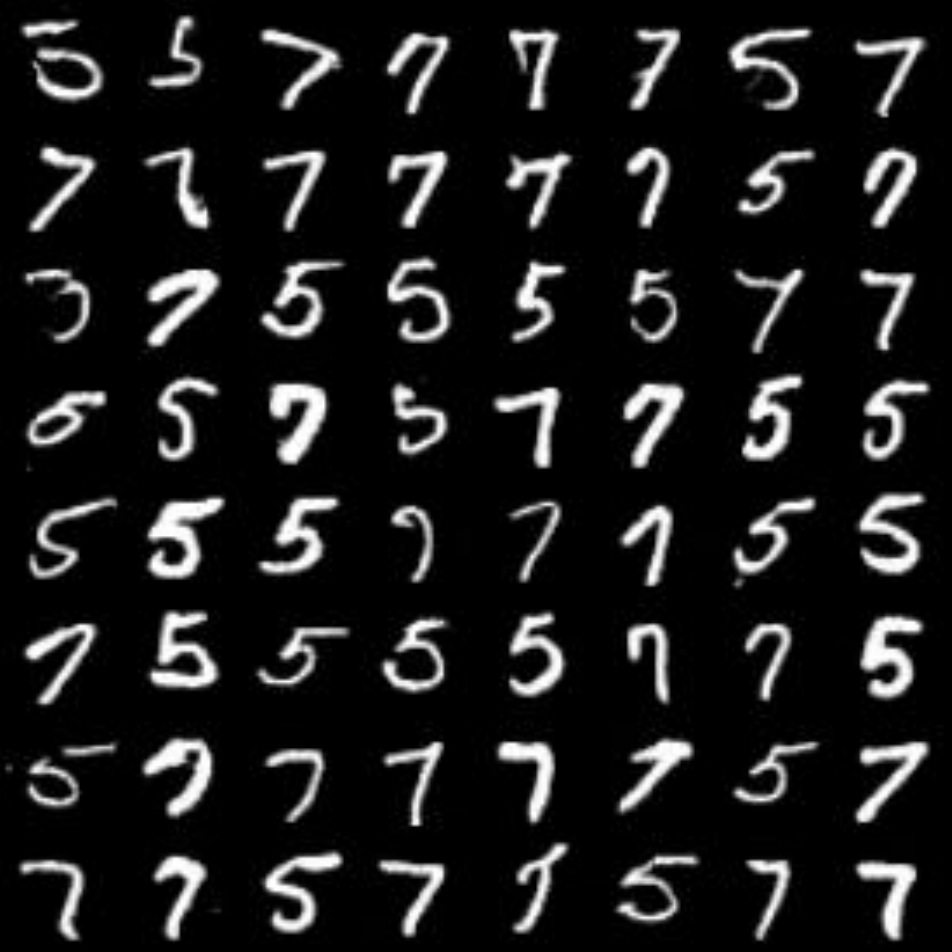} &
 \includegraphics[width=2.3cm, height = 2.3cm]{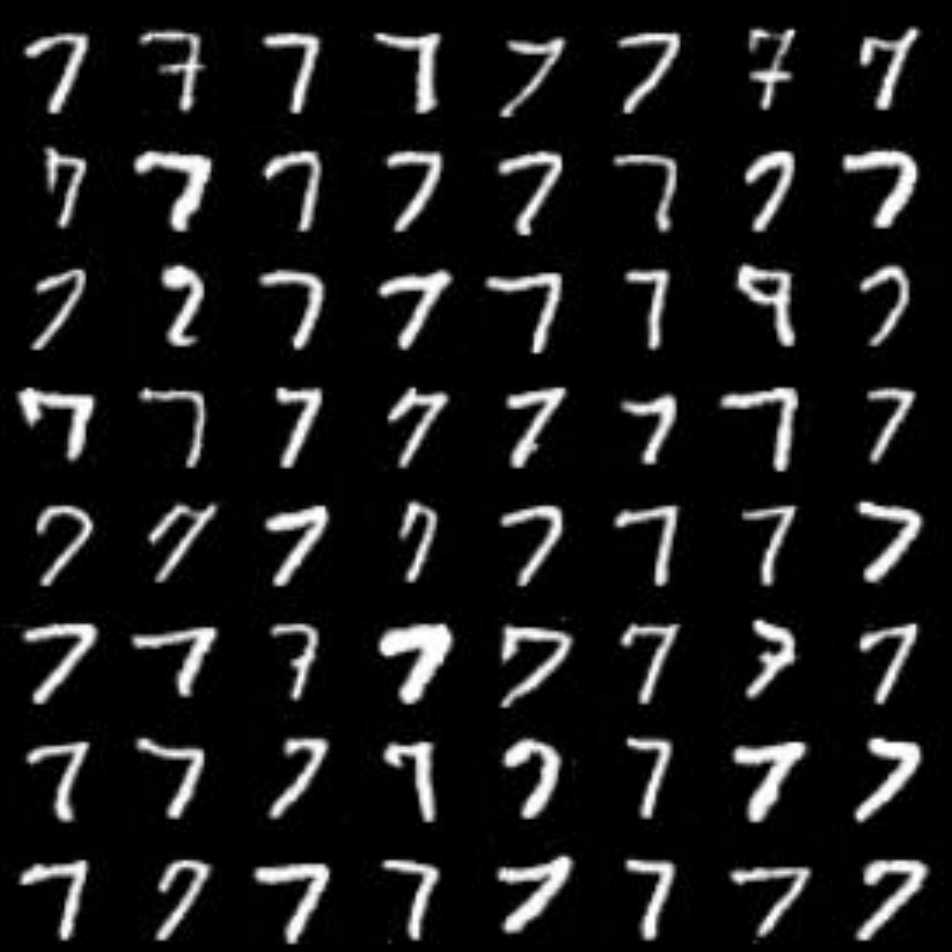} & \includegraphics[width=2.3cm, height = 2.3cm]{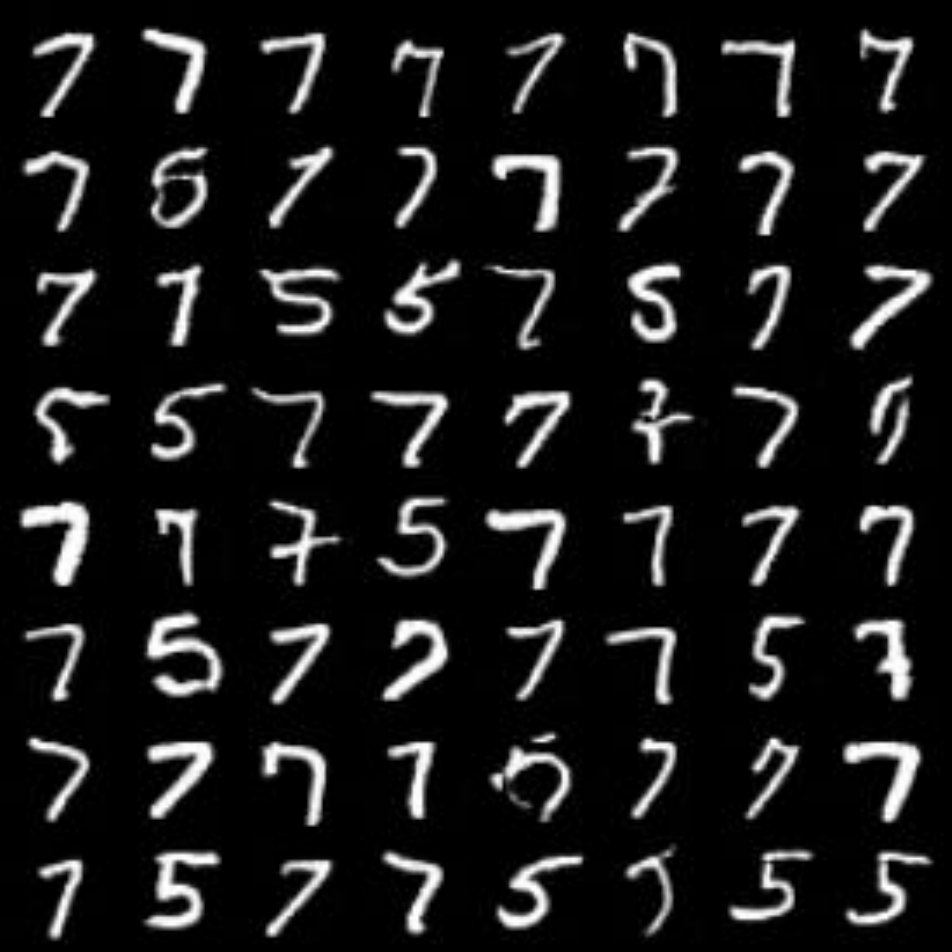} \\
 {\scriptsize  (a) balanced MNIST: JS-GAN } & 
 {\scriptsize  (b) balanced MNIST: RS-GAN } &     
 {\scriptsize  (c) imbalanced MNIST: JS-GAN  } 
 & {\scriptsize   (d) imbalanced MNIST: RS-GAN }
    \end{tabular}
    \vspace{-0.1cm}
    \caption{Balanced and Imbalanced MNIST setting: Comparison of JS-GAN and RS-GAN.
    }
    \label{mnist imbalance}
\vspace{-0.3cm}
\end{figure}

\section{Experiments of Bad Initialization }\label{app: bad initial experiment}

\iffalse 
In this part, we will perform ``bad initialization experiment'' to further 
support the improved landscape of RpGAN over SepGAN. 
 For simplicity, we only compare RS-GAN and JS-GAN. 
 \fi 

A bad optimization landscape does not mean the algorithm always converges to bad local minima\footnote{Technically since we are not dealing with a pure minimization problem,  we should say ``the algorithm converges to a bad attractor''. But for simplicity of illustration, we still call it ``local minimum.''}.
A `bad' landscape means is that there exists a ``bad'' initial point  (the blue point in Fig.~\ref{fig:landscapeillustration}(a))  that it will lead to a  `bad' final solution  upon training. In contrast, a good landscape is more robust to the initial point: starting from any initial point (e.g., two points shown in Fig.~\ref{fig:landscapeillustration}(b)), the algorithm can still find a good solution.
Therefore, bad optimization landscape of  JS-GAN does not mean the performance of JS-GAN is bad for \textit{any} initial point, but it should imply that JS-GAN is bad for \textit{certain} initial points. 

\begin{figure}[!b]
\vspace{-0.5cm}
\centering
\begin{tabular}{cc}
 \includegraphics[width=50mm,height=1.5cm]{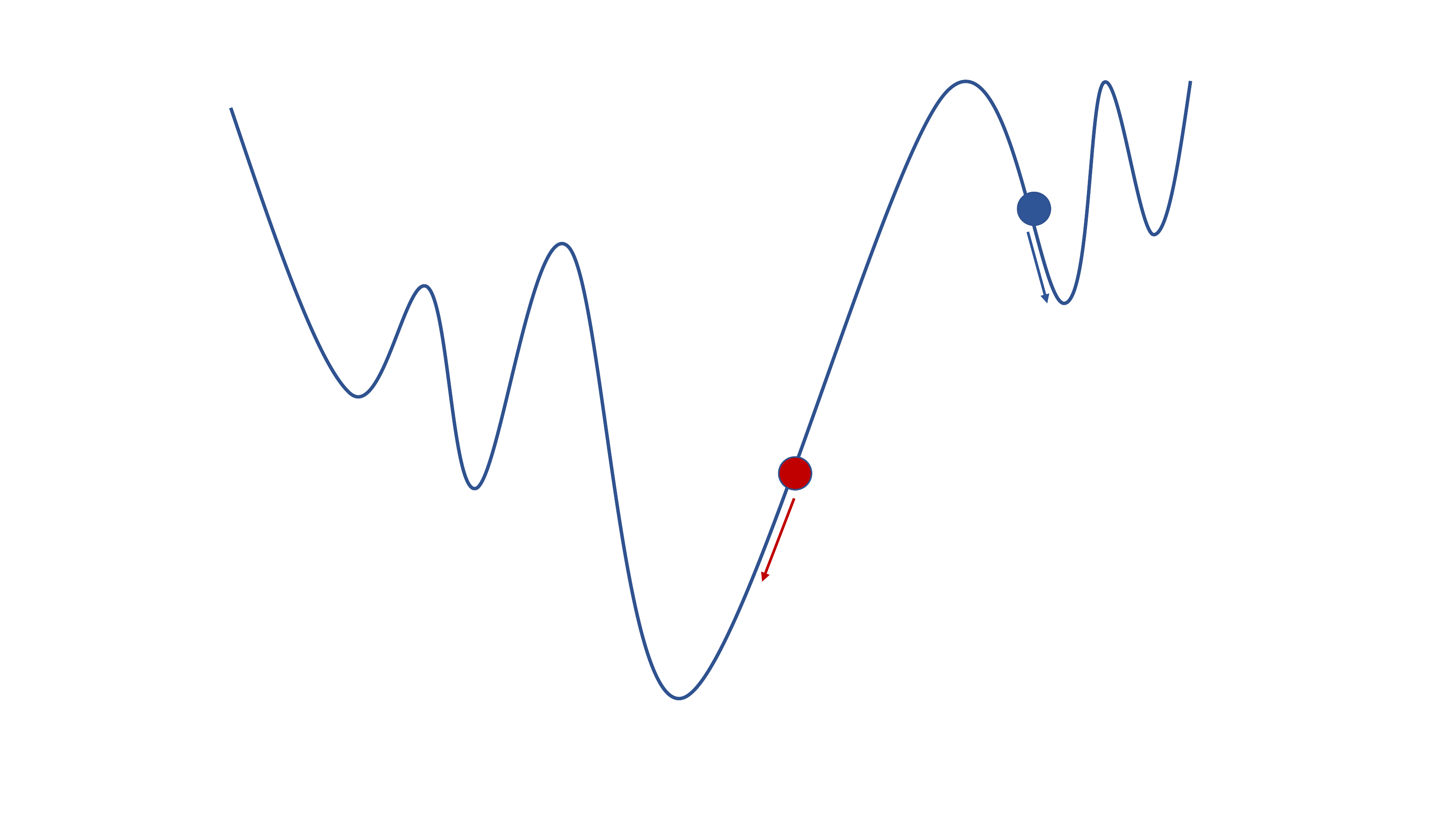}
 &\includegraphics[width=50mm, height = 1.5cm]{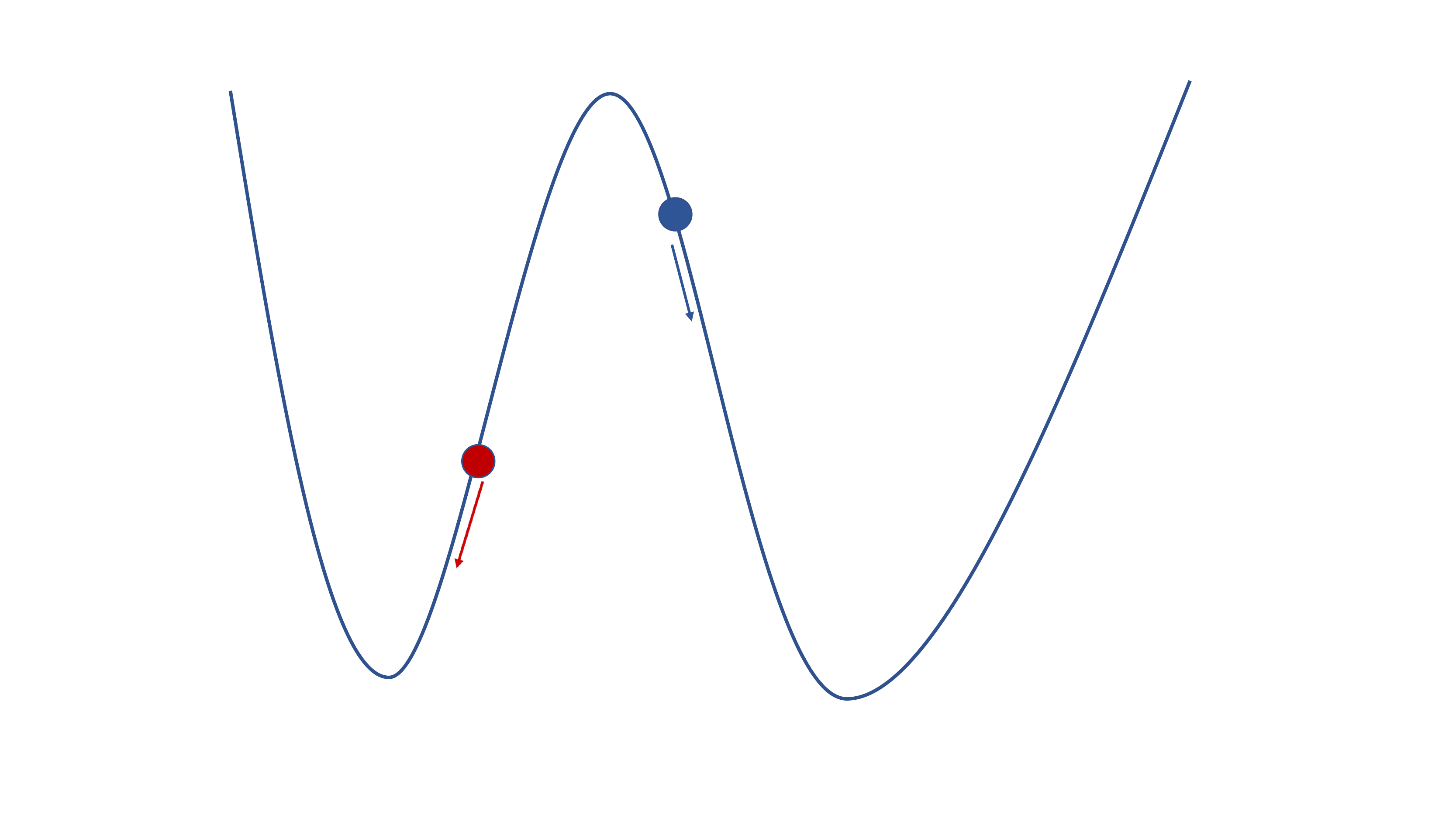} \\
 {\scriptsize  (a) bad landscape with bad local minima } & 
 {\scriptsize  (b) good landscape with multiple global minima }
\end{tabular}
\vspace{-0.1cm}
\caption{Left: for a bad landscape, a good initial point (red) leads to convergence to  a global optima while a bad one (blue) does not.
Right: for a good landscape, two initial points both converge to global minima. 
}
\label{fig:landscapeillustration}
\end{figure}

Next, we will show experiments that support this prediction.

\iffalse 
\vspace{-0.2cm}
\subsection{5-Gaussian Experiments}
\label{sec:badinittoy}
\vspace{-0.2cm}
\fi

\iffalse 
, i.e., we no longer optimize directly over $Y$.  
When using a single cluster as the initial point for the 5-Gaussian experiment we found that JS-GAN can split them (unlike in the previous experiment).  Likely this is due to the fact that it is hard to find a non-trivial generator neural-net which  yields generated points that are close.  Therefore, we design another method to reveal the difference between JS-GAN and RSGAN. 
\fi 

\iffalse 
According to our theory, the bad basin of JS-GAN occurs near a certain discriminator (the optimal discriminator for given fake data). In practice, since we are using one step of gradient descent for the discriminator, we do not necessarily get close to the optimal discriminator and thus the procedure may not get stuck at a bad basin. An implication is that a bad basin does not necessarily exist near a mode collapse configuration, but may require a special discriminator parameter. Starting from a random discriminator in this experiment makes it not easy to distinguish JS-GAN and RS-GAN. 
\fi 

\textbf{5-Gaussian Experiments}.
We  consider a 2-dimensional 5-Gaussian distribution as illustrated in Fig.~\ref{fig:5gaussianillustration}(a). 
We design a procedure to find an  initial discriminator and generator. 
For JS-GAN or RS-GAN, in some runs we obtain mode collapse and in some runs
we obtain perfect recovery.
Firstly, for the runs achieving perfect recovery (Fig.~\ref{fig:5gaussianillustration}(b)) in JS-GAN and RS-GAN respectively,
we pick the generators at the converged solution, which we denote as $ G_\text{JS0} $ and $G_\text{RS0}$ respectively.
Secondly,  for the runs attaining mode collapse (Fig.~\ref{fig:5gaussianillustration}(c)) in JS-GAN and RS-GAN respectively,
we pick the discriminators at the converged solution, referred to as  $D_\text{JS0}$ and $D_\text{RS0}$,  Then we re-train both JS-GAN
and RS-GAN from  $ ( D_\text{JS0} , G_\text{JS0} )$ 
and  $ ( D_\text{RS0} , G_\text{RS0} )$ 
respectively.

\begin{figure}[t]
\centering
\scalebox{0.95}{
 \begin{tabular}{cccc}
 \includegraphics[width=30mm]{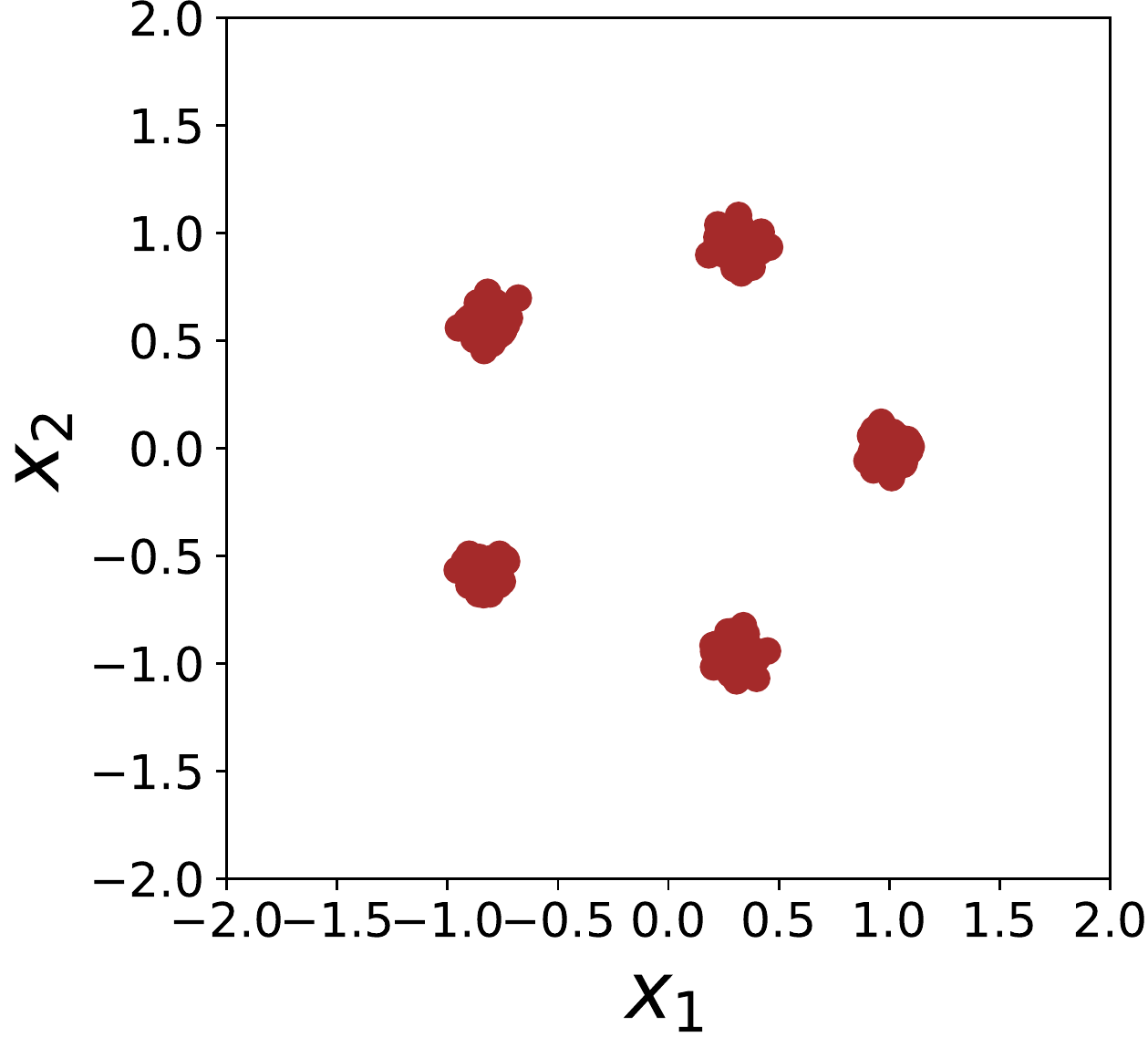}&\includegraphics[width=30mm]{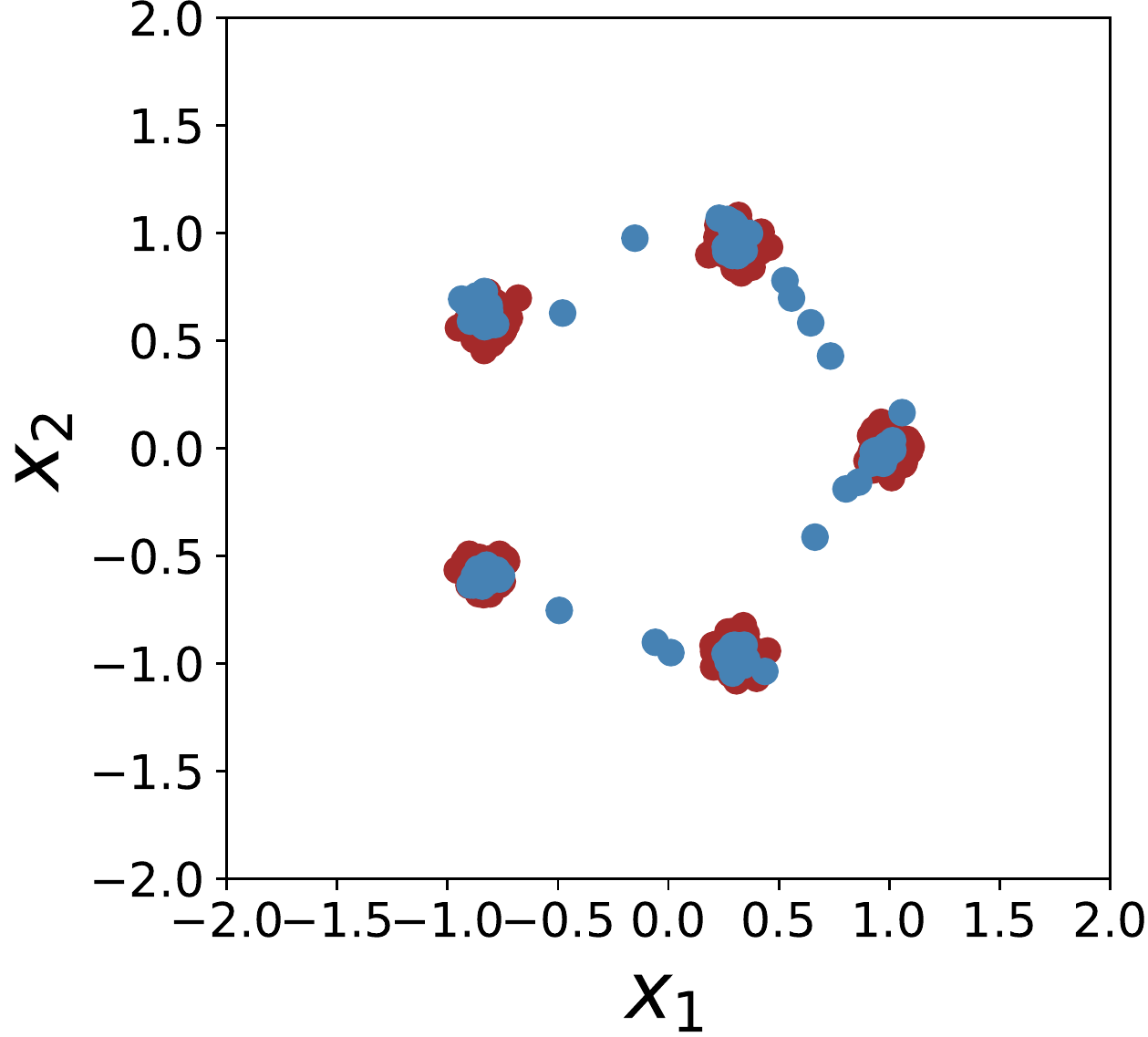}&\includegraphics[width=30mm]{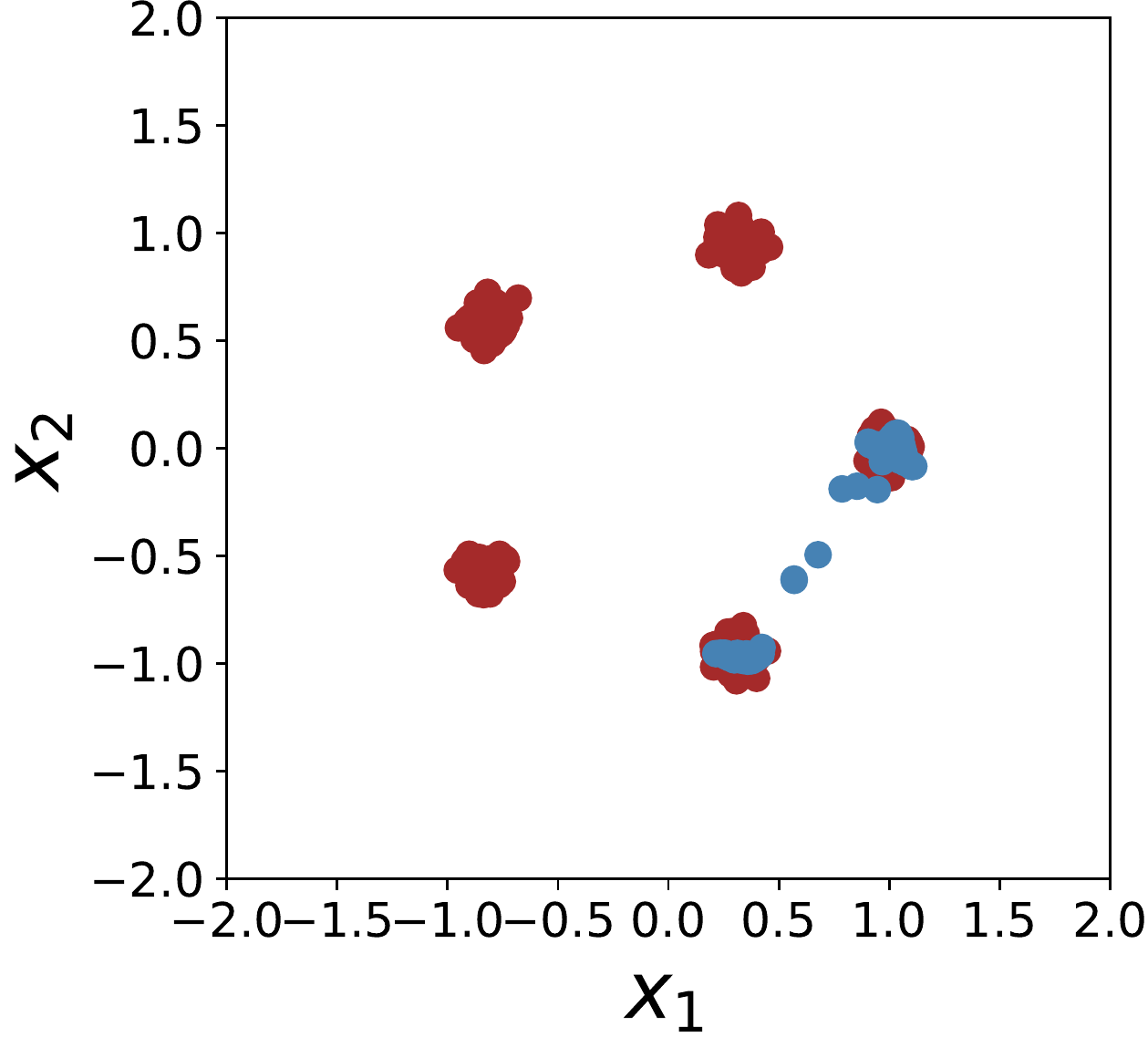}
& \includegraphics[width= 35mm,
 height=3cm]{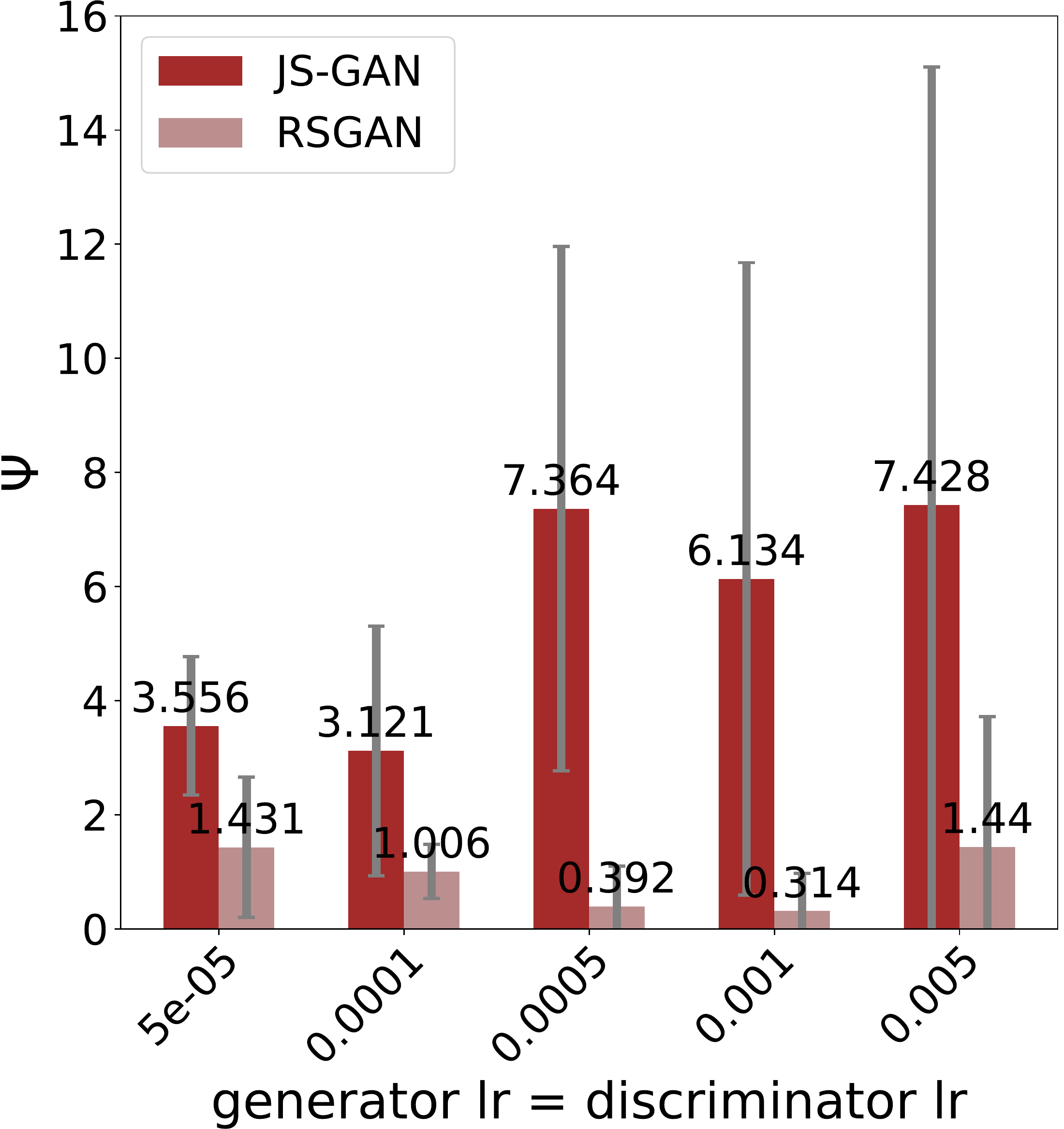}
 \\
(a) & (b) & (c) & (d)
\end{tabular}
}
\vspace{-0.3cm}
\caption{ Five Gaussian experiment. (a): ground truth. (b): generated data covers all five clusters. (c): mode collapse happens and only two clusters get covered.
(d) JS-GAN and RSGAN's loss $\Psi$ under different lr (generator lr = discriminator lr). }
\label{fig:5gaussianillustration}
\vspace{-0.2cm}
\end{figure}

\iffalse 
\begin{figure}[t]
\centering
    \begin{tabular}{cc}
 \raisebox{-0.5\height}{ \includegraphics[width=0.3\linewidth,
 height=3cm]{figure/5Guassian_goodG_crapyD_11layerD.pdf} } & \raisebox{-0.5\height}{\includegraphics[width=0.45\linewidth, height = 3cm]{figure/MNSIT_FID_goodGcrapyD.pdf} }\\
 (a) & (b) 
 \end{tabular}
 \vspace{-0.3cm}
    \caption{ (a) Bad initial point experiment on 5 Gaussian data:  (b) Bad initial point experiment on MNIST data: JS-GAN and RSGAN's FID after the 30,000-th iteration under different generator learning rate = discriminator learning rate settings. 
    }
    \label{fig:losscomp}
\vspace{-0.3cm}
\end{figure}
\fi

\iffalse 
\vspace{-0.2cm} 
\subsection{MNIST Experiments}
\label{sec:badinitmnist}
\vspace{-0.2cm}
\fi

\begin{wrapfigure}{r}{0.28\textwidth}
\vspace{-0.3cm}
 \includegraphics[width=1\linewidth, height=1.6cm]{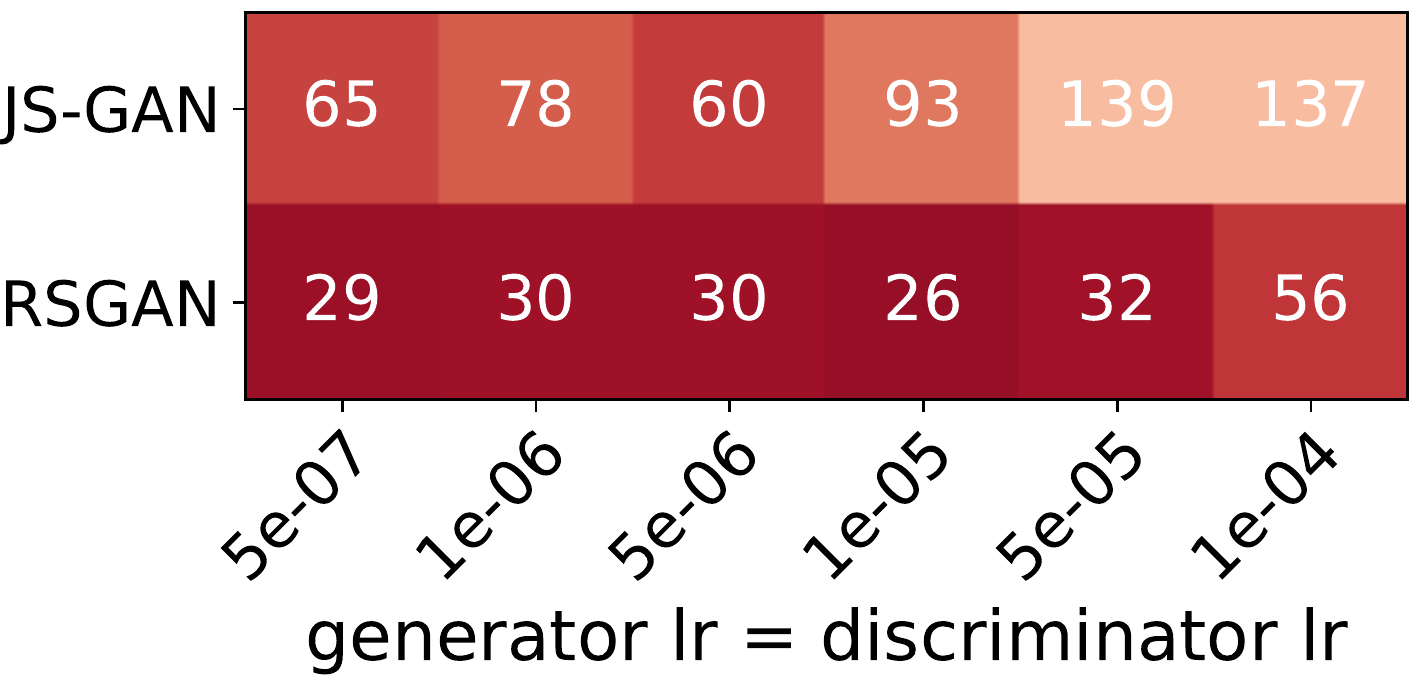}
 \captionsetup{font={scriptsize}} 
      \caption{MNIST experiment}
 \label{fig: MNIST bad ini}
 \vspace{-0.5cm}
\end{wrapfigure}
We define an evaluation metric $\Psi =  \sum_{k=1}^K \min_{1\leq i \leq 10^4}(\alpha \|x_i-C_k\|)$,
where $C_k$'s are the cluster centers, $\alpha$ is a scalar
and $x_i$'s are $10^4$ true data samples. 
We repeat the experiment $S=50$ times and compute the average $\Psi$. The larger the
metric, the worse the generated points. 
As shown in Fig.~\ref{fig:5gaussianillustration}(a),
 the metric $\Phi$ is much higher for JS-GAN than for RS-GAN, for various learning rates lr. 
 
\textbf{MNIST Experiments}.
We use a similar strategy to find initial parameters for MNIST data.
Fig. \ref{fig: MNIST bad ini} (also in 
Sec. \ref{sec:experiments}) shows that RS-GAN generates much lower FID scores (30+ gap) than JS-GAN.

\iflonger 
See Fig.~\ref{fig:mnistillustration} for the generated digits.  
\fi

\iflonger 
\begin{figure}[t]
\centering
 \begin{tabular} {ccccc}
 \rotatebox{90}{\hspace{0.8cm}{\small JS-GAN}} &  \includegraphics[width=20mm,height=2cm]{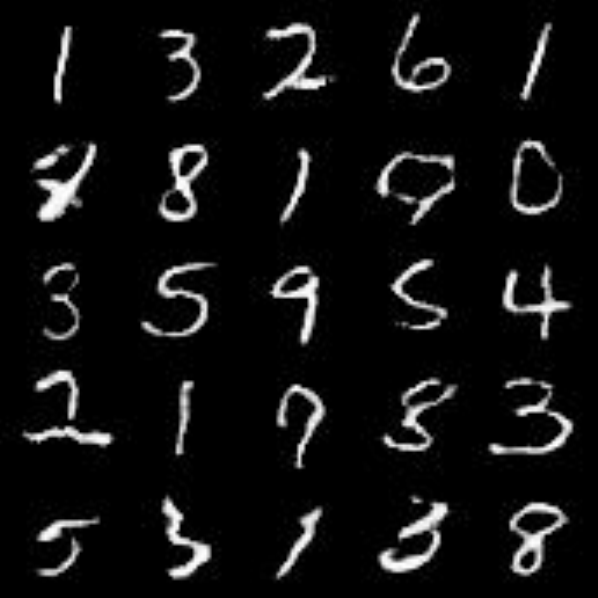}
 &\includegraphics[width=20mm,height=2cm]{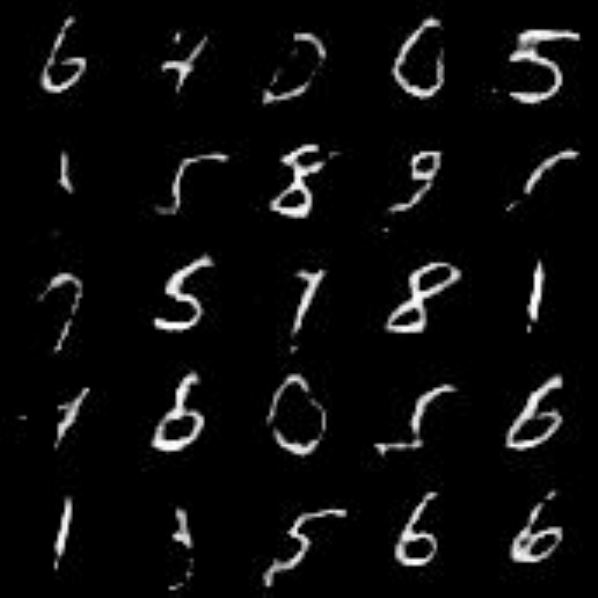}
 &\includegraphics[width=20mm,height=2cm]{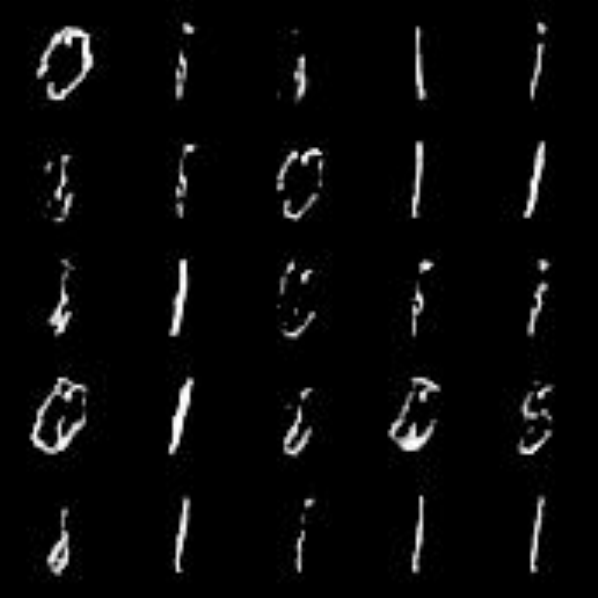}
 &\includegraphics[width=20mm,height=2cm]{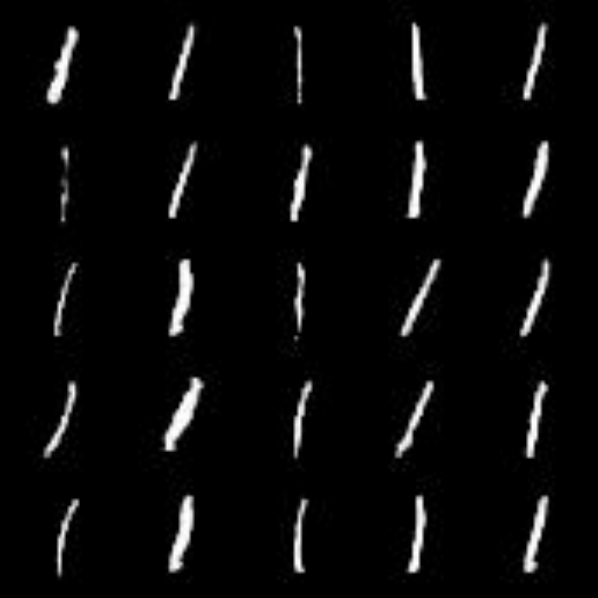}\\
  \rotatebox{90}{\hspace{0.8cm}{\small RSGAN}} &
  \includegraphics[width=20mm]{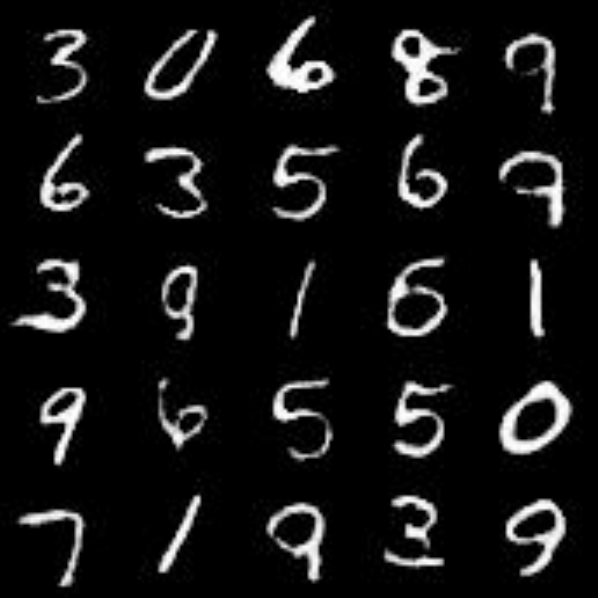}&\includegraphics[width=20mm]{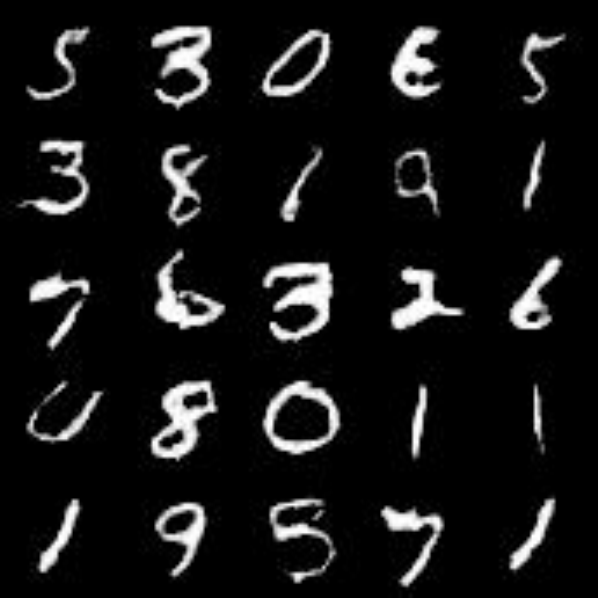}&\includegraphics[width=20mm]{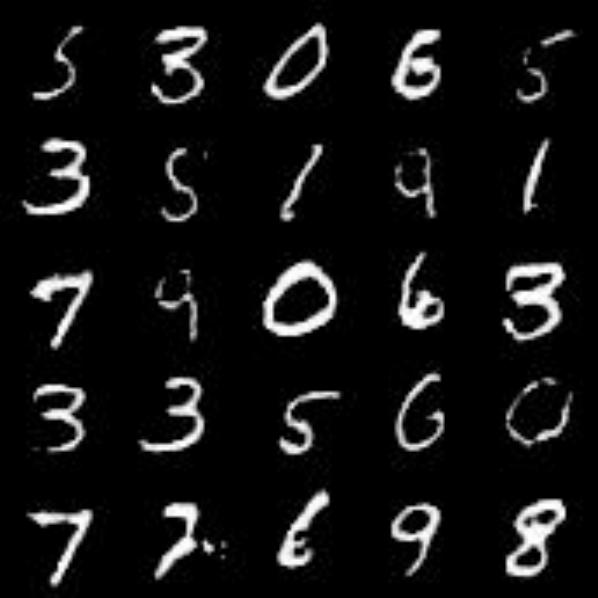}&\includegraphics[width=20mm]{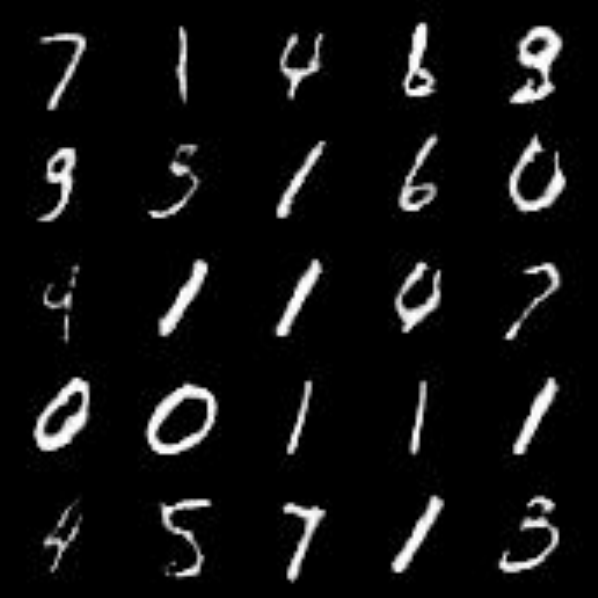}\\
  & $G_0(z)$ &lr=5e-07 & lr=1e-05 & lr=1e-04
\end{tabular}
\caption{Bad initial point experiment: generated MNIST data under different learning rate settings. See Fig.~\ref{fig:losscomp}(b) for corresponding FID scores. 
}
\label{fig:mnistillustration}
\vspace{-0.3cm}
\end{figure}
\fi

\iffalse 
In order to illustrate the effectiveness of RS-GAN, we conduct a set of experiments on both Mixure of Gaussians and real datasets. This session is structured as follows. First, we will compare the simulation performance of JS-GAN, WGAN-GP and RSGAN on the two commonly used datasets CIFAR-10 and STL-10. Second, 
\fi

\iffalse 
In Section~\ref{sec:experiments} (a subsection on the experiments) in the main paper, we have discussed the advantage of RS-GAN over JS-GAN when starting from a certain initial point, by running the methods on 
the MNIST dataset. 
We use $10k$ images to extract embedding features through the Inception net for both true data and generated samples. And we use the DCGAN structure (details in Table~\ref{table: mnist structure}).
\fi 

The two experiments verify our prediction that RS-GAN is more robust to initialization, which supports our theory that RS-GAN enjoys a better landscape than JS-GAN.

\iflonger 
\begin{table}[t]
\tablesize 
\begin{minipage}{.5\textwidth}
\centering
\begin{tabular}{c}
\toprule
\textbf{(a) Generator} \\
\midrule
$z \in \mathbb{R}^{128} \sim {\mathcal N}(0, I)$ \\\midrule
128 $\rightarrow $ 512, dense, linear \\\midrule
$4 \times 4$, stride 2 deconv, 256, BN, ReLU  \\\midrule
$4 \times 4$, stride 2 deconv, 128, BN, ReLU  \\\midrule
$4 \times 4$, stride 2 deconv, 64, BN, ReLU  \\\midrule
$4 \times 4$, stride 2 deconv, 1, Tanh  \\
\bottomrule
\end{tabular}
\end{minipage}%
\begin{minipage}[c]{.5\textwidth}
\centering
\begin{tabular}{c}
\toprule
\textbf{(b) Discriminator} \\
\midrule
RGB image $x \in [-1, 1]^{H \times W \times 1}$ \\ \midrule
$4 \times 4$, stride 2 conv, 64, LReLU 0.1  \\\midrule
$4 \times 4$, stride 2 conv, 128, LReLU 0.1 \\\midrule
$4 \times 4$, stride 2 conv, 256, LReLU 0.1 \\\midrule
$256\times4\times4\rightarrow s$, dense, linear \\
\bottomrule
\end{tabular}
\end{minipage}
\caption{CNN models for MNIST used in our experiments on image Generation. }
\label{table: mnist structure}
\end{table}
\fi

\section{Experiments of Regular Training: More Details and More Results}
\label{appen: experiments}

In this section, we present details of the regular experiments
in Sec.~\ref{sec:experiments} and a few more experiments.

\subsection{Experiment Details and More Experiments with Logistic Loss}
\label{expsetting}

\textbf{Non-saturating version.} Following the standard practice \cite{goodfellow2014generative},
if $ \lim_{t \rightarrow \infty } h( t ) = 0 $, we use the non-saturating version 
of RpGAN in practical training: 
{\equationsizeReg 
\begin{equation}\label{RpGAN non saturating version}
\begin{split}
 \min_{\theta}  L_{D}(\theta; w ) & \triangleq 
 \frac{1}{ n } \sum_{i}  h (  f_{\theta}(x_i)) -  f_{\theta}(G_w(z_i)) ),   \\
\min_{w}  L_{G}( w ; \theta  ) &  \triangleq \frac{1}{ n }
\sum_{i} h (  f_{\theta}(G_w(z_i)) - f_{\theta}(x_i))  ).
\end{split}
\end{equation}
}
For logistic and hinge loss, we use Eq.~\eqref{RpGAN non saturating version}. 
For least-square loss, we use the original min-max version (check Appendix~\ref{sec: lsreal} for more). We use alternating stochastic GDA to solve this problem. 

\textbf{Neural-net structures:} We conduct experiments on two datasets: CIFAR-10 ($32\times 32$ size) and STL-10 ($48\times 48$ size) on both standard CNN and ResNet. As mentioned in Sec.~\ref{sec:experiments}, we also conduct experiments on the narrower nets: we reduce the number of channels for all convolutional layers in the generator and discriminator to (1) half, (2) quarter and (3) bottleneck (for ResNet structure),
The architectures are shown in Tab.~\ref{table: CNN structure} (CNN), Tab.~\ref{table: cifar_regualar_resnet } (ResNet for CIFAR) and Tab.~\ref{table: stl_regular_resnet } (ResNet for STL) and Tab.~\ref{table: cifar_bottleneck_resnet } (Bottleneck for CIFAR) and Tab.~\ref{table: stl_bottleneck_resnet } (Bottleneck for STL).
\iflonger 
We also illustrate the difference between the regular ResBlock and BottleNeck ResBlock in Fig.~\ref{fig:resblockstructure}.
\fi 

\textbf{Hyper-parameters:} We use a batchsize of 64. For CIFAR-10 on ResNet we set $\beta_1=0$ and $\beta_2=0.9$ in Adam. For others, $\beta_1=0.5$ and $\beta_2=0.999$. We use $\text{GIter} =1$ 
for both CNN and ResNet. We also use $\text{DIter} =1$ for CNN and $\text{DIter} =5$ for ResNet.
We fix the learning rate for the discriminator (dlr) to be 2e-4.
For RpGANs, we find that the
learning rate for the generator (glr) needs to be larger than dlr to keep the training balanced. Thus we tune glr using parameters in the set {2e-4, 5e-4, 1e-3, 1.5e-3}. For SepGAN, we set glr = 0.0002 for SepGANs (JS-GAN,hinge-GAN) as suggested by \cite{miyato2018spectral, radford2015unsupervised} \footnote{We tuned glr in the set {2e-4, 5e-4, 1e-3, 1.5e-3} and find that glr = 2e-4 performs the best in most cases
for SepGAN, so we follow the suggestion of \cite{miyato2018spectral, radford2015unsupervised}.}.
See Tab.~\ref{Tab:lrsetting} for the learning rate of RS-GAN and hyper-parameters of WGAN-GP.

\iffalse 
We set the learning rate for the generator ($\text{glr}$) = 0.0002 for 
SepGANs (JS-GAN, hinge-GAN) as suggested by \citet{miyato2018spectral, radford2015unsupervised}. For RpGANs, we find that $\text{glr}$ needs to be larger than $\text{dlr}$ to keep the training balanced. Thus we tune $\text{glr}$ using parameters in the set \{2e-4, 5e-4, 1e-3, 1.5e-3, 2e-3, 2.5e-3\}. 
\fi

\begin{table}[t]
\tablesize 
\centering
\setlength{\tabcolsep}{2pt}
\begin{tabular}{l|cc|cc|cc}
\toprule
 & \multicolumn{2}{c}{ \bf{CIFAR-10}} & \multicolumn{2}{c}{ \bf{CIFAR-10+EMA}}& \multicolumn{2}{c}{ \bf{STL-10+EMA}}  \\ \midrule
& IS $\uparrow$    &  FID $\downarrow$    &    IS $\uparrow$       & FID $\downarrow$& IS $\uparrow$    &  FID $\downarrow$      \\ \midrule
\multicolumn{5}{l}{\bf{ResNet}}    \\
JS-GAN+SN  & 8.03$\pm$0.10  & 20.06$\pm$0.18   & 8.41$\pm$0.09  & 17.79$\pm$0.43  & 9.14$\pm$0.12  & 33.06   \\
RS-GAN+SN  &  7.94$\pm$0.09 & 19.79$\pm$0.57  & 8.37$\pm$0.10  &  17.75$\pm$0.56 &   9.23$\pm$0.08  &  31.87  \\
JS-GAN+SN+GD channel/2  & 7.77$\pm$0.08  & 23.36$\pm$0.46  & 8.24$\pm$0.08  & 20.55$\pm$0.59  & 8.69$\pm$0.08  & 42.05     \\
RS-GAN+SN+GD channel/2  & 7.76$\pm$0.07  & 21.63$\pm$0.51  & 8.21$\pm$0.09  & 18.91$\pm$0.45  & 8.77$\pm$0.13  & 39.31   \\
JS-GAN+SN+GD channel/4  & 6.75$\pm$0.06  & 44.39$\pm$4.38  & 7.18$\pm$0.06  & 38.75$\pm$6.28  & 8.42$\pm$0.06  &  52.38   \\
RS-GAN+SN+GD feature/4  & 7.20$\pm$0.07  & 31.40$\pm$0.78  & 7.60$\pm$0.06  & 26.85$\pm$0.56  & 8.43$\pm$0.10  & 48.92  \\ 
JS-GAN+SN+BottleNeck   & 7.51$\pm$0.07  & 27.33$\pm$1.05  & 7.99$\pm$0.10  &  23.71$\pm$0.86 & 8.37$\pm$0.08  &   47.97    \\
RS-GAN+SN+BottleNeck  & 7.52$\pm$0.10  & 25.05$\pm$0.35  & 8.06$\pm$0.11  & 21.29$\pm$0.22  & 8.48$\pm$0.06  & 44.60 \\\bottomrule
\end{tabular}
\caption{Repeat the experiments (logistic loss) in Tab.~\ref{Tab:metric} with at least three seeds. 
}
\label{Tab:metric+moreseeds}
\vspace{-0.5cm}
\end{table}

\textbf{More details of EMA:}
In Sec.~\ref{sec:experiments}, we conjectured that the effect of EMA (exponential moving average)~\cite{yazici2018unusual} and RpGAN are additive.
Suppose $w^{(t)}$ is the generator parameter
in $t$-th iteration of one run,
the EMA generator at the $t^\text{th}$ iteration is computed as follows 
$    w^{(t)}_\text{EMA} = \beta w^{(t-1)}_\text{EMA} + (1-\beta)w^{(t)},$
where $w^{(0)}_\text{EMA} = w^{(0)}$.
Note that EMA is a post-hoc processing step, and does not affect the training process. Intuitively, the EMA generator is closer to the bottom of a basin while the real training is circling around a basin due to the minmax structure.
We set $\beta=0.9999$. 
As Tab.~\ref{Tab:metric+moreseeds} shows, while EMA improves both JS-GAN and RS-GAN, 
RS-GAN is still better than JS-GAN.
\iffalse 
Remarkably, RS-GAN with EMA can achieve a similar result 
to the baseline (FID 20.13) using 16.8\% parameters (Resnet with bottleneck). 
\fi 

\textbf{Results on Logistic Loss with More Seeds:} 
Besides the result in Tab.~\ref{Tab:metric}, we run at least  3 extra seeds for all experiments with ResNet structure on CIFAR-10 to show that the results are consistent across different runs. We report the results in  Tab.~\ref{Tab:metric+moreseeds},
and find RS-GAN is still better than JS-GAN and the gap 
increases as the networks become narrower. 

\textbf{Samples of image generation:} 
Generated samples obtained upon training on CIFAR-10  are given in Fig.~\ref{fig:cifarcnnsample} for CNN, Fig.~\ref{fig:cifarresnetsample} for ResNet. Generated samples obtained upon training on STL-10 dataset are given in Fig.~\ref{fig:stlcnnsample} for CNN, Fig.~\ref{fig:stlesnetsample} for ResNet. Instead of cherry-picking, all sample images are generated from random sampled Gaussian noise.

\subsection{Experiments with Hinge Loss}\label{sec: hingereal}
Hinge loss has become popular in GANs~\cite{Tran2017DeepAH, miyato2018spectral,brock2018large}. 
The empirical loss of hinge-GAN is %
{\equationsize 
\begin{align*}
 \min_{\theta} L^\text{Hinge}_{D}(\theta; w ) &  \triangleq   \frac{1}{2 n } \left[ \sum_{i} \max(0, 1-D_{\theta}(x_i)) + \sum_{i} \max(0, 1+D_{\theta}(G_w (z_i))\right],  \\
\min_{w}  L^\text{Hinge}_{G}( w; \theta )  &
 \triangleq   -\frac{1}{ n } \sum_{i} D_{\theta}(G_w(z_i)). 
\end{align*}
}
Note that Hinge-GAN applies the hinge loss for the discriminator, and linear loss
for the generator. This is a variant of SepGAN with $h_1(t) = h_2(t) = -\max (0, 1 - t)  $.

\iffalse 
We consider two versions of Rp-hinge-GAN.  The first one only uses the hinge loss
in the D objective, thus is also a mixture of hinge loss and logistic loss (similar to hinge-GAN). The second one uses hinge loss in both D objective and G objective. 
Specifically, version 1 of Rp-hinge-GAN is:
{\scriptsize 
\begin{align*}
 \min_{\theta}  L^{\text{R-Hinge}_\text{v1}}_{D}(\theta; w ) & \triangleq  \frac{1}{ n }  \sum_{i} \max(0, 1+( f_{\theta}(G_w(z_i)) - f_{\theta}(x_i))),  \\
\min_{w}  L^{\text{R-Hinge}_\text{v1}}_{G}( w; \theta ) & \triangleq  \frac{1}{n } \sum_{i} \log(1+\exp( f_{\theta}(x_i) - f_{\theta}(G_w(z_i)))).
\end{align*}
}
In version 2, the D loss is $L^{\text{R-Hinge}_\text{v2}}_{D}(\theta; w)
 = L^{\text{R-Hinge}_\text{v1}}_{D}(\theta; w ) $, and the generator solves
 \fi 
 The Rp-hinge-GAN is RpGAN given in Eq.~\eqref{RpGAN non saturating version} 
  with $ h(t) = -\max (0, 1 - t)  $:
{\scriptsize 
\begin{align*}
 \min_{\theta}  L^{\text{R-Hinge}}_{D}(\theta; w ) & \triangleq  \frac{1}{ n }  \sum_{i} \max(0, 1+( f_{\theta}(G_w(z_i)) - f_{\theta}(x_i))),  \\
\min_{w}  L^{\text{R-Hinge}}_{G}( w ; \theta  ) &  \triangleq \frac{1}{ n }
\sum_{i} \max(0, 1 + ( f_{\theta }(x_i) - f_{\theta}(G_w(z_i)))).
\end{align*}
}
\iffalse 
We refer to the above three models as Hinge-GAN,
Rp-Hinge-GAN-HL and Rp-Hinge-GAN-HH in  Tab.~\ref{Tab:hingemetric}.
\fi 
We compare them on ResNet with the hyper-parameter settings in Appendix~\ref{expsetting}. As  Tab.~\ref{Tab:hingemetric} shows, 
Rp-Hinge-GAN (both versions) performs better than Hinge-GAN.
For narrower networks, the gap is $ 4$ to $ 9$ FID scores, larger than the gap for the logistic loss.

\begin{table}[t]
\tablesize 
\centering
\setlength{\tabcolsep}{2pt}
\begin{tabular}{l|ccc|ccc}
\toprule
& \multicolumn{3}{c}{ \bf{CIFAR-10}} & \multicolumn{3}{c}{ \bf{CIFAR-10 + EMA}} \\ \midrule
         & IS $\uparrow$    &  FID $\downarrow$   &  FID Gap &  IS $\uparrow$      & FID $\downarrow$ &   FID Gap       \\ \midrule
\multicolumn{ 6 }{l}{\bf{ResNet + Hinge Loss}}  \\
Hinge-GAN  &  7.92$\pm$0.08   &  21.30  &  & 8.44$\pm$0.10 &  17.43  &      \\
Hinge-GAN +GD channel/2  &   7.63$\pm$0.05   &  27.21  & & 7.90$\pm$0.08   &  24.35  &  \\
Hinge-GAN +GD channel/4  &  6.79$\pm$0.09    &  37.51  & & 7.39$\pm$0.07   & 34.45    & \\
Hinge-GAN +BottleNeck   &  7.16$\pm$0.10    &   33.24 & & 7.91$\pm$0.09 & 26.56&     \\
\bottomrule
Rp-Hinge-GAN  &   7.84$\pm$0.09   &  19.10     & 2.20 & 8.21$\pm$0.09  & 17.19  &  0.24  \\
Rp-Hinge-GAN +GD channel/2  &  7.77$\pm$0.08  &  21.10 & 6.11 & 8.34$\pm$0.11  & 19.19  & 5.17 \\
Rp-Hinge-GAN +GD channel/4  & 7.21$\pm$0.11   & 29.41  &  8.10 & 7.77$\pm$0.08   &  25.57  &  8.88 \\ 
Rp-Hinge-GAN +BottleNeck   &  7.52$\pm$0.07   &   23.28 & 9.96 & 8.05$\pm$0.07  &  22.03 & 4.53  \\
\bottomrule
\end{tabular}
\caption{Comparison of Hinge-GAN and Rp-Hinge-GAN. We also show the FID gap between Rp-Hinge-GAN with Hinge-GAN (e.g. $ 2.20 = 21.30-19.10 $ and $ 9.96 = 33.24 - 23.28)$.}
\label{Tab:hingemetric}
\vspace{-0.5cm}
\end{table}

\subsection{Experiments with Least Square Loss}\label{sec: lsreal}
We consider the least square loss.
The LS-GAN \cite{lsgan2016mao}  is defined as follows: 
{\equationsize
\begin{align*}
  \min_{\theta}   L^\text{LS}_{D}(\theta;  w ) & \triangleq  \frac{1}{2 n } \left[ \sum_{i} ( f_{\theta}(x_{i}) - 1)^2   + \sum_{i}  f_{\theta}(G_w(z_i))^2\right],  \\
 \min_{w}  L^\text{LS}_{G}( w;  \theta )  &
\triangleq \frac{1}{ n } \sum_{i} ( f_{\theta}(G_w(z_i)) - 1)^2. 
\end{align*}
}
This is a non-zero-sum variant of SepGAN with $h_1(t) = -(1 - t)^2, h_2(t) = -t^2  $. 

 Rp-LS-GAN addresses the following objectives: 
{\equationsize
\begin{equation}\label{RpLSGAN}
\begin{split}
 \min_{\theta}  L^\text{Rp-LS}_{D}(\theta; w) & \triangleq   \frac{1}{n } \sum_{i} ( f_{\theta}(x_{i}) - f_{\theta}(G(z_i)) - 1)^2  ,  \\
  \min_{ w }  L^\text{Rp-LS}_{G}( w; \theta )  &
\triangleq - L^\text{Rp-LS}_{D}(\theta; w) = - \frac{1}{n } \sum_{i} ( f_{\theta}(x_{i}) - f_{\theta}(G_w(z_i)) - 1)^2 .
\end{split}
\end{equation}
}
\iflonger 
For Rp-LS-GAN, we are not using the non-saturating version, because the gradient vanishing
issue does not appear for quadratic loss (recall that the gradient vanishing
issue pointed out in \cite{goodfellow2014generative} is due to the fact
that the binary Logistic loss $h(t) = - \log (1 + \exp(-t))$ has almost flat curve as $t$ goes to infinity). 
\fi 
For least square loss $h(t) = -(t-1)^2$, the gradient vanishing issue due to $ h $
does not exist, thus we can use the min-max version given in Eq.~\eqref{RpLSGAN} in practice. 
Our version of Rp-LS-GAN is actually different from the version of Rp-LS-GAN in \cite{jolicoeur2018relativistic} which is similar to Eq.~\eqref{RpGAN non saturating version}
with least square $h$.

In Tab.~\ref{Tab:lsmetric} we compare LS-GAN and Rp-LS-GAN on CIFAR-10 with CNN architectures detailed in Tab.~\ref{table: CNN structure}. As Tab.~\ref{Tab:lsmetric} shows, Rp-LS-GAN is slightly worse than LS-GAN
in regular width, but is better than LS-GAN (with 5.7 FID gap) when using 1/4 width.

\begin{table}[t]
\centering
\setlength{\tabcolsep}{2pt}
{\tablesize 
\begin{tabular}{l|ccc|ccc|ccc}
\toprule
& \multicolumn{3}{c}{Regular width} & \multicolumn{3}{c}{channel/2 }& \multicolumn{3}{c}{channel/4} \\
\midrule
& IS & FID & FID Gap &  IS & FID & FID Gap & IS & FID & FID Gap  \\
\midrule 
LS-GAN & 6.91$\pm$0.10 & \textbf{32.93} &   & 6.63$\pm$0.08 & 37.83 &  & 5.69$\pm$0.10 
& 48.63 &   \\
\midrule
Rp-LS-GAN & \textbf{7.09}$\pm$0.07 & 34.78 & -1.85 & \textbf{6.94}$\pm$0.04 & \textbf{34.34}  & 3.49 & \textbf{6.22}$\pm$0.10& \textbf{42.86} & 5.77 \\
\bottomrule
\end{tabular}
\caption{Comparison of LS-GAN and Rp-LS-GAN on CIFAR-10 with the CNN structure. }
\label{Tab:lsmetric}
}
\vspace{-0.5cm}
\end{table}

\begin{figure}[t]
	\centering
	\scalebox{1}{
		\begin{tabularx}{\linewidth} {cc}
			\includegraphics[width=0.5\linewidth]{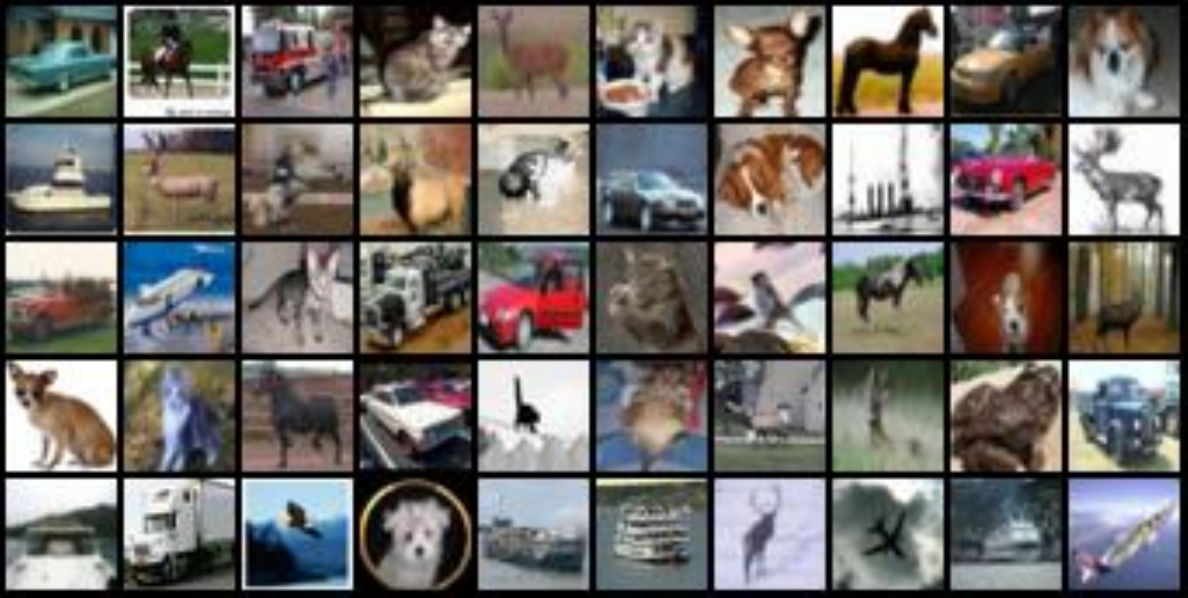} & \includegraphics[width=0.5\linewidth]{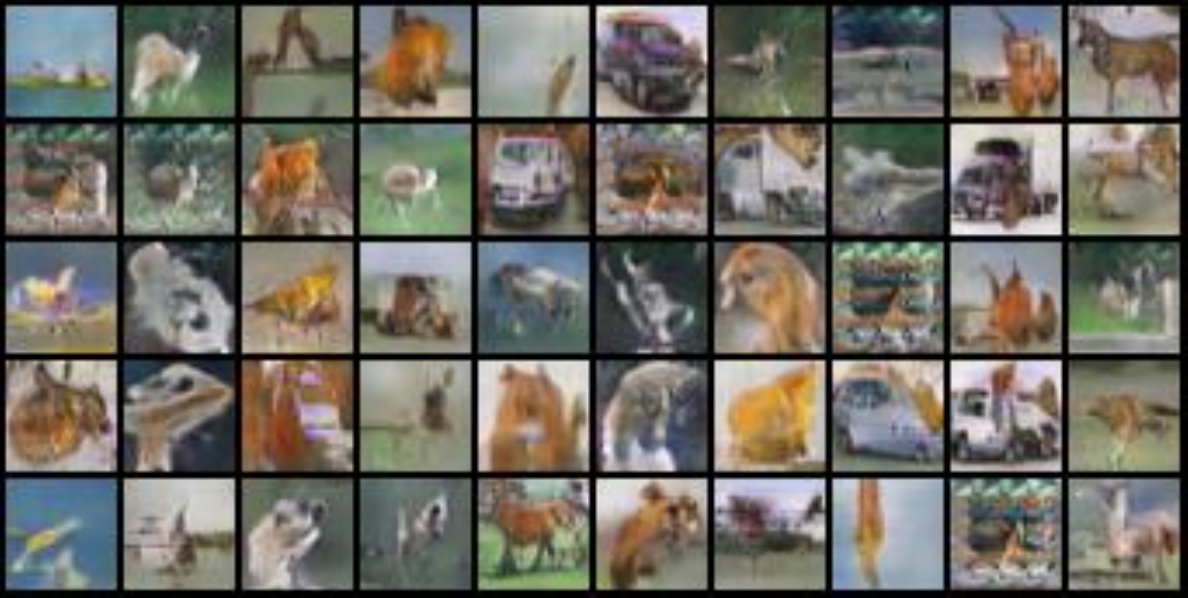}\\
			(a) real data & (b) JS-GAN + BatchNorm \\
			\includegraphics[width=0.5\linewidth]{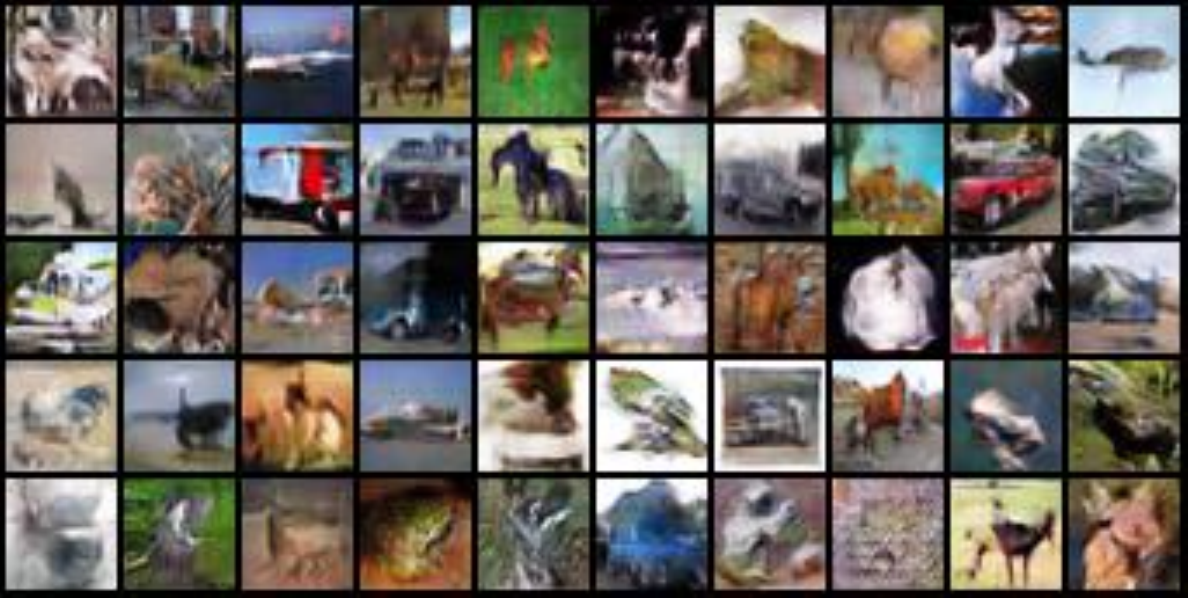} & \includegraphics[width=0.5\linewidth]{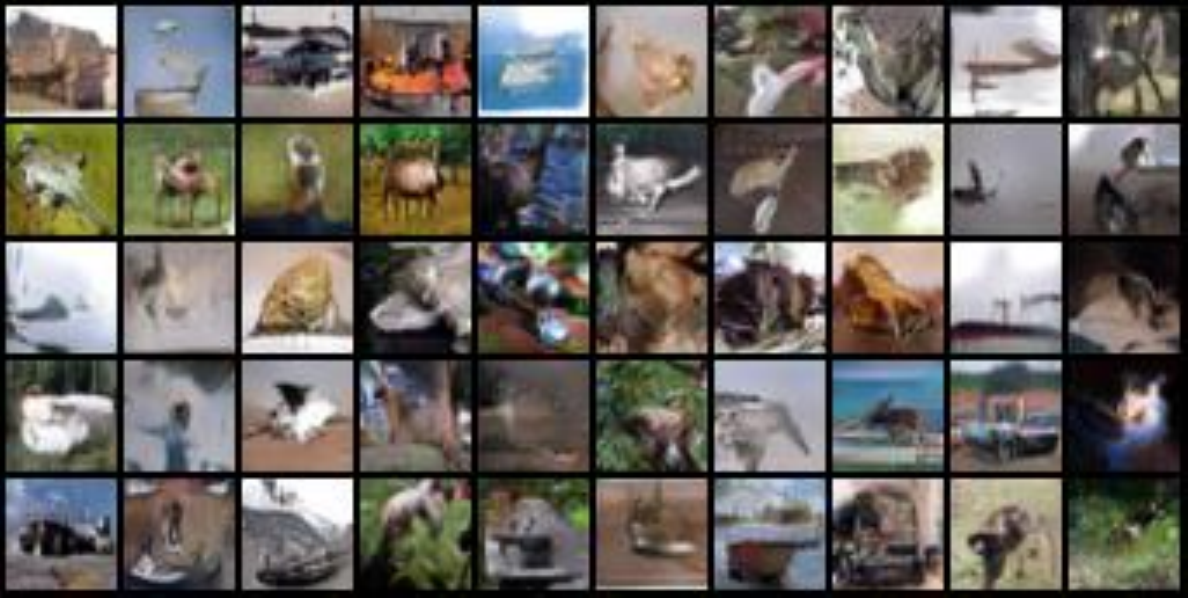}\\
			(c) WGAN-GP & (d) RS-GAN \\
			\includegraphics[width=0.5\linewidth]{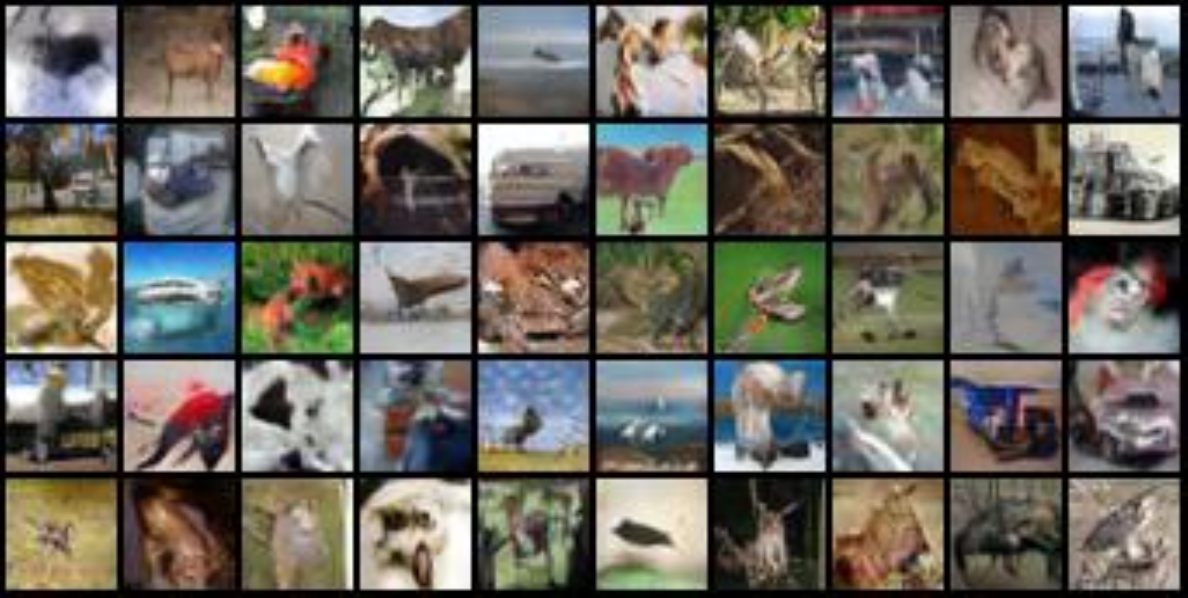} & \includegraphics[width=0.5\linewidth]{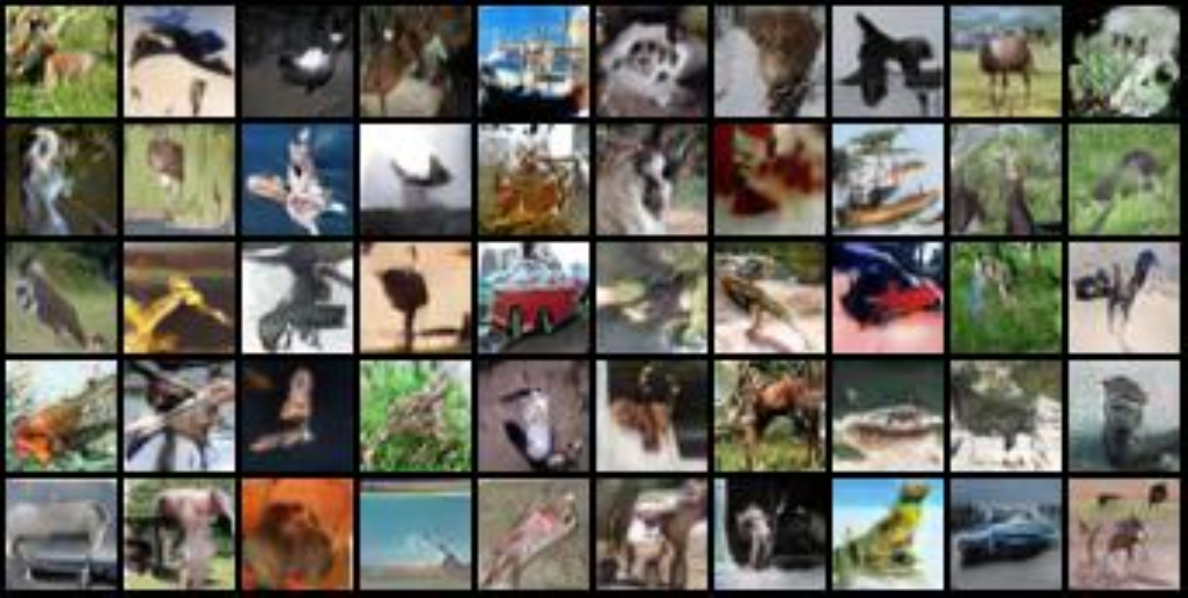}\\
			(e) JS-GAN + Spectral Norm + Regular CNN & (f) RS-GAN + Spectral Norm + Regular CNN \\
			\includegraphics[width=0.5\linewidth]{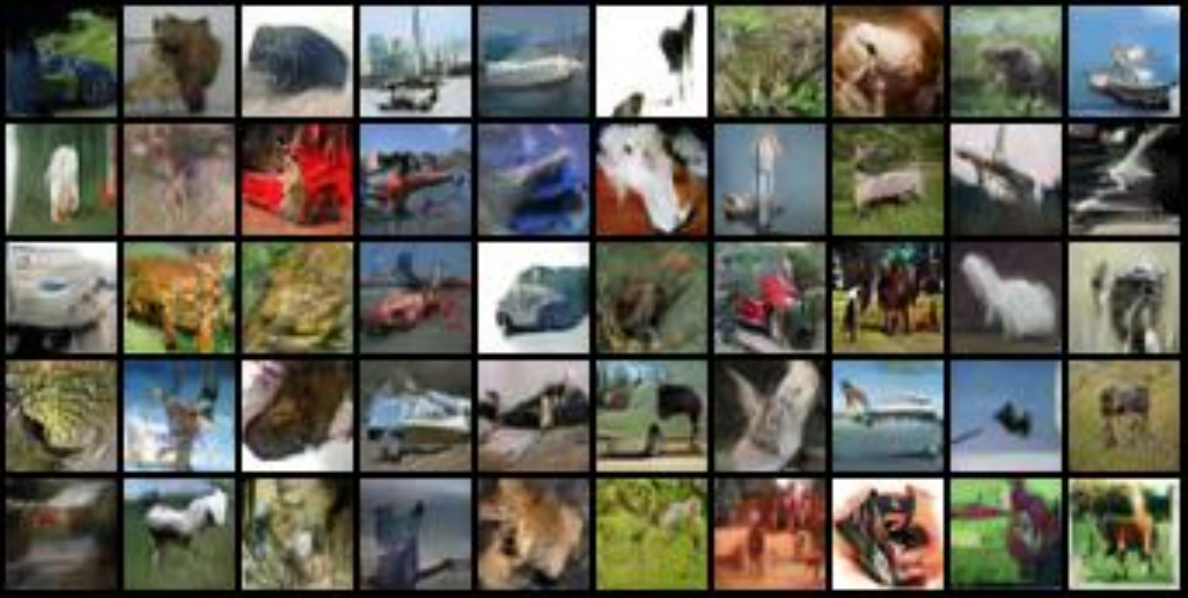} & \includegraphics[width=0.5\linewidth]{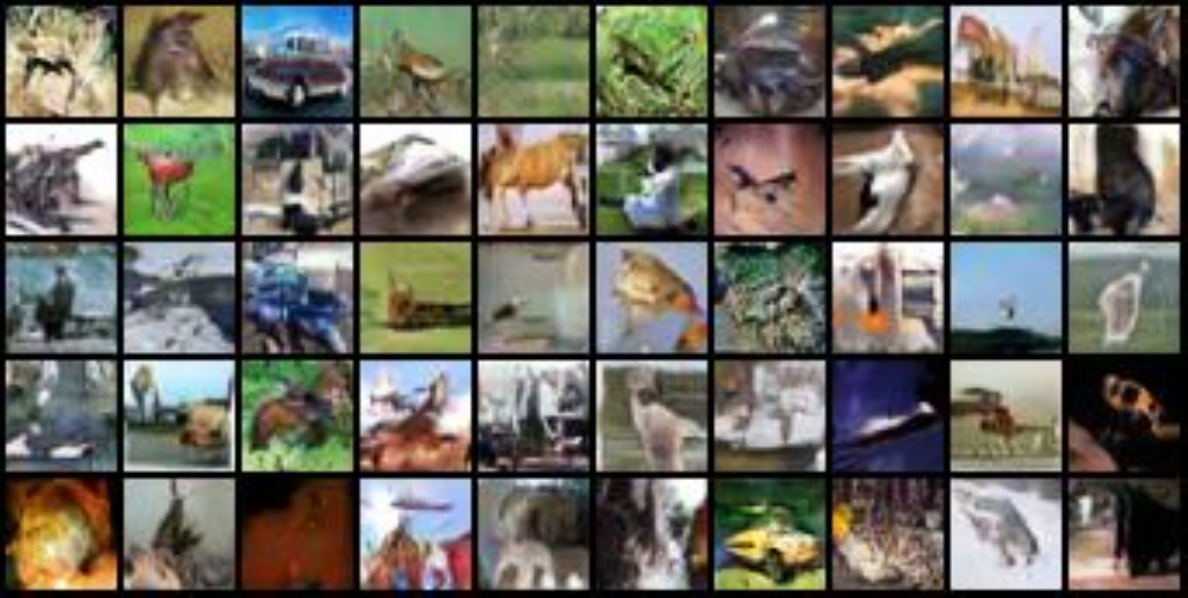}\\
			(g) JS-GAN + Spectral Norm + Channel/2 & (h) RS-GAN + Spectral Norm + Channel/2 \\
			\includegraphics[width=0.5\linewidth]{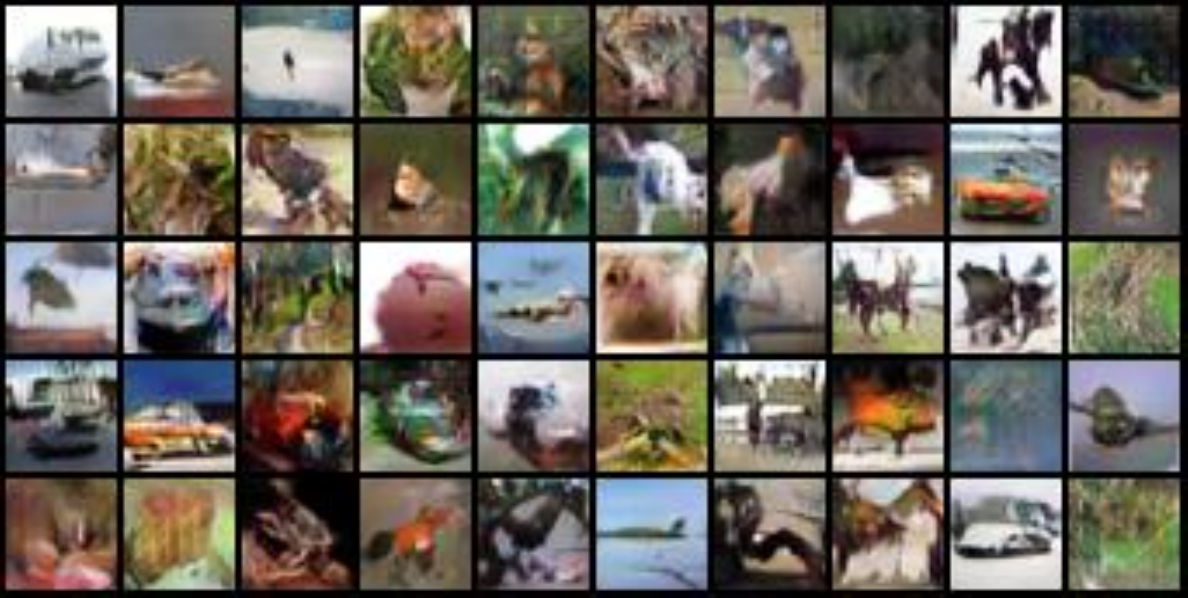} & \includegraphics[width=0.5\linewidth]{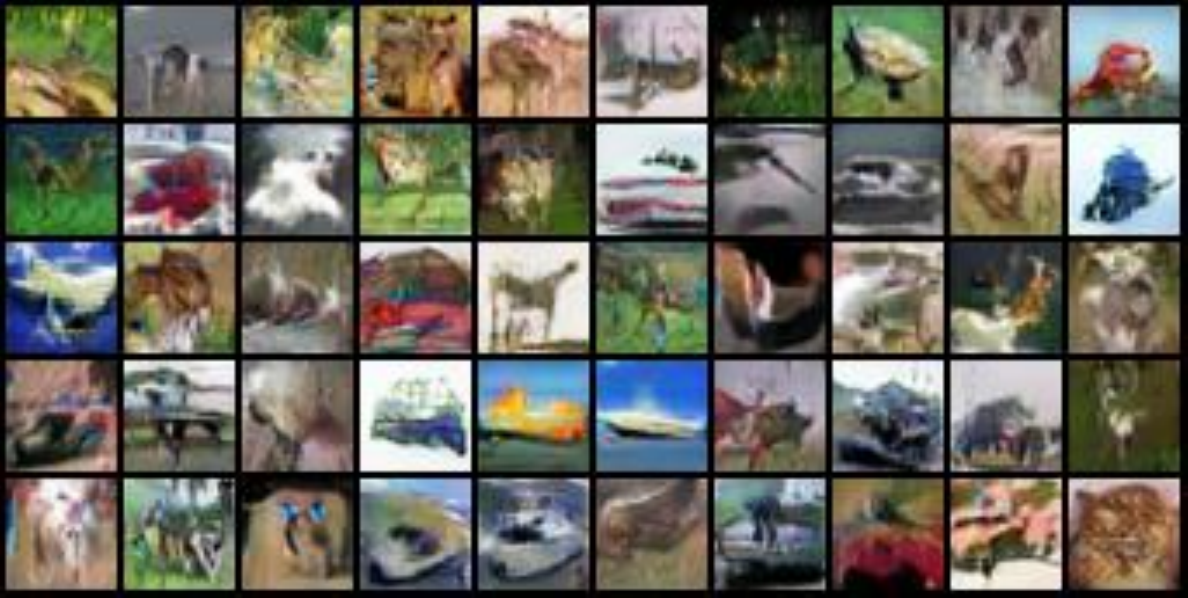}\\
			(i) JS-GAN + Spectral Norm + Channel/4 & (j) RS-GAN + Spectral Norm + Channel/4 \\

		\end{tabularx}
	}
	\caption{Generated CIFAR-10 samples with CNN.}
	\label{fig:cifarcnnsample}
	
\end{figure}

\begin{figure}[t]
	\centering
	\scalebox{1}{
		\begin{tabularx}{\linewidth} {cc}
			\includegraphics[width=0.5\linewidth]{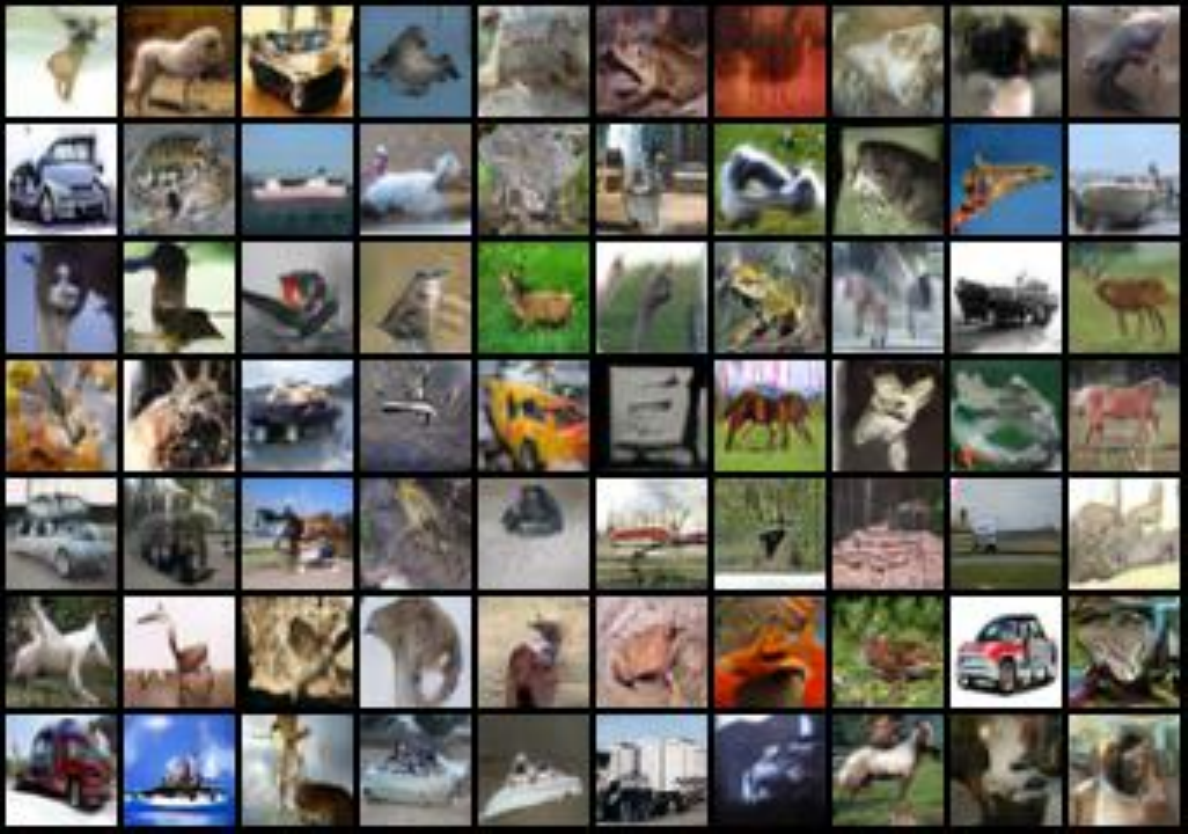} & \includegraphics[width=0.5\linewidth]{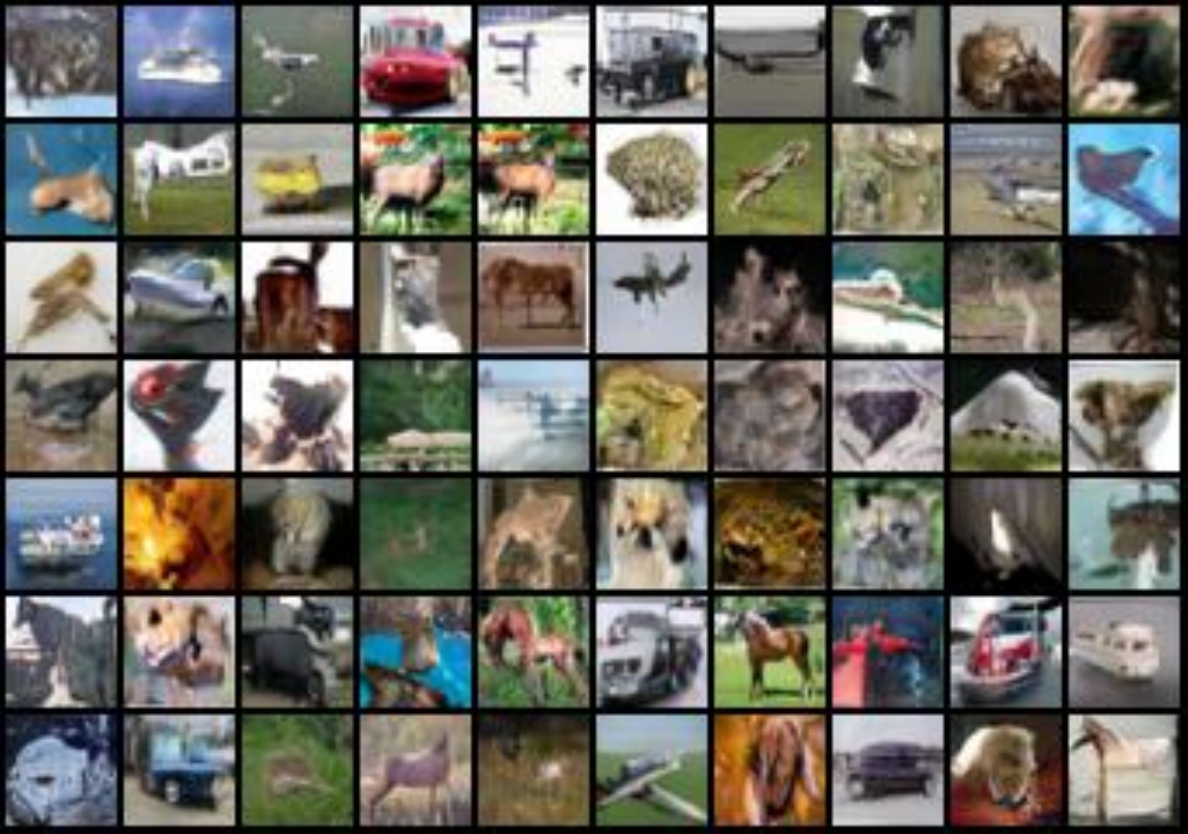}\\
			(a) JS-GAN + Spectral Norm + Regular ResNet & (b) RS-GAN + Spectral Norm + Regular ResNet \\
			\includegraphics[width=0.5\linewidth]{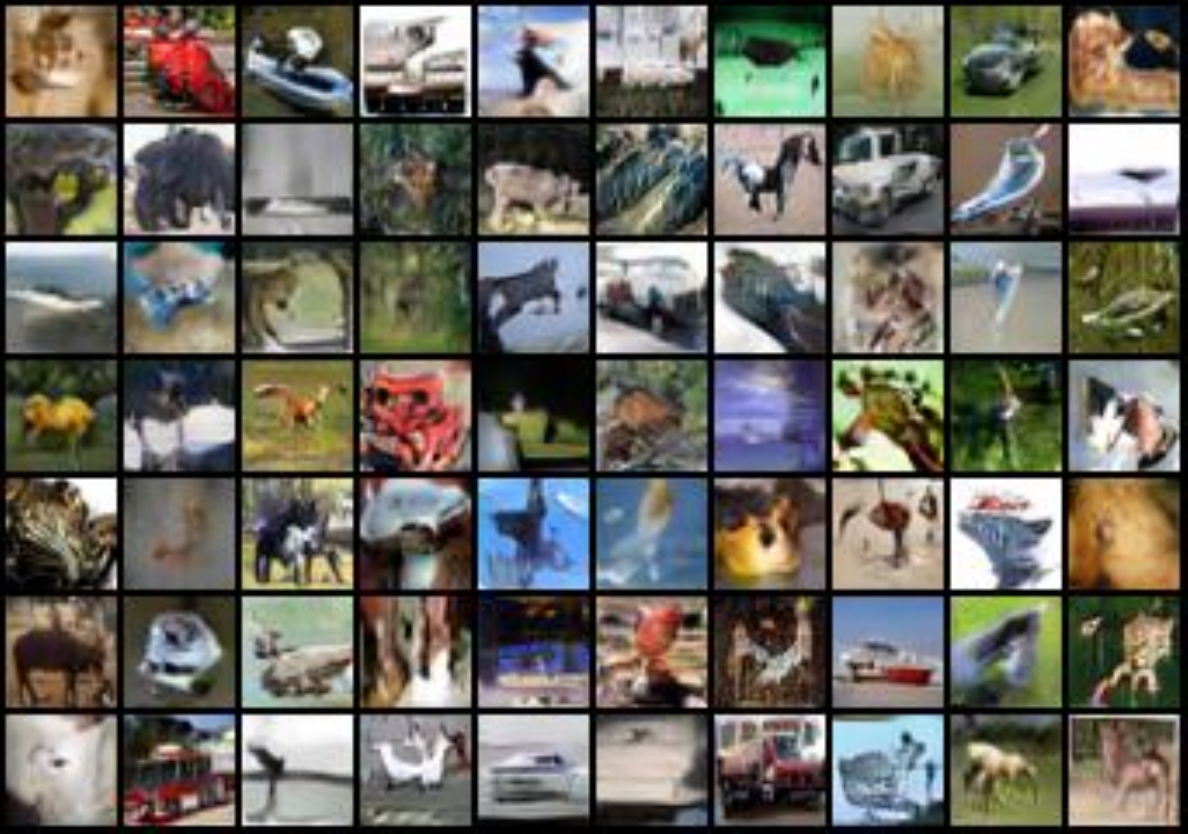} & \includegraphics[width=0.5\linewidth]{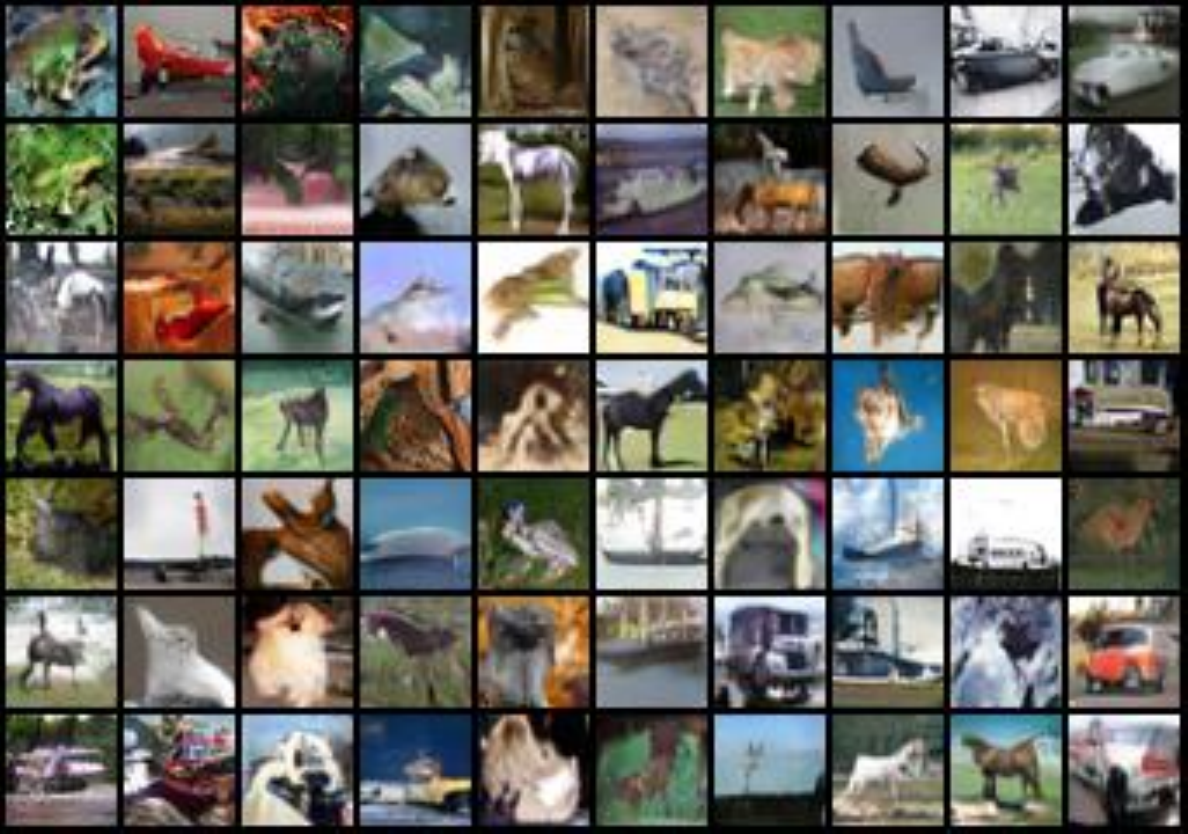}\\
			(c) JS-GAN + Spectral Norm + Channel/2 & (d) RS-GAN + Spectral Norm + Channel/2 \\
			\includegraphics[width=0.5\linewidth]{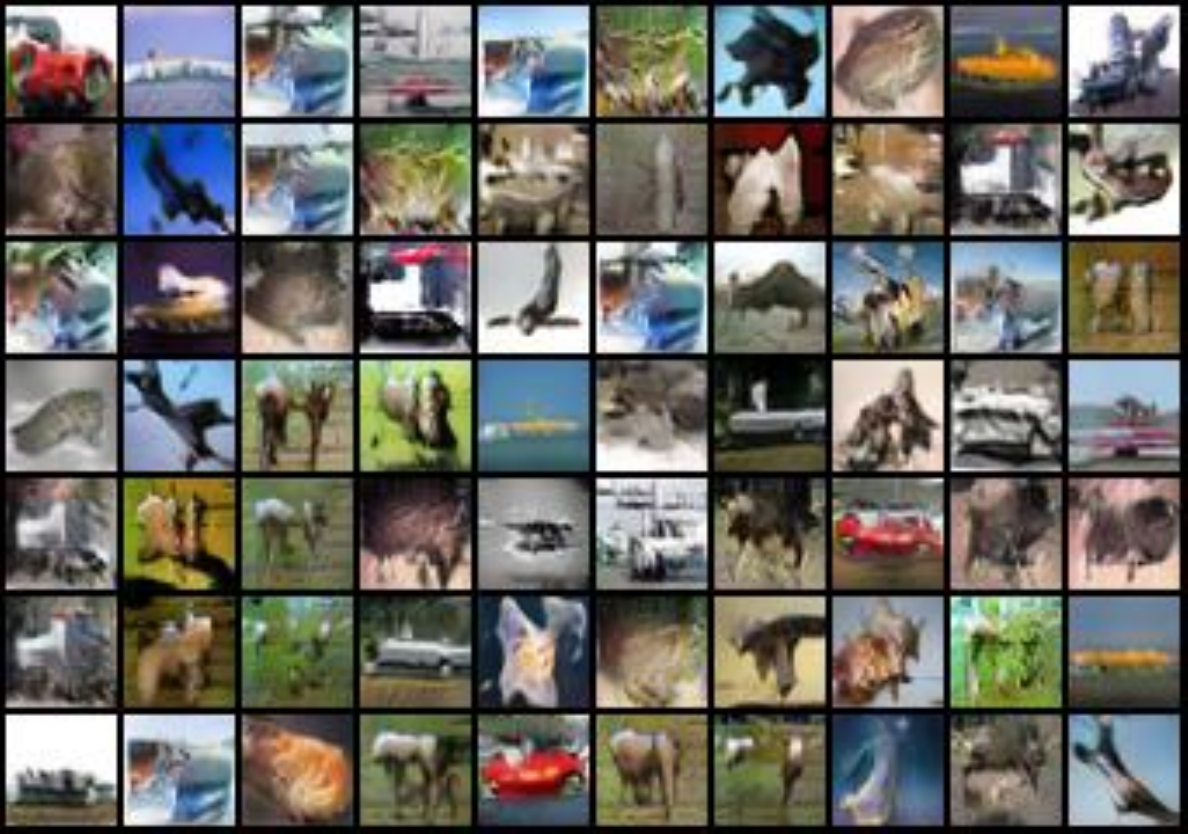} & \includegraphics[width=0.5\linewidth]{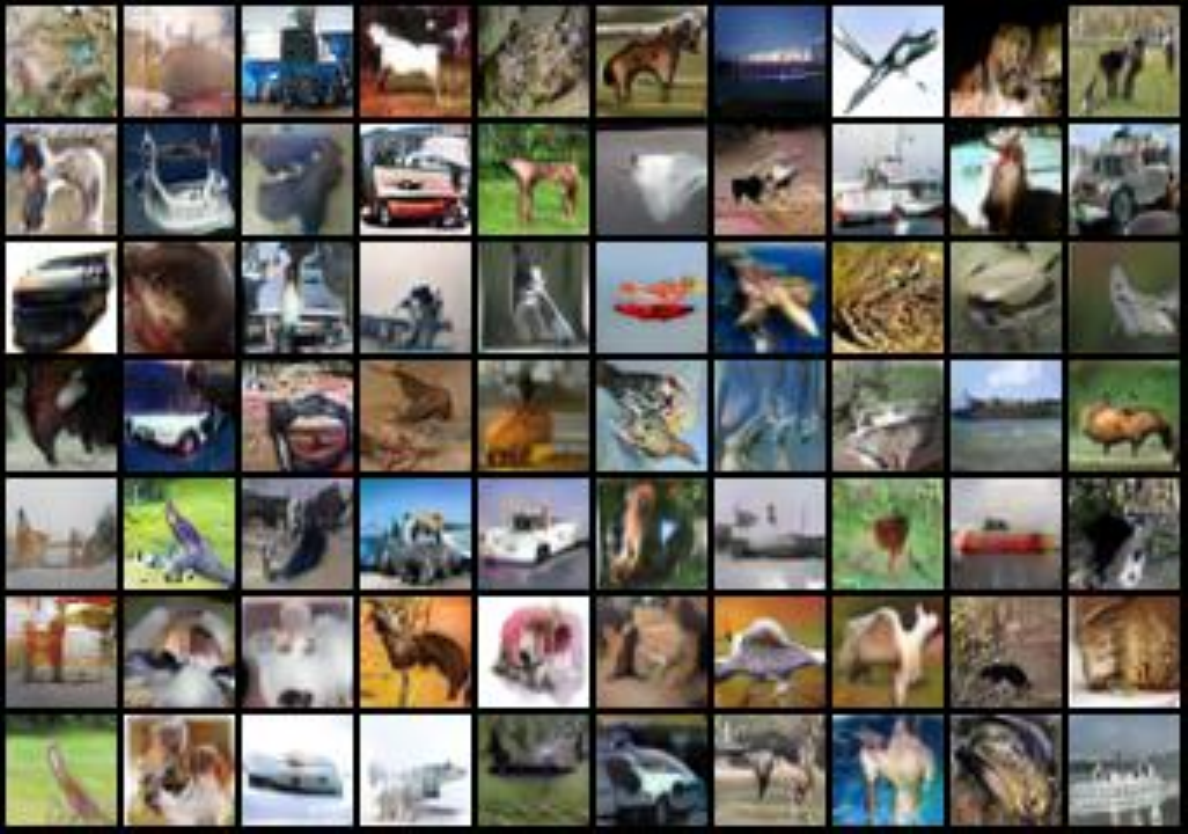}\\
			(e) JS-GAN + Spectral Norm + Channel/4 & (f) RS-GAN + Spectral Norm + Channel/4 \\
			\includegraphics[width=0.5\linewidth]{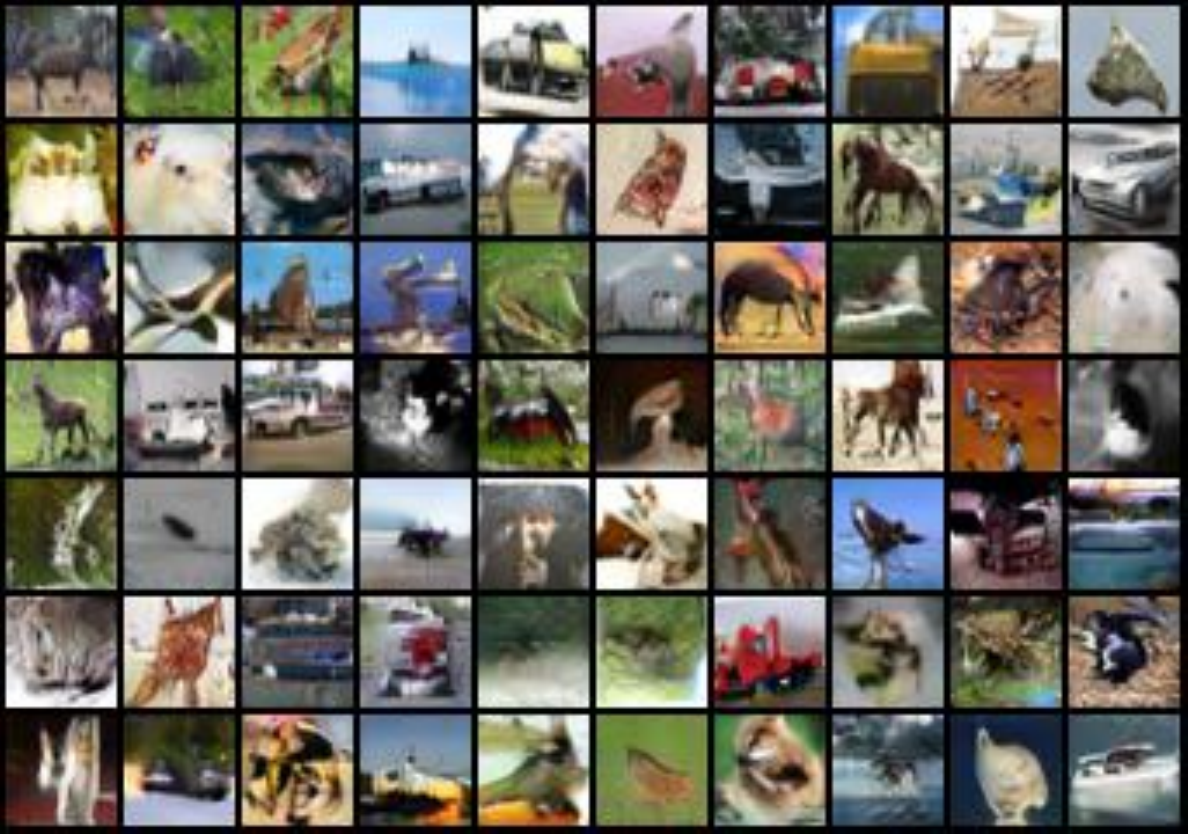} & \includegraphics[width=0.5\linewidth]{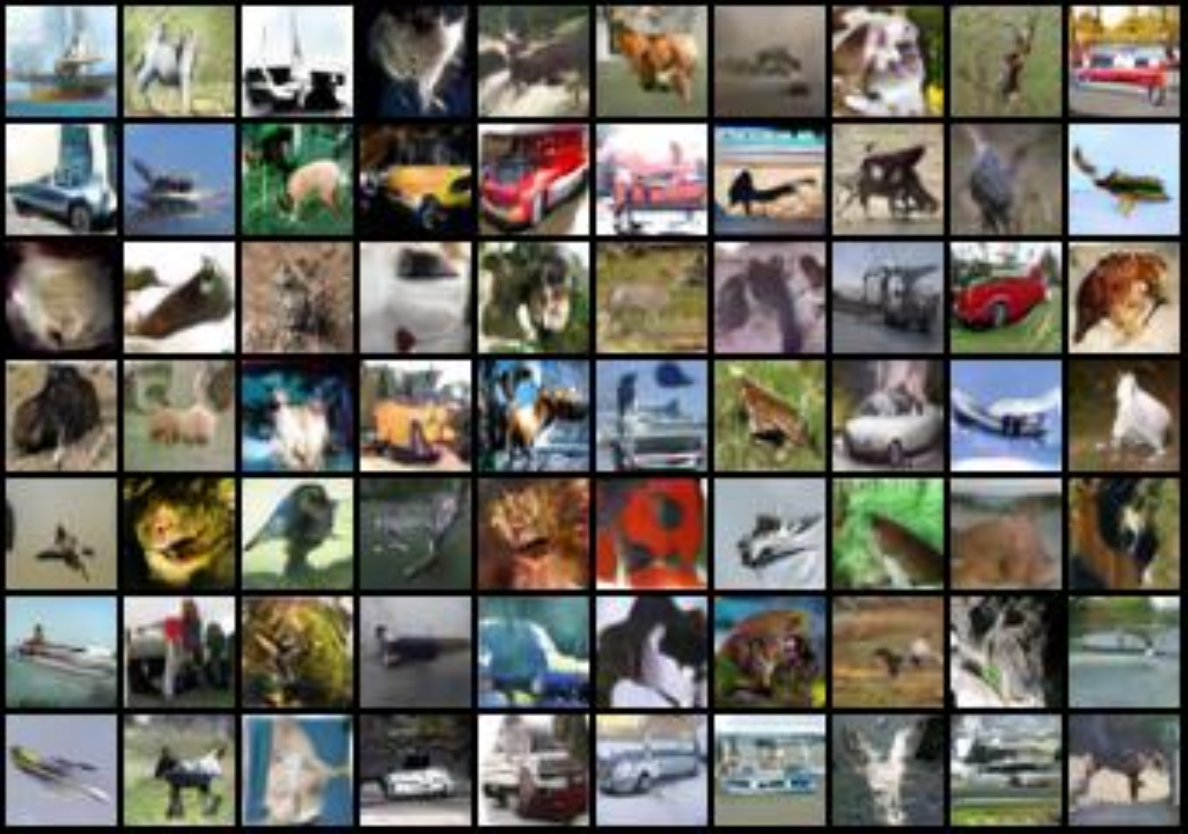}\\
			(g) JS-GAN + Spectral Norm + BottleNeck & (h) RS-GAN + Spectral Norm + BottleNeck \\
		\end{tabularx}
	}
	\caption{Generated CIFAR-10 samples on ResNet.}
	\label{fig:cifarresnetsample}
\end{figure}

\begin{figure}[t]
	\centering
	\scalebox{1}{
		\begin{tabularx}{\linewidth} {cc}
			\includegraphics[width=0.5\linewidth]{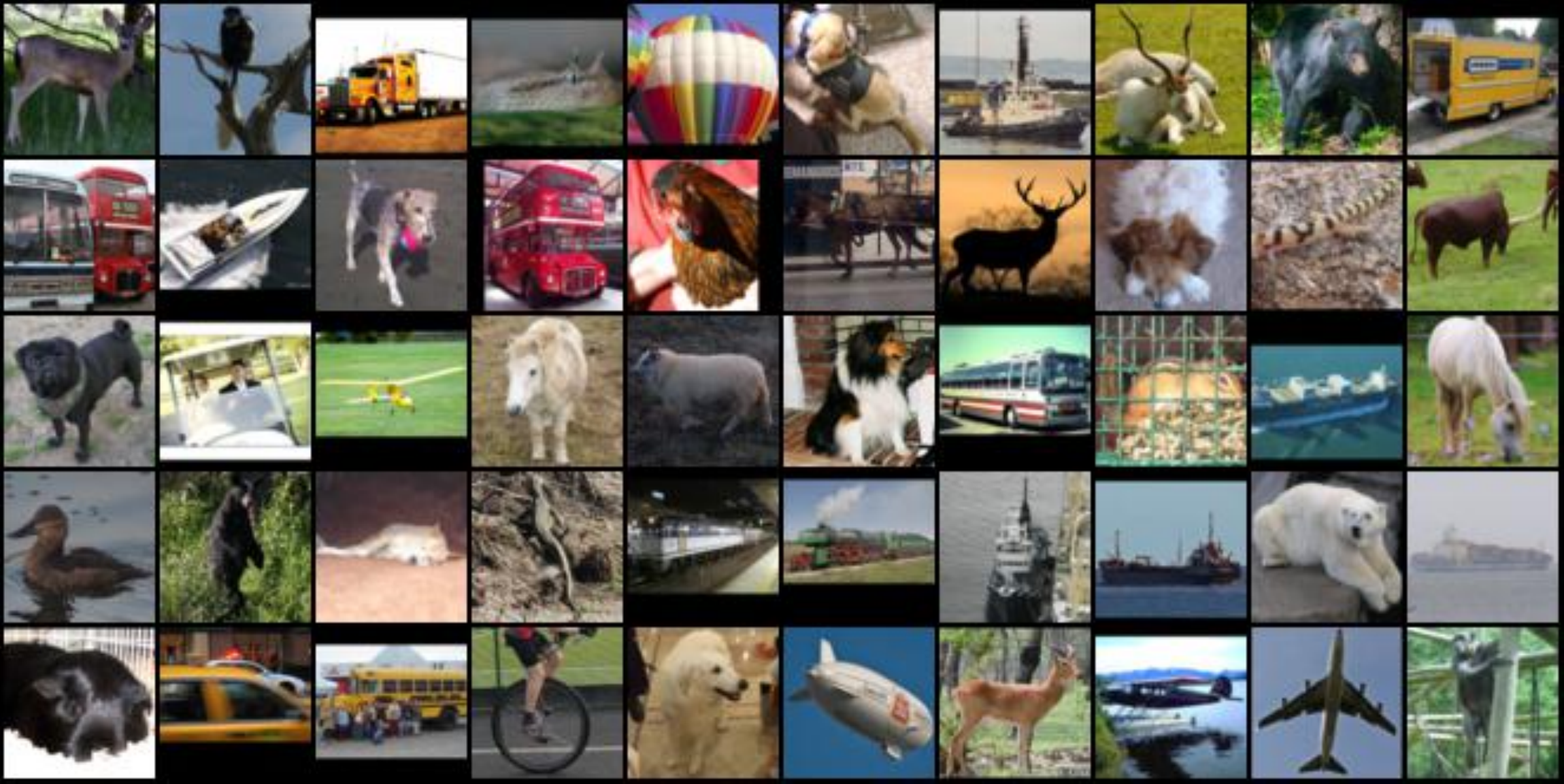} & \includegraphics[width=0.5\linewidth]{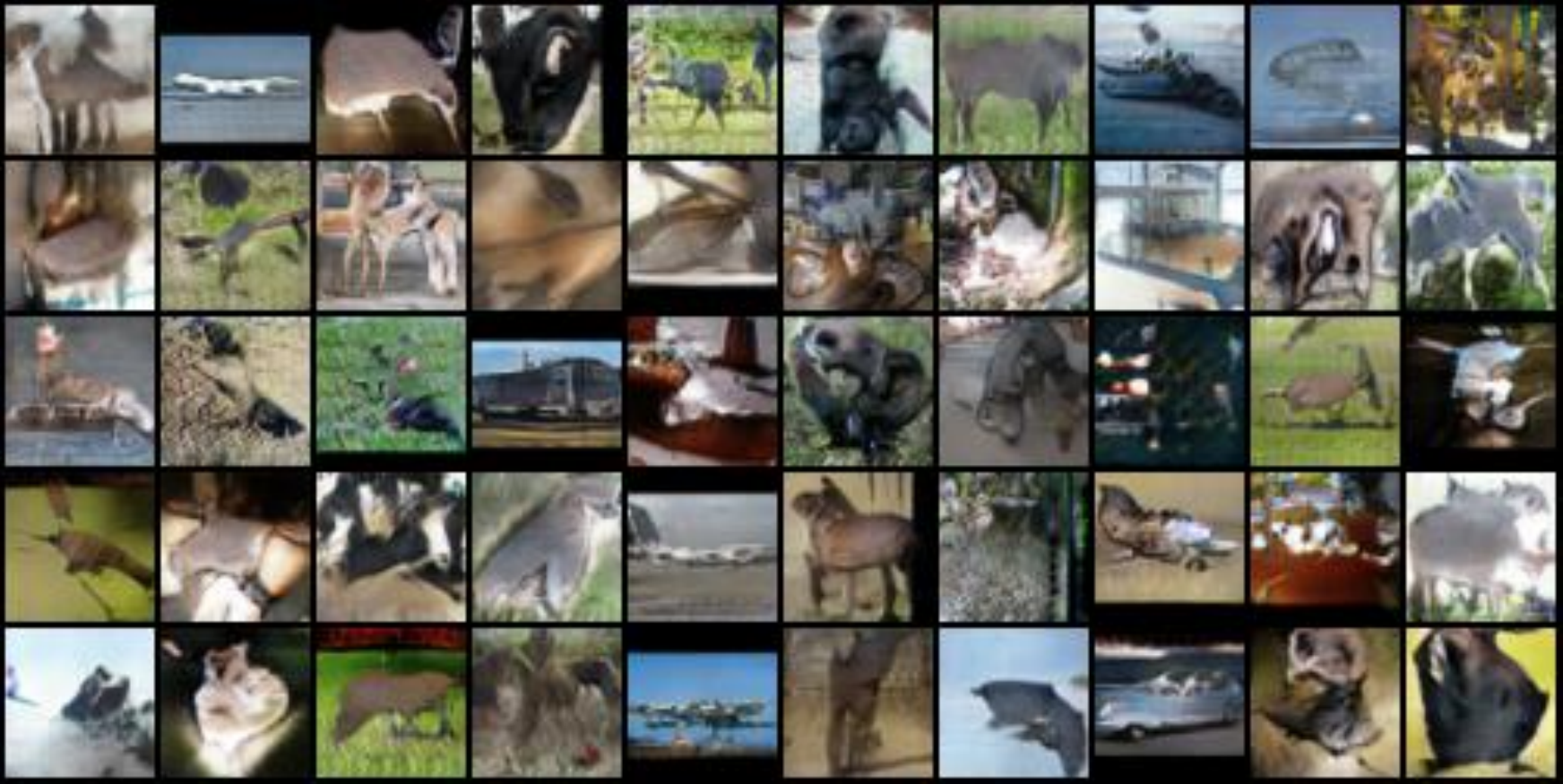}\\
			(a) real data & (b) JS-GAN + BatchNorm \\
			\includegraphics[width=0.5\linewidth]{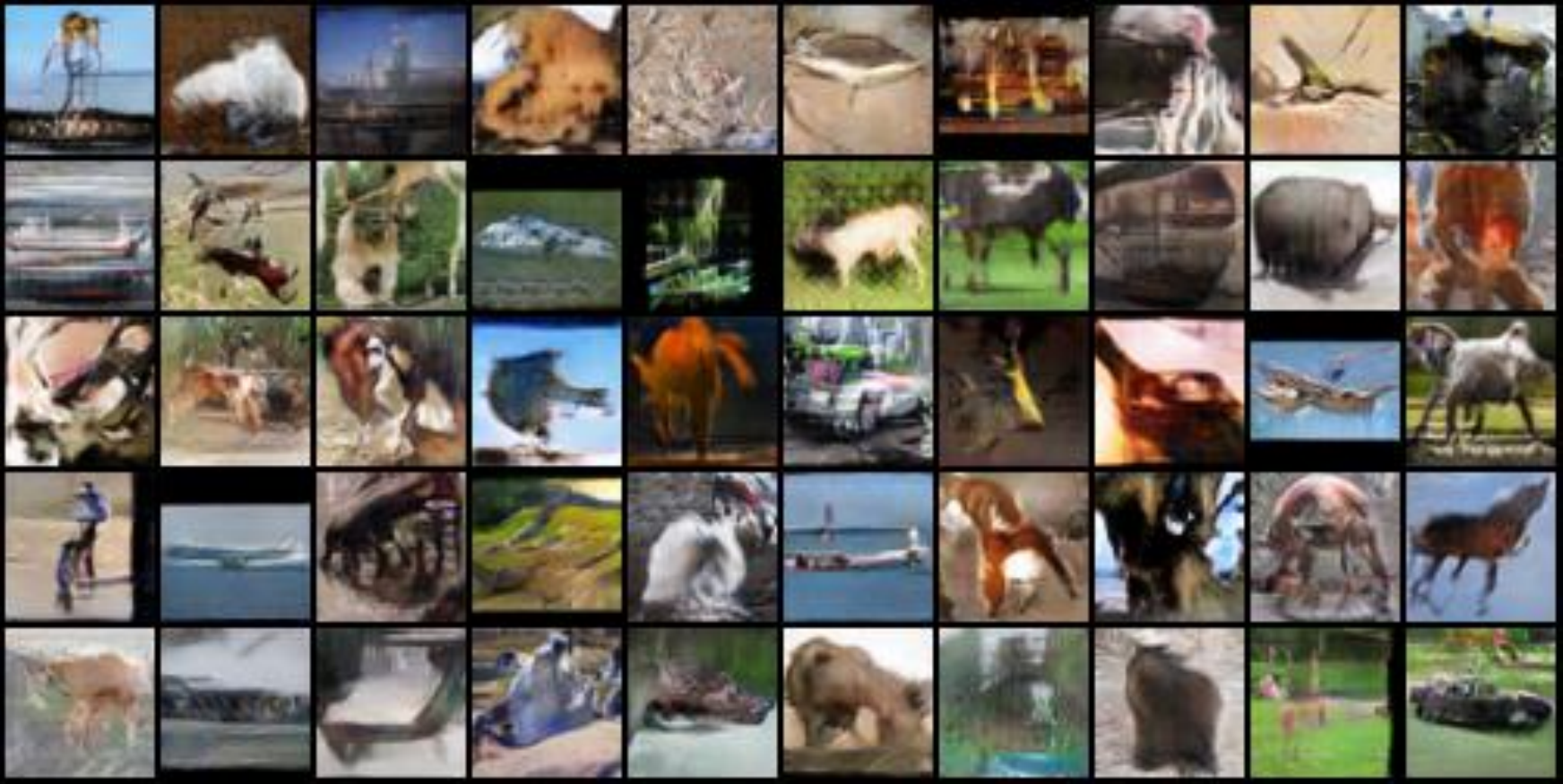} & \includegraphics[width=0.5\linewidth]{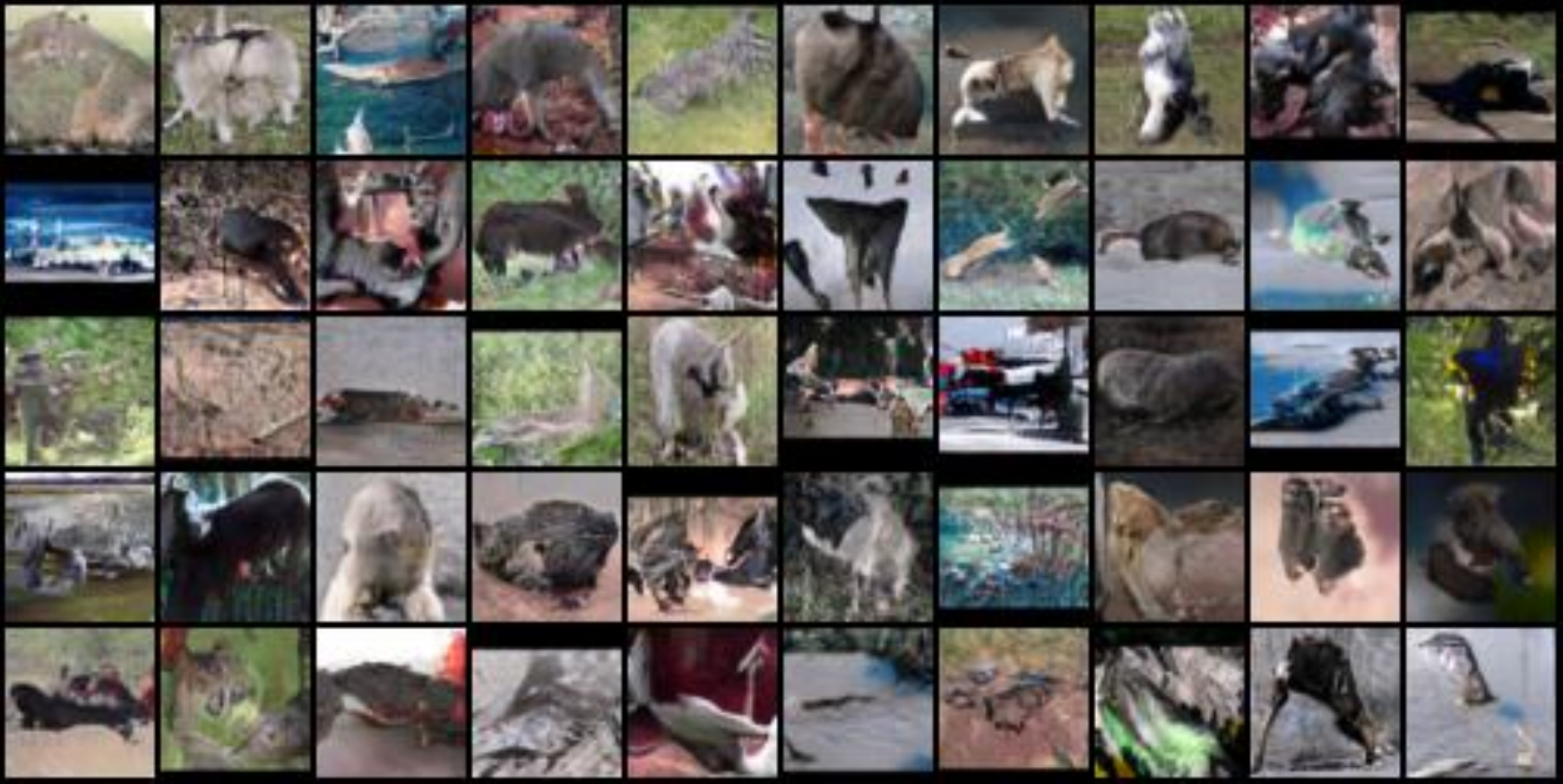}\\
			(c) WGAN-GP & (d) RS-GAN \\
			\includegraphics[width=0.5\linewidth]{FigApp/JS_regular_CNN.pdf} & \includegraphics[width=0.5\linewidth]{FigApp/RS_regular_CNN.pdf}\\
			(e) JS-GAN + Spectral Norm + Regular CNN & (f) RS-GAN + Spectral Norm + Regular CNN \\
			\includegraphics[width=0.5\linewidth]{FigApp/JS_32_CNN.pdf} & \includegraphics[width=0.5\linewidth]{FigApp/RS_32_CNN.pdf}\\
			(g) JS-GAN + Spectral Norm + Channel/2 & (h) RS-GAN + Spectral Norm + Channel/2 \\
			\includegraphics[width=0.5\linewidth]{FigApp/JS_16_CNN.pdf} & \includegraphics[width=0.5\linewidth]{FigApp/RS_16_CNN.pdf}\\
			(i) JS-GAN + Spectral Norm + Channel/4 & (j) RS-GAN + Spectral Norm + Channel/4 \\

		\end{tabularx}
	}
	\caption{Generated STL-10 samples with CNN.}
	\label{fig:stlcnnsample}
	
\end{figure}

\begin{figure}[t]
	\centering
	\scalebox{1}{
		\begin{tabularx}{\linewidth} {cc}
			\includegraphics[width=0.5\linewidth]{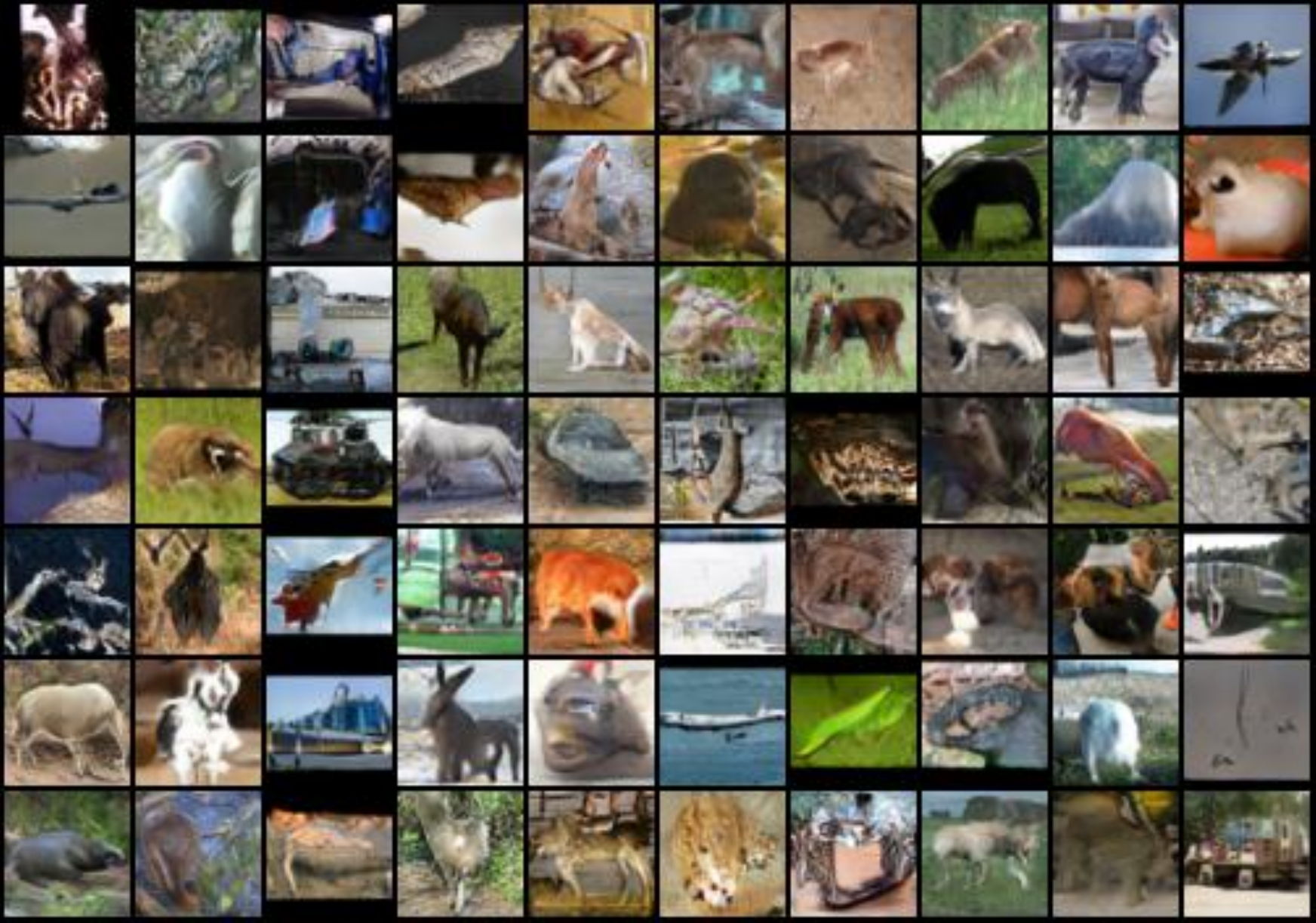} & \includegraphics[width=0.5\linewidth]{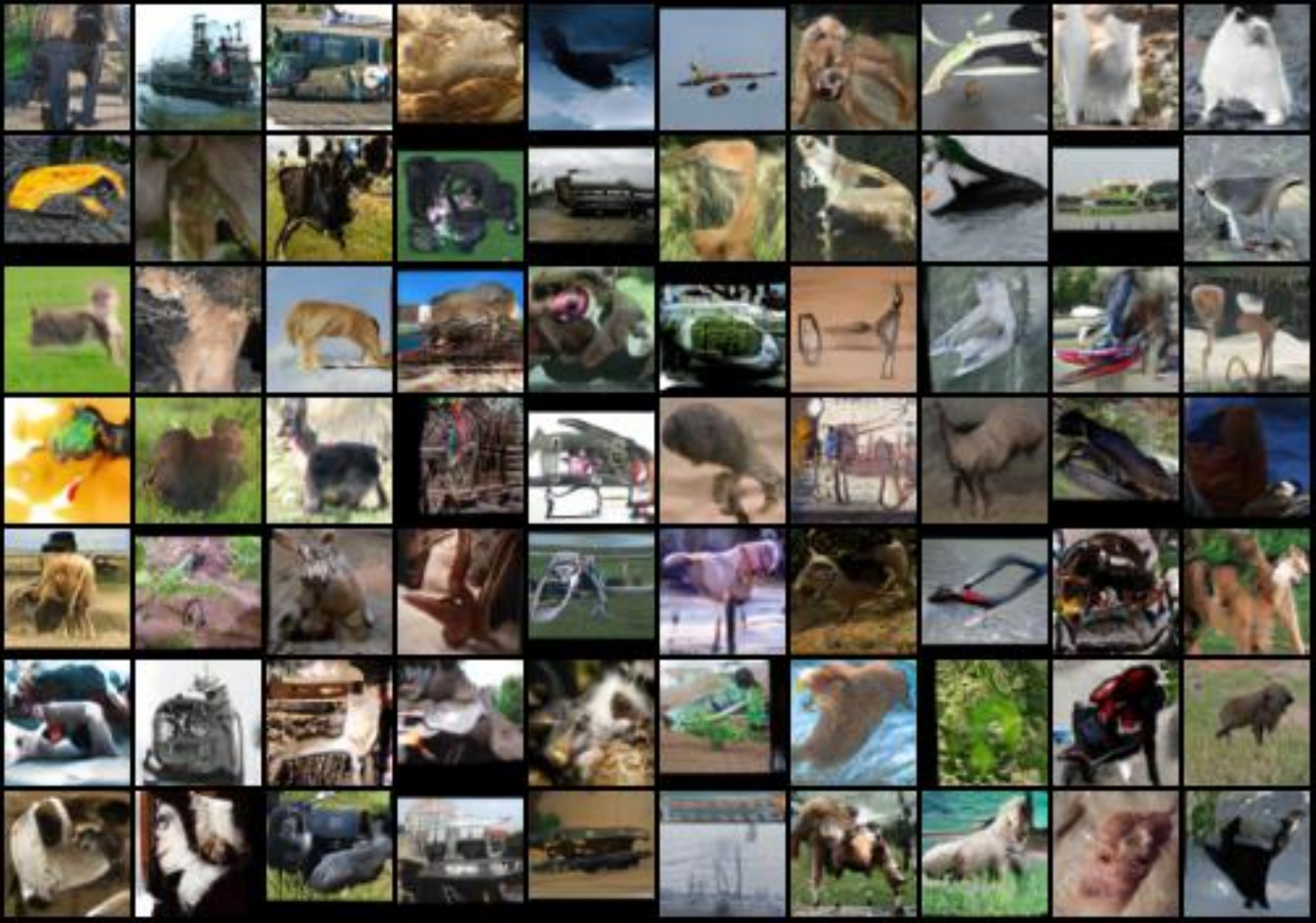}\\
			(a) JS-GAN + Spectral Norm + Regular ResNet & (b) RS-GAN + Spectral Norm + Regular ResNet \\
			\includegraphics[width=0.5\linewidth]{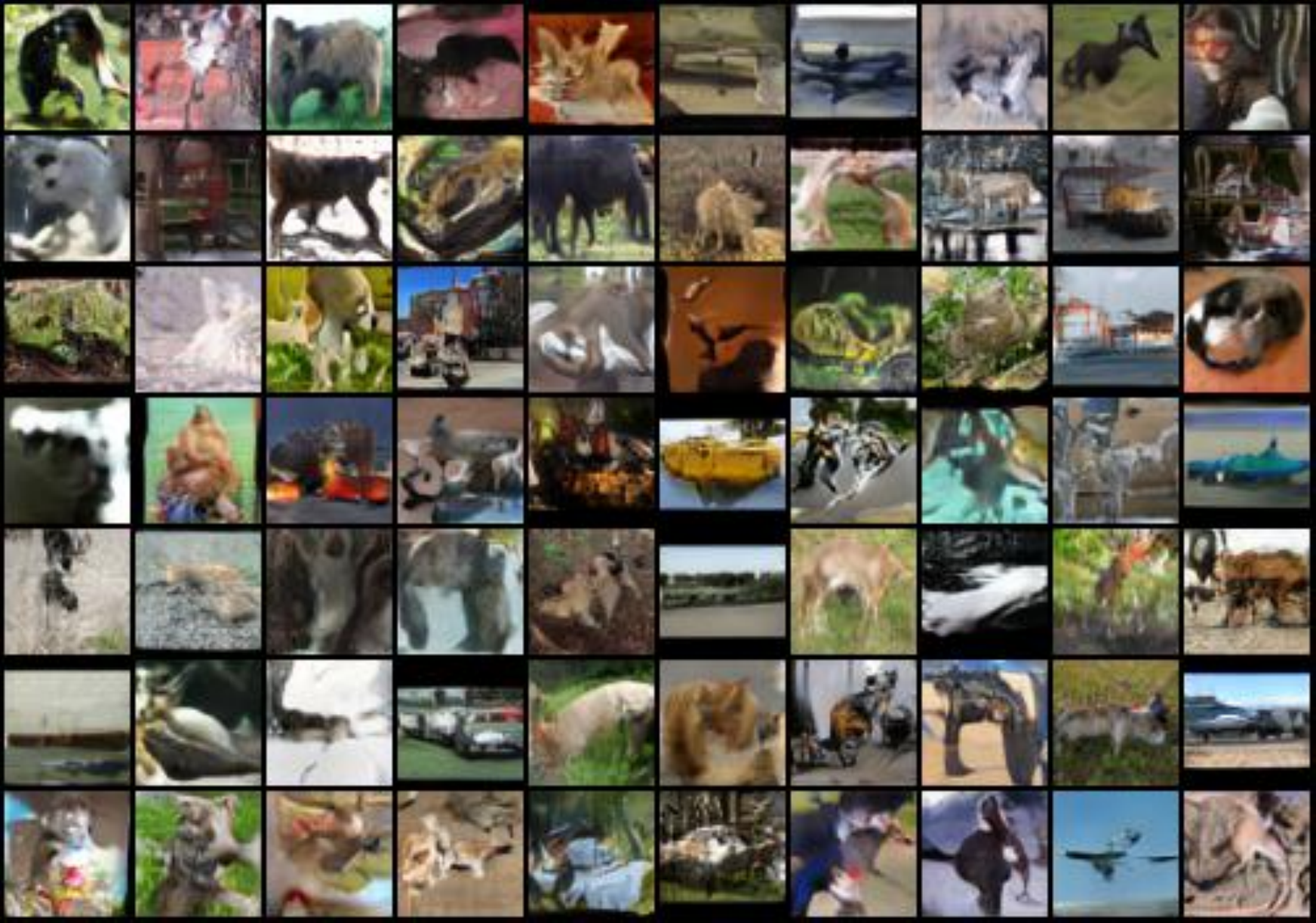} & \includegraphics[width=0.5\linewidth]{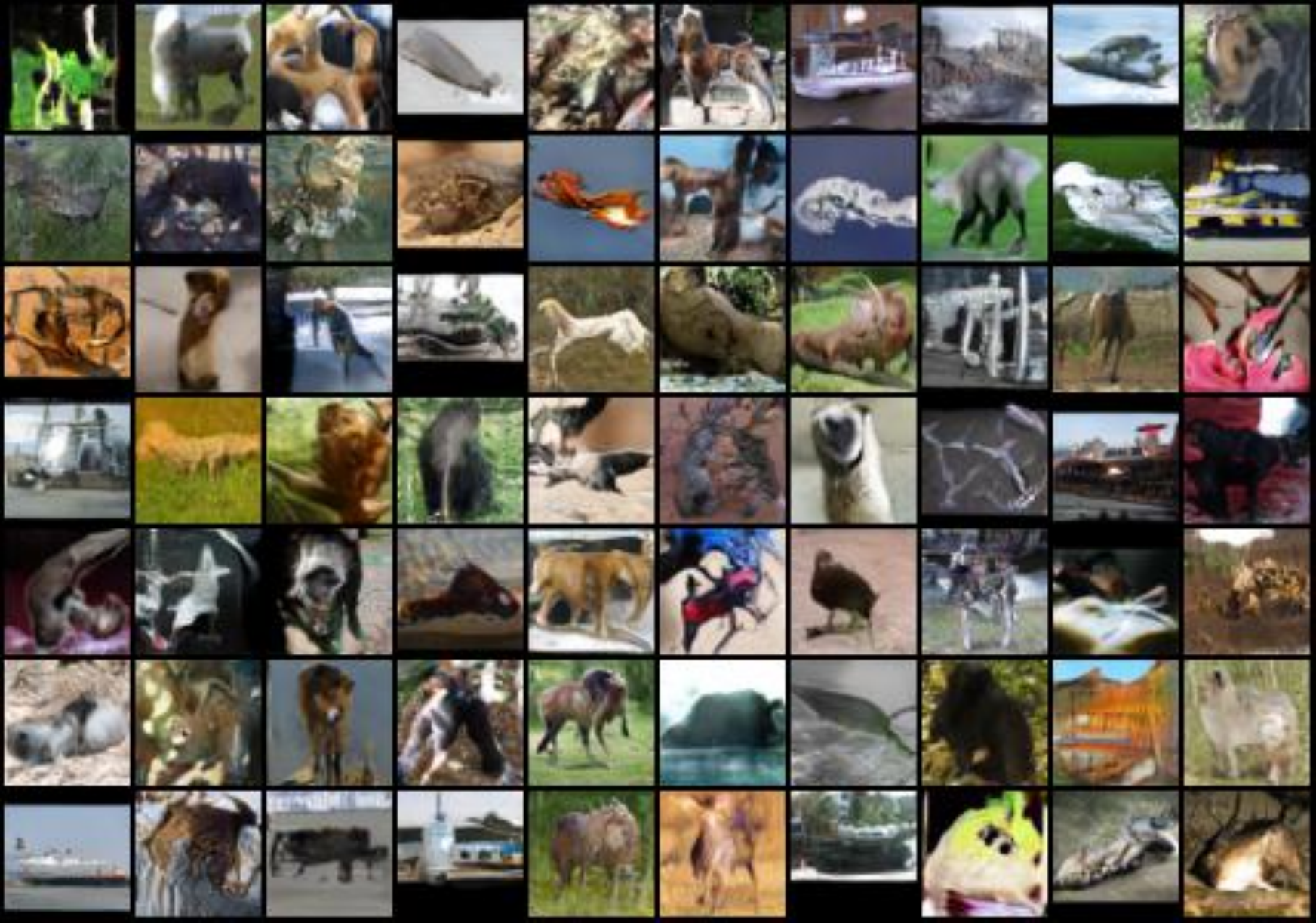}\\
			(c) JS-GAN + Spectral Norm + Channel/2 & (d) RS-GAN + Spectral Norm + Channel/2 \\
			\includegraphics[width=0.5\linewidth]{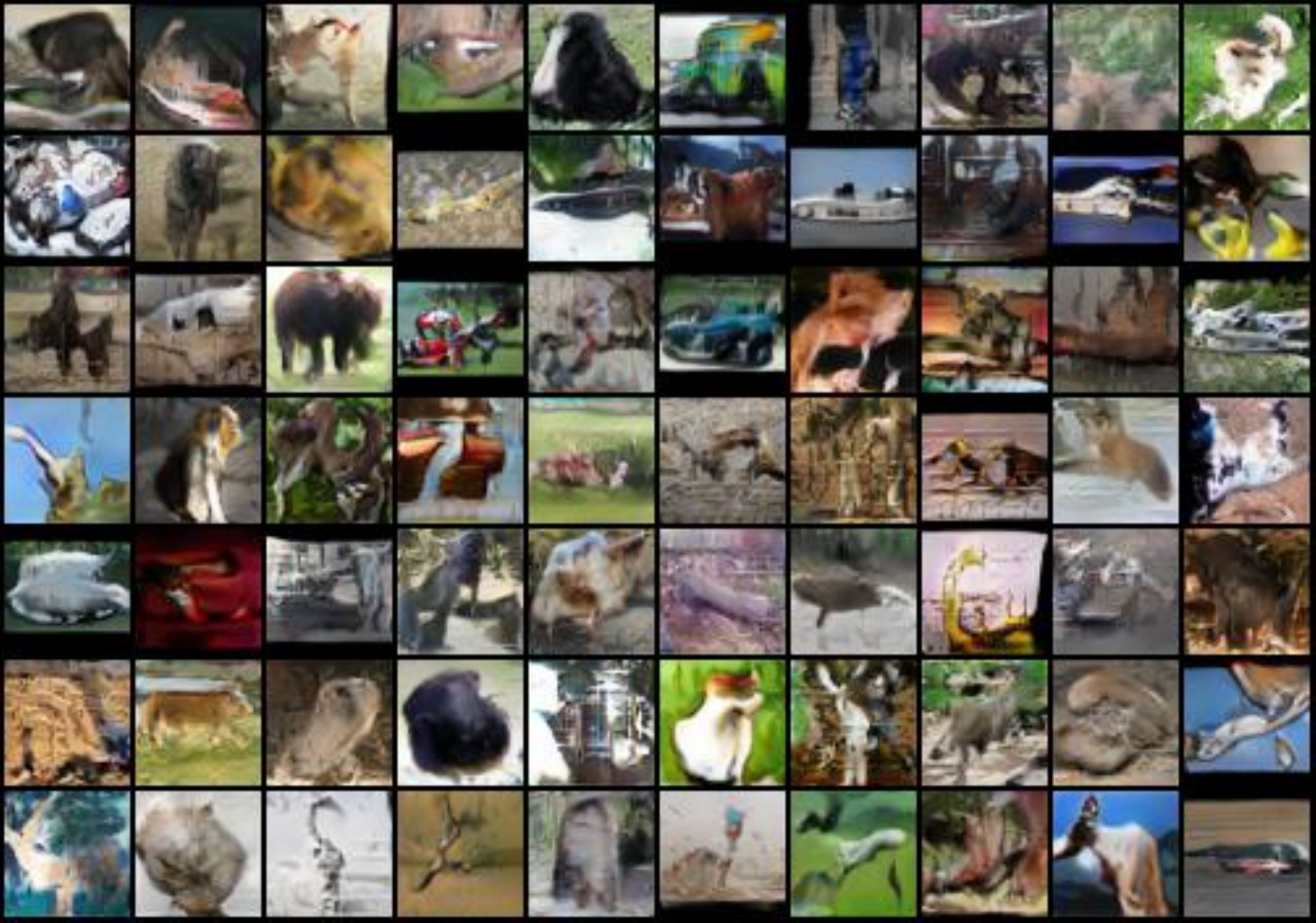} & \includegraphics[width=0.5\linewidth]{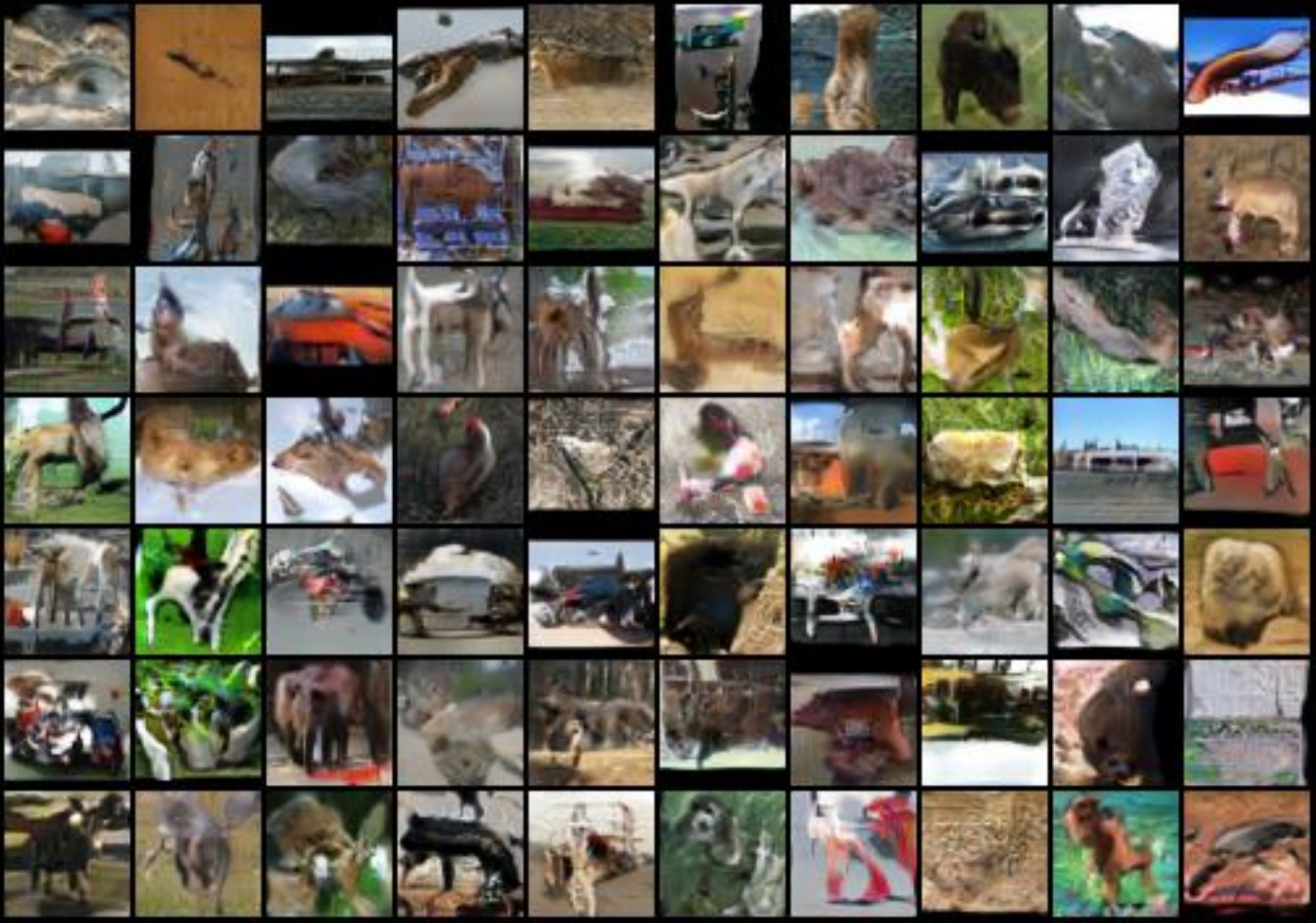}\\
			(e) JS-GAN + Spectral Norm + Channel/4 & (f) RS-GAN + Spectral Norm + Channel/4 \\
			\includegraphics[width=0.5\linewidth]{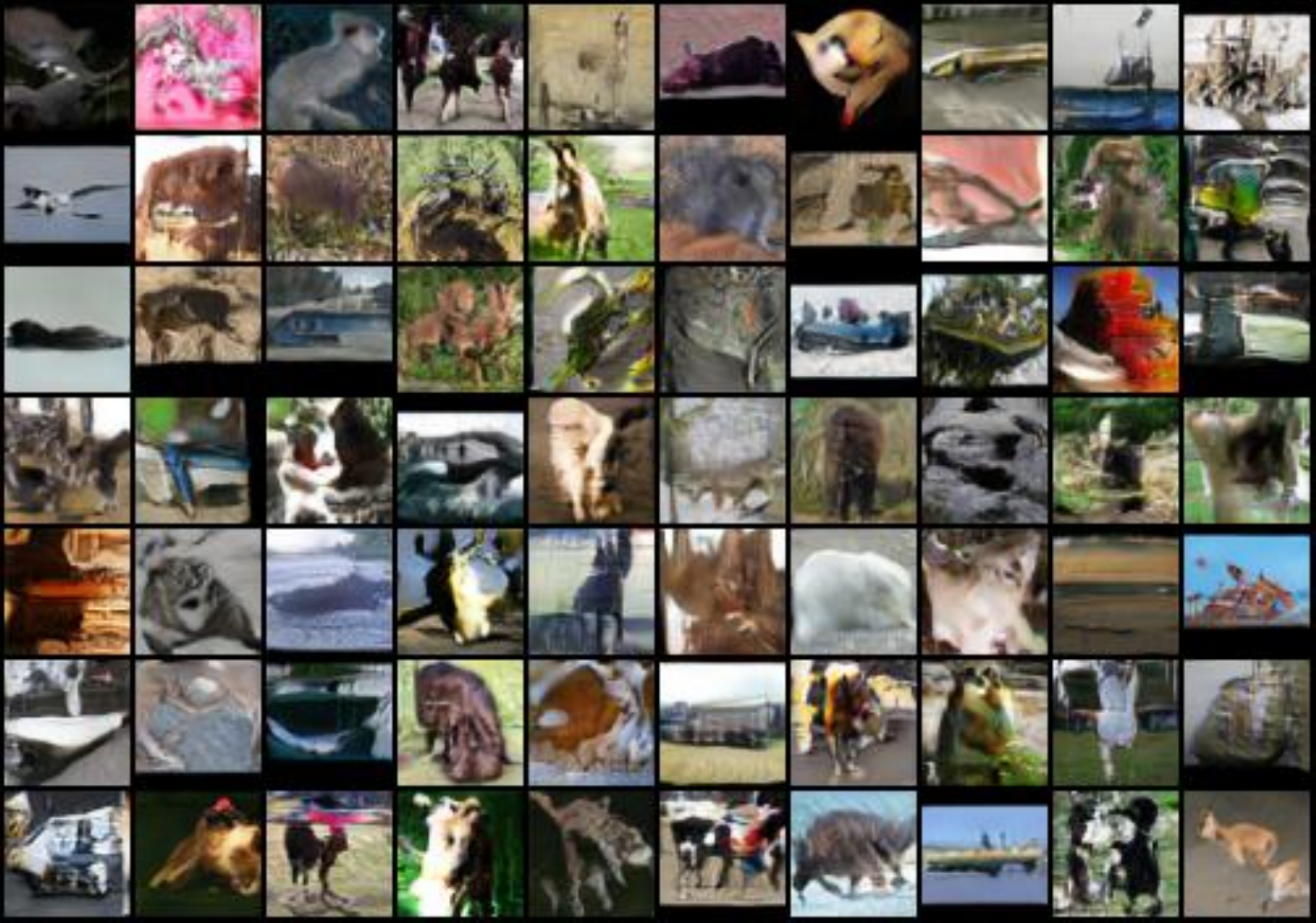} & \includegraphics[width=0.5\linewidth]{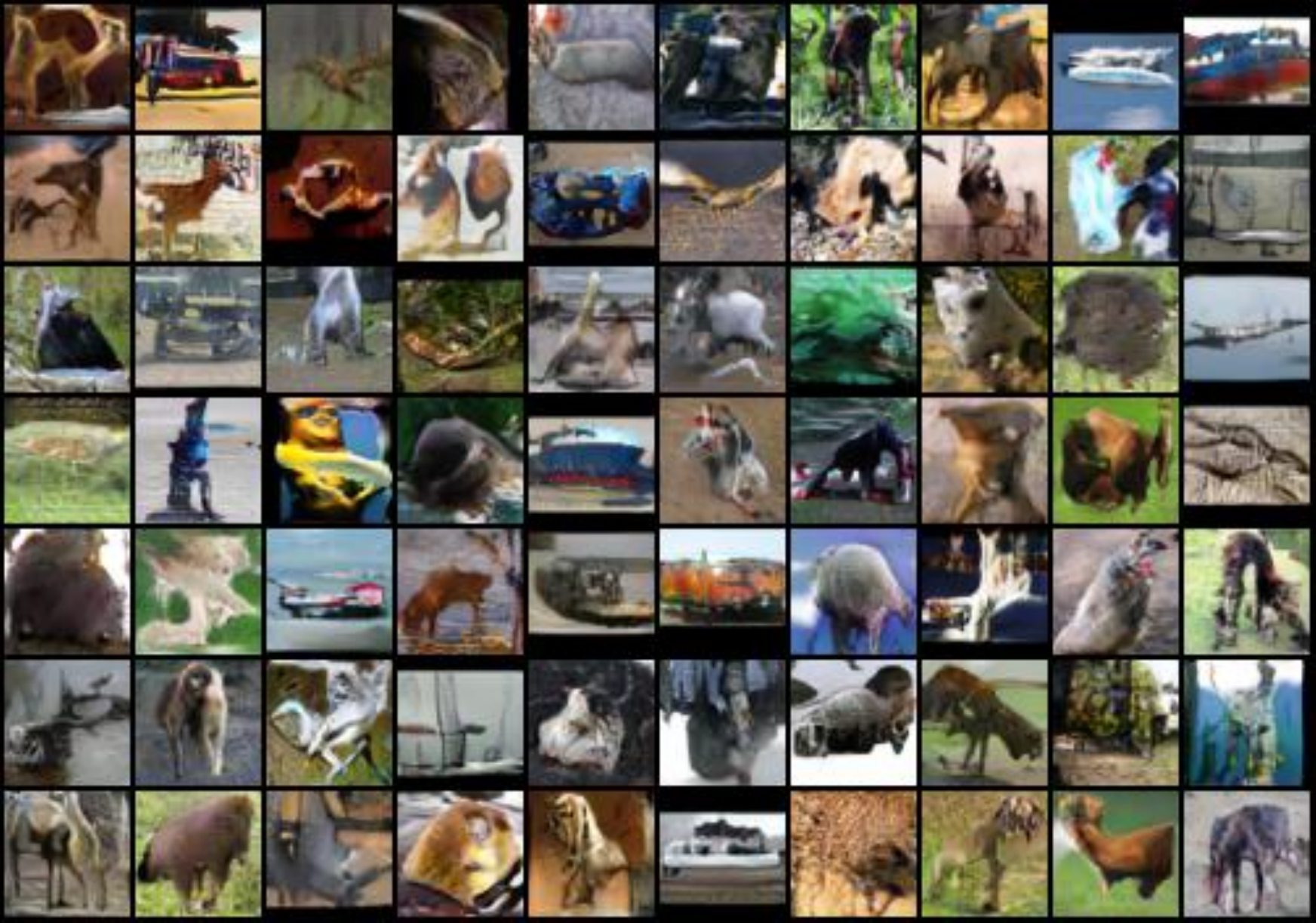}\\
			(g) JS-GAN + Spectral Norm + BottleNeck & (h) RS-GAN + Spectral Norm + BottleNeck \\
		\end{tabularx}
	}
	\caption{Generated STL-10 samples with ResNet.}
	\label{fig:stlesnetsample}
	
\end{figure}

\section{Experiments on High Resolution Data}\label{sec: high resolution}
There are two approaches to achieve a good landscape: one uses a wide enough neural net~\citep{nguyen2018loss,li2018over}, and the other uses a large enough number of samples (approaching convexity of pdf space). 
As we discuss in  Sec.~\ref{sec: empirical and population}  (see also Appendix~\ref{app: particle v.s. prob.}), 
when the number of samples is far from enough for filling the data space, the convexity (of pdf space) may vanish.
A higher dimension of data implies a larger gap between empirical loss and population loss,  thus the non-convexity issue will become more severe.
Thus we conjecture that JS-GAN suffers more for higher resolution data generation.

We consider $256\times256$ LSUN Church and Tower datasets with CNN architecture in Tab.~\ref{table: lsun_cnn_model}. For RS-GAN, we set glr = 1e-3 and dlr = 2e-4
We train $100,000$ iterations with batchsize $64$. 
The generated images are presented in Fig.~\ref{fig:lsunsample}. For both datasets, RS-GAN outperforms JS-GAN visually.

\begin{figure}[t]
	\centering
	\scalebox{1}{
		\begin{tabularx}{\linewidth} {ccc}
		\includegraphics[width=0.5\linewidth]{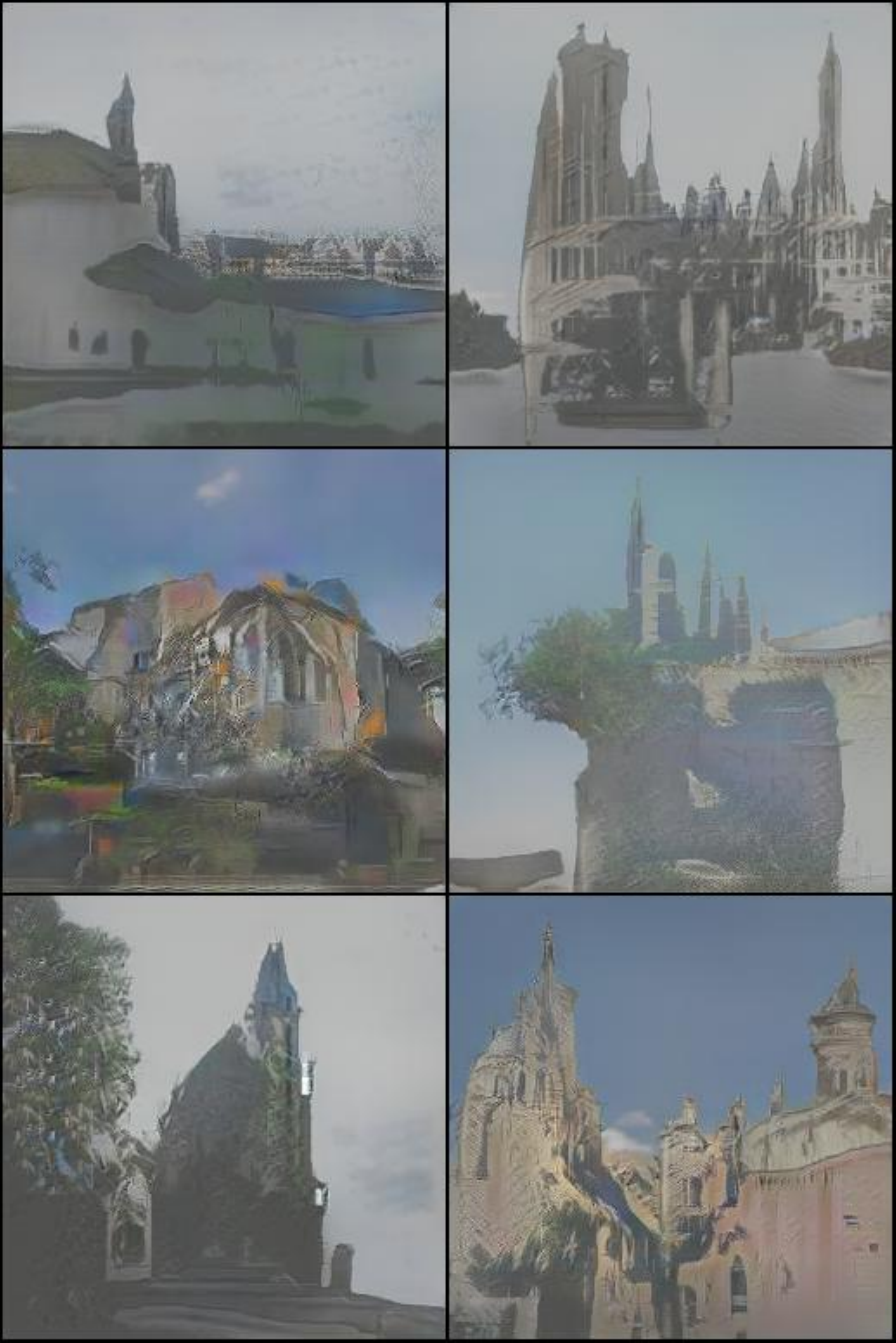} & &\includegraphics[width=0.5\linewidth]{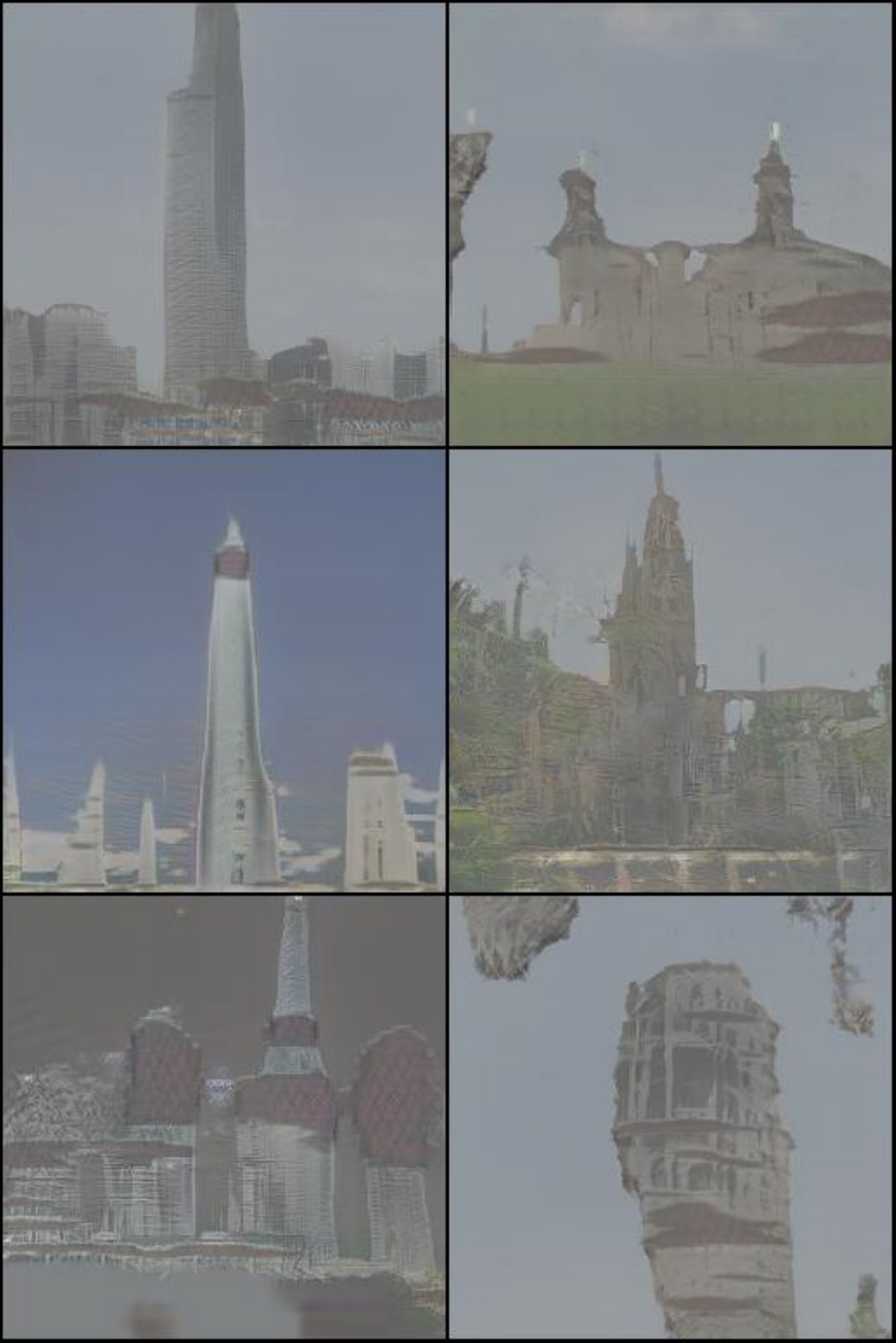}\\
(a) LSUN Church by JS-GAN & & (b) LSUN Tower by JS-GAN \\
	\includegraphics[width=0.5\linewidth]{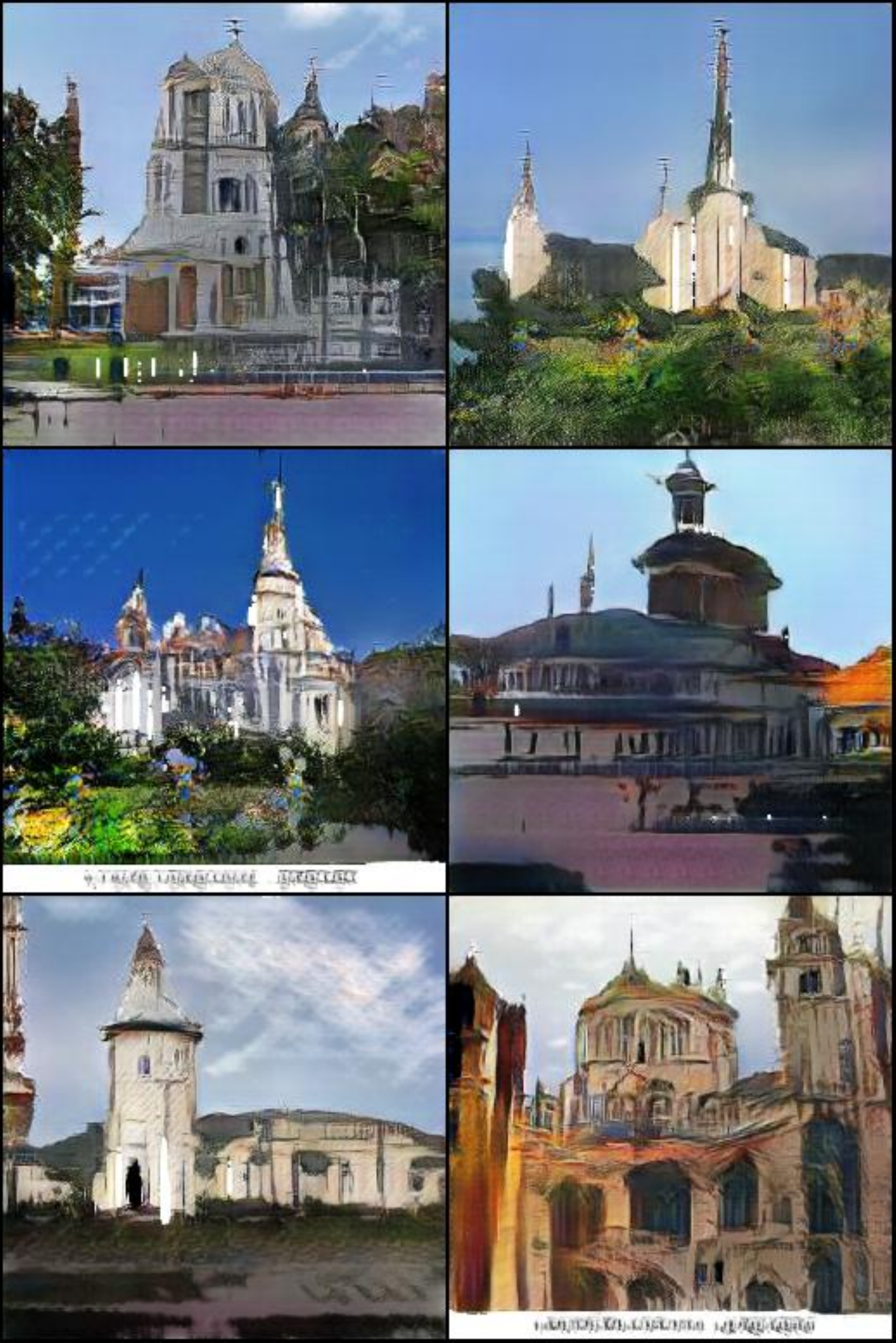} & &\includegraphics[width=0.5\linewidth]{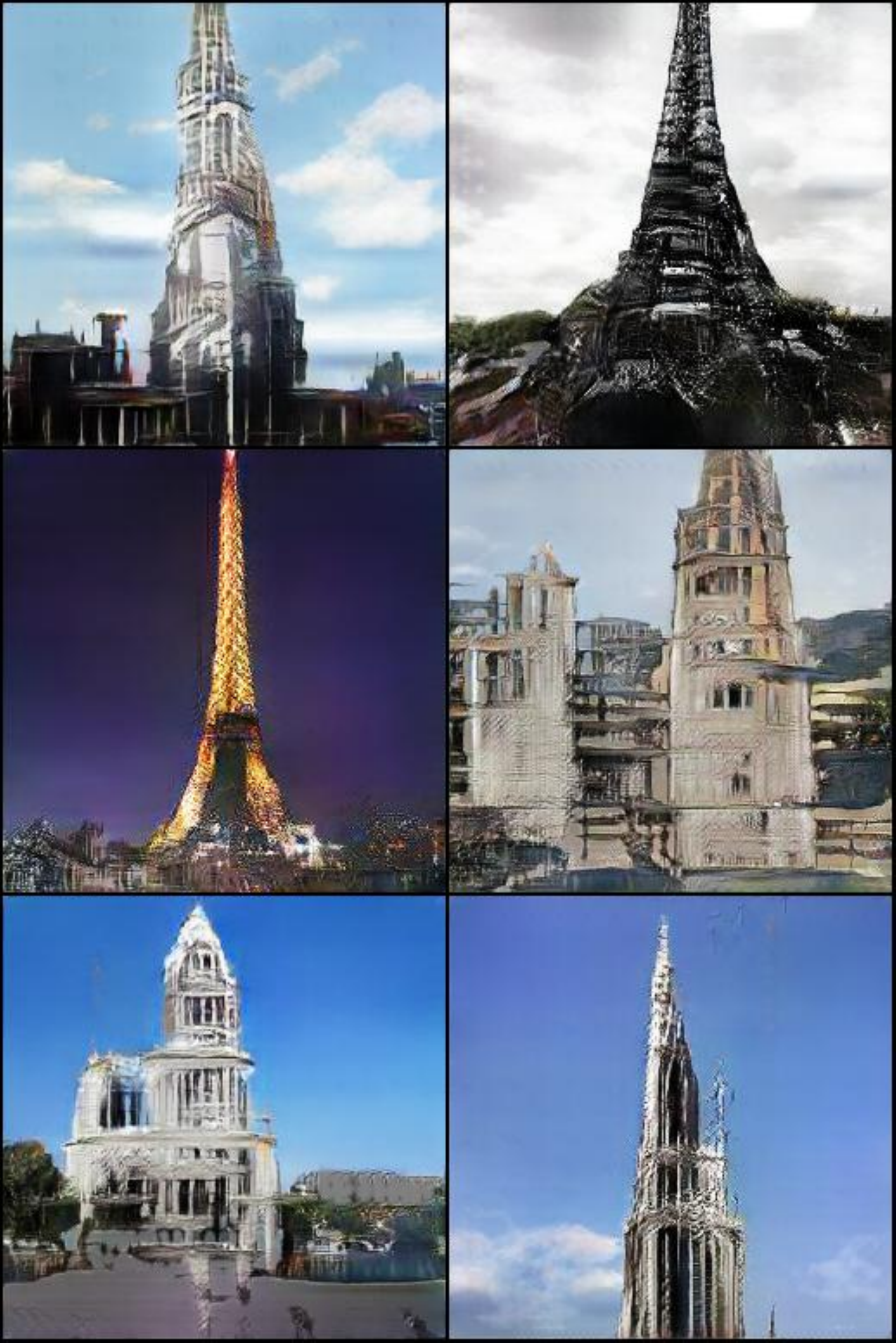}\\
(c) LSUN Church by RS-GAN & & (d) LSUN Tower by RS-GAN
		\end{tabularx}
	}
	\caption{Generated $256\times256$ Church and Tower Image by JS-GAN and RS-GAN.}
	\label{fig:lsunsample}
	
\end{figure}

\iffalse 
\input{append2p0_imbalancedata}
\input{append2_real_experiments}
\input{append2p2_HR_image}
\fi

\vspace{-0.3cm}

\section{Discussions on Empirical Loss and Population Loss (complements Sec.~\ref{sec: empirical and population}) }\label{Sec: Parameter Space}

As mentioned in Sec. \ref{sec: empirical and population}, the pdf space view (the population loss) was first used in \cite{goodfellow2014generative}, and became quite popular for GAN analysis. See, e.g., 
\cite{nagarajan2017gradient,johnson2019framework,chu2019probability}. 
In this part, we provide more discussions on the relation of
empirical loss and population loss in GANs.

\iffalse 
The motivation for analyzing the empirical version is two-fold. 
First, as argued in the main paper,
the convexity due to optimization over the probability density
is a universal property and not strongly related to the loss function.
Second, it is a common practice in machine learning to separately analyze the population version and the empirical version,
as elaborated in Appendix \ref{appen: generalization} below.
Third, as we mentioned earlier, ``the $n$-point distribution is an estimation of the $n$-mode distribution,'' as elaborated in Appendix \ref{app: macro learning}.
\fi

\subsection{Particle space or probability space?}\label{app: particle v.s. prob.}

 Suppose $p_z = \mathcal{N}(0, I_{d_z})$ (or other distributions) is the  distribution of the latent variable $z$, and $Z = (z_1, \dots, z_n)$
 are the samples of latent variables. 
 During training, the parameter $w$ of the generator net $ G_w $ is
 moving, and, as a result, both the pdf $ p_g = G_w( p_z ) $ 
 and the particles $ y_j = G_w( z_j ) $ move accordingly. 
Therefore, GAN training can be viewed as either probability space optimization
or particle space optimization. 
 The two views (pdf space and particle space) are illustrated in Figure~\ref{fig1:fig_two_views}. 
 
 \iflonger 
 \begin{figure}[H]
	\vspace{-0.6cm}
	\begin{tabular}{cc}
		\includegraphics[width=0.3\linewidth, height = 1.3cm]{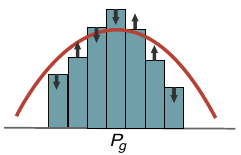}
		&
		\includegraphics[width=0.4\linewidth,height = 1.3cm]{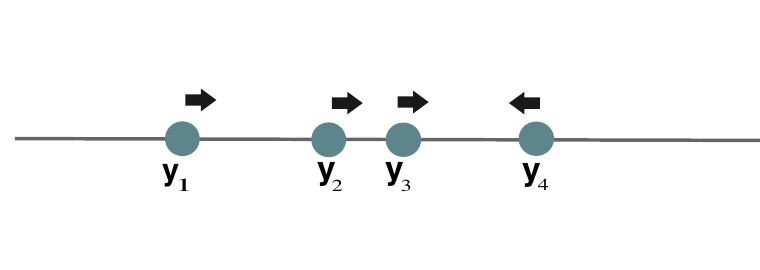}
		\\
		(a) & (b)
	\end{tabular}
	\caption{(a) Population version: probability density changes;  (b) Empirical version: samples are moving. }
	\label{fig:fig_two_views}
\end{figure}
\fi 

\iflonger 
To be more precise, there are three models: pdf $p_g$ moves freely, particles $Y$ move freely, and pdf moves according to the movement of $Y$.
\citet{goodfellow2014generative} adopt the first model, we adopt
 the second model and \cite{unterthiner2018coulomb} adopts the third model. 
\fi 
In the probability space view, an implicit assumption
is that the pdf $p_g$ moves \textit{freely};
in the particle space view, we assume the particles move \textit{freely}. 
Free-particle-movement implies free-pdf-movement if the particles almost occupy the whole space (a one-mode distribution), as shown in Fig.~\ref{fig:fig7_1mode}.
However, for multi-mode distributions in high-dimensional space,
 the particles are sparse in the space, and free-particle-movement does NOT
 imply free-pdf-movement. This gap was also pointed out in \cite{unterthiner2018coulomb};
 here, we stress that the gap becomes larger for sparser samples
 (eiher due to few samples ore high dimension). 
 This forms the foundation for experiments in App. \ref{sec: high resolution}.

To illustrate the gap between free-pdf-movement and free-particle-movement, 
we use an example of learning a two-mode distribution $\Pd$.
Suppose we start from an initial two-mode distribution $\Pg$, as shown
  Figure~\ref{fig:fig8_2mode}. 
 To learn $\Pd$, we need to do two things:
first, move the two modes of $\Pg $ to roughly overlap with the two modes of $\Pd$ which we call ``macro-learning''; second, adjust the distributions of each mode to match those of $\Pd$, which we call ``micro-learning.''
This decomposition is illustrated in  Fig. \ref{fig:fig8_2mode} and \ref{fig:fig9_decompose}.  In micro-learning, the pdf can move freely, but in macro-learning, the whole mode has to move together and cannot
 move freely in the pdf space.

\iffalse 
Therefore, we have explained the two reasons for analyzing the empirical version (the $n$-point distribution).
 First, it is the practically used version. Second, it captures the ``macro-learning'' behavior of learning a multi-mode distribution.
 \fi 

\iffalse 
Finally, the two types of learning are related to the two types of mode collapse mentioned by \citet{lin2018pacgan}: 
``entire modes from the input data are never generated, or the generator only  creates images within a subset of a particular mode.''
Failure of macro-learning can cause missing modes  in the generated distributions, which corresponds to the first type of mode collapse. 
 Failure of micro-learning can  cause a sub-mode of a mode to be missed, which also corresponds to the second type of mode collapse. See more discussions of related works on mode collapse in Appendix~\ref{sec: related works}. 
\fi 

	\begin{minipage}[t]{0.48\textwidth}
	\begin{figure}[H]
	\vspace{-0.8cm}
	\begin{center}
		\includegraphics[width=4cm, height = 1.6cm]{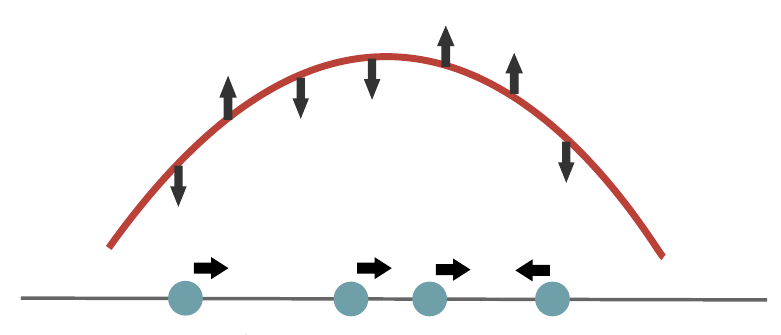}
		\captionsetup{font={scriptsize}}
		\caption{ {\scriptsize Illustration of  the learning process of the single mode. 
		The generated samples are moving, which  corresponds to  adjustment of the probability	densities. }
		}
		\label{fig:fig7_1mode}
	\end{center}
	\vspace{-0.0cm}
\end{figure}
	\end{minipage}
\hfill 
\begin{minipage}[t]{0.48\textwidth}
\begin{figure}[H]	
	\begin{center}
	\vspace{-0.8cm}
		\includegraphics[width= 4 cm, height = 1.6cm]{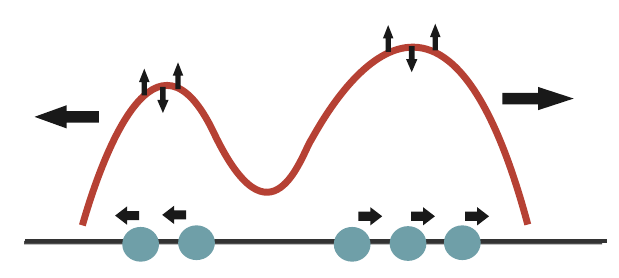}
	\captionsetup{font={scriptsize}}
		\caption{ {\scriptsize
 Illustration of  the process of learning a multi-mode distribution.  We decompose this process into two parts in the next figure. } }
		\label{fig:fig8_2mode}
	\end{center}
\end{figure}
	\end{minipage}

\begin{figure}[H]
	\vspace{-0.5cm}
	\begin{tabular}{ccc}
		\includegraphics[width=0.5\linewidth, height = 1.5cm]{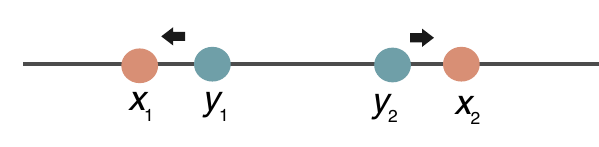}
		&
		\includegraphics[width=0.4\linewidth, height = 1.5 cm]{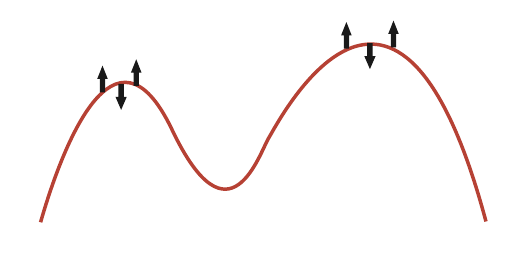}
		\\
	{\scriptsize	(a) Macro-learning } &
	{\scriptsize (b) Micro-learning }
	\end{tabular}
	\captionsetup{font={scriptsize}}
	\caption{ {\scriptsize Decomposing learning a multi-mode distribution 
	into macro-learning and  micro-learning.
	Macro-learning refers to the movement of the whole mode towards the underlying data mode. Micro-learning refers to the adjustment of the distribution within each mode. 
	If macro-learning fails, then an entire mode is missed in the generated distributions, which corresponds to mode collapse. 
	}
	}
	\label{fig:fig9_decompose}
\end{figure}

\iffalse 
GANs are implicit generative models, as opposed to  explicit models like 
energy-based models  \citep{lecun2006tutorial}  and restricted Boltzmann machines \citep{salakhutdinov2009deep}.
For implicit models, there is no explicit parameterization of the probability density.
The benefit of implicit models is that they are easier to train and sample from. During the training of GANs,  when the parameters of the mapping (i.e., generator) change,  it is the generated samples  that are moving, not that an explicit probability density is moving. 
\fi

  \iffalse 
  We emphasize again the discrepancy here:
   the probability space view says
   that there exists a certain energy function 
   that provides a energy-decreasing path from any initial $w_0$
   to an optimal $w^*$; however, 
   due to the limited feedback from real data,
   this path is hard to be recognized by the parameter-space 
   optimization. Instead, the particle space view can better
   match the practice faced by the parameter-space optimizers.
   \fi

\iflonger  
 In many GANs, a Gaussian distribution
 is used as an example.  An implicit assumption of those examples is that
 there are enough data to cover the whole space of interest,
 thus the movement of the parameter
  can be well approximated by the movement of the probability density.
  See Figure~\ref{fig:fig7_1mode} for an  illustration. 
 In the two-cluster experiments, in the second stage of training
 the fake data points are spread out, thus the probability
 space view becomes useful again. 
 Our goal is not to discard the probability space view;
 rather, we aim to provide an alternative view that
 is less studied in GAN reseach. 
 In the particle space view, the landscape of JS-GAN contains a bad attractor which is
  the mode collapse, which can better explain the phenomenon that
  the training loss gets stuck at around $0.48$, as shown
  in Section~\ref{sec: case study}.
 In addition, this view can  lead to the design of RpGAN, and
can explain and predict some benefits of RpGAN.
 \fi

 \subsection{Empirical loss and population loss}

 \iffalse 
 We remark that the empirical loss we use is just
an approximation of the practical training procedure. 
In practical GAN training, we are not performing batch gradient methods
on the objective we discuss; rather,
 we use ``stochastic'' gradient method.
 \fi 
 
The population version of RpGAN \cite{jolicoeur2018relativistic} is   $  
    \min_{\Pd  }  \phi_{\rm R, E} ( \Pg, \Pd ) , $ where
 \begin{equation}\label{phi R def, population}
  \phi_{\rm R, E} ( \Pg , \Pd )  = 
   \sup_{  f \in C( \mathbb{R}^d ) } 
   \Exp_{ (x, y) \sim  (\Pg, \Pd ) }    [  h (  f( x ) - f(y ) )  ]. 
 \end{equation}
Suppose we sample $ x_1, \dots, x_n \sim \Pd  $
and  $ y_1, \dots, y_n \sim \Pg $, 
then  $\frac{1}{ n }  \sum_{i = 1}^n    [  h (  f( x_i ) - f(y_i ) )  ] $
is an approximation of $ \Exp_{ (x, y) \sim  (\Pg, \Pd ) }    [  h (  f( x ) - f(y ) )  ]. $ 
The empirical version of RpGAN addresses   $  
    \min_{ Y \in \mathbb{R}^{d \times n}  }  \phi_{\rm R} ( Y, X ) ,
    $ 
    where
\begin{equation}\label{phi R def, 1st}
      \phi_{\rm R} ( Y, X )  = 
\sup_{  f \in C( \mathbb{R}^d ) } \frac{1}{ n }  \sum_{i = 1}^n    [  h (  f( x_i ) - f(y_i ) )  ]. 
\end{equation}

Our analysis is about the geometry of $  \phi_{\rm R} ( Y, X ) $ in Eq.~\eqref{phi R def, 1st}. In practical SGDA (stochastic GDA), 
at each iteration we draw a mini-batch 
of samples and update the parameters based on the mini-batch. 
The samples of true data $x_i$ are re-used
multiple times (similar to SGD for a finite-sum optimization), but the samples of latent variables $z_i$ are fresh (similar to on-line optimization). 
Due to the re-use of true data, stochastic GDA
 shall be viewed as an online optimization algorithm for solving Eq.~\eqref{phi R def, 1st} where $x_i$'s can be the same.
 Recall that in the main results, we have assumed that $x_i$'s are distinct,
 thus there is a gap between our results and practice.
 Extending our results to the case of non-distinct $x_i$'s requires extra work.
 This was done in Claim \ref{claim of JS GAN imbalanced} for the 2-cluster setting.  
 But for readability we do not further study this setting in the more general cases. 
 We leave this to future work.

\subsection{Generalization and overfitting of GAN}\label{appen: generalization}

One may wonder whether fitting the empirical distribution can cause memorization 
and failure to generate new data.
\citet{arora2017generalization} proved  that for many GANs (including JS-GAN)
with neural nets, only a polynomial number of samples are needed to achieve a small generalization error. We suspect that a similar generalization bound can be derived for RpGAN.

\iflonger 
Nevertheless, fitting training data is indeed what the practical GANs doing and not the invention of our paper. In practice, the direct objective of GANs is to fit the data (similar to neural nets in supervised learning), and as an indirect outcome GANs can generate new data.
Explaining why and when GANs can generalize is beyond the scope of this paper.
Nevertheless, for completeness, we briefly discuss the issue of generalization from the following perspectives:
first, the existing theory on the generalization  of GANs; second, some intuition on why GANs do not overfit.

First, we discuss the existing generalization bounds for GANs. 
Suppose  the true distribution  is $ \Pd  $  and the generated distribution is $ \Pg  $. 
Further, suppose we sample $n$ points $  x_i$ from $\Pg$ and sample $n $ points from $\Pd $, and use $ \mu $ and $\nu  $ to represent the two empirical distributions (discrete distributions with equal probability). \citet{arora2017generalization} proved  that for a large class of GAN problems (including JS-GAN using neural-net discriminator and generators), only a polynomial number of samples are needed to achieve a small generalization error. As a result of the generalization bound,   \citet{arora2017generalization} stated that if the GAN \textit{successfully minimized the   empirical distance} (i.e., the distance between $ \mu $ and $\nu  $), then the population distance (the distance between the two distributions $ \Pd  $  and $\Pg $) is also small. Therefore, there is a generalization guarantee under suitable conditions. We suspect that a similar generalization bound can be derived for RpGAN.
\fi

\iffalse 
Of course, there is much space in improving the generalization bound of \citet{arora2017generalization}, and that is an orthogonal line of research. 
Our goal of this paper is mainly to study how to ``\textit{successfully minimize the empirical distance}'' (i.e., fit the empirical distribution $\mu$). 
\fi 
 \iffalse 
     Technically speaking, the form of \citet{arora2017generalization} does not cover RSGAN, but
      the proof can be easily extended to RSGAN. We briefly explain the proof idea below. 
   The proof of  \citet{arora2017generalization} is based
     on a standard covering argument which uses a $\epsilon$-net to cover
     the parameter space, and then bound the generalization error based on the covering error.
     This proof does not rely too much  on the specific form of the loss function. Therefore, it applies to a large set of loss functions including RSGAN. 
     \fi

\begin{wrapfigure}{r}{0.28\textwidth}
\vspace{-0.5cm}
 \includegraphics[width=1\linewidth, height=2cm]{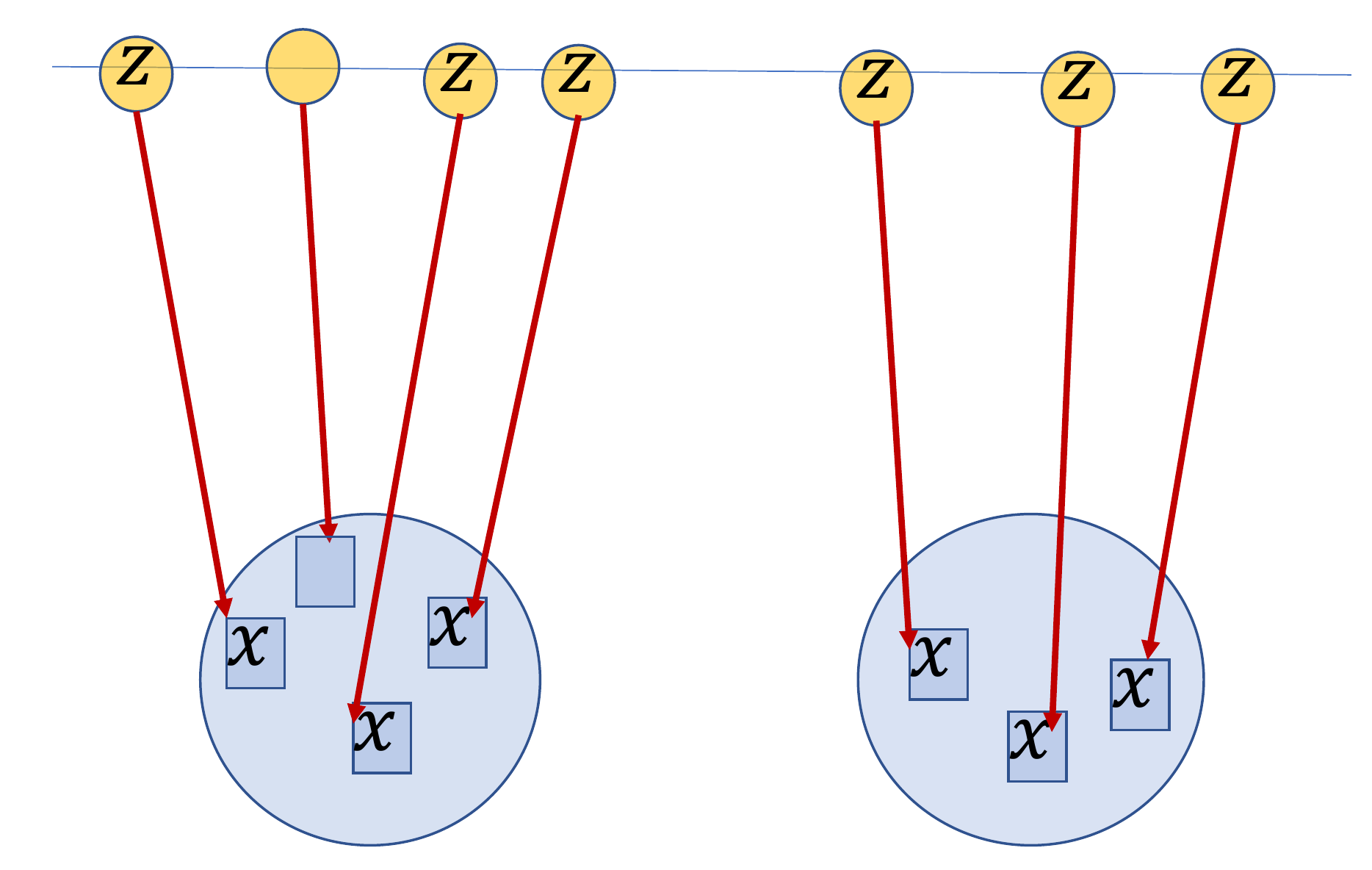}
 \captionsetup{font={scriptsize}} 
      \caption{How to generate new point.}
\label{fig:fig9_generalization}
 \vspace{-0.5cm}
\end{wrapfigure}
We provide some intuition why fitting the empirical data distribution via a GAN
may avoid overfitting. Consider  learning a two-cluster distribution as shown in Fig.~\ref{fig:fig9_generalization}.
During training, we learn a generator that maps the latent
samples $z_i$ to $x_i$, thus fitting the empirical distribution.
If we sample a new latent sample  $z_i$, then the generator will map $z_j$ to a new point $x_j$ in the underlying data distribution (due to the 
continuity of the generator function).
Thus the continuity of the generator (or the restricted power
of the generator) provides regularization for achieving generalization.

   \iffalse 
Third, instead of studying directly how to learn the true distribution, we
can decompose the problem into two parts: learning the empirical distribution (optimization part) and generalization to the true distribution (generalization part).
Generalization theory for GANs has been studied earlier (e.g.,~\citet{arora2017generalization}), while we focus on learning of the empirical distribution in this work. 
This decomposition resembles the decomposition in classical machine learning theory.
Instead of studying directly how to learn a classifier on certain data distribution, the problem is decomposed into two parts: learning the classifier on the empirical distribution (optimization part) and generalization to the true distribution (generalization part).
 It is very common to focus on one aspect in a paper; for instance,
 generalization theory papers (e.g.  \cite {bartlett2017spectrally}) often assume that the global minima  of the empirical optimization problem has been found, and
 optimization theory papers only study how to find the global minima
 (see Appendix \ref{sec: related works} which discussed many recent works
 on the pure optimization part of training deep neural nets).
One could criticize an optimization paper that it ignores the generalization theory
 and similarly could criticize a generalization paper that it ignores the convergence 
 analysis, but due to the complication of the problem it is often
 difficult to cover both sides (if so, it often comes with a price such
 as a much simpler model).
   \fi

\section{Proofs for Section \ref{sec: intuition and toy results} (2-Point Case)
and Appendix \ref{sec: imbalance} (2-Cluster Case)  }
\label{appen: toy proof}
We now provide the proofs for the toy results
(i.e., the case $n = 2$).

\subsection{Proof of  Claim \ref{2-dim GAN all values} and Corollary \ref{coro of bad basin in GAN} (for JS-GAN)}\label{app sub: proof of Claim 1 JS GAN 2 point}

 \textbf{Proof of  Claim \ref{2-dim GAN all values}}: 
We will  compute values of $ \phi_{\rm JS}(Y, X) $ for all $Y$.
Recall $ D $ can be any continuous  function with range $ ( 0 ,   1 ) .  $ 
Recall that
$ 
\phi_{\rm JS}(  Y , X )  
   =   \sup_{  D  }   \frac{1}{ 2 n } \left[    \sum_{i = 1}^n \log ( D( x_i ) )  +   \sum_{i = 1}^n \log ( 1 -  D( y_i  ) ) \right]. 
   $ 
Consider four cases. 
Denote a multiset $ \mathcal{Y} = \{ y_1, y_2 \} $,
and let $ m_i = |\mathcal{Y} \cap \{ x_i \}  | , i \in\{ 1, 2\}. $

\textbf{Case 1} (state 1): %
$ m_1 = m_2 = 1 $. 
Then the objective is 
{\equationsizeReg
\begin{align*}
&\sup_{  D }  \frac{1}{2} \left[  \frac{1}{2} \log (D(x_1) ) +  \frac{1}{2} \log( 1 - D( x_1 ) ) + \frac{1}{2} \log (D(x_2 ) ) +  \frac{1}{2} \log( 1 - D( x_2 ) )  \right].
\end{align*}
}
The optimal value is  $ -  \log 2 $,  which is achieved when $ D( x_1 ) = D(x_2 ) = \frac{1}{2 } $.

\textbf{Case 2} (state 1a): $ \{ m_1, m_2 \} = \{ 0, 1\}. $
WLOG, assume $m_1 = 1, m_2 = 0$, and $y_1 = x_1 ,  y_2 \notin \{ x_1, x_2  \} . $ 
The objective becomes
{\equationsizeReg
\[ 
\sup_{  D } \frac{1}{2} \left[ \frac{1}{2} \log (D(x_1) )  + \frac{1}{2} \log (D(x_2 ) ) +  \frac{1}{2} \log( 1 - D( x_1 ) )+  \frac{1}{2} \log( 1 - D( y_2 ) )
\right] .
\] 
}
The optimal value $ - \log 2/2  $ is achieved when 
$ D(x_1 ) =  1/2 $, $ D(x_2)  \rightarrow 1  $ and $  D(y_2) \rightarrow 0 $.

\textbf{Case 3} (state 1b): $ \{ m_1, m_2 \} = \{ 0, 2 \}. $
WLOG, assume $ y_1 = y_2 = x_1.  $ The objective becomes
{\equationsizeReg
\[ 
\sup_{  D } \frac{1}{2} \left[ \frac{1}{2} \log (D(x_1) ) +   \log( 1 - D( x_1 ) ) + \frac{1}{2} \log (D(x_2 ) )  \right]  .
\] 
}
The optimal value  $  \frac{1}{4} \log \frac{1}{3} + \frac{1}{2} \log \frac{2}{ 3 }  \approx -0.4774  $   is achieved
when $ D(x_1) = 1/3 $ and $ D(x_2 ) \rightarrow  1 $.

\textbf{Case 4} (state 2):  $ m_1 = m_2 = 0 ,$ i.e., 
$ y_1, y_2 \notin \{ x_1, x_2 \}  $.
The objective is:
{\equationsizeReg
\[ 
\sup_{  D } \frac{1}{2} \left[  \frac{1}{2} \log (D(x_1) )  + \frac{1}{2} \log (D(x_2 ) ) +  \frac{1}{2} \log( 1 - D( y_1 ) )+  \frac{1}{2} \log( 1 - D( y_2 ) )  \right] .
\] 
}
These terms are independent, thus each term can achieve its supreme $ \log 1 = 0 $. Then the optimal value $ 0 $ is achieved when 
$ D(x_1 ) = D(x_2) \rightarrow 1 $ and $ D(y_1) = D(y_2) \rightarrow 0 $.

\iffalse 
We re-state the corollary below.
\begin{coro}
	Suppose   $ \bar{Y} = (   \bar{y}_1,   \bar{y}_2   ) $  satisfies  
	$  \bar{y}_1=   \bar{y}_2  = x_1 $, then it is a sub-optimal strict  local minimum of the problem. 
\end{coro}
\fi 

\textbf{Proof of  Corollary \ref{coro of bad basin in GAN}:} 
Suppose $ \epsilon $ is the minimal non-zero distance between two points of $x_1, x_2,  y_1, y_2 . $
Consider a small perturbation of $ \bar{Y} $ as $Y =  ( \bar{y_1} + \epsilon_1 ,  \bar{y_2} + \epsilon_2       ) $,
where $  | \epsilon_i  | <  \epsilon    $.
We want to verify that
\begin{equation}\label{perturbed smaller, 1st}
\phi( \bar{Y} , X )   >    \phi( Y  , X )   \approx  -0.48. 
\end{equation}

There are two possibilities. 
\textbf{Possibility 1}:  $\epsilon_1 = 0$ or $\epsilon_2  = 0 $. WLOG, assume $\epsilon_1 = 0$, then we must have
$ \epsilon_2 > 0$.  Then we still have $ y_1 = \bar{y}_1 = x_1$.
Since the perturbation amount is small enough, we have $ y_2 \notin \{ x_1 , x_2 \} $.
According to Case 2 above, we have 
$
\phi( \bar{Y} , X )   = - \log 2 \approx -0.35 >  -0.48. 
 $ 
\textbf{Possibility 2}:   $\epsilon_1 >  0, \epsilon_2 > 0$. 
Since the perturbation amount $\epsilon_1$ and $\epsilon_2$  are small enough, 
we have   $   y_1 \notin \{ x_1 , x_2 \},    y_2 \notin \{ x_1 , x_2 \} $. 
According to Case 4 above, we have 
$
\phi( \bar{Y} , X )   = 0  >  -0.48.  
$ 
Combining both cases, we have  proved Eq.~\eqref{perturbed smaller, 1st}.  \hfill$\Box$

\subsection{Proof of Claim \ref{claim: 2-dim RS-GAN all values} (for RS-GAN) }
This is the result of RS-GAN for $n = 2 $.
WLOG, assume $ x_1 = 0, x_2 = 1 $. 
Denote $ g_{\rm RS}(Y) \triangleq  \phi_{\rm RS}(Y, X) =  \sup_{  f  \in C(\dR^d )  } \frac{1}{2} \log  \frac{1}{   1 +   \exp( f (  0 )  -  f  ( y_1 )   )   }   + \frac{1}{2} \log  \frac{1}{   1 +  
	\exp(  f (  1 )  -  f  ( y_2 )    )    }  .
$ 
Denote $ m_i = | \{ y_i \} \cap \{ x_i \}  | , i = 1, 2 $; note this definition
is different from JS-GAN in App. \ref{app sub: proof of Claim 1 JS GAN 2 point}. 
Consider three cases. 

\textbf{Case 1}: $ m_1 = m_2 = 1 $. 
If $ y_1 = 0 ,  y_2 = 1$, 
then
$ 
g_{\rm RS}(Y)  =  \frac{1}{2}  [ \log 0.5 + \log 0.5 ] = - \log 2 \approx -0.6937. 
$ 
If $ y_1 = 1 ,  y_2 =  0 $,  then 
{\equationsizeReg
\begin{align*}
g_{\rm RS}(Y)   &= \sup_{  f  \in \mathcal{F}  }    \frac{1}{2} \log  \frac{1}{   1 +   \exp( f (  0 )  -  f  ( 1  )   )   }   + 
\frac{1}{2} \log  \frac{1}{   1 +  
	\exp(  f (  1 )  -  f  ( 0  )    )    }           \\
& = \sup_{  t \in \mathbb{R} }  
\left[  \frac{1}{2} \log  \frac{1}{   1 +   \exp(  t  )   }  +    \frac{1}{2} \log  \frac{1}{   1 +   \exp(  - t  )   } \right]  =   -\log 2  .  
\end{align*}
} 

\textbf{Case 2}:  
$ \{ m_1, m_2 \} = \{ 0, 1\}. $
WLOG, assume $y_1 = 0 ,  y_2  \neq 1 $  (note that $y_2 $ can be $0$). 
Then 
{\equationsizeReg
\begin{align*}
g_{\rm RS}(Y) 
&\geq    \sup_{f \in \mathcal{F}     }
\frac{1}{2}  \log   \frac{1}{   1 +   \exp( f( 0 ) - f( 0 )   )   }   + \frac{1}{2} \log  \frac{1}{   1 +   \exp(   f( 1 ) - f( y_2 )   )    }      \\ 
	&=     -	\frac{1}{2} \log 2   +     \sup_{ t \in \mathbb{R}   }       \frac{1}{2} \log  \frac{1}{   1 +   \exp(   t   )    }  =   -	\frac{1}{2} \log 2   \approx  -0.3466. 
\end{align*}
The  value is achieved when $ f( 1 ) - f( y_2 )  \rightarrow   -  \infty   $.
}

\textbf{Case 3}:
$ m_1 = m_2 = 0 . $
Then
{\equationsizeReg
\begin{align*}
g_{\rm RS}(Y) 
  \geq &   \sup_{f \in \mathcal{F}      }
\frac{1}{2} \log   \frac{1}{   1 +   \exp( f( 0 ) - f( y_1 )   )   }   + \frac{1}{2} \log  \frac{1}{   1 +   \exp(   f( 1 ) - f( y_2 )   )    }      \\ 
	=  &  \sup_{ t_1 \in \mathbb{R},  t_2 \in \mathbb{R}  }   
\frac{1}{2} \log  \frac{1}{   1 +   \exp(  t_1  )    }  +  \frac{1}{2} \log  \frac{1}{   1 +   \exp(    t_2  )    }  = 0  . 
\end{align*}
}
The  value is achieved when $    f( 1 ) -   f(  y_2  )  \rightarrow  -  \infty  $ and $   f(0 )  - f(y_2)  \rightarrow  -  \infty .  $

The global minimal value is $- \log 2$, and the only global minima are $ \{ y_1, y_2 \} = \{ x_1, x_2  \}$.
In addition, from any $Y$, it is easy to verify that there is a non-decreasing path from $Y$ to a global minimum.

\subsection{Proofs for 2-Cluster Data (Possibly Imbalanced) }\label{app-sub: proof of imbalanced}

\textbf{Proof of Claim \ref{claim of JS GAN imbalanced}. }
The proof is built on the proof of Claim \ref{2-dim GAN all values} in Appendix
 \ref{app sub: proof of Claim 1 JS GAN 2 point}. 

We first consider a special case
$ |  \mathcal{X}_1 \!\cap\! Y |\!\!=\!\! m , | \mathcal{X}_2 \!\cap\!  \mathcal{Y} |\!\!=\!\! 0  $. 
This means that $m$ generated points are in mode 1, and the rest are in neither modes.
The loss value can be computed as follows: 
{\equationsizeReg
\begin{align*}
 \phi_{\rm JS}(Y, X) = & \frac{1}{2n} \left[ \alpha n \log (\frac{\alpha n}{\alpha n+m}) +  m \log( 1 - \frac{\alpha n}{\alpha n+m} ) \right]\\
=&  \frac{\alpha}{2}\log (\alpha n) + \frac{m}{2n}\log m - \frac{\alpha n+m}{2n}\log (\alpha n+m) )  = q_{\alpha}(m).
\end{align*}
}
In general, if $ | \mathcal{X}_1 \!\cap\!  \mathcal{Y}  |\!\!=\!\! m_1  , | 
\mathcal{X}_2 \!\cap\! \mathcal{Y} |\!\!=\!\! m_2  $,
then 
$ \phi_{\rm JS}(Y, X)$  can be divided into three parts: 
 the first part is the sum of the terms that contain $x_1$ (including $x_i$'s and $y_j$'s
 that are equal to $x_1$),
 the second part is the sum of the terms that contain $x_n$ (including
 $x_i$'s and $y_j$'s that are equal to $x_n$), and the third
 part is the sum of the terms that contain $y_j$'s that are not in $\{ x_1, x_n\}$.
 Similar to Case 3 above, the value of the first part is
 $ q_{\alpha}(m_1) $, and the value of the second part is $ q_{1 - \alpha}(m_2) $.
 Similar to the above special case, the value of the third part is $ 0 $.
 Therefore, the loss value is
$  \phi_{\rm JS}(Y, X)  = q_{\alpha}(m_1) + q_{1 - \alpha}(m_2) .  $

It is easy to show that $ q_{\alpha}(m_1) + q_{1 - \alpha}(m_2) \geq - \log 2 $,
and the equality is achieved iff $ m_1 = n \alpha, m_2 = n (1 - \alpha )$,
i.e.,  $\mathcal{Y} = \mathcal{X}_1 \cup  \mathcal{X}_2 $.
 $ \Box  $

\textbf{Proof sketch of Corollary \ref{coro of strict local min imbalanced}.} 
After a small enough perturbation, we must have 
$   m_1 \triangleq  | \mathcal{X}_2 \!\cap\! \mathcal{Y} | \leq  n_1 
,  m_2 \triangleq  | \mathcal{X}_1 \!\cap\!  \mathcal{Y}  | \leq  n_2   . $ 
Since $q_{\alpha}(m)$ and  $q_{1 - \alpha}(m)$
are strictly decreasing functions of $m$, we have 
\[ \phi(Y, X) = q_{\alpha}(m_1 ) +  q_{1 - \alpha}(m_2 )
\leq q_{\alpha}( n_1 ) +  q_{1 - \alpha}( n_2 )=  \phi( \hat{Y} , X) .  \]
The equality holds iff $ (m_1, m_2) = (n_1, n_2) $, i.e., $ Y =  \hat{Y}. $
This means that if $ ( n_1, n_2) \neq ( n \alpha, n (1 - \alpha) ), $
then $ \hat{Y} $ is a sub-optimal strict local minimum.  $\Box $

We skip the detailed proof, since other parts are  similar to
the proof of Corollary~\ref{coro of bad basin in GAN}. 

\iflonger 
We need to prove the following: there exists $\epsilon > 0 $
such that for any $ \| Y - \hat{Y} \|_F \leq \epsilon $, we have
$ \phi(Y, X) >  \phi( \hat{Y}, X) .  $ 
In other words, the loss value at $\hat{Y} $ is strictly smaller
than the loss values of $Y$ in a small neighborhood of $\hat{Y}$.

The assumption $n_1 + n_2 = n$ means that each $\hat{y}_j $ is either identical to $x_1$ or $x_2 $. As a result, a small enough perturbation of $ \hat{Y} $ will not move
a point in one mode to another.
More formally, if we pick any $ \epsilon < \| x_1 - x_n \| ,$ then 
for any $ \| Y - \hat{Y} \|_F \leq \epsilon $, 
we have $  y_j = \hat{y}_j$ or $ y_j \notin \{ x_1, x_n \} ,$ $  ~ \forall ~ j$ . 
For such $Y$, we must have 
$   m_1 \triangleq  | \mathcal{X}_2 \!\cap\! \mathcal{Y} | \leq  n_1 
,  m_2 \triangleq  | \mathcal{X}_1 \!\cap\!  \mathcal{Y}  | \leq  n_2   . $ 
Define
$ \eta(t) =  t \log t -  ( \alpha n + t  ) \log( \alpha n + t) $,
then  $ \eta'(t) = \log t + 1 - \log( \alpha n + t )  - 1 = 
 \log t  - \log( \alpha n + t ) < 0 $.
 This further implies $q_{\alpha}(m)$ is strictly decreasing function of $m$.
 Similarly,  $q_{1 - \alpha}(m)$ is strictly decreasing function of $m$.

According to Claim \ref{claim of JS GAN imbalanced}, 
we have 
\[ \phi(Y, X) = q_{\alpha}(m_1 ) +  q_{1 - \alpha}(m_2 )
\leq q_{\alpha}( n_1 ) +  q_{1 - \alpha}( n_2 )=  \phi( \hat{Y} , X) .  \]
The equality holds iff $ (m_1, m_2) = (n_1, n_2) $, i.e., $ Y =  \hat{Y}. $
 Therefore,  $ \hat{Y} $ is a strict local minimum. 
 
 It is easy to show $ q_{\alpha}( n_1 ) +  q_{1 - \alpha}( n - n_1 ) 
  >  q_{\alpha}( n \alpha  ) +  q_{1 - \alpha}( n - n \alpha )   $
 if $  n_1 \neq  n \alpha $. 
This means that if $ ( n_1, n_2) \neq ( n \alpha, n (1 - \alpha) ), $
then $ \hat{Y} $ is a sub-optimal strict local minimum.
$ \Box $
\fi

\iffalse 
\textbf{Case 3 Proof}: if $|  X_1 ~ \cap ~ Y |\!\!=\!\! 0  , |  X_2 \!\cap\! Y |\!\!=\!\! m,$ then $m$ of the generated points are in mode $X_2$, the rest are in neither modes.
\begin{align*}
& \frac{1}{2n} \left[ (1-\alpha) n \log (\frac{(1-\alpha)n}{(1-\alpha)n+m}) +  m \log( 1 - \frac{(1-\alpha)n}{(1-\alpha)n+m} ) \right]\\
=& \frac{1}{2n} \left[ (1-\alpha)n\log((1-\alpha)n) - (1-\alpha)n \log((1-\alpha)n + m) + m\log m - m \log((1-\alpha)n + m) \right]\\
=&  \frac{1-\alpha}{2}\log((1-\alpha)n) + \frac{m}{2n}\log m - \frac{(1-\alpha)n+m}{2n}\log((1-\alpha)n+m) .
\end{align*}
\fi 

\iffalse 
Remark: In the above corollary, we do not consider all possible points in the landscape.
  In particular, we do not consider the case $|  X_1 \!\cap\! Y |\!\! >0   , |  X_2 \!\cap\! Y |\!\!  > 0 $ but $  \mathcal{Y} \neq  \mathcal{X}_1 \cup  \mathcal{X}_2  $.
\fi

\section{Proof of Theorem \ref{prop: GAN all values, extension} (Landscape of Separable-GAN)}
\label{appen: proof of Thm 1}

Denote  $ F( D; Y )  =  \frac{1}{2n } 
       \sum_{i = 1}^n    [  h_1 (  f( x_i ) ) + h_2 ( - f(y_i ) )  ]  \leq 0 $
(since $ h_i(t ) \leq 0 , i=1, 2 $ for any $ t $).

\textbf{Step 1:} \textbf{Compute the value of $  \phi(\cdot,  X  )   $   for each $Y$.}
For any $i $, 
denote 
$    M_i =  \{  j:   y_j = x_i   \}   ,    m_i = | M_i | \geq 0,   i = 1, 2, \dots, n.   $
Then $ m_1 + \dots + m_n = n $.
Denote $ \Omega  =  M_1 \cup M_2 \dots \cup M_n  $.
 Then 
 {\equationsize 
 \begin{subequations}
\begin{align}
\phi (  Y , X )  = 
&    \frac{1}{2n }  \sup_{ f }  \sum_{i = 1}^n    [  h_1 (  f( x_i ) ) + h_2 ( - f(y_i ) )  ]    \notag     
=    \frac{1}{2 n }   \sup_{  f }   \left(  \sum_{ i = 1}^n   [ h_1 (  f_1 ( x_i ) )  +   m_i  h_2 ( - f(x_i ) ) ] 
  + \sum_{  j  \notin  \Omega  }   h_2 ( - f(y_i ) )      \right)      \notag       \\
\overset{  (i)   }{  =  }   &     \frac{1}{ 2 n }
   \left(   \sum_{ i = 1}^n    \sup_{  t_i \in \mathbb{R}  }  [  h_1 (  t_i  )  +   m_i  h_2 ( - t_i ) ]  + 
      |\Omega^c|   \sup_{  t \in \mathbb{R}  }   h_2(t) 
     \right) 
\overset{  (ii)   }{  =    }         \frac{1}{ 2 n }  \sum_{i = 1}^n  \xi(m_i )    \label{eq: App-L: expression of phi}    \\
\overset{  (iii)   }{  \geq   }  &    \frac{1}{ 2 n }  \sum_{i = 1}^n  m_i \xi( 1 ) = \frac{1}{2} \xi(1). \notag 
\end{align}
\end{subequations}
} 
Here (i) is because  $ f(y_j) ,   j \in \Omega   $  are independent of 
$ h(x_i)$'s and thus can be any values; (ii) is by the definition $ \xi(m)
  = \sup_t [  h_1(t) + m h_2( -t )  ] $
 and Assumption \ref{assumption 1}
  that $ \sup_t h_2(t) = 0$;
(iii) is due to the convexity of $ \xi  $ (note that $\xi$ is the supreme of linear functions). 
Furthermore, if there is a certain $ m_i > 1 $, then $ \xi(m_i )  + (m_i - 1) \xi(0) =\xi(m_i ) > m_i \xi(1) $ (according
to Assumption \ref{assumption 2}), causing
 (iii) to become a strict inequality.
Thus the equality in (iii) holds iff
 $m_i = 1,  ~\forall i $, i.e., $ \{ y_1, \dots, y_n\} = 
 \{ x_1, \dots, x_n\} . $
 Therefore, we have proved that  $  \phi(  Y , X ) $ achieves the minimal value $  \frac{1}{ 2  }   \xi( 1 ) $ iff $ \{ y_1, \dots, y_n\} = 
 \{ x_1, \dots, x_n\} . $

\textbf{Step 2:} \textbf{Sufficient condition for strict local-min. }
Next, we show that if $ Y $ satisfies $  m_1 + m_2 + \dots + m_n = n $    
then $  Y  $ is a strict local-min. 
Denote 
$ 
\delta   =  \min_{ k \neq l  }   \|  x_k - x_l  \|. 
$
Consider a small perturbation of $ Y  $ as $ \bar{Y} =
( \bar{y_1}  ,  \bar{y_2}  , \dots,   \bar{y_n }      )   =  (  y_1  + \epsilon_1 ,  
y_2  + \epsilon_2 ,   \dots,   y_n  + \epsilon_n     ) $, 
where $  \| \epsilon_j  \| <  \delta   , \forall j  $ and 
$ \sum_j  \| \epsilon_j \|^2 > 0 . $
We want to prove 
$ \phi_{\rm }( \bar{Y} , X )   >    \phi_{\rm }( Y  , X )  .  $ 

Denote 
$
\bar{m}_i    =   | \{  j:   \bar{y}_j = x_i   \} |,     i  = 1, 2, \dots, n.
$
Consider an arbitrary $j $. Since $ y_j \in \{ x_1, \dots, x_n\} $, there must
be some $i$ such that $y_j = x_i $. Together with 
$   \|  \bar{y}_j  - y_j  \|  =  \|  \epsilon_j  \|    < \delta   =  \min_{ k \neq l  }   \|  x_k - x_l  \|  $,
we have $ \bar{y}_j  \notin  (   \{ x_1, x_2, \dots, x_n  \}  \backslash \{ x_i  \} ). $
In other words, the only possible point in $ \{ x_1, \dots, x_n \} $ that can coincide with 
$  \bar{y}_j $ is $ x_i  $, and this happens only when $ \epsilon_j  = 0 $.
This implies $\bar{m}_i   \leq m_i , \forall i $. %
Since we have assumed $ \sum_j  \| \epsilon_j \|^2 > 0  $, for at least
one $i$ we have $ \bar{m}_i < m_i. $
Together with Assumption \ref{assumption 3} that $ \xi( m ) $ is 
a strictly decreasing function in $ m \in [ 0, n ] $,  we have
$ \phi_{\rm }( \bar{Y} , X ) =   \frac{1}{ n }  \sum_{i = 1}^n   
\xi( \bar{m}_i )   
>  \frac{1}{ n }  \sum_{i = 1}^n  \xi( m_i ) 
=     \phi_{\rm }(  Y  , X  ).  
$

\textbf{Step 3}: \textbf{Sub-optimal strict local-min.}
Finally, if $ Y $ satisfies that  $  m_1 + m_2 + \dots + m_n = n $   and  $ m_k  \geq 2  $ for some $ k $,  then  $\phi_{\rm }(  Y , X )  >   \frac{1}{2}\xi(0) . $
Thus $Y $ is a sub-optimal strict local minimum.
\textbf{Q.E.D.}

\textbf{Remark 1}:  $\xi( m ) $ is convex (it is the supreme
of linear functions),  thus we always have 
$ \xi(m) =  \xi( m ) + (m - 1) \xi(0 ) \geq m \xi(1) $. 
Assump. \ref{assumption 2} states that the inequality is strict, thus
it is slightly stronger than the convexity of $\xi$.  
By Assump. \ref{assumption 1}, 
we also have $ h_1(t) + (m+1 ) h_2( - t)  \leq  h_1(t) + m h_2( - t)   $,
thus $ \xi(n ) \leq  \xi( n -1 ) \leq \dots \leq \xi(0) $.
Assumption \ref{assumption 3} states that the inequalities are strict. 
This holds if the maximizer of $ h_1(t) + m  h_2( - t)  $
does not coincide with the maximizer of $h_2(t)$.
Intuitively, if $h(t)$ is ``substantially different'' from a constant function, 
then Assump. \ref{assumption 2}  and Assump. \ref{assumption 3} hold. 

\textbf{Remark 2}: The upper bound $0$ in Assumption~\ref{assumption 1} is not essential, and can be relaxed to any
finite numbers (change other two assumptions accordingly). We skip the details.

\iflonger 
More specifically, 
 Theorem \ref{prop: GAN all values, extension}
holds under the followings assumptions (we skip the proof): 
\begin{assumption}\label{assump gen 0}
 $ \tau_1 \triangleq   \sup_t h_1(t) < \infty  ,  \tau_2 \triangleq   \sup_t h_2(t) < \infty. $
\end{assumption} 
\begin{assumption}\label{assump gen 1}
$ \xi(m) + ( m - 1 ) \tau_1  >  m \xi(1) ,   \forall ~  m  > 1,$
where $ \xi(m) =\sup [ h_1(t) +  m h_2(t)   ].  $ 
  \end{assumption} 
\begin{assumption}\label{assump gen 2}
$ \xi(m)    <    \xi(m - 1 ) +   \tau_2 ,   \forall ~  m  \geq 1. $ 
\end{assumption} 
\fi

\section{ Proof of Theorem \ref{prop: RS-GAN all values, extension} (Landscape of RpGAN)  }
\label{appen: proof of Thm 2}
This proof is the longest one in this paper. We will focus on a proof for the special case of RS-GAN. The proof for general RpGAN is quite similar, and presented in Appendix \ref{appen:
proof of Thm 4}. %
Recall
$ 
 \phi_{\rm RS}( Y, X )  =  \sup_{  f }   \frac{1}{ n }  \sum_{i = 1}^n \log \frac{1}{  1 +   \exp( f  ( y_i )   -  f (  x_i )    )       } .
$

\begin{thm2}\label{prop: special case of Thm 2}
(special case of Theorem \ref{prop: RS-GAN all values, extension} for RS-GAN)
	Suppose   $ x_1 , x_2, \dots, x_n  \in \mathbb{R}^d $ are distinct. 
	The global minimal value of $ \phi_{\rm RS}( Y, X ) $ 
	is $ -  \log 2 $,  which is achieved iff 
	$  \{ x_1, \dots, x_n  \} = \{  y_1, \dots, y_n  \} $.
	Furthermore, any point is global-min-reachable for the function. 
\end{thm2}
\iffalse 
\fi 

\textbf{Proof sketch.}
We compute the value of $g(Y)  = \phi_{\rm RS}( Y, X ) $ for any $Y$,
using the following steps: 

(i) We build a graph with vertices representing distinct values of $x_i, y_i$ 
and draw directed edges from $ x_i$ to $ y_i$. 
This graph can be decomposed into cycles and trees.

(ii) Each vertex in a cycle contributes $ - \frac{1}{n} \log 2 $ to the value $g(Y)$.

(iii) Each vertex in a tree contributes $ 0 $ to the value $g(Y)$.

(iv) The value $g(Y)$ equals $ - \frac{1}{n} \log 2 $ times the number of vertices in the cycles.

The outline of this section is as follows.
In the first subsection, we analyze an example as warm-up. 
Next,  we prove 
Theorem \ref{prop: special case of Thm 2}.
The proofs of some technical lemmas will be provided in the 
following subsections.  Finally, in Appendix \ref{appen:
proof of Thm 4} we present the proof for Theorem \ref{prop: RS-GAN all values, extension}.

\subsection{Warm-up Example}

We prove that if $ \{ y_1, y_2, \dots, y_n   \}  = \{  x_1, \dots, x_n  \}$, then $Y$ is a global minimum of   $  g(Y) $.

Suppose $ y_i  = x_{ \sigma(i) } $, where $ ( \sigma(1), \sigma(2),  \dots, \sigma(n) ) $ is a permutation of $ (1,2, \dots, n) $.
\iflonger 
We view $ \sigma$ as  a mapping from $ \{  1, 2, \dots, n   \} $ to $   \{  1, 2, \dots, n   \}  $.
Pick an arbitrary $ i  $, then in  the infinite sequence $ i ,   \sigma(i),  \sigma (\sigma(i)),  \sigma^{(3)}(i), \dots   $ there exists at least two  numbers that are the same.
Suppose $ \sigma^{(k_0 )}(i) = \sigma^{( k_0 + T )} (i) $ for some $k_0 , T$,
then since $\sigma$ is a one-to-one mapping we have $ i  = \sigma^{(  T )} (i) . $   
Then we obtain a cycle $  C =  (   i,  \sigma(i),   \sigma^{(2)}(i),  \dots,  \sigma^{(T-1)}( i )   ) . $
The permutation defines a bipartite graph (add a directed edge from $i$ to $n + \sigma(i)$,
and and edge from any $n + i$ to $i$), and the connected components
 of this graph are cycles. 
\fi 
We can divide $ \{ 1, 2, \dots, n   \} $ into finitely many cycles 
$  C_1,  C_2, \dots, C_K  $, where each cycle  $ C_k  =  (    c_k(1) ,  c_k(2),  \dots, c_k( m_k )   )  $ satisfies $  c_k( j + 1 ) = \sigma(  c_k (j) ) ,  j \in\{ 1, 2, \dots,  m_k \}    $. Here $ c_k( m_k +1 ) $ is defined as $ c_k (1) $. 
Now we calculate the value of $ g(Y) $.
{\equationsize 
\begin{align*}
g(Y)  &  =  \sup_{  f }   \frac{1}{ n }  \sum_{i = 1}^n \log \frac{1}{  1 +   \exp( f  ( y_i )   -  f (  x_i )    )     )  } 
\overset{ \text{(i)} }{= }  -   \inf_{  f }   \frac{1}{ n }  \sum_{ k = 1}^K      
\sum_{ i \in C_k  }    \log   \left(  1 +   \exp( f  ( y_i )   -  f (  x_i )    )       \right)    \\
&   =  - \inf_{  f }   \frac{1}{ n }  \sum_{ k = 1}^K      
\sum_{ j = 1  }^{ m_k   }    \log  \left(  1 + 
e^{ f  ( x_{  c_k( j + 1 ) } )   -  f (  x_{  c_k( j ) } ) ) }  \right)  
\overset{ \text{(ii)} }{= }    -   \frac{1}{ n }  \sum_{ k = 1}^K      
\inf_{  f }   \sum_{ j = 1  }^{ m_k  }    \log  \left(  1 + 
e^{ f  ( x_{  c_k( j + 1 ) } )   -  f (  x_{  c_k( j ) } )  }   \right)    \\
&   =    -   \frac{1}{ n }  \sum_{ k = 1}^K
\inf_{   t_1, t_2, \dots, t_{m_k} \in \mathbb{R} }  
\left[   \sum_{ j = 1  }^{ m_k - 1  }   \log    \left(  1 +   \exp( t_{j+1}  -    t_j     )  \right)  +  \log   \left(     1 +   \exp( t_{ 1 }  -    t_{m_k}    )  \right)  \right]     \\
&   \overset{ \text{(iii)} }{= }    -   \frac{1}{ n }  \sum_{ k = 1}^K      
m_k    \log (  1 +   \exp( 0    )  )   =     - \log 2. 
\end{align*}
}
Here (i) is because $ \{ 1,2, \dots, n  \} $ is the combination of $C_1, \dots, C_K $ and 
$  i \in   C_k $ means that $i  =  c_k(j) $ for some $j$. 
(ii)  is because  $C_k $'s are disjoint and $f$ can be any continuous function; more specifically, the choice of $ \{ f(x_i  ) :  i \in  C_{k }   \} $   is independent of the choice of   $ \{ f(x_i  ) :  i \in  C_{ l  }   \} $ for any $k \neq l $, thus we can take the infimum over each cycle (i.e., put ``inf'' inside the sum over $k$). 
(iii) is because 
$  \sum_{ j = 1  }^{ m - 1  }   \log  (  1 +   \exp( t_{j+1}  -    t_j     )  )  +  
\log     \left(     1 +   \exp( t_{ 1 }  -    t_{m }    )  \right)  $
is a convex function of $t_1, t_2, \dots, t_{m } $ and the minimum  is achieved 
at $ t_1 = t_2 = \dots = t_{m } = 0$.

\iflonger 
\textbf{Warm-up example 2: mode dropped.}
Suppose $   y_j     \in \{  x_1, \dots, x_n  \} , \forall  j   $, and there exist
some   $ x_{i_0 }  $ that is not equal to any $y _j $.  
We show that  $Y$ is not a global minimum. 
The computation will illustrate how a ``free''  variable
reduces the objective value $g(Y)$ by \textit{at least} $ - \frac{1}{ n } \log 2 $.

Consider the term $  \log (   1  +  \exp (  f( y_{i_0}  )  -    f( x_{i_0}  )     )     )  $. Since $ x_{i_0} $ does not appear in any other term in
$ \sum_{i}  \log  (   1  +  \exp (  f( y_{i }  )  -    f( x_{i}  )     )     ) $,
the choice of $ f( x_{i_0} ) $ is free.
Therefore, no matter what  values of $ f(x_1), \dots, f(x_{i_0 - 1}), f(x_{i_0 + 1}) , \dots, f(x_n) $ and $ f(y_1), \dots, f(y_n) $ are, 
we can always pick $ f(x_{i_0}) $ so that  $  f( y_{i_0}  )  -    f( x_{i_0} )  \rightarrow  - \infty $,
making the term $  \log (   1  +  \exp (  f( y_{i_0}  )  -    f( x_{i_0}  )     )     )  \rightarrow  0 .$ Thus
{\footnotesize 
\begin{align*}
g(Y)  &  
=  -   \inf_{  f }   \frac{1}{ n }  \sum_{ i = 1}^n 
\log   \left(  1 +   e^ {  f  ( y_i )   -  f (  x_i )    }     )  \right)   =  -   \inf_{  f }   \frac{1}{ n }  \sum_{ i \neq i_0 } 
\log   \left(  1 +    e^ {  f  ( y_i )   -  f (  x_i )    }    )  \right)   + 0    
\geq - \frac{n-1}{ n } \log 2. 
\end{align*}
} 
\fi

\subsection{Proof of 
Theorem \ref{prop: special case of Thm 2}  }
This proof is divided into three steps. 
In Step 1, we compute the value of $g(Y)$ if all $y_i \in \{ x_1, \dots, x_n \}$. This is the major step of the whole proof. 
In Step 2, we compute the value of $g(Y)$ for any $Y $. 
In Step 3, we show that  there is a non-decreasing continuous path from $Y$ to a global minimum. 

\textbf{Step 1: Compute $g(Y)$ that all $y_i \in \{ x_1, \dots, x_n \}$}.
Define 
\begin{equation}\label{assumption that y lies in X}
  R(X) = \{ Y :   y_i  \in  \{  x_1, \dots, x_n  \}, \forall i \}. 
\end{equation}
\textbf{Step 1.1: Build a graph and decompose it.}
We fix $Y \in R(X)$. 
We build a directed graph $G = (V, A )$ as follows.
The set of vertices $V = \{  1, 2, \dots, n   \} $ represent $x_1, x_2, \dots, x_n$.
A directed edge  $  ( i  , j ) \in A    $  if $ y_i = x_j  $. In this case, there is a term   $ \log (1  +  \exp (   f(x_j) - f(x_i)  )  ) $    in $g(Y)$. 
It  is possible to have a self-loop $ (i, i) $, which corresponds
to the case $ y_i = x_i $. 
By Eq.~\eqref{assumption that y lies in X}, we have
{\equationsizeReg
\begin{equation}\label{transform to graph sum}
\begin{split}
g(Y)   &=  -   \inf_{  f }   \frac{1}{ n }  \sum_{ i = 1}^n 
\log   \left(  1 +   e^{ f  ( y_i )   -  f (  x_i )    }      \right) 
=   -   \inf_{  f }   \frac{1}{ n }  \sum_{  (i, j) \in A  } 
\log   \left(  1 +   e^{  f  ( x_j  )   -  f (  x_i )    }    \right) . 
\end{split}
\end{equation}
}

Each $y_i $ corresponds to a unique $x_j$, thus the out-degree of $ i $, denoted as 
$ \text{outdegree}(i) $, must be exactly $1$.
The in-degree of each $i$, denoted as $\text{indegree}(i)$,  can be any number in $ \{ 0, 1, \dots, n  \}$.  

We will show that the graph $G$ can be decomposed into the union of
cycles and trees (see App.~\ref{appen: proof of lemma 1} for its proof, and definitions of cycles and trees). 
A graphical  illustration is given in Figure \ref{fig of graph decomposition2}. 

\begin{lemma}\label{lemma: decomposition into cycles and trees}
	Suppose $ G = (V, A)$ is a directed graph and $\text{outdegree}(v) = 1, \forall v \in V$.
	Then: 
	
	(a) There exist cycles $ C_1, C_2, \dots, C_K $ and subtrees $ T_1, T_2, \dots, T_M $
	such that each edge  $v \in A $ appears either in exactly one of the cycles or in exactly one of the subtrees. 
	
	(b) The root of each subtree $u_m$ is a vertex of a certain cycle $C_k$.
	In addition, each vertex of the graph appears in exactly one of the following sets:
	$ V(C_1), \dots, V(C_K), V( T_1)\backslash \{ u_1 \}, \dots, V( T_M )\backslash \{ u_M \} $.

	(c) There is at least one cycle in the graph.
\end{lemma}

\begin{figure}[H]
\vspace{-0.5cm}
	\centering
	\qquad
	\subfigure[Eg 1 for Lemma \ref{lemma: decomposition into cycles and trees}]{
		\label{fig3a2}
		\includegraphics[width=0.2\textwidth, height=2.5cm]{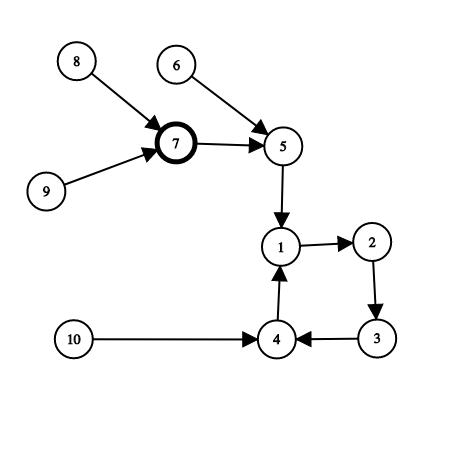}}
	\subfigure[Eg 2, with self-loop]{
		\label{fig3b2}
		\qquad\qquad
		\includegraphics[width=0.25 \textwidth, height= 2.5cm]{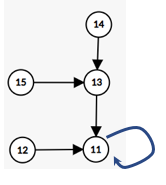}}\hfill 
	\subfigure[Example graph for general case]{ \label{another case}
		\includegraphics[width=0.3\textwidth, height= 2.5 cm]{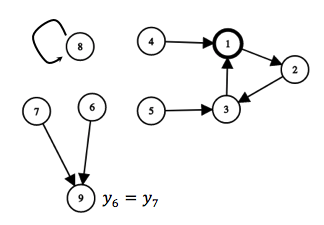}}\hfill
\captionsetup{font={scriptsize}}
	\caption{The first two figures are two connected component of a graph
	representing the case $y_i \in \{x_1,\dots, x_n \}, \forall i$.
	The first figure  contains $ 10$ vertices and $10$ directed edges.  It can be decomposed into a cycle $(1, 2, 3, 4)$ and two subtrees:
one subtree consists of edge $ (10, 4)$ and vertices $10 , 4$,
		and another consists of edges $ (8, 7), (9, 7), (7, 5), (6,5), (5, 1) $. 
	The second figure has one cycle being a self-loop, and two trees attached to it. 
		The third figure is an example graph of the  case
		that some $y_i \notin \{x_1,\dots, x_n \}$.
	In this example, $n = 8$ (so $8$ edges), and all $y_i$'s are in	$ \{x_1,\dots, x_n \} $ except $  y_6 , y_7 $.
	The two edges $ ( 6 , 9 ) $ and $(6, 9 )$
indicate the two terms $ h( f( y_6 ) - f(x_6) ) $ and $ h( f(y_7) - f(x_7) ) $ in $g(Y)$.  They have the same head $9$, thus $y_6 = y_7$. 
The vertice $ 9 $ has out-degree $0$, indicating that $ y_6 = y_7 \notin 
 \{x_1,\dots, x_n \}  $.   This figure can be decomposed into two cycles and three subtrees. Finally, adding a self-loop $(9, 9 )$ will generate a graph
 where each edge has outdegree $1$ (this is the reduction done in Step 2). 
 }
	\label{fig of graph decomposition2}
\end{figure}

Denote $ \xi(y_i , x_i ) = \log   \left(  1 +   e^{  f  ( y_i )   -  f (  x_i )   } \right)   $.
According to   Lemma \ref{lemma: decomposition into cycles and trees}, we have
{\equationsizeReg
\begin{equation}\label{g bound by g cyc}
\begin{split}
&  - n g(Y)    =   \inf_{  f }    \sum_{ i = 1}^n 
\xi(y_i , x_i )         
\geq      \inf_{  f }  \left[    \sum_{k=1}^K    \sum_{   i \in V(C_k ) } 
\xi(y_i , x_i ) 
\right]       
\triangleq     g_{\rm cyc}.     
\end{split}
\end{equation}
}

\iflonger 
{\equationsizeReg
\begin{equation}\label{g bound by g cyc}
\begin{split}
&  - n g(Y)    =   \inf_{  f }    \sum_{ i = 1}^n 
\log   \left(  1 +   e^{  f  ( y_i )   -  f (  x_i )   } \right)           \\
=   &  \inf_{  f }  \left[    \sum_{k=1}^K    \sum_{   i \in V(C_k ) } 
\log   \left(  1 +  e^{  f  ( y_i )   -  f (  x_i )   }  \right)    +  
\sum_{ m =1}^M    \sum_{   i \in V( T_m ) \backslash \{ u_m \}  } 
\log   \left(  1 +   e^{  f  ( y_i )   -  f (  x_i )   }  \right)      \right]       \\
\geq    &  \inf_{  f }  \left[    \sum_{k=1}^K    \sum_{   i \in V(C_k ) } 
\log   \left(  1 +   e^{  f  ( y_i )   -  f (  x_i )   }  \right)    
\right]       
\triangleq     g_{\rm cyc}.     
\end{split}
\end{equation}
}
\fi 

\textbf{Step 1.2: Compute} $ g_{\rm cyc}.  $
We then compute $ g_{\rm cyc}.  $
Since $C_k$ is a cycle, we have $ X_k  \triangleq   \{  x_i  :  i \in C_k    \}   = \{  y_i :  i \in C_k   \}$.
Since $C_k $'s are disjoint, we have $ X_k \cap X_l = \emptyset , \forall k \neq l .$
This implies that $ f(x_i), f(y_i) $ for $i$ in  one cycle $C_k$ are independent of the values corresponding to other cycles.  
Then $g_{\rm cyc} $ can be decomposed according to different cycles: 
{\equationsizeReg 
\begin{align*}
g_{\rm cyc}   & =      \inf_{  f }  \left[    \sum_{k=1}^K    \sum_{   i \in V(C_k ) } 
\log   \left(  1 +   \exp( f  ( y_i  )   -  f (  x_i )    )   \right)  \right]        =          \sum_{k=1}^K    
\inf_{  f } \sum_{   i \in V(C_k ) }  \log   \left(  1 +   \exp( f  ( y_i  )   -  f (  x_i )    )    \right) .
\end{align*}
} 
Similar to Warm-up example 1, 
the infimum for each cycle is achieved when $f(x_i) = f(x_j) , \forall i, j  \in V(C_k)$.
\iflonger 
More specifically, pick any $k$, and suppose the edges of $ C_k $ are 
$ (v_1, v_2)$, $(v_2, v_3), \dots,$ $(v_{r-1}, v_r  )$, $(v_r, v_1) $,
where $r = | V(C_k) |$ is the number of vertices in $C_k$. 
Denote $v_{r+1} = v_1$. 
Then 
{\equationsize 
\begin{equation}\label{g one cycle value}
\begin{split}
 \inf_{  f } \sum_{   i \in V(C_k ) }  \log   & 
 \left(  1 +   \exp( f  ( y_i  )   -  f (  x_i )    )     \right)     
=     \inf_{  f }     \sum_{   j = 1 }^r  \log   \left(  1 +   \exp( f  (x_{v_{j+1}}  )   -  f (  x_{v_j } )    )         \right)   \\
=   &   \inf_{   t_1, t_2, \dots, t_r \in \mathbb{R}  }      \left[
\sum_{   j = 1 }^{r - 1 }    \log   \left(  1 +   \exp(  t_{j+1 } - t_j     )           \right)     +  \log   \left(  1 +   \exp(  t_{1 } - t_r     )           \right)   \right]      
=    - r \log 2   = - | V(C_k) |   \log 2. 
\end{split}
\end{equation}
} 
The infimum is achieved when $ f( x_{v_1} )  =  \dots =  f( x_{v_r} )   $, or equivalently, $f(x_i) = f(x_j),  \forall i,j \in V(C_k). $
\fi 
In addition,
\begin{equation}\label{g cycle value}
g_{\rm cyc}  =  - \log 2  \sum_{k = 1}^K   | V(C_k) |. 
\end{equation}
\textbf{Step 1.3: Compute} $ g(Y).  $
According to Eq.~\eqref{g bound by g cyc} and Eq.~\eqref{g cycle value}, we have
{\equationsizeReg
\begin{equation}\label{g lower bound}
- n g(Y) \geq  \sum_{k = 1}^K   | V(C_k) |  \log 2 . 
 \end{equation}
 }
Denote 
$ F( Y ;  f )  = -\frac{1}{n} \sum_{i=1}^n \log   \left(  1 +   e^ { f  ( y_i )   -  f (  x_i )    } \right)   $, then $g(Y) =  \inf_f F( Y ;  f ) . $ 
We claim that for any $\epsilon > 0, $ there exists a continuous function $ f$ such that
\begin{equation}\label{g upper bound}
- n F(Y; f)  <     \sum_{k = 1}^K   | V(C_k) |  \log 2  + \epsilon. 
\end{equation}

Let $ N $ be a large positive number  such that 
\begin{equation}\label{N bound}
n   \log   \left(  1 +   \exp(  - N  )     )  \right)  < \epsilon. 
\end{equation}
Pick a continuous function $f$ as follows. 
\begin{equation}\label{f choice}
f( x_i ) = \begin{cases}
0,   &     i \in     \bigcup_{k = 1}^K  V(C_k)  ,  \\
N \cdot  \text{depth}( i ),  &  i \in  \bigcup_{m = 1}^M  V(T_m).   
\end{cases}
\end{equation}
Note that the root $u_m$ of a tree $T_m$ is also in a certain cycle  $C_k$, thus the value 
$ f(x_{u_m}) $ is defined twice in Eq.~\eqref{f choice}, but in both definitions its value is $0$, thus the definition of $f$ is valid. 
For any $i \in  V(C_k)$,  suppose $y_i = x_j $, then both $i, j \in V(C_k)$ which implies
$ f(y_i) - f(x_i) = f(x_j) - f(x_i) = 0.  $
For any  $ i   \in V(T_m) \backslash \{u_m \}$, suppose $ y_i = x_j $, then by the definition of the graph
$  ( i, j  )   $ is a directed edge of the tree $T_m$, which means that
$ \text{depth} ( i ) = \text{depth}(j) + 1 $.
Thus $ f(y_i) - f(x_i) = f(x_j) - f(x_i) = - N.  $
In summary, for the choice of $f$ in Eq.~\eqref{f choice}, we have
\begin{equation}\label{f dif values}
f(y_i) -   f( x_i ) = \begin{cases}
0,   &     i \in     \bigcup_{k = 1}^K  V(C_k)  ,  \\
-N ,  &  i \in  \bigcup_{m = 1}^M  V(T_m).   
\end{cases}
\end{equation}

Denote $ p =  \sum_{k = 1}^K   | V(C_k) | \log 2  $.
For the choice of $f$ in Eq.~\eqref{f choice}, we have
{\equationsizeReg
\begin{equation}\label{g path value}
\begin{split}
 - n F(Y; f )       
 = &   \sum_{ i = 1}^n \log   \left(  1 +   e^ { f  ( y_i )   -  f (  x_i )    } \right)           \\
=   &   \left[    \sum_{k=1}^K    \sum_{   i \in V(C_k ) } 
\log   \left(  1 +    e^ { f  ( y_i )   -  f (  x_i )    }     \right)    +  
\sum_{ m =1}^M    \sum_{   i \in V( T_m ) \backslash \{ u_m   \} } 
\log   \left(  1 +    e^ { f  ( y_i )   -  f (  x_i )    }  \right)      \right]       \\
\overset{ (\ref{f dif values}) }{ = }  &   \left[    \sum_{k=1}^K    \sum_{   i \in V(C_k ) } 
\log   \left(  1 +   e^{  0 }     \right)    +  
\sum_{ m =1}^M    \sum_{   i \in V( T_m ) \backslash \{ u_m   \} } 
\log   \left(  1 +   e^{  - N  }       \right)      \right]       \\
=  &      \sum_{k = 1}^K   | V(C_k) |  \log 2
+   \sum_{k = 1}^M   ( | V( T_m)  | - 1 )  \log   \left(  1 +  e^{  - N  }    \right)   
\leq    p  +    n   \log   \left(  1 +   e^{  - N  }   \right)  
\overset{ (\ref{N bound}) }{<}      p    + \epsilon. 
\end{split}
\end{equation}
} 
This proves Eq.~\eqref{g upper bound}.
Combining  the two relations given in Eq.~\eqref{g upper bound}
 and Eq.~\eqref{g lower bound}, we have 
{\equationsizeReg
\begin{equation}
g(Y) = \inf_f F(Y; f) = \frac{ 1}{n } \sum_{k = 1}^K   | V(C_k) |  \log 2 ,
 \; \forall \; Y \in R(X). 
\end{equation}
}

\textbf{Step 2: Compute $g(Y)$ for any $ Y $}\label{appJ proof of Thm 2, any Y value}.

In the general case, not all $y_i$'s lie in  $ \{  x_1, \dots, x_n  \}  . $
We will reduce to the previous case.
Denote 
$$
H = \{   i  :   y_i  \in  \{  x_1, \dots, x_n  \}    \},   \quad  H^c = \{   j  :   y_j  \notin  \{  x_1, \dots, x_n  \}    \}.
$$

Since $ y_j $'s in $H^c$ may be the same,  we define the set of such distinct values of $y_j$'s
as
$$
Y_{\text{out}} = \{    y \in \mathbb{R}^d  :   y  = y_j, \text{ for some }  j \in H^c         \}.
$$
Let  $ \bar{n}  =   | Y_{\text{out}} | $, then
there are total $n +  { \bar{n} }   $ distinct values in $ x_1, \dots, x_n, y_1, \dots, y_n  $. 
WLOG, assume $y_1, \dots, y_{ \bar{n} }  $ are distinct (this is because the value of $g(Y)$ does not change if we re-index $x_i$'s and $y_i$'s as long as  the subscripts of $x_i, y_i$ change together), then 
$$
Y_{\text{out}} =   \{   y_1, \dots, y_{ \bar{n} }    \}. 
$$

We create artificial ``true data'' and ``fake data''
$ x_{n+1} =  x_{n+1} = y_1, \dots,  x_{n+ { \bar{n} } } = y_{n+ { \bar{n} }  } = y_{ \bar{n} }  $. 
Define $ F_{\rm auc} (Y, f )
=  - \sum_{ i = 1}^{n+m}  \log   \left(  1 +   e^ { f  ( y_i )   -  f (  x_i )    } \right)  $
$ g_{\rm auc} = -  \inf_f  F_{\rm auc} (Y, f ) .  $
Clearly, $ F_{\rm auc} (Y, f ) =  n F(Y, f) - { \bar{n} }  \log 2  $ and
$  n g(Y) = g_{\rm auc}  -  { \bar{n} }  \log 2 $. 
 
Consider the new configurations  $ \hat{X} = (x_1, \dots, x_{n+ { \bar{n} }  }) $ 
 and $ \hat{ Y } = ( y_1, \dots, y_{n+ { \bar{n} }  }) $.
 For the new configurations, we can build a graph $ \hat{G} $
 with $n + { \bar{n} }  $ vertices and $ n + { \bar{n} }  $ edges.
 There are  $ K  $ self-loops $ C_{K+1}, \dots, C_{ K + \bar{n} }  $
  at the vertices corresponding to $ y_{1}, \dots, y_{ \bar{n}} $.
 Based on Lemma \ref{lemma: decomposition into cycles and trees}, we have:
 	(a) There exist cycles $ C_1, C_2, \dots, C_K, C_{K+1}, \dots, C_{ K + \bar{n} } $ and subtrees $ T_1, T_2, \dots, T_M $
 	(with roots $u_m$'s) s.t. each edge  $v \in A $ appears in exactly one of the cycle or subtrees.
	(b) $u_m$ is a vertex of a certain cycle $C_k$ where $1 \leq k \leq K + \bar{n}$.  
	(c) 
Each vertex of the graph appears in exactly one of the following sets:
	$ V(C_1), \dots, V(C_{K+ \bar{n} }), V( T_1)\backslash \{ u_1 \}, \dots, V( T_{M } )\backslash \{ u_{M } \}. $
According to the proof in Step 1, we have
$ g_{\rm auc} = \sum_{k=1}^{K + \bar{n}} |V(C_k)| \log 2 
=  \sum_{k=1}^{K } |V(C_k)| \log 2  +  \bar{n} \log 2  $.
Therefore, 
\[ n g(Y) = g_{\rm auc}  -  { \bar{n} }  \log 2 
 = \sum_{k=1}^{K } |V(C_k)| \log 2.
\]

We build a graph $G$ by removing the self-loops $C_{K+j} = ( y_j, y_j), j=1, \dots,
\bar{n} $ in $\hat{G}$. 
The new graph $G $ consists of $n + \bar{n}$ vertices
corresponding to $x_1, \dots, x_n$ and $y_1, \dots, y_{\bar{n}}$
and $ n $ edges.
The graph can be decomposed into 
cycles $ C_1, C_2, \dots, C_K $ (since $\bar{n}$ cycles
are removed from $\hat{G}$) and subtrees $ T_1, T_2, \dots, T_M $. 
The value  $ n g(Y) = \sum_{k=1}^{K } |V(C_k)| \log 2 $,
where $ C_k $'s are all the cycles of $G$.

\textbf{Step 3: Finding a non-decreasing path to a global minimum}.
Finally, we prove that for any $Y$,  there is a non-decreasing continuous path from  $Y$ to one global minimal $Y^*$.  
\iflonger 
In other words, there is a continuous function   $ \eta: [0, 1] \rightarrow  \mathbb{R}^{d \times n }  $
such that   $ \eta(0)  = \bar{Y},  \eta(1) = Y^* $ and  $ g( \eta(t) ) $ is a non-decreasing function with respect to $ t \in [0,1] $. 
In this proof, we will just describe the path in words, and skip the rigorous definition of the continuous function $\eta$, since it should be clear from the context of how to define $\eta$.
\fi 
The following claim shows that we can increase the value of $Y$ incrementally. See the proof in Appendix \ref{appen: proof of Claim of improvement}.

\begin{claim}\label{claim: incremental improvement}
	For an arbitrary $Y$ that is not a global minimum, there exists another $ \hat{Y}  $  and a non-decreasing continuous  path from  $Y$ to
	$ \hat{Y}  $   such that $ g(   \hat{Y}  ) - g( Y )  \geq   \frac{1}{ n } \log 2  $. 
\end{claim}

For any $Y$ that is not a global minimum, we apply Claim \ref{claim: incremental improvement} for finitely many times (no more than $n$ times), then we will arrive at one global minimum $Y^*$. We connect all non-decreasing continuous paths and get a non-decreasing continuous path
from $Y$ to  $Y^*$. This finishes the proof.

\subsubsection{Graph Preliminaries and Proof of Lemma \ref{lemma: decomposition into cycles and trees}}\label{appen: proof of lemma 1}

We present a few definitions from standard graph theory.
\begin{Def} (walk, path, cycle) 
	In a directed graph $G = (V, A)$,
	a walk $W = (v_0,  e_1, v_1, e_2, $   $ \dots,  v_{m-1}, e_m, v_m)$ is a  sequence of vertices and edges  
	such that $ v_i \in V , \forall ~ i \in \{ 0,1, \dots, m \} $  
	and $e_i = (v_{i-1}, v_i) \in A, \forall ~ i \in \{  1, \dots, m \} $.  
	If $v_0, v_1, \dots, v_m$ are distinct, we call it path (with length $m$). 
	If $v_0, v_1, \dots, v_{m-1}$ are distinct and $v_m = v_0$, we call it a cycle.
\end{Def}

Any $v$ has a path to itself (with length $0$), no matter whether there is an edge between $v$ to itself or not.  This is because the degenerate walk $W = (v)$ satisfies  the above definition. 
The set of vertices and edges in $W$ are denoted as $V( W ) $ and $A(W)$
	respectively. 

\begin{Def} (tree)
	A directed tree is a directed graph $T = (V, A)$ with a designated node $r \in V$, the root, such that  there is exactly one path from $ v$ to $r$ for each node $v \in V$ and there  is no edge from the root $r$ to itself. 
	The depth of a node is the length of the path from the node to the root (the depth of the root is $0$). 
	A subtree of a directed graph $G$  is a subgraph 
	$T  $ which is a directed tree.  
\end{Def}

\textbf{Proof of Lemma \ref{lemma: decomposition into cycles and trees}:}

We slightly extend the definition of  ``walk'' to allow infinite length. 
We present two observations. 

 \textbf{Observation 1}: 
Starting from any vertex $v_0 \in V(G)$, there is a unique walk with infinite length 
$$
W(v_0) \triangleq (v_0, e_1, v_1, e_2, v_2, \dots, v_i, e_i, v_{i+1}, e_{i+1}, \dots ),
$$ 
where $ e_i $ is an edge in $A(G)$ with tail $v_{i-1}$ and head $v_i$.

Proof of Observation 1: 
At each vertex $v_i$, there is a unique outgoing edge $e_i = (v_i, v_{i+1})$
which uniquely defines the next vertex $v_{i+1}$. Continue the process, we have proved Observation 1. 

\textbf{Observation 2}: 
The walk
$
W(v_0) \triangleq (v_0, e_1, v_1, e_2, v_2, \dots, v_i, e_i, v_{i+1}, e_{i+1}, \dots )
$ 
can be decomposed into two parts $ 
W_1 (v_0)     =   (  v_0, e_1, v_1, e_2, v_2, \dots, v_{i_0 - 1}, e_{i_0 }, v_{ i_0 }   ) ,     $ 
$ W_2 (v_0 ) = (   v_{i_0}, e_{i_0 + 1 }, v_{ i_0 + 1 },  e_{i_0 + 2 },   v_{ i_0 + 2 }, \dots  ),  $
where $ W_1 ( v_0 )  $ is a path from $v_0$ to $v_{i_0}$  (i.e.
$ v_0, v_1, \dots, v_{i_0} $ are distinct), and
$W_2(v_0)$ is the repetition of a certain cycle (i.e., there exists $T$
such that $v_{i + T } = v_i$, for any $i \geq i_0 $).  
This decomposition is unique, and we say the ``first-touch-vertex'' of $v_0$ is $v_{i_0}$.

\textbf{Proof of Observation 2}: 
Since the graph is finite,  then some vertices must appear at least twice in $W(v_0)$.
Among all such vertices, suppose $ u $ is the one that appears the earliest in the walk $W(v_0)$, and the first two appearances are $ v_{i_0 } =  u  $ and $v_{i_1 } = u $
and $i_0  < i_1 $.  Denote $T = i_1 - i_0 $.
Then it is easy to show $ W_2 (v_0 ) $  is the repetitions of the cycle consisting of vertices $ v_{i_0}, v_{i_0 + 1}, \dots, v_{i_1 - 1} $, and $W_1(v_0)$ is a directed path from $v_0$ to $v_{i_0}$.
\iflonger 
Since there is a unique edge going out from any vertex,
thus $ v_{i_0 + 1}  $ must be the same as $ v_{i_1 + 1} = v_{i_0 + 1 + T} $.
Continue the process, we have $ v_i = v_{i+T}  $ for any $i \geq i_0$.
Thus starting from $u = v_{i_0 }$, the walk  $ W_2 (v_0 ) $  will be repetitions of the cycle consisting of vertices $ v_{i_0}, v_{i_0 + 1}, \dots, v_{i_1 - 1} $, and we denote this cycle as $C_{k_0}$. 
If the vertices before $v_{i_1}$ are not distinct, then there are at least
two vertices $v_j =  v_l $ where $0 \leq j < l \leq i_0$.
This contradicts the definition of $i_0$. 
Therefore,  $W_1(v_0)$ is a directed path from $v_0$ to $v_{i_0}$.  
\fi

The first-touch-vertex $u = v_{i_0}$  has  the following properties:   (i)  $u \in C_{k}$ for some $k$;  
(ii) there exists a path from $v$ to $u $; 
(iii) any paths  from $v$ to any vertex in the cycle $C_{k}$ other than $u $ must pass $u $. 
Note that if $ u  $ is in some cycle, then its first-touch-vertex is  $u$  itself.

As a corollary of Observation 2,  there is at least one cycle. 
Suppose all cycles of $G$ are $C_1, C_2, \dots, C_K $.
Because the outdegree of each vertex is $1$,  these cycles must be disjoint, i.e.,  $V(C_i ) \cap V(C_j)  = \emptyset $ and $ A(C_i) \cap A(C_j)  = \emptyset $,    for any $ i \neq j$. 
Denote the set of vertices in the cycles as 
\begin{equation}\label{set of cycle vertices}
V_c = \bigcup_{k=1}^K  V( C_1  ) \cup  \dots \cup V( C_K ) .  
\end{equation}
Let $u_1, \dots, u_M$  be the vertices of $C_1, \dots, C_m$ with indegree at least 
$2$. 

Based on Observation 2, starting from any vertex outside $V_c $ there is a unique path that reaches $V_c$. Combining all vertices that reach the cycles at $u_m$ (denoted as $V_m$), and the paths from these vertices to $u_m$, we obtain a directed subgraph $T_m$, which is connected with $V_c$ only via the vertex $u_m$. The subgraphs $T_m$'s are disjoint from each other since they are connected with $V_c$  via different vertices. In addition, each vertex outside of $V_c$ lies in exactly one of the subgraph $T_m$. Thus, we can partition the whole graph into the union of the cycles $C_1, \dots, C_K$ and the subgraphs $T_1, \dots, T_M$. 

We then show $T_m$'s are trees. 
For any vertex $v_0 $ in the subgraph $T_m$, 
consider the walk $W(v_0). $   Any path starting from $v_0$ must be part of $W(v_0)$. 
Starting from $v_0 $ there is only one path from $v_0 $ to $u_m$ which is $W_1 (v_0)$, according to Observation 2. 
Therefore, by the definition of a directed tree, $T_m$ is a directed tree with the root $u_m$. 
Therefore, 
we can partition the whole graph into the union of the cycles $C_1, \dots, C_K$ and subtrees $T_1, \dots, T_M$ with disjoint edge sets;
in addition, the edge sets of the cycles are disjoint, and 
the root of $T_l $ must be in certain cycle $C_k $. 
It is easy to verify the properties stated in Lemma \ref{lemma: decomposition into cycles and trees}.
This finishes the proof.

\vspace{-0.3cm}
\subsubsection{Proof of Claim \ref{claim: incremental improvement}}\label{appen: proof of Claim of improvement}
We first prove the case for $d \geq 2$. 
Suppose the corresponding graph for $Y$ is $G$, and $G$ is decomposed into the union of  cycles $C_1, \dots, C_K$
and  trees $ T_1, \dots, T_m$.
We perform the following operation: pick an arbitrary tree $T_m$ with  the root $u_m$.
The tree is non-empty, thus there must be an edge $e$ with the head $u_m$.

Suppose $ v $ is the tail of the edge $e $. 
Now we remove the edge $e = (v, u_m)$ and create a new edge $e' = (v, v)$.
The new edge corresponds to $ y_v = x_v$.  The old edge $ (v, u_m) $ corresponds 
to $ y_{ v} = x_{u_m}  $ (and a term $ h( f( x_{u_m}) - f(x_v ) ) $) if $u_m \leq n$ or  
$ y_{ v} =  y_{u_m - n } \notin \{ x_1 , \dots, x_n \}  $   (and a term $h( f( y_{u_m - n}) - f(x_v ) ) $)    if $u_m > n $. 
This change corresponds to the change of $y_v$: we change
$ y_{ v } = x_{u_m}  $ (if $u_m \leq n $) or $y_v = y_{u_m -n}$ (if $u_m > n$)
to $ \hat{y}_{v } = x_v $. 
Let $ \hat{y}_i = y_i $ for any $i \neq v$,
and $\hat{Y} = (\hat{y}_1, \dots, \hat{y}_n)$ is the new point. 

Previously $v$ is in a tree $T_m$ (not its root), now $v$ is the root of a new tree, and also part of the new cycle (self-loop) $C_{K+1} = (v, e', v)$. 
In this new graph, the number of vertices in cycles increases by $1$, thus the value
of $g$ increases by $ -\frac{1}{n} \log 2$, i.e., $ g(   \hat{Y}  ) - g( Y ) =   \frac{1}{ n } \log 2  $. 

Since  $d   \geq 2 $, we can find a path in $\dR^d $ from a point to another point without passing any of the points in $ \{ x_1, \dots, x_n \} $.
In the continuous process of moving $ y_v  $ to $ \hat{y}_v  $, the function value will not change except at the end that $y_v = x_v$.
Thus there  is a non-increasing path from $Y  $ to $\hat{Y}  $, in the sense that along this path the function value of $g$ does not decrease. 

The illustration of this proof is given below. 
\begin{figure}[H]
\vspace{-0.5cm}
	\centering
	\subfigure[Original graph]{
		\label{fig3a}
		\includegraphics[width=0.35\textwidth, height = 2.5cm]{figure/fig4_graph_gene}}
	\subfigure[Modified graph, with improved function value]{
		\label{fig3b}
		\includegraphics[width=0.4 \textwidth, height = 2.5cm]{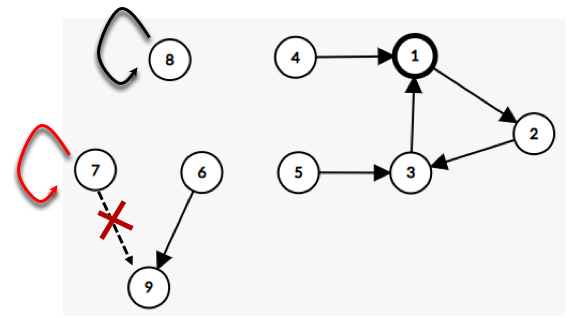}}\hfill 
	\caption{   Illustration of the proof of Claim  \ref{claim: incremental improvement}.  
		For the figure on the left,	we pick an arbitrary tree with the head being vertex $9$, which corresponds to $y_6 = y_7 $. 
		We change $y_7 $ to $ \hat{y}_7 = x_7 $ to obtain the figure on the right.  Since one more cycle is created, the function value
		increases by $ - \frac{1}{ n } \log 2 . $  }
	\label{fig of graph decomposition}
\end{figure}

For the case $d = 1$, the above proof does not work. The reason is that the path from $y_v$ to $\hat{y}_v $ may touch other points
in $\{ x_1, \dots, x_n \}$ and thus may change the value of $g$.
We only need to make a small modification:
we move $ y_v $ in $\dR $ until it touches a certain $x_i$ that corresponds to a vertex in the tree $T_m$, at which point
a cycle is created, and the function value increases by at least   $  \frac{1}{ n } \log 2   $. 
This path is a non-decreasing path, thus the claim is also proved.

\subsection{Proof of Theorem \ref{prop: RS-GAN all values, extension} }\label{appen:
proof of Thm 4}

\iflonger 
\textbf{Step 1: Optimal value. }
We first show that the optimal value is $ h(0) $.
 $g(Y) = \phi_{\rm R} ( Y, X )  = \frac{1}{ n } 
\sup_{  f \in C( \mathbb{R}^d ) } 
\sum_{i = 1}^n    [  h (  f( x_i ) - f(y_i ) ) ]
\geq  h ( 0) ,
$ 
thus $ \min_Y  \phi_{\rm R} ( Y, X )  \geq  h ( 0) $.
In addition, if  $ y_i = x_i, \forall i $,
then we have $  \phi_{\rm R} ( Y, X ) = h(0) $.
Thus the global minimal value is $ h (0) $, and is achieved 
if  $ y_i = x_i, \forall i $ (not the only choices of optimal solutions, as discussed below).
\fi

Obviously, $g(Y) \triangleq  \phi_{\rm R} ( Y, X )  = \frac{1}{ n } 
\sup_{  f \in C( \mathbb{R}^d ) } 
\sum_{i = 1}^n    [  h (  f( x_i ) - f(y_i ) ) ]
\geq  h ( 0) $ (by picking $f = 0$).  

\textbf{Step 1: achieving optimal $g(Y)$.} We prove if $ \{ y_1, \dots, y_n \} = \{ x_1, \dots, x_n \} $, then $ g(Y) = h(0) $.
\begin{claim}\label{small claim concave}
Assume $ h$ is concave. %
Then the function $ \xi_{\rm R}( m ) = \sup_{ (t_1, \dots, t_k) \in ZO( m ) } 
  \sum_{i=1}^m  h( t_i  )   $  satisfies   $ \xi_{\rm R }( m ) = m h(0)  $,
  where the set $ ZO( m ) = \{ t_1, t_2, \dots, t_m \in   \mathbb{R} : 
 \sum_{i=1}^m t_i = 0 \}  $.
\end{claim}
The proof of this claim is obvious and skipped here.
When $ \{ y_1, \dots, y_n \} = \{ x_1, \dots, x_n \} $,
we can divide $[n]$ into multiple cycles
$ C_1 \cup \dots \cup C_K $, each with length $m_k$, 
and obtain
$
\phi_{\rm R} ( Y, X )  = \frac{1}{ n }   \sup_{  f \in C( \mathbb{R}^d ) } 
\sum_{ k = 1}^K \sum_{i=1}^{m_k }  [  h (  f( x_i ) - f(y_i ) )  ]
=   \frac{1}{ n }  \sum_{ k = 1}^K \xi {\rm R }( m_k )
= \frac{1}{ n }  \sum_{ k = 1}^K  m_k h ( 0)
 = h(0). 
$

\textbf{Step 2: compute $g(Y)$ when $ y_i  \in  \{  x_1, \dots, x_n  \}, \forall i. $}
Assume $  y_i  \in  \{  x_1, \dots, x_n  \}, \forall i. $ 
We build a directed graph $G = (V, A )$ as follows (the same graph as in
Appendix \ref{appJ proof of Thm 2, any Y value}).
The set of vertices $V = \{  1, 2, \dots, n   \} $ represents $x_1, x_2, \dots, x_n$.
We draw a directed edge  $  ( i  , j ) \in A    $  if $ y_i = x_j  $.
Note that it is possible to have a self-loop $ (i, i) $, which corresponds
to the case $ y_i = x_i $. 

According to Lemma \ref{lemma: decomposition into cycles and trees}, this graph can be decomposed into  cycles $ C_1, C_2, \dots, C_K $ and subtrees $ T_1, T_2, \dots, T_M $. 
We claim that
\begin{equation}\label{major f val}
\phi_{\rm R} ( Y, X ) =  \frac{ 1 }{ n }  \sum_{k=1}^K |V(C_k)| h(0) \geq h(0) .
\end{equation}
The proof of the relation in  Eq.~\eqref{major f val}
is similar to the proof of Eq.~\eqref{g cycle value} used in the proof
of Theorem \ref{prop: RS-GAN all values, extension},
and briefly explained below. 
One major part of the proof is to show that the contribution of the nodes
in the cycles is $ \sum_{k=1}^K |V(C_k)| h(0)  $.
This is similar to Step 1, and is based on 
Claim \ref{small claim concave}. 
Another major part of the proof is to show that the contribution
of the nodes in the subtrees is zero, similar to the proof of
Eq.~\eqref{g path value}.
This is because  we can utilize Assumption \ref{assumption cp-1}
to construct a sequence of $ f $ values
(similar to Eq.~\eqref{f choice})
so that 
\begin{equation}\label{f dif values general}
f(y_i) -   f( x_i ) = \begin{cases}
0,   &     i \in     \bigcup_{k = 1}^K  V(C_k)  ,  \\
 \alpha_N  ,  &  i \in  \bigcup_{m = 1}^M  V(T_m).   
\end{cases}
\end{equation}
Here $ \{ \alpha_N \}_{N=1}^{\infty } $ is a sequence of real numbers so that $ \lim_{N \rightarrow \infty } h( \alpha_N )  = \sup_t h (t) = 0 $.
In the case that $ h (\infty ) = 0 $ like RS-GAN,
we pick $\alpha_N = N$. In the case that $ h ( a ) = 0 $ for a certain finite number $a$,
we can just pick $\alpha_N = a , \forall N$ (thus we do
not need a sequence but just one choice).

Since the expression of $\phi_{\rm R}(Y, X)$
in Eq.~\eqref{major f val} is a scaled version
of the expression of $\phi_{\rm RS}(Y, X)$ 
(scale by $- \frac{ \log 2 } {  h(0) } $), the rest of the proof is the same as the proof of Theorem \ref{prop: RS-GAN all values, extension}.

\textbf{Step 3: function value for general $Y$ and GMR.}
This step is the same as the proof of Theorem \ref{prop: special case of Thm 2}.
For the value of general $Y$, we build an ``augmented graph''
and apply the result in Step 2 to obtain $g(Y)$.
To prove GMR, the same construction as 
the proof of Theorem \ref{prop: special case of Thm 2} suffices.

\section{Results in Parameter Space}\label{appen: param space results}

We will first state the technical assumptions and then present
the formal results in parameter space. 
The results become somewhat technical due to the complication
of neural-nets. 
Suppose the discriminator neural net is $ f_{\theta}  $ where
$ \theta \in \mathbb{R}^{ J } $
and the generator net is $ G_w $ where $ w  \subset \mathbb{R}^{ K } .$

\begin{assumption}\label{D repres assumption}(representation power of discriminator net):
  For any distinct vectors  $  v_1, \dots, v_{2 n } \in \mathbb{R}^{d} $ ,
 any $  b_1, \dots, b_{ 2n }   \in \mathbb{R}  $,
  there exists $ \theta \in \mathbb{R}^{ J } $ such that 
  $ f_{\theta } ( v_i ) = b_i,   ~ i =1, \dots, 2n. $
  \end{assumption}
  
  \begin{assumption}\label{G repres assumption 1}(representation power of generator net in $\mathcal{W}$)
 For any distinct $z_1, \dots, z_n \in \mathbb{R}^{d_z} $ and any $ y_1, \dots, y_n  \in \mathbb{R}^{d} $, there exists
 $w \in \mathcal{W}$  such  that $ G_w ( z_i )  = y_i , i=1, \dots, n $.
\end{assumption}

For any given $ Z = ( z_1, \dots, z_n ) \in \mathbb{R}^{d_z \times n } $,
and any $ \in \mathcal{W} \subseteq \mathbb{R}^K $,  
we define a set $ G^{-1}(Y; Z) $ as follows:
  $ w \in  G^{-1}(Y; Z ) $ iff $ G_w( Z) = Y $ and $ w \in \mathcal{W}$.  
  \begin{assumption}\label{G repres assumption 2, dup}
  (path-keeping property of generator net; duplication of Assumption~\ref{G repres assumption 2}):
  For any distinct $z_1, \dots, z_n \in \mathbb{R}^{d_z} $, the following holds: 
 for any continuous path $Y(t), t \in [0, 1]$ in the space $ \mathbb{R}^{d \times n}  $ and any $w_0 \in G^{-1}(Y (0) ; Z )  $, 
 there is continuous path $ w(t), t \in [0,1]  $ such that
 $ w(0) = w_0 $ and $ Y(t) = G_{w(t)}(Z)  , t \in [0, 1] $. 
\end{assumption}

We will present sufficient conditions for these assumptions later. %
Next we present two main results on the landscape of GANs in the parameter space. 

\begin{prop2}\label{prop-formal: JSGAN param space}(formal version
of Proposition \ref{prop: GAN param space})
    Consider the separable-GAN problem  $  \min_{ w \in \mathbb{R}^{K}  } 
    \varphi_{\rm sep}( w ) , $ where
$ \varphi_{\rm sep}( w )  = \sup_{ \theta }
  \frac{1}{2 n} \sum_{i=1}^n   [ h_1 (   f_{ \theta}( x_i )  ) + h_2 (  - f_{\theta} ( G_w( z_i ) )  )   ] $
  Suppose $h_1, h_2$ satisfy the same assumptions of Theorem \ref{prop: GAN all values, extension}. 
Suppose $ G_{w}  $ satisfies Assumption \ref{G repres assumption 1} and
Assumption \ref{G repres assumption 2} (with certain $\mathcal{W}$). 
Suppose $  f_{\theta} $  satisfies Assumption \ref{D repres assumption}. Then there exist at least $ (  n^n - n !  )$ distinct $ w  \in \mathcal{W } $ that are not global-min-reachable. 
\end{prop2}

\begin{prop2}\label{prop-formal: CPGAN param space}(formal version
of Prop. \ref{prop: CPGAN param space})
    Consider the RpGAN problem  $  \min_{ w \in \mathbb{R}^{K}  }  \varphi_{\rm R}( w ) , $  where  
$  \varphi_{\rm R}( w )  = \sup_{ \theta }  \frac{1}{ n } \sum_{i=1}^n   [ h ( f_{ \theta}( x_i ) )  -  f_{ \theta}( G_w(z_i) )  ]. $
Suppose $h $ satisfies  the same assumptions of Theorem \ref{prop: RS-GAN all values, extension}. Suppose $ G_{w}  $ satisfies Assumption \ref{G repres assumption 1} and
Assumption \ref{G repres assumption 2} (with certain $\mathcal{W}$). 
Suppose $  f_{\theta} $  satisfies Assumption \ref{D repres assumption}. 
 Then any $ w  \in \mathcal{W} $ is global-min-reachable for $ \varphi_{\rm R}( w )  $.
\end{prop2}

We have presented two generic results that relies on a few 
properties of the neural-nets. These properties can be satisfied by certain neural-nets,
 as discussed next. Our results largely rely on recent advanced in neural-net
 optimization theory.

\subsection{Sufficient Conditions  for the Assumptions}
\label{app-sub: suffcient conditions for repres assumptions}

 In this part, we present a set of conditions on neural nets that
 ensure the assumptions to hold. 
 We will discuss more conditions in the next subsection. 

\begin{assumption}\label{assump 1 of overpara}(mildly wide)
 The last hidden layer has at least $ \bar{ n } $ neurons, where $\bar{n}$ is the number of input vectors.
\end{assumption}
The assumption of width is common in recent  theoretical works in neural net optimization (e.g. \cite{li2018over,nguyen2018loss,allen2018convergence}).
For the generator network, we set $\bar{n} = n$; for the discriminator network,
we set  $ \bar{ n } = 2n . $

\begin{assumption}\label{assump 2 of overpara}(smooth enough activation)
	The activation function $\sigma$ is an analytic function, and
 the $k$-th order derivatives $ \sigma^{(k)}(0) $ are non-zero, 
for $k = 0, 1, 2, \dots,  \bar{ n } , $ where $\bar{n}$ is the number of input vectors.
\end{assumption}
The assumption of the neuron activation is satisfied by sigmoid, tanh, SoftPlus, swish, etc.

 For the generator network, consider a fully neural network $ G_w( z ) = W_H \sigma( W_{H-1} \dots  W_2 \sigma( W_1  z ) )  $ that maps $ z \in \mathbb{R}^{d_z} $ to 
 $ G_w( z ) \in \mathbb{R}^{d} $.
Define $ T_k (z) = \sigma( W_{ k -1} \dots W_2 \sigma( W_1  z ) ) 
\in \mathbb{R}^{ d_k } $ where $d_k $ is the number of neurons in the $k$-th hidden layer.
Then we can write $ G_w(z) =  W_H T_H(z)  $, where $ W_H \in \mathbb{R}^{d \times d_H } $. 
Let $Z = (z_1, \dots, z_n )$ and let $T_k(Z) = (T_k(z_1), \dots, T_k(z_n)) \in \mathbb{R}^{ d_k  \times n },  $ 
$ k = 1, 2, \dots, H.  $ Define $\mathcal{W} = \{ w = (W_1, \dots, W_H): T_H(Z) \text{is full rank}  \} $.

We will prove that under these two assumptions on the neural nets,
 the landscape of RpGAN is better than that of SepGAN. 
\begin{prop2}\label{prop: first realization using dense W}
Suppose $h_1, h_2, h$ sastify assumptions in Theorem \ref{prop: GAN all values, extension}
and  Theorem \ref{prop: RS-GAN all values, extension}. 
Suppose $ G_{w} , f_{\theta} $ satisfies Assump. 
\ref{assump 2 of overpara} and \ref{assump 1 of overpara}
($\bar{n} = n $ for $G_w $, and $\bar{n} = 2 n $ for $f_{\theta}$).
 Then there exist at least $ (  n^n - n !  )$ distinct $ w  \in \mathcal{W } $ that are not GMR for $ \varphi_{\rm Sep}( w )  $.
In contrast, any $ w  \in \mathcal{W} $ is global-min-reachable for $ \varphi_{\rm R}( w )  $.
\end{prop2}

This proposition is the corollary of Prop. \ref{prop-formal: JSGAN param space}
and Prop. \ref{prop-formal: CPGAN param space}; we only need
 to verify the assumptions in the two propositions. 
 The following series of claims provide such verification.

\begin{claim}\label{representation claim of G}
Suppose Assumptions \ref{assump 1 of overpara} and \ref{assump 2 of overpara} hold
for the generator net  $ G_w $ with distinct input $z_1, \dots, z_n $. 
 Then $\mathcal{W} = \{ (W_1, \dots, W_H):
T_H(Z) \text{ is full rank}  \} $ is a dense set  in $ \mathbb{R}^K $. 
In addition, Assumption \ref{G repres assumption 1}  holds. 
 \end{claim}

This full-rank condition was used in a few works of neural-net landscape analysis (e.g. \cite{nguyen2017loss}).
In GAN area, \cite{bai2018approximability} studied  invertible generator nets $ G_w $ where the weights are restricted to a subset of $ \mathbb{R}^K $ to avoid singularities.  As the set $\mathcal{W} $ is dense, intuitively  the iterates will stay in this set for most of the time. However, rigorously proving that the iterates stay in this set is not easy, and is  one of the major challenges of current neural-network analysis.  For instance, \cite{jacot2018neural}) shows that for very wide neural networks with proper initialization along the training trajectory of gradient descent the neural-tangent kernel (a matrix related to $T_H(Z)$) is full rank. A similar analysis can prove that the matrix $T_H(Z)$ stays full rank during training under similar conditions. We do not attempt to develop the more complicated convergence analysis for general neural-nets here and leave it to future work.

\begin{claim}\label{claim of continuous path, full rank case}
Suppose Assumptions \ref{assump 1 of overpara} and \ref{assump 2 of overpara} hold
for the generator net  $ G_w $ with distinct input $z_1, \dots, z_n $. 
 Then it satisfies Assumption 
 \ref{G repres assumption 2} with $\mathcal{W}$
 defined in Claim \ref{representation claim of G}.
 \end{claim}

Assumption \ref{D repres assumption} can be shown to hold under a similar condition to that in Claim \ref{representation claim of G}.

\begin{claim}\label{representation claim of D}
 Consider a fully connected neural network $ f_{\theta}( z ) = \theta_H \sigma( \theta_{H-1} \dots
 \theta_2 \sigma( \theta_1  z ) )  $ that maps $ u  \in \mathbb{R}^{d } $ to  $ f_{\theta}( u ) \in \mathbb{R}  $
 and suppose Assumptions \ref{assump 1 of overpara} and \ref{assump 2 of overpara} hold. 
 Then Assumption \ref{D repres assumption}  holds. 
 \end{claim}
 
 The proofs of the claims are given in Appendix \ref{appen: repr claim proof}.

\iflonger 
Remark: Assumptions \ref{D repres assumption} and \ref{G repres assumption 2} just
state that $f_{\theta }$ and $ G_w $ can represent any data set with size no more than
 its final hidden layer. This is about the representation power of neural nets. 
The standard universal approximation theorem claims that wide neural networks 
 can represent any continuous function. A small difference with classical representation theory is that we require the exact equivalence $ G_w(z_i) = y_i $ for finite samples while classical representation theory
allows some error $ \| G_w(z_i) - y_i \| < \epsilon  $. 
\fi 

With these claims, we can immediately prove Prop. \ref{prop: first realization using dense W}.

\textbf{Proof of Prop. \ref{prop: first realization using dense W}}: 
According to Claim  \ref{claim of continuous path, full rank case},
\ref{representation claim of G}, \ref{representation claim of D},
the assumptions of Prop. \ref{prop: first realization using dense W}
imply the assumptions of Prop. \ref{prop-formal: JSGAN param space}
and Prop. \ref{prop-formal: CPGAN param space}.
Therefore, the conclusions of Prop. \ref{prop-formal: JSGAN param space}
and Prop. \ref{prop-formal: CPGAN param space} hold. 
Since the conclusion of Prop. \ref{prop: first realization using dense W}
is the combination of the the conclusions of Prop. \ref{prop-formal: JSGAN param space}
and Prop. \ref{prop-formal: CPGAN param space},
 it also holds.  $\Box $

\subsection{Other Sufficient Conditions}

Assumption \ref{G repres assumption 2, dup} (path-keeping property) is the key assumption.
Various results in neural-net theory can ensure this assumption (or its variant)
holds, and we have utilized one of the simplest such results in the last subsection.  
We recommend to check \cite{sun2020global} which describes a bigger picture
about various landscape results. 
In this subsection, we briefly discuss other possible results applicable to GAN.

We start with a strong conjecture about neural net landscape, which only requires
 a wide final hidden layer but no condition on the depth and activation. 
\begin{conjecture}\label{conjecture of deep net GMR}
Suppose $ g_{\theta} $ is a fully connected neural net with any depth and any continuous activation, and it satisfies Assumption \ref{assump 1 of overpara} (i.e. a mildly wide final hidden layer). 
Assume $\ell( y , \hat{y} ) $ is convex in $ \hat{y}$, then the empirical loss function of a supervised learning problem $ \sum_{i=1}^n \ell( y_i, g_{\theta}(x_i )  ) $ is global-min-reachable  for any point.
\end{conjecture}

We then describe a related conjecture for GAN, which is easy to prove if Conjecture 
\ref{conjecture of deep net GMR} holds. 

\textbf{Conjecture 1} (informal): Suppose $G_w$ is a fully connected net 
satisfying Assump. \ref{assump 1 of overpara} (i.e.
a mildly wide final hidden layer). 
Suppose $G_w$ and $f_{\theta}$ are expressive enough (i.e. Assump. \ref{G repres assumption 1} and Assump. \ref{D repres assumption} hold). 
Then the RpGAN loss has a benign landscape, in the sense that any point is GMR
for $ \varphi_{\rm R}( w )  $. In contrast, the SepGAN loss does not have this property.

Unfortunately, we are not aware of any existing work that has proved
Conjecture \ref{conjecture of deep net GMR}, thus we are not able to 
prove Conjecture 1 above for now. 
\citet{venturi2018spurious} proved a special case of Conjecture \ref{conjecture of deep net GMR} for $L = 1$ (one hidden layer),
and other works such as  \citet{li2018over} prove a weaker version
of Conjecture \ref{conjecture of deep net GMR}; see \cite{sun2020global} for other
related results. 
The precise version of Conjecture \ref{conjecture of deep net GMR} seems non-trivial to prove.

\iffalse 
These assumptions can be relaxed to a neural net with one wide layer (not necessarily the final hidden layer) and  activation functions being increasing functions in the layers above this wide layer as analyzed in \cite[Theorem 2]{li2018over}. Alternative conditions on the activation functions are provided in \cite{nguyen2017loss,nguyen2018loss}.
It is not hard to follow \cite[Theorem 2]{li2018over} and the result of  \cite{nguyen2018loss}.
 to present two different types of results; for simplicity,  we just follow the conditions of
 \cite[Theorem 1]{li2018over} to present one representative result. 
 \fi 

We list two results on GAN that can be derived from weaker versions of Conjecture \ref{conjecture of deep net GMR}; both results apply to the whole space instead of the dense subset $\mathcal{W}$.

  \textbf{Result 1} (1-hidden-layer):
Suppose $G_w$ is 1-hidden-layer network  with any continuous activation. Suppose 
it satisfies Assump. \ref{assump 1 of overpara} (i.e.
a mildly wide final hidden layer). 
Suppose $G_w$ and $f_{\theta}$ are expressive enough (i.e. Assump. \ref{G repres assumption 1} and Assump. \ref{D repres assumption} hold). 
Then the RpGAN loss satisfies GMR for any point.
This result is based on  \citet{venturi2018spurious}.

   \textbf{Result 2}: 
  Suppose $G_w$ is a fully connected network
  with any continuous activation and any number of layers.
  Suppose it satisfies Assump. \ref{assump 1 of overpara} (i.e.
a mildly wide final hidden layer). Suppose $G_w$ and $f_{\theta}$ are expressive enough (i.e. Assump. \ref{G repres assumption 1} and \ref{D repres assumption} hold). 
Then the RpGAN loss has no sub-optimal set-wise local minima (see \cite[Def. 1]{li2018over} for the definition). This  result is based on  \citet{li2018over}.

Due to space constraint, we do not present the proofs of the above two results
(combining them with GANs is somewhat cumbersome).
The high-level proof framework is similar to that of
 Prop. \ref{prop: first realization using dense W}.

\subsection{Proofs of Propositions for Parameter Space }

\textbf{Proof of Proposition \ref{prop-formal: JSGAN param space}.} 
The basic idea is to build a relation between
 the points in the parameter space to the points in the function space.

Denote 
$ \mathcal{L}_{\rm sep}( w ; \theta ) 
= \frac{1}{2 n} \sum_{i=1}^n   [ h_1 (   f_{ \theta}( x_i )  ) + h_2 ( -   f_{\theta} ( G_w( z_i ) )  )   ] $, 
then 
$    \varphi_{\rm sep}( w )  = \sup_{ \theta }  \mathcal{L}_{\rm sep} ( w ; \theta ). 
$ 
Denote 
$ L_{\rm sep}( Y ;  f )  = \frac{1}{2 n} \sum_{i=1}^n   
[ h_1 (  f(x_i) ) + h_2 (  - f( y_i) ) ]  $,
and $ \phi (Y, X) = \sup_f L_{\rm sep} (Y; f) .  $ 
Note that in the definition of the two functions above, the discriminator is hidden
 in the $ \sup$ operators, thus we have freedom to pick the discriminator values (unlike
 the generator space which we have to check all $w $ in the inverse of $Y$). 

Our goal is to analyze the landscape of   $  \varphi_{\rm sep}( w ) $,
based on the previously proved result on the landscape of $ \phi (Y, X) $.
We first show that the image of
$ \varphi_{\rm sep}( \hat{w }  ) $ is the same as that of  $ \phi_{\rm sep}( \hat{ Y } , X) $.

Define  $  G^{-1}(Y) \triangleq  \{   w:  G_w(z_i) =  y_i, i=1, \dots, n  \} . $ 
We first prove that
\begin{equation}\label{eq App J proof of value corresp}
	 \phi_{\rm sep}( \hat{ Y } , X)  = \varphi_{\rm sep}( \hat{w }  ), ~
	 \forall ~ \hat{w } \in G^{-1}( \hat{ Y } ).
\end{equation}
Suppose $ \phi_{\rm sep}(   \hat{ Y } , X)  = \alpha $. This implies that  $ L_{\rm sep}( \hat{Y} ;  f ) \leq \alpha $ for any $f$; in addition, for any $\epsilon > 0 $ there exists $ \hat{f} \in C ( \mathbb{R}^d ) $
such that  
\begin{equation}\label{eq App J temp L}
	 L_{\rm sep}( \hat{Y} ;  \hat{f} )   \geq  \alpha  - \epsilon .  
	\end{equation}

According to Assumption \ref{D repres assumption},  there exists $ \theta^* $ such that
$ f_{\theta^* }( x_i ) = \hat{f}( x_i ),  ~ \forall ~ i$, and 
$ f_{\theta^*  }( u )  = \hat{f}( u )  , \forall ~ u \in   \{ y_1, \dots, y_n \} \backslash \{ x_1 , \dots, x_n \}  $.
In other words, there exists $ \theta^* $ such that
 \begin{equation}\label{eq App Jmapping}
  f_{\theta^* }( x_i ) = \hat{f}( x_i ), ~   f_{\theta^* }( y_i ) = \hat{f}( y_i ) , ~ \forall ~ i  .
 \end{equation}
  \iflonger
  \footnote{ We can slightly modify Assumption \ref{D repres assumption}: we allow $v_i$'s are not distinct, but
 then we add the trivial requirement that $b_i = b_j$ whenever $v_i = v_j$. With this
modification, we can directly apply the assumption to obtain \eqref{eq App Jmapping}.   }
\fi
Then we have
{\equationsizeReg
\begin{align*}
 \mathcal{L}_{\rm sep}( \hat{w } ; \theta^*(\epsilon ) ) 
&  =  \frac{1}{ 2n } \sum_{i=1}^n   [ h_1 (   f_{  \theta^*   } ( x_i )  ) +
 h_2 ( -   f_{ \theta^* } ( G_{ \hat{w} } ( z_i ) )  )   ]    \overset{\rm (i) }{ = }     \sum_{i=1}^n   [  h_1 (   f_{  \theta^*   } ( x_i )  )  +
 h_2 ( -   f_{ \theta^* } ( \hat{y}_j  )  )  ] \\
 &    \overset{\rm (ii) }{ = }     \frac{1}{ 2n }  \sum_{i=1}^n   [ h_1 (  \hat{f} ( x_i )  ) + h_2 ( -  \hat{f} ( \hat{y}_i )   )   ]     =   L_{\rm sep}( \hat{Y} ;  \hat{f} )    \overset{\rm (iii) }{ \geq }   \alpha - \epsilon.
\end{align*}
}
In the above chain, (i) is due to the assumption $ \hat{w } \in G^{-1}( \hat{ Y } ) $ (which implies  
$G_{ \hat{w} } ( z_j  )  =  \hat{y}_j  $), 
(ii) is due to   the choice of $\theta^* $.
 (iii) is due to  \eqref{eq App J temp L}.

Therefore, we have 
$  \varphi_{\rm sep}( \hat{w } )  = \sup_{ \theta } 
       \mathcal{L}_{\rm sep} ( \hat{w } ; \theta )  \geq   \mathcal{L}_{\rm sep}( \hat{w } ; \theta^*(\epsilon ) )  \geq  \alpha -  \epsilon .
       $ 
  Since this holds for any $ \epsilon $, 
  we have $   \varphi_{\rm sep}( \hat{w } ) \geq \alpha . $
 Similarly, from $ \mathcal{L}_{\rm sep}(  \hat{w }  ;  \theta  ) 
 \leq \alpha $ we can obtain $  \varphi_{\rm sep}( \hat{w } ) \leq \alpha . $
Therefore  $
  \varphi_{\rm sep}( \hat{w }  ) = \alpha = \phi_{\rm sep}( \hat{ Y } , X) .
$
This finishes the proof of  (\ref{eq App J proof of value corresp}).

Define 
{\equationsizeReg
\begin{align*}
	Q(X) \triangleq  \{  Y = (y_1, \dots, y_n) \mid  y_i \in \{ x_1, \dots, x_n \} , i\in \{ 1, 2, \dots, n\}; 
	y_i = y_j \text{ for some } i \neq j \}. 
\end{align*}
}
\iflonger 
According to Thm. \ref{prop: GAN all values, extension}, 
$ Y^* $ is a global minimum of  $ \phi_{\rm sep}(Y , X) $ 
iff $ Y^* \in P(X) $. %
Thus any $ w^* \in G^{-1}(Y^*) $  is a global minimum of 
$  \varphi_{\rm sep}( w )  $. %
On the contrary, any global minimum $w^* $ such that $  \varphi_{\rm sep}( w^* ) 
 =  \tau_1 $ must satisfy $ ( G_{w^*}(z_1), \dots, G_{w^*}(z_n) ) \in P(X) $
since otherwise it will contradict Theorem \ref{prop: GAN all values, extension}. 
\fi 
Any $Y  \in Q(X)$ is a mode-collapsed pattern. 
According to Theorem \ref{prop: GAN all values, extension}, any $ Y \in  Q (X) $ is a strict local minimum of $ \phi_{\rm sep}(Y , X) $, and thus $Y$ is not GMR.
Therefore $ \hat{w} \in G^{-1} (Y) $ where $Y \in Q(X) $
is not GMR; this is because a non-decreasing path in the parameter space
 will be mapped to a non-decreasing path in the function space,
 causing contradiction. 
\iflonger 
Assume the contrary, that there is a continuous path $ w(t) $
from $ w(0)=  \hat{w}  $ to $ w(1) =  w^* $, where $ w^* $ is a certain global minimum,
along which the function value $  \varphi_{\rm sep}(   w(t) )   $ is non-increasing. 
Since $ G_{ w} (z_1, \dots, z_n ) $ is a continuous function of $ w $,
the path $ Y(t) =  G_{ w(t) }(z_1,\dots, z_n ), t \in [ 0 , 1 ] $ is a continuous 
 path in the space of $Y$. 
 The starting point is $  G_{ w(0) }(z_1,\dots, z_n ) = G_{\hat{w} }(z_1,\dots, z_n ) =  \hat{Y}  $,
 and the ending point is   $  G_{ w( 1 ) }(z_1,\dots, z_n ) = G_{ w^* }(z_1,\dots, z_n ) =  Y^* $.
Therefore, we have found a continuous path from $ \hat{Y} $ to $ Y^* $
such that the loss value  $  \varphi_{\rm sep}( \hat{w }  ) = \phi_{\rm sep}( \hat{ Y } , X) $
is non-increasing.  This contradicts the conclusion of Theorem \ref{prop: GAN all values, extension}  that any
 $ \hat{Y} \in Q (X )  $ is not a strict local-min. 
 Therefore, we have proved that  any $ \hat{w} \in G^{-1} (Y) $ where $Y \in \Omega $,
is not global-min-reachable.
\fi 
Finally, according to Assumption  \ref{G repres assumption 1}, for any $ Y $  there exists at least one pre-image $ w  \in G^{-1}( Y ) \cap  \mathcal{W}   $.
 There are $  ( n^n - n ! )  $ elements in $ Q(X)  $, thus 
 there are at least $ (n^n  - n! ) $ points in $  \mathcal{W}  $ that are not global-min-reachable. 
$\Box $

\textbf{Proof of Proposition \ref{prop-formal: CPGAN param space}.} 
Similar to Eq.~\eqref{eq App J proof of value corresp}, we have
$ \varphi_{\rm R} ( w)= \phi_{\rm R} (Y, X) $ for any $w \in G^{-1}(Y)$. 
\iflonger 
The rest of the proof is slightly different from that of Prop. \ref{prop-formal: JSGAN param space} because we want to prove the existence of a non-increasing path in the parameter space (a positive result). We explain how this is enabled by the assumptions. 
\fi 
We need to prove that there is a non-decreasing path from any $ w_0 \in \mathcal{W}  $ to $ w^* $, where $ w^* $ is a certain global minimum.  Let $ Y_0 = G_{ w_0 }(z_1,\dots, z_n ) $. According to Thm. \ref{prop: RS-GAN all values, extension}, there is a continuous path $Y(t)$ from $ Y_0 $ to $ Y^* $ along which the loss value $  \phi_{\rm R}( Y(t) , X) $ is non-increasing.
According to Assump. \ref{G repres assumption 2}, there is a continuous path $ w(t) $ such that $ w(0) = \hat{w} $, $ Y(t) =  G_{w(t)} (Z)  , t \in [ 0 , 1]$. Along this path, the value $ \varphi_{\rm R}( w(t) ) = \phi_{\rm R}( Y(t) , X)  $ is non-increasing, and at the end the function value $ \varphi_{\rm R}( w(1) ) = \phi_{\rm R}( Y^* , X) $ is the minimal value of $ \varphi_{\rm R}(w) $. 
 Thus the existence of such a path is proved. 
$\Box $

\subsection{A technical lemma}

We  present a technical lemma, that slightly generalizes
\cite[Proposition 1]{li2018over}.

\begin{assumption}\label{assump 3 general}
	$v_1 , v_2, \dots, v_{ \barn } \in \mathbb{R}^d $ are distinct, i.e., $v_i \neq v_j$ for any $ i \neq j$. 
\end{assumption}

\begin{lemma}\label{lemma zero measure}
	Define $ T_H ( V ) = (  \sigma( W_{ H -1} \dots W_2 
	\sigma( W_1  v_i ) ))_{i=1}^{\barn} \in \mathbb{R}^{d_H \times \barn }  $. 
	Suppose Assumptions \ref{assump 1 of overpara}, \ref{assump 2 of overpara}
	and \ref{assump 3 general} hold. 
	Then the set    $ \Omega = \{  (W_1, \dots, W_{H-1}) :   \text{rank}(T_H( V ) ) 
	<  \barn  \} $
	has zero measure. 
\end{lemma}

This claim is slightly different from \cite[Proposition 1]{li2018over}, which requires the input vectors to have one distinct dimension (i.e., there exists $j$ such that $v_{1j}, \dots, v_{\barn, j}$ are distinct);
 here we only require the input vectors to be distinct.
 It is not hard to  link ``distinct vectors'' to ``vectors with one distinct dimension'' 
by a variable transformation.

\begin{claim}
	Suppose $ v_1,  \dots, v_m \in \mathbb{R}^d $ are distinct. 
Then for generic matrix $W \in \mathbb{R}^{ d \times d }$,
for the vectors $ \bar{v}_i = W v_i \in \mathbb{R}^d, i=1, \dots, n$, there exists $j$ such that $ \bar{v}_{1j}, \dots, \bar{v}_{ m, j}$ are distinct. 
\end{claim}

\begin{proof}
Define the set $ \Omega_0 = \{ u \mid  u \in \mathbb{R}^{ 1 \times d }, \exists i \neq j \text{ s.t. } u^T v_i = u^T v_j \} $. This is the union of $d(d-1)$ hyperplanes $ \Omega_{ij} \triangleq \{ u \mid  u \in \mathbb{R}^{ 1 \times d }, u^T v_i = u^T v_j  \} $. Each hyperplane $\Omega_{ij}$ is a zero-measure set, thus the union of them $\Omega_0$  is also a zero-measure set. 
Let $u$ be the first row of $W$, then $u$ is generic vector and thus not
in  $ \Omega_0 $, which implies  $ \bar{v}_{1 1}, \dots, \bar{v}_{ m, 1}$ are distinct. 
\end{proof}

\iflonger 
This means in the neural network problem, as long as the input vectors are distinct, then we can transform them into input vectors with a dinstinct dimension. 
For instance,
 a neural network $ f(W_2, W_1) = W_2 \sigma( W_1 X ) $ can be transformed
 to $ g(W_2, \hat{W}_1) = W_2 \sigma ( \hat{W}_1 \hat{X}) $ where $\hat{X} = A X $ and $\hat{W_1} = W_1 A^{-1}$,
 and $A$  is a  $d \times d$  generic matrix. 
This transformation does not affect the topological property (it may affect quantitative properties though), thus all existing results on the topological properties of $g$ apply to $f$. For instance, if there is a decreasing path in the space of $(W_2, \hat{W}_1)$, then after the continuous transformation of $\eta (W_2, \hat{W_1}) = (W_2, \hat{W_1} A^{-1}) = (W_2, W_1)$ we obtain a decreasing path in the space of $(W_2, W_1)$. 
\fi

\textbf{Proof of Lemma \ref{lemma zero measure}:}
 Pick a generic matrix $A \in \mathbb{R}^{ d_v \times d_v } $, then
  $ \bar{v}_i  = A v_i  \in \mathbb{R}^{ d_v \times 1 } $ has one distinct dimension,
  i.e.,  there exists $j$ such that $ \bar{v}_{1j}, \dots, \bar{v}_{\barn, j}$ are distinct. 
 In addition, we can assume $A $ is full rank (since it is generic).  
	Define 
	\[  \bar{T}_H ( \bar{V} ) = (  \sigma( W_{ H -1} \dots W_2 \sigma( \bar{W}_1  \bar{v}_1 ) ), \dots, 
\sigma( W_{ H -1} \dots W_2 \sigma( \bar{W}_1  \bar{v}_{\barn} )  )
\in \mathbb{R}^{d_H \times \barn }  . \] 
According to \cite[Prop. 1]{li2018over}, the set $ \bar{\Omega} = \{  ( \bar{W}_1, W_2, W_3, \dots, W_{H-1}):  \text{rank}( \bar{T}_H ( \bar{V} ) ) <  \barn  \} $ has zero measure. 
With the transformation $ \eta_0 ( \bar{W}_1 ) =  \bar{W}_1 A^{-1}  $, 
we have $  \sigma( W_{ H -1} \dots W_2 \sigma( \bar{W}_1  \bar{v}_i ) )
 =   \sigma( W_{ H -1} \dots W_2 \sigma(W_1 v_i  ) )  ,  ~ \forall ~ i  $ and
 thus $  \bar{T}_H ( \bar{V} )   =   T_H( V ) . $
 Define $ \eta ( \bar{W_1}, W_2, \dots, W_m ) = (  \bar{W}_1 A^{-1} , W_2, \dots, W_m ) $,
then $\eta$ is a homeomorphism between $  \bar{\Omega} $ and $ \Omega$. 
Therefore the set    $ \Omega = \{  (W_1, \dots, W_{H-1}) :   \text{rank}(  T_H( V ) ) <  \barn  \}   $
has zero measure.   $\Box $

\subsection{Proof of claims}\label{appen: repr claim proof}
\textbf{Proof of Claim \ref{representation claim of G}:}
According to Lemma \ref{lemma zero measure}, $\mathcal{W}$ is a dense subset of 
$ \mathbb{R}^J $ (in fact, $\Omega$ is defined for a general neural network,
and  $\mathcal{W}$ is defined for the generator network, thus an instance of $\Omega$). 
As a result, there exists $ (W_1, \dots, W_{H-1}) $ such that
$ T_H(Z) $ has rank at least  $ n $.    %
Thus for any $ y_1, y_2, \dots, y_n \in \mathbb{R}^d  $, 
there exists $W_H$ such that $ W_H T_H(Z) = (y_1, \dots, y_n) $. 
$\Box $

\textbf{Proof of Claim \ref{claim of continuous path, full rank case}: }
 For any continuous path $Y(t), t \in [0, 1]$ in the space $ \mathbb{R}^{d \times n}  $, any $w_0 \in G^{-1}(Y (0) ) $ and any $\epsilon > 0$,
our goal is to show that there exists a continuous path $ w(t), t \in [0,1]  $ such that  $ w( 0 ) = w_0  $
 and $ Y(t) = G_{w(t)}(Z)  , t \in [0, 1] $.

Due to the assumption of $w_0 \in \mathcal{W}$, we know that
$ w_0 $ corresponds to a rank-$n$ post-activation matrix $ T_{ H } (Z) $.
Suppose $ w_0 = ( W_1, \dots, W_{H} ) $ 
and $T_{H }(Z) = (T_{H } (z_1),
\dots, T_{H }(z_n)) \in \mathbb{R}^{ d_{H }  \times n }  $ has rank $ n $.
Since $ T_H(Z) $ is full rank, for any path
from $ Y(0) $ to  $ Y(1)$, we can continuously change 
$ W_H  $ such that the output of $ G_w(Z) $ changes from $ Y(0) $ to $ Y(1) $.
Thus there exists a continuous path $ w(t), t \in [0,1]  $ such that 
$   w(0) = w_0  $ and $ Y(t) = G_{w(t)}(Z)  , t \in [0, 1] $. 
$\Box $

\textbf{Proof of Claim \ref{representation claim of D}:}
This is a direct application of  Lemma \ref{lemma zero measure}.
Different from Claim \ref{claim of continuous path, full rank case}, here we
apply  Lemma \ref{lemma zero measure} to the discriminator network.   $\Box $

\section{Discussion of Wasserstein GAN}\label{appen: W-GAN}

W-GAN is a popular formulation of GAN, so a natural question
is whether we can prove a similar landscape result for W-GAN. 
Consider W-GAN formulation (empirical version)
$
\min_{Y} \phi_{\rm W}(Y, X),
$ 
where
\[ 
\phi_{\rm W}(Y, X) = \max_{ | f|_L \leq 1 } 
\frac{1}{n} \sum_{i=1}^n [ f(x_i)  - f(y_i) ].
\] 
For simplicity we consider the same number
of generated samples and true samples.
It can be viewed as a special case of RpGAN where $h(t) = - t $;
it can also be viewed as a special case of 
SepGAN where $h_1 (t) = h_2(t) = - t$.

However, the major complication is the Lipschitz constraint.
It makes the computation of the function values much harder. 
For the case of $n = 2$, the function value of $ \phi_{\rm W}(Y, X )  $
is provided in the following claim. 
\begin{claim}
Suppose $n = 2$. 
 Denote $ a_1 = x_1, a_2 = x_2, a_3 = y_1, a_4 = y_2 $. 
 The value of $ \phi_{\rm W}(Y, X )  $ is
\begin{align*}
  \max_{ u_1, u_2, u_3 , u_4 \in \mathbb{R} }  &  u_1 + u_2 - u_3 - u_4,  \\
  \text{s.t.} \quad  &  | u_i - u_j |  \leq  \| a_i - a_j \|,  \forall i, j \in \{ 1, 2, 3, 4 \}. 
\end{align*}
\end{claim}
This claim is not hard to prove, and we skip the proof here.

This claim indicates that computing 
$ \phi_{\rm W}(Y, X )  $ is equivalent to solving a linear program (LP).
Solving LP itself is computationally feasible, 
but our landscape analysis requires to infer about the global
landscape of $ \phi_{\rm W}(Y, X )  $ as a function of $ Y $.
In classical optimization, it is possible to 
state that the optimal value of an LP is a convex function
of a certain parameter (e.g. the coefficient of the objective).
But in our LP $y_i$'s appear in multiple positions of the LP,
and we are not aware of an existing result that can be readily applied. 
 
Similar to  Kantorovich-Rubinstein Duality,
we can write down the dual problem of the LP 
where the objective is  linear combination of $ \| a_i - a_j  \| $.
However, it is still not clear what to say about the global landscape,
 due to the lack of closed-form solutions. 

Finally, we remark that although W-GAN has a strong theoretical appeal, it did not replace JS-GAN or simple variants of JS-GAN in recent GAN models. For instance, SN-GAN \cite{miyato2018spectral} and BigGAN \cite{brock2018large} use hinge-GAN.

\newpage

\begin{table}[t]
\tiny 
\begin{minipage}{.47\textwidth}
\centering
\begin{tabular}{cc}
\toprule
\textbf{(a) Generator} & \textbf{(b) Discriminator} \\
\midrule
$z \in \mathbb{R}^{128} \sim {\mathcal N}(0, I)$ & image $x \in [-1, 1]^{H \times W \times 3}$\\\midrule
128 $\rightarrow h \times w \times$ 512/c, dense, linear & $3 \times 3$, stride 1 conv, 64/c  \\\midrule
$4 \times 4$, stride 2 deconv, 256/c, BN, ReLU & $4 \times 4$, stride 2 conv, 128/c  \\
& $3 \times 3$, stride 1 conv, 128/c \\\midrule
$4 \times 4$, stride 2 deconv, 128/c, BN, ReLU & $4 \times 4$, stride 2 conv, 256/c \\
& $3 \times 3$, stride 1 conv, 256/c \\\midrule
$4 \times 4$, stride 2 deconv, 64/c, BN, ReLU & $4 \times 4$, stride 2 conv, 512/c \\
& $3 \times 3$, stride 1 conv, 512/c \\\midrule
$3 \times 3$, stride 1 conv, 3, Tanh  & $h \times w \times 512/c \rightarrow s$, linear\\
\bottomrule
\end{tabular}
\captionsetup{font={scriptsize}} 
\caption{CNN models for CIFAR-10 and STL-10 used in our experiments on image Generation. h = w = 4, H = W = 32 for CIFAR-10. h = w = 6, H = W = 48 for STL-10. c=1, 2 and 4 for the regular, 1/2 and 1/4 channel structures respectively.
All layers of D use LReLU-0.1 (except the final dense `'linear'' layer).
}
\label{table: CNN structure}
\end{minipage}
\hfill 
\begin{minipage}{.47\textwidth}
\centering
\begin{tabular}{cc}
\toprule
\textbf{(a) Generator} & \textbf{(b) Discriminator}\\
\midrule
$z \in {\mathbb{R}}^{128} \sim {\mathcal N}(0, I)$ &  $x \in [-1, 1]^{256 \times 256 \times 3}$\\\midrule
reshape $\rightarrow$ $128 \times 1 \times 1$ & $4 \times 4$, stride 2 conv, 32,  \\\midrule
$4 \times 4$, stride 1 deconv, BN, 1024  & $4 \times 4$, stride 2 conv, 64 \\\midrule
$4 \times 4$, stride 2 deconv, BN, 512 & $4 \times 4$, stride 2 conv, 128 \\\midrule
$4 \times 4$, stride 2 deconv, BN, 256 & $4 \times 4$, stride 2 conv, 256 \\\midrule
$4 \times 4$, stride 2 deconv, BN, 128 & $4 \times 4$, stride 2 conv, 512 \\\midrule
$4 \times 4$, stride 2 deconv, BN, 64 & $4 \times 4$, stride 2 conv, 1024 \\\midrule
$4 \times 4$, stride 2 deconv, BN, 32  & dense $\rightarrow$ 1 \\\midrule
$4 \times 4$, stride 2 deconv, 3, Tanh & \\
\bottomrule
\end{tabular}
\captionsetup{font={scriptsize}} 
\caption{CNN model architecture for size 256 LSUN used in our experiments on high resolution image generation. All layers of G use ReLU (except one layer with Tanh);
all layers of D use LReLU-0.1.
}
\label{table: lsun_cnn_model}
\end{minipage}%
\end{table}

\begin{table}[t]
\tiny 
\begin{minipage}{.5\textwidth}
\centering
\begin{tabular}{cc}
\toprule
\textbf{(a) Generator} & \textbf{(b) Discriminator}\\
\midrule
$z \in \mathbb{R}^{128} \sim {\mathcal N}(0, I)$ & image $x \in [-1, 1]^{32 \times 32 \times 3}$\\\midrule
dense,  $4 \times 4 \times 256$/c & ResBlock down 128/c\\\midrule
ResBlock up 256/c & ResBlock down 128/c \\\midrule
ResBlock up 256/c & ResBlock down 128/c \\\midrule
ResBlock up 256/c & ResBlock down 128/c \\\midrule
BN, ReLU, $3 \times 3$ conv, 3 Tanh & LReLU 0.1 \\\midrule
 & Global sum pooling  \\\midrule
&dense $\rightarrow$ 1\\
\bottomrule
\end{tabular}
\captionsetup{font={scriptsize}}
\caption{Resnet architecture for CIFAR-10. c=1, 2 and 4 for the regular, 1/2 and 1/4 channel structures respectively.}
\label{table: cifar_regualar_resnet }
\end{minipage}\qquad
\begin{minipage}{.5\textwidth}
\centering
\begin{tabular}{cc}
\toprule
\textbf{(a) Generator} & \textbf{(b) Discriminator}\\
\midrule
$z \in \mathbb{R}^{128} \sim {\mathcal N}(0, I)$ & image $x \in [-1, 1]^{48 \times 48 \times 3}$\\\midrule
dense,  $6 \times 6 \times 512$/c & ResBlock down 64/c\\\midrule
ResBlock up 256/c & ResBlock down 128/c \\\midrule
ResBlock up 128/c & ResBlock down 256/c \\\midrule
ResBlock up 64/c & ResBlock down 512/c \\\midrule
BN, ReLU, $3 \times 3$ conv, 3 Tanh &ResBlock down 1024/c\\\midrule
& LReLU 0.1 \\\midrule
 & Global sum pooling  \\\midrule
&dense $\rightarrow$ 1\\
\bottomrule
\end{tabular}
\captionsetup{font={scriptsize}}
\caption{Resnet architecture for STL-10. c=1, 2 and 4 for the regular, 1/2 and 1/4 channel structures respectively.}
\label{table: stl_regular_resnet }
\end{minipage}%
\end{table}

\iflonger 
\begin{figure}[t]
\centering
\begin{tabular}{ccc}
 \includegraphics[height = 4 cm, width = 2.5cm]{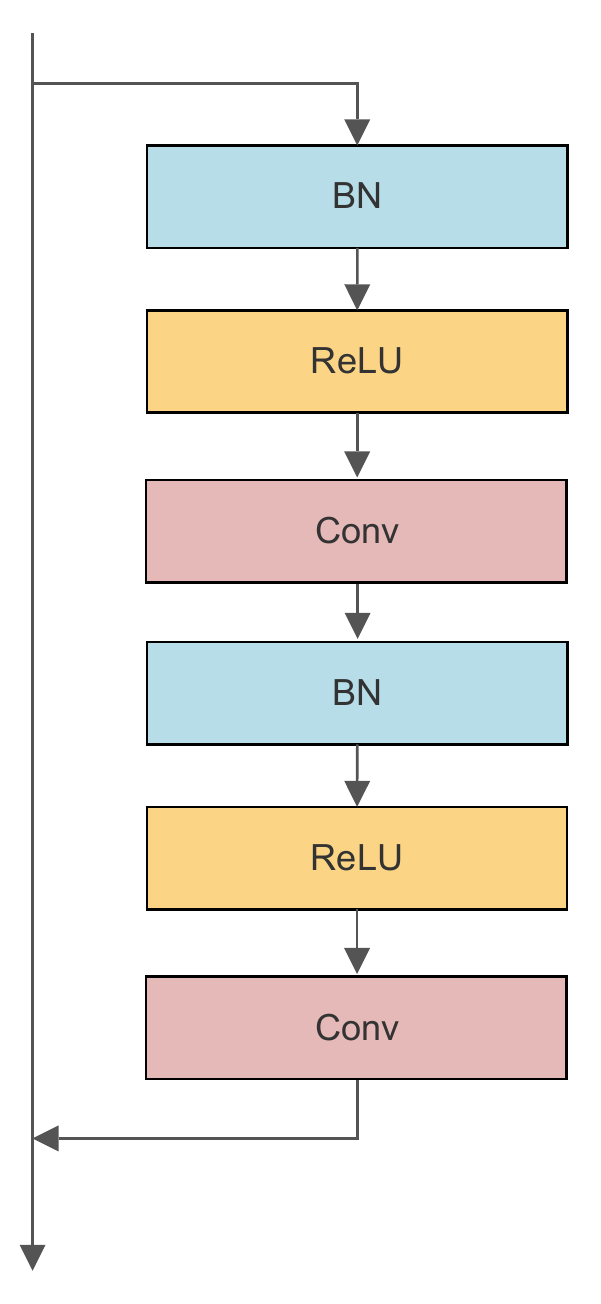} & 
 \hspace{2cm} &\includegraphics[height = 4 cm, width = 2.5cm]{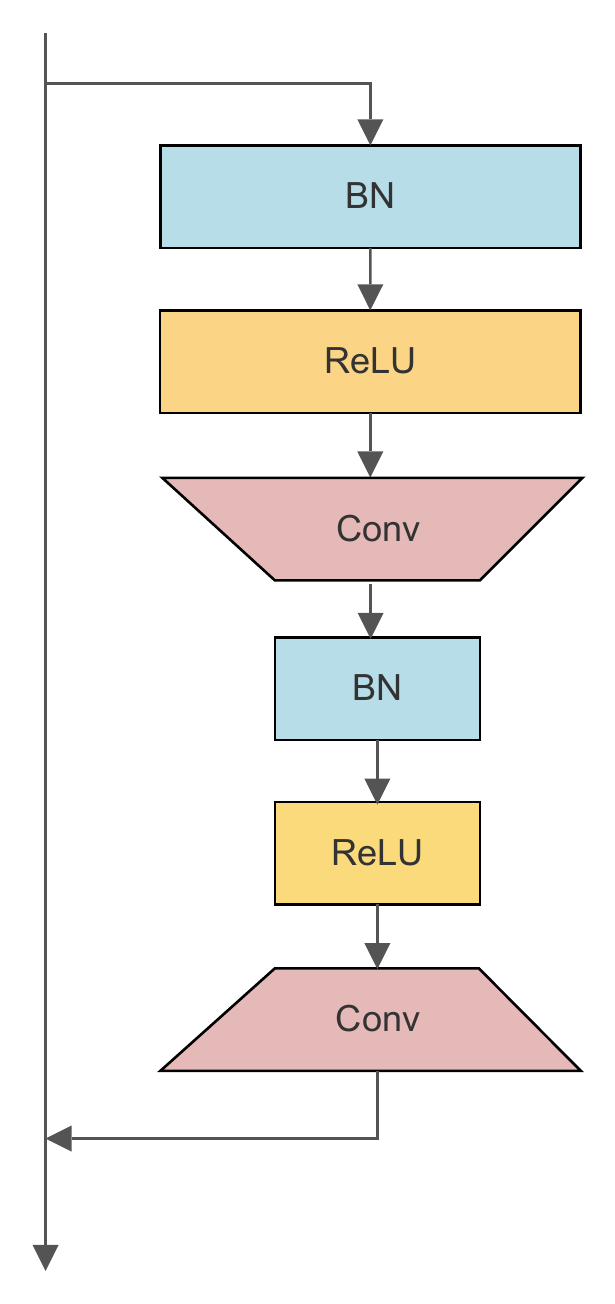}\\
 (a) && (b)
\end{tabular}
\vspace{-0.3cm}
\caption{(a) Regular ResBlock structure. (b) BottleNeck ResBlock structure. 
}
\label{fig:resblockstructure}
\vspace{0.5cm}
\end{figure}
\fi

\begin{table}[t]
\tiny
\begin{minipage}{.5\textwidth}
\centering
\begin{tabular}{cc}
\toprule
\textbf{(a) Generator} & \textbf{(b) Discriminator}\\
\midrule
$z \in \mathbb{R}^{128} \sim {\mathcal N}(0, I)$ & image $x \in [-1, 1]^{32 \times 32 \times 3}$\\\midrule
dense,  $4 \times 4 \times 128$ & BRes down (64, 32, 64) \\\midrule
BRes up (128, 64, 128) & BRes down (64, 32, 64)   \\\midrule
BRes up (128, 64, 128) & BRes down (64, 32, 64) \\\midrule
BRes up (128, 64, 128) & BRes down (64, 32, 64) \\\midrule
BN, ReLU, $3 \times 3$ conv, 3 Tanh & LReLU 0.1  \\\midrule
& Global sum pooling  \\\midrule
& dense $\rightarrow$ 1\\
\bottomrule
\end{tabular}
\captionsetup{font={scriptsize}}
\caption{BottleNeck Resnet models for CIFAR-10. BRes refers to BottleNeck ResBlock. 
BRes $(a, b, c)$ refers to the Bottleneck resblock with (input, hidden and output) being $(a, b , c)$.
}
\label{table: cifar_bottleneck_resnet }
\end{minipage}\qquad
\begin{minipage}[c]{.5\textwidth}
\centering
\begin{tabular}{cc}
\toprule
\textbf{(a) Generator} & \textbf{(b) Discriminator} \\
\midrule
$z \in {\mathbb{R}}^{128} \sim {\mathcal N}(0, I)$ & image $x \in [-1, 1]^{48 \times 48 \times 3}$ \\ \midrule
dense,  $6 \times 6 \times 256$ & BRes down (3, 16, 32)  \\\midrule
BRes up (256, 64, 128) & BRes down (32, 16, 64)   \\\midrule
BRes up (128, 32, 64)& BRes down (64, 32, 128)  \\\midrule
BRes up (64, 16, 32) & BRes down (128, 64, 256)  \\\midrule
BN, ReLU, $3 \times 3$ conv, 3 Tanh & BRes down (256, 128, 512)  \\\midrule
& LReLU 0.1 \\\midrule
& Global sum pooling  \\\midrule
& dense $\rightarrow$ 1\\
\bottomrule
\end{tabular}
\captionsetup{font={scriptsize}}
\caption{BottleNeck Resnet models for STL-10. 
}
\label{table: stl_bottleneck_resnet }
\end{minipage}
\end{table}

\begin{table}[t]
\tablesize 
\begin{minipage}{.5\textwidth}
\centering
\begin{tabular}{cccc}
\toprule
\multicolumn{4}{c}{\textbf{RS-GAN generator learning rate}} \\
\midrule
& & CIFAR-10 & STL-10\\
\multirow{4}{*}{CNN} & No normalization & 2e-4 & 5e-4\\
& Regular + SN & 5e-4 & 5e-4\\
& channel/2 + SN & 5e-4 & 5e-4\\
& channel/4 + SN & 2e-4 & 5e-4\\ \midrule
\multirow{4}{*}{ResNet} & Regular+SN & 1.5e-3 & 1e-3\\
& channel/2 + SN & 1.5e-3 & 1e-3\\
& channel/4 + SN & 1e-3 & 5e-4\\ 
& BottleNeck & 1e-3 & 1e-3\\
\bottomrule
\end{tabular}
\end{minipage}%
\begin{minipage}[c]{.6\textwidth}
\centering
\begin{tabular}{lc}
\toprule
\multicolumn{2}{c}{ \textbf{WGAN-GP Hyper-parameters }} \\
\midrule
generator learning rate &1e-4 \\
discriminator learning rate &1e-4\\
$\beta_1$ & 0.5\\
$\beta_2$ & 0.9\\
Gradient penalty $\lambda$ &10\\
\# D iterations per G iteration & 5 \\
\bottomrule
\end{tabular}
\end{minipage}
\caption{Learning rate for RS-GAN in each setting. Hyper-parameters used for WGAN-GP  }
\label{Tab:lrsetting}
\end{table}

\fi 

\end{document}